\documentclass{article}
\usepackage[verbose=true,letterpaper]{geometry}
\AtBeginDocument{
	\newgeometry{
		textheight=9in,
		textwidth=5.5in,
		top=1in,
		headheight=12pt,
		headsep=25pt,
		footskip=30pt
	}
}

\usepackage[utf8]{inputenc} % allow utf-8 input
\usepackage[T1]{fontenc}    % use 8-bit T1 fonts
\usepackage{lmodern}

\usepackage{amsmath}
\makeatletter
\newcommand{\leqnomode}{\tagsleft@true\let\veqno\@@leqno}
\makeatother
\usepackage{amssymb}

\usepackage{algorithm}
\usepackage{algpseudocode}

\usepackage{graphicx} % crop included graphics
\usepackage{caption}
\usepackage{placeins} %\FloatBarrier

\usepackage{booktabs}
\usepackage{makecell}

\usepackage{tikz}
\usepackage{import}
\usepackage{graphicx} % Required for inserting images
\usepackage{tabularx}
\usepackage{amsmath}
\usepackage{amsthm} %http://www.ams.org/arc/tex/amscls/amsthdoc.pdf
\usepackage{amsfonts}
\usepackage{bbm}% bb math symbols, eg 1
\usepackage[utf8]{inputenc}
\usepackage{enumitem}
\usepackage{tikz}
\usetikzlibrary{arrows}
\usepackage{xfrac}		% \sfrac (nice-frac)
\usepackage{enumitem}	% [label=...] for enums
\usepackage{braket}
\usepackage[normalem]{ulem} % strike-out text via \sout

\usepackage{todonotes}
\usepackage{dirtytalk} % why would you even ask? [provides quotation-mark helpers, see '\inquotes' macro below]

\usepackage{hyperref}
\hypersetup{allcolors=black, colorlinks}

\makeatletter
\newcommand*{\centernot}{%
	\mathpalette\@centernot
}
\def\@centernot#1#2{%
	\mathrel{%
		\rlap{%
			\settowidth\dimen@{$\m@th#1{#2}$}%
			\kern.5\dimen@
			\settowidth\dimen@{$\m@th#1=$}%
			\kern-.5\dimen@
			$\m@th#1\not$%
		}%
		{#2}%
	}%
}
\makeatother

\newcommand{\independent}{\perp\mkern-9.5mu\perp}
\newcommand{\notindependent}{\centernot{\independent}}
\newcommand{\dependent}{\notindependent}
\newcommand{\onehalf}{\sfrac{1}{2}}

\newcommand{\eg}{e.\,g.\ }
\newcommand{\Eg}{E.\,g.\ }
\newcommand{\ie}{i.\,e.\ }
\newcommand{\Ie}{I.\,e.\ }
\newcommand{\st}{s.\,t.\ }
\newcommand{\wrt}{w.\,r.\,t.\ }
\newcommand{\asswlog}{w.\,l.\,o.\,g.\ }
\newcommand{\Wlog}{W.\,l.\,o.\,g.\ }
\newcommand{\BreakingSmallSpace}{\hspace{.16667em}}
\newcommand{\Slash}{\,/\BreakingSmallSpace}

\DeclareMathOperator{\Pa}{Pa}

\DeclareMathOperator{\Anc}{Anc}
\DeclareMathOperator{\Dec}{Desc}

\DeclareMathOperator{\const}{const}

\DeclareMathOperator{\boolOr}{or}
\DeclareMathOperator{\boolAnd}{and}

\DeclareMathOperator{\boolFalse}{false}
\DeclareMathOperator{\boolTrue}{true}

\DeclareMathOperator{\Var}{Var}
\DeclareMathOperator{\BinomDist}{Binom}

\DeclareMathOperator{\Pattern}{Patt}
\DeclareMathOperator{\NormalDist}{\mathcal{N}}
\DeclareMathOperator{\UniformDist}{Unif}

\DeclareMathOperator{\median}{median}

\DeclareMathOperator{\dx}{dx}

\NewDocumentCommand{\txt}{m}{\text{\normalfont{#1}}}
\NewDocumentCommand{\inquotes}{m}{\say{#1}}
\NewDocumentCommand{\setelemcount}{m}{|#1|}

\DeclareMathOperator{\Map}{Map}

\NewDocumentCommand{\norm}{m}{\Vert #1 \Vert}

\NewDocumentCommand{\val}{m}{\mathcal{#1}}

\theoremstyle{plain}
\newtheorem{lemma}{Lemma}[section]
\newtheorem{thm}{Theorem}
\newtheorem{prop}[lemma]{Proposition}
\newtheorem{cor}[lemma]{Corollary}
\theoremstyle{definition}

\newtheorem{definition}[lemma]{Definition}
\newtheorem{assumption}[lemma]{Assumption}

\newtheorem{lemmaNotation}[lemma]{Lemma\Slash{}Notation}

\newtheorem{notation}[lemma]{Notation}
\newtheorem{convention}[lemma]{Convention}
\newtheorem{example}[lemma]{Example}

\newtheorem{rmk}[lemma]{Remark} % there is already an environment called "remark" in some packages

\usepackage[square,numbers,sort&compress]{natbib}

\usepackage{microtype}

\usepackage{hyperref}       % hyperlinks

\usepackage{graphicx}
\graphicspath{{./aux_files/plots/}} %Setting the graphicspath

\NewDocumentCommand{\defName}{m}{#1}

\NewDocumentCommand{\Multiindex}{}{\vec{i}}
\NewDocumentCommand{\MultiindexSet}{}{I^{2+*}}

\NewDocumentCommand{\IStructs}{}{\mathcal{I}}
\NewDocumentCommand{\IStructsM}{}{\mathcal{I}_\txt{marked}}
\NewDocumentCommand{\GClasses}{}{\mathcal{G}}
\NewDocumentCommand{\GClassesGeneral}{}{\GClasses_{\txt{gt}}}
\NewDocumentCommand{\MClasses}{}{\mathcal{M}}
\NewDocumentCommand{\Rmodel}{}{R^{\txt{model}}}
\NewDocumentCommand{\NonTrivialIndicators}{}{\mathcal{R}}
\DeclareMathOperator{\IStruct}{IS}
\DeclareMathOperator{\IStructM}{IM}
\DeclareMathOperator{\IStructR}{IR}
\DeclareMathOperator{\IStructX}{IX}
\NewDocumentCommand{\IStructPseudo}{m}{\IStruct_\txt{pseudo}^{#1}}

\DeclareMathOperator{\Lookup}{L}
\DeclareMathOperator{\subsets}{Subsets}
\DeclareMathOperator{\IStructExt}{\IStructX}
\DeclareMathOperator{\IStructOracle}{\IStruct_{\txt{oracle}}}
\DeclareMathOperator{\IStructOracleM}{\IStructM_{\txt{oracle}}}
\DeclareMathOperator{\IStructOracleRel}{\IStructR_{\txt{oracle}}}
\DeclareMathOperator{\IStructOracleX}{\IStructX_{\txt{oracle}}}

\DeclareMathOperator{\CDAlg}{CD}
\DeclareMathOperator{\CDAlgPC}{PC}
\DeclareMathOperator{\CDAlgFCI}{FCI}
\DeclareMathOperator{\im}{img}
\DeclareMathOperator{\linspan}{span}
\NewDocumentCommand{\gUnion}{m}{#1^{\txt{union}}}

\NewDocumentCommand{\halfquad}{}{\mkern9mu}

\NewDocumentCommand{\Underbar}{}{\_\allowbreak{}}

\NewDocumentCommand{\algMarkedIndependence}{}
	{\texttt{marked\_\allowbreak{}independence}}
\NewDocumentCommand{\algConstructStateSpace}{}
	{\texttt{construct\_\allowbreak{}state\_\allowbreak{}space}}
\NewDocumentCommand{\algRunCD}{}
	{\texttt{run\_\allowbreak{}cd}}
\NewDocumentCommand{\algPseudoCIT}{}
	{\texttt{pseudo\_\allowbreak{}cit}}

\NewDocumentCommand{\mCIToutR}{}{\mathfrak{R}}

\NewDocumentCommand{\CodeRepo}{}{\href{https://github.com/martin-rabel/Causal\_GLDF}{https://github.com/martin-rabel/Causal\_GLDF}}

\title{Context-Specific Causal Graph Discovery with Unobserved Contexts: Non-Stationarity, Regimes
	and Spatio-Temporal Patterns}
	
\author{Martin Rabel\textsuperscript{$*$,a,b}, Jakob Runge\textsuperscript{a,b}}
\date{\today}

\counterwithin{figure}{section}
\counterwithin{table}{section}

\usetikzlibrary {shapes.geometric}
\usetikzlibrary {arrows.meta}

\NewDocumentCommand{\mayincludegraphics}{O{}m}{\includegraphics[#1]{#2}}

\begin{document}
	\maketitle
	\noindent\begin{minipage}{\textwidth}
		\centering
		{\footnotesize\textsuperscript{a}University of Potsdam, Institute of Computer Science,
			An der Bahn 2, 14476 Potsdam, Germany}\\[-0.1em]
		{\footnotesize\textsuperscript{b}Center for Scalable Data Analytics and Artificial Intelligence (ScaDS.AI) Dresden\Slash{}Leipzig, Germany}\\[-0.1em]
		{\footnotesize\textsuperscript{$*$}corresponding author,
			martin.rabel@uni-potsdam.de}
	\end{minipage}

	\begin{abstract}
		Real-world problems, for example in climate applications,
		often require causal reasoning on
		spatially gridded time series data or data with comparable structure.
		While the underlying system is often believed to behave similarly
		at different points
		in space and time, those variations that do exist are relevant twofold:
		They often encode important information in and of themselves. And they
		may negatively affect the stability and validity of results if not accounted for.
		We study the information encoded in changes of the causal graph,
		with stability in mind.
		Two core challenges arise,
		related to the complexity of encoding system-states and
		to statistical convergence properties in the presence
		of imperfectly recoverable non-stationary structure.
		We provide a framework realizing principles
		conceptually suitable to overcome these challenges
		-- an interpretation supported by numerical experiments.
		Primarily, we modify constraint-based causal discovery approaches on
		the level of independence testing.
		This leads to a framework which is additionally highly modular,
		easily extensible and widely applicable.
		For example, it allows to
		leverage existing constraint-based causal discovery methods
		(demonstrated on PC, PC-stable, FCI, PCMCI, PCMCI+ and LPCMCI),
		and to systematically divide the problem into simpler subproblems
		that are easier to analyze and understand and relate more clearly
		to well-studied problems like change-point-detection, clustering,
		independence-testing and more.
		Code is available at \CodeRepo{}.
	\end{abstract}
	
	{\small\noindent{}\emph{AMS Mathematics Subject Classification:}
	62G02, 68T02.}
	
	% "List keywords for the work presented (maximum of 6), separated by commas.
	% We suggest that keywords do not replicate those used in the title"
	{\small\noindent{}\emph{Keywords:}
	causality,
	%context-specific graphs,
	%non-stationarity, regimes,
	%patterns,
	change point detection, clustering,
	independence-structures, distribution shifts.}
	
	\section{Introduction}\label{sec:intro}

	In science and technology, the researcher often is interested in causal relationships
	\citep{PearlBook,Elements}.
	One reason being
	that causality is key to understanding
	actions and their consequences.
	Even without the intent to actually take an action, scientific reasoning and insight often
	revolve around the impact an action would have.
	Understanding the causes of things has always been integral to scientific curiosity.
	
	The flavor of causal reasoning we will focus on is the study from \emph{observational}
	data:
	In practice, potential consequences may prohibit the experimental
	exploration of possible actions. In other cases an experiment may not be feasible;
	for example, because an action, like the release of greenhouse gases, may not be taken
	voluntarily or repeatable,
	yet the consequences, like climate change, may mandate scientific study.
	In such cases, causal knowledge must be extracted from observations alone.
	
	Generally, the study of causal relationships,
	causal inference \citep{PearlBook,Elements,spirtes2001causation},
	encompasses many tasks,
	ranging from the estimation of effects (of actions\Slash{}interventions)
	to the answering of counter-factual queries (what would have happened?).
	All these tasks build on the knowledge of basic cause-effect relationships;
	a knowledge often represented by a causal graph.
	This causal graph may not be known a priori,
	in which case it has to be discovered from data.
	This graph-discovery task is referred to as causal discovery (CD).
	Especially CD from observational data
	has recently garnered substantial attention from climate- and other applications
	\citep{granger1969investigating,PearlCIOverview09,runge2019inferring}.
	
	In real-world scenarios CD faces many challenges \citep{Runge2023}.
	These challenges may stem, for example from	the type and quality of data
	or the violation of simplifying assumptions required to make the CD task
	tractable. A particularly common assumption is that of IIDness
	-- or stationarity in the
	time series case: A single underlying model is postulated to describe all data-generation
	uniformly
	-- the location in space, time or other information relating multiple data-points
	is considered insubstantial. Yet in the real world, almost any observed model is subject to some degree of variability \eg in space or time.
	A local climate might experience qualitatively different dry and moist regimes
	over time in a single place, or in different places at the same time;
	many climate and weather parameters are subject to seasonality;
	a policy might affect public health in ways related to
	untracked parameters
	or a technological device may break or otherwise change its internal state for some time
	or some firmware-version.
	
	Indeed these examples reveal another challenge, entirely separate from
	assumption violations and incurred stability issues:
	A mayor point of interest oftentimes is the occurrence and form of a change itself;
	it is the existence of and insight into changes in itself that is the subject of study.
	With the goal of understanding precisely these \emph{variations} of the causal model,
	the challenge to CD studied in this paper is to extract additional information:
	Does a causal model -- and in particular its causal graph -- change within a data-set,
	and if so how?
	Thus we study context-dependent causal properties
	\citep{bareinboim2012transportability,LDAG_definition,CD-NOD,JCI,JPCMCI}
	in the case where the context is unknown\Slash{}hidden \citep{Saggioro2020},
	specifically hidden-context context-specific causal discovery (HCCD)
	see Fig.\ \ref{fig:intro_toy_model}.

	\begin{figure}[ht]
		\begin{minipage}{0.58\textwidth}
			\begin{tikzpicture}
				\draw (0,0) node[anchor=north west] {\mayincludegraphics[width=\textwidth]{intro_toy_model}};
				\draw (0.75em,-0.25em) node[anchor=north west] {(a)};
			\end{tikzpicture}
		\end{minipage}
		\begin{minipage}{0.4\textwidth}
			\colorlet{nodecolor}{black}
\providecommand{\introSpaceQuandrantsScale}{0.5}

\begin{tikzpicture}
	[
		scale=\introSpaceQuandrantsScale,
		inner sep=0,
		outer sep=0.15em,
		varnode/.style={
			circle,
			draw=nodecolor,
			thick,
			fill=nodecolor!20,
			align=center,
			minimum height=1.5em
		},
		mechedge/.style={
			very thick
		}
	]
	\draw (-5,5) node[anchor=north west]{(b)};
	
	\draw[very thick] (-5,0) -- (5,0);
	\draw[very thick] (0,-5) -- (0,5);
	
	% top left
	\begin{scope}[xshift=-2.5cm, yshift=2.5cm]
		\draw (0,1.5) node(a)[varnode]{$X$};
		\draw (-1.5,0) node(b)[varnode]{$Y$};
		\draw (0,-1.5) node(c)[varnode]{$Z$};
		\draw (1.5,0) node(d)[varnode]{$W$};
		
		\draw[mechedge, ->] (a) -- (c);
%		\draw[mechedge, ->] (d) -- (b);
%		\draw[mechedge, <->] (b) -- (c);
	\end{scope}
	
	% top right
	\begin{scope}[xshift=2.5cm, yshift=2.5cm]
		\draw (0,1.5) node(a)[varnode]{$X$};
		\draw (-1.5,0) node(b)[varnode]{$Y$};
		\draw (0,-1.5) node(c)[varnode]{$Z$};
		\draw (1.5,0) node(d)[varnode]{$W$};
		
		\draw[mechedge, ->] (a) -- (c);
%		\draw[mechedge, ->] (d) -- (b);
		\draw[mechedge, <->] (b) -- (c);
	\end{scope}
	
	% bottom left
	\begin{scope}[xshift=-2.5cm, yshift=-2.5cm]
		\draw (0,1.5) node(a)[varnode]{$X$};
		\draw (-1.5,0) node(b)[varnode]{$Y$};
		\draw (0,-1.5) node(c)[varnode]{$Z$};
		\draw (1.5,0) node(d)[varnode]{$W$};
		
		\draw[mechedge, ->] (a) -- (c);
		\draw[mechedge, ->] (d) -- (b);
%		\draw[mechedge, <->] (b) -- (c);
	\end{scope}
	
	% bottom right
	\begin{scope}[xshift=2.5cm, yshift=-2.5cm]
		\draw (0,1.5) node(a)[varnode]{$X$};
		\draw (-1.5,0) node(b)[varnode]{$Y$};
		\draw (0,-1.5) node(c)[varnode]{$Z$};
		\draw (1.5,0) node(d)[varnode]{$W$};
		
		\draw[mechedge, ->] (a) -- (c);
		\draw[mechedge, ->] (d) -- (b);
		\draw[mechedge, <->] (b) -- (c);
	\end{scope}
\end{tikzpicture}
		\end{minipage}
		\caption{Simple toy-model to illustrate our method's output (a; applied with FCI):
		Data is given on a spatial $100\times 100$ grid (\ie
		10,000 data points in total), the link marked in orange is present in the southern
		half, the one marked in blue is present in the eastern half.
		This spatial distribution is unknown\Slash{}hidden.
		In the graph, the circle-marks on edges at $X$ and $Y$ indicate that no conclusive
		decision was possible about these edge-marks (\eg both $Z \leftarrow X$
		and $Z\leftrightarrow X$ are compatible with encountered independence relations).
		Our focus is on how the graph on the left-hand side (including the designation of
		changing links) can be recovered.
		In this example, different contexts correspond to the four (spatial) quadrants (b).
		While this illustrates the interpretation of output (a),
		the association of changing links to spatial directions is
		a very special case, see Fig.\ \ref{fig:illustrate_global_vs_local_structure}.}
		\label{fig:intro_toy_model}
	\end{figure}
	
	Our focus is on the \emph{extraction} of information about changes of the causal graph,
	with \emph{robustness} against non-graphical non-stationarity.	
	From the perspective that contexts are intervened instances of a shared model,
	this means we gain knowledge about a notion of soft intervention
	that modifies parent-sets.
	Such interventions seem to occur in the real world: From the closing of a window disrupting
	temperature exchange or the damaging of an electrical insulation enabling new currents,
	via phase-transitions qualitatively changing physical interactions or
	the day-night-cycle (or cloud-cover) controlling the effect of solar parameters,
	to climate-science phenomena like ENSO \citep{webster1997past, Saggioro2020}
	or soil-moisture feedbacks
	\citep{seneviratne2010investigating, Popescu2024RegimeCausalClimate}.
	We discuss acyclic (no \emph{contemporaneous} cycles in the time series case) models, with
	numerical experiments focused on the causally sufficient (no hidden confounders) case.
	These are not fundamental limitations of our approach, rather choices to keep the presentation
	reasonably succinct and to ensure the availability of existing methods for comparison.

	\paragraph{Choice of Inductive Bias:}
	
	Non-IIDness and non-stationarities can take a vast number of possible forms.
	Indeed, even the question about the presence of non-IID structure cannot
	be answered from data in full generality \citep{WaldWolfowitzRuntest, RandomnessTestsOnline} (see also
	Prop.\ \ref{prop:imposs_A}).
	It is therefore not possible to provide adequate quality of estimation
	over arbitrary non-IID structure. This makes a detailed understanding
	and choice of inductive biases,
	and considerations for deciding the applicability of methods a priori,
	indispensable.
	By choice of inductive bias we mean here specifically, what form of
	non-IID structure we intend to optimize performance for.
	
	First, the principle of causal modularity
	\cite[p.\,63]{PearlBook}\cite[2.1 (p.\,19)]{Elements}
	suggests that causal mechanisms
	change independently of each other. From this perspective it appears reasonable
	to direct our focus to changes that occur \emph{locally in the model}.
	To this end, we will expand around the limit where changes in different mechanisms
	(more precisely: different links) occur independently. 
	
	Second, variations of the generating model seem particularly relevant for large
	data-sets. By relevant we really mean two things: Changes in the generating model
	seem more likely to occur in large data-sets, but from a practical perspective
	this challenging problem may also not possess satisfying solutions on too small datasets.
	Thus we pay particular attention to the careful statistical description
	and scaling properties -- both statistically and concerning compute-resources --
	with increasing sample-size.
	
	Also the notion of large structures, like the temporal extend of persistent regimes,
	requires further comment: We specifically account for
	scaling of the sample-size $N$ \emph{not} necessarily being matched by typical (time-)scales $\ell$
	of context-changes; instead of a limit with $\sfrac{N}{\ell}$ fixed, an uninformative prior
	(for example uniformly distributed mechanism-typical length-scales
	of $\ell$ over scales from $1$ to $N$) seems to capture
	finite sample properties better.
	This can also be phrased in terms of information contained in the regime-structure.
	Proportional scaling would assume that there is always an equal absolute amount of information in
	the regime-structure independently of data-set size.
	We instead focus on limits where the regime-structure will typically contain more information	on larger data-sets.

	\paragraph{Challenges:}
	For a better understanding of the challenges involved in the problem of HCCD
	it also helps to briefly outline possible approaches to the problem
	considered in the literature.
	Particularly simple are sliding-window approaches that divide the data into smaller
	segments and apply a conventional CD-algorithm locally in time\Slash{}space.
	Similarly, one may apply
	a change-point detection (CPD) method first and apply CD on detected segments.
	Besides finite-sample stability
	a difficulty in both cases is how to aggregate the resulting graphs.
	There is no simple way to leverage possibly reoccurring contexts (especially not
	local contexts), but for example one may cluster graphs.
	Indeed, one may also first cluster data,
	then apply CD per cluster.
	There are also more sophisticated methods that fit multiple models and optimize over
	regime-assignments by predictive quality with \citep{BalsellsRodas2023} or without \citep{Saggioro2020}
	parametric assumptions about regime-structure.
	While evidently the HCCD problem is connected to
	CD, clustering, CPD, optimality and more,
	this connection is as of now not well systematized making the combination of different
	solutions to the different aspects difficult.

	\begin{figure}[ht]
		\centering
		\colorlet{nodecolor}{black}
\colorlet{color_c}{black!10!orange}
\colorlet{color_b}{gray!30!blue}
\colorlet{color_a}{black!30!green}
\colorlet{color_d}{black!30!purple}

\providecommand{\introLocalScale}{1.0}
\providecommand{\introLocalScaleSubX}{1.2}
\providecommand{\introLocalScaleSubY}{0.8}
\providecommand{\introGlobalScaleX}{1.0}
\providecommand{\introGlobalScaleY}{0.25}
\providecommand{\introLocalGlobalLineWidth}{1.2pt} % very thick

	\begin{tikzpicture}[scale=\introLocalScale]
		
		\draw (-1.75, 0) node[rotate=90]{\textbf{local}};
		
		\fill[color_b!30!white] (-1.5,0) rectangle (1.5,1);
		\fill[color_b!30!white] (7.5,0) rectangle (10.5,1);
		\fill[color_b!10!white] (1.5,0) rectangle (7.5,1);
		\draw[color_b, very thick] (1.5,0) -- (1.5,1);
		\draw[color_b, very thick] (7.5,0) -- (7.5,1);
		
		\fill[color_a!30!white] (-1.5,-1.2) rectangle (4.5,-0.2);
		\fill[color_a!10!white] (4.5,-1.2) rectangle (10.5,-0.2);
		\draw[color_a, very thick] (4.5,-1.2) -- (4.5,-0.2);
		
		\draw (0,0) node (A) {
			\begin{tikzpicture}
				[
				xscale=\introLocalScaleSubX,
				yscale=\introLocalScaleSubY,
				inner sep=0,
				outer sep=0.15em,
				varnode/.style={
					circle,
					draw=nodecolor,
					thick,
					fill=nodecolor!20,
					align=center,
					minimum height=1.5em
				},
				mechedge/.style={
					->,
					line width=\introLocalGlobalLineWidth
				}
				]		
				\draw (0,0) node(a)[varnode]{};
				\draw (1,0) node(b)[varnode]{};
				\draw (0.5,1) node(c)[varnode]{};
				\draw (1.5,1) node(d)[varnode]{};
				\draw (1,2) node(e)[varnode]{};
				
				\draw[mechedge,color_a] (a) -- (b);
				\draw[mechedge] (a) -- (c);
				\draw[mechedge] (b) -- (d);
				
				\draw[mechedge,color_b] (e) -- (c);
				\draw[mechedge] (e) -- (d);
			\end{tikzpicture}
		};
		\draw (3,0) node (B) {
			\begin{tikzpicture}
				[
				xscale=\introLocalScaleSubX,
				yscale=\introLocalScaleSubY,
				inner sep=0,
				outer sep=0.15em,
				varnode/.style={
					circle,
					draw=nodecolor,
					thick,
					fill=nodecolor!20,
					align=center,
					minimum height=1.5em
				},
				mechedge/.style={
					->,
					line width=\introLocalGlobalLineWidth
				}
				]		
				\draw (0,0) node(a)[varnode]{};
				\draw (1,0) node(b)[varnode]{};
				\draw (0.5,1) node(c)[varnode]{};
				\draw (1.5,1) node(d)[varnode]{};
				\draw (1,2) node(e)[varnode]{};
				
				\draw[mechedge,color_a] (a) -- (b);
				\draw[mechedge] (a) -- (c);
				\draw[mechedge] (b) -- (d);
				
				%\draw[mechedge,color_b] (e) -- (c);
				\draw[mechedge] (e) -- (d);
			\end{tikzpicture}
		};
		\draw (6,0) node (C) {
			\begin{tikzpicture}
				[
				xscale=\introLocalScaleSubX,
				yscale=\introLocalScaleSubY,
				inner sep=0,
				outer sep=0.15em,
				varnode/.style={
					circle,
					draw=nodecolor,
					thick,
					fill=nodecolor!20,
					align=center,
					minimum height=1.5em
				},
				mechedge/.style={
					->,
					line width=\introLocalGlobalLineWidth
				}
				]		
				\draw (0,0) node(a)[varnode]{};
				\draw (1,0) node(b)[varnode]{};
				\draw (0.5,1) node(c)[varnode]{};
				\draw (1.5,1) node(d)[varnode]{};
				\draw (1,2) node(e)[varnode]{};
				
				%\draw[mechedge,color_a] (a) -- (b);
				\draw[mechedge] (a) -- (c);
				\draw[mechedge] (b) -- (d);
				
				%\draw[mechedge,color_b] (e) -- (c);
				\draw[mechedge] (e) -- (d);
			\end{tikzpicture}
		};
		\draw (9,0) node (D) {
			\begin{tikzpicture}
				[
				xscale=\introLocalScaleSubX,
				yscale=\introLocalScaleSubY,
				inner sep=0,
				outer sep=0.15em,
				varnode/.style={
					circle,
					draw=nodecolor,
					thick,
					fill=nodecolor!20,
					align=center,
					minimum height=1.5em
				},
				mechedge/.style={
					->,
					line width=\introLocalGlobalLineWidth
				}
				]		
				\draw (0,0) node(a)[varnode]{};
				\draw (1,0) node(b)[varnode]{};
				\draw (0.5,1) node(c)[varnode]{};
				\draw (1.5,1) node(d)[varnode]{};
				\draw (1,2) node(e)[varnode]{};
				
				%\draw[mechedge,color_a] (a) -- (b);
				\draw[mechedge] (a) -- (c);
				\draw[mechedge] (b) -- (d);
				
				\draw[mechedge,color_b] (e) -- (c);
				\draw[mechedge] (e) -- (d);
			\end{tikzpicture}
		};
	\end{tikzpicture}\\
	\begin{tikzpicture}[xscale=\introGlobalScaleX,yscale=\introGlobalScaleY]		
		\draw (-1.75, 0) node[rotate=90]{\textbf{global}};
		
		\fill[color_a!30!white] (-1.5,-1.2) rectangle (1.5,1.2);
		\fill[color_b!30!white] (1.5,-1.2) rectangle (4.5,1.2);
		\fill[color_c!30!white] (4.5,-1.2) rectangle (7.5,1.2);
		\fill[color_d!30!white] (7.5,-1.2) rectangle (10.5,1.2);
		
		\draw[black, very thick] (1.5,-1.2) -- (1.5,1.2);
		\draw[black, very thick] (7.5,-1.2) -- (7.5,1.2);
		\draw[black, very thick] (4.5,-1.2) -- (4.5,1.2);
		
		\draw (0,0) node (A) {A};
		\draw (3,0) node (B) {B};
		\draw (6,0) node (C) {C};
		\draw (9,0) node (D) {D};
	\end{tikzpicture}\\[-0.75em]
		\caption{Illustration of complexity-scaling for local vs.\ global description of contexts
			on a time-series example.
			Locally there are two contexts (top panel), globally the number of contexts increases
			exponentially in the number of local contexts (bottom panel).
			The global contexts can always be decomposed in terms of local contexts, this does not
			require the existence of corresponding index-set\Slash{}spatial directions as
			those used in Fig.\ \ref{fig:intro_toy_model}b.}
		\label{fig:illustrate_global_vs_local_structure}
	\end{figure}
	
	When comparing these quite diverse approaches to what we picked as an interesting inductive bias above,
	we note that they all share two deep conceptual challenges:
	First, they all are assigning contexts globally in the model\Slash{}graph, while we came to believe that
	the model changes primarily locally. But given $\kappa$ local changes
	(binary context indicators)
	one observes up to a total of $2^{\kappa}$ global combinations (contexts);
	thus global (in this sense)
	methods scale exponentially in the number of local changes.
	This is illustrated in Fig.\ \ref{fig:illustrate_global_vs_local_structure}.
	Additionally, the local signal (independent
	of node-count) is obfuscated by the all noise globally
	(with entropy approximately proportional to node-count); thus global (in this sense)
	methods suffer from large node-counts.
	
	Second, less evident, but much more worrisome is an observation concerning the scaling to larger
	sample-size. All the above approaches (barring sliding windows)
	follow in some way the intuitive logic of reconstructing regimes, then running CD per-regime.
	However, the first step necessarily produces imperfect reconstructions, so we find ourselves
	in the difficult situation, where the second step (CD) \emph{must not} converge fast enough
	to detect these imperfections. We are effectively racing our CD (or independence test) vs.\ our
	method of regime-assignment.
	Especially, if the information contained in the regime-structure is not artificially held constant
	by matching typical regime-scales to sample-size, the reconstruction of regimes will often
	converge slower than the independence tests used by CD.
	Indeed, this leads in numerical experiments §\ref{sec:num_experiments} to a systematic and
	substantial degradation of results on large data-sets.
	This is illustrated in Fig.\ \ref{fig:illustrate_direct_vs_indirect}.
	
	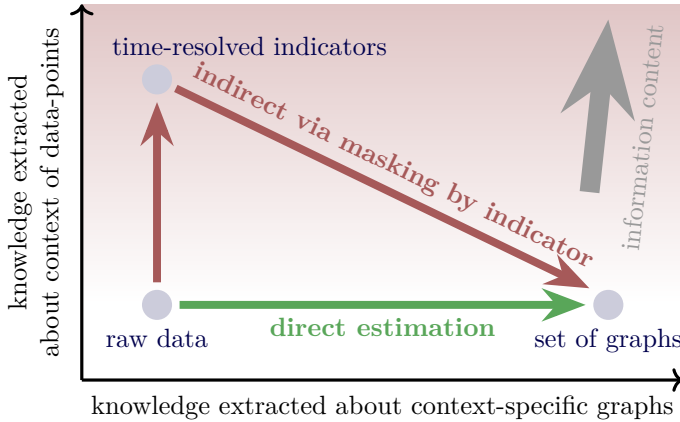
\begin{figure}[ht]
		\begin{minipage}{0.65\textwidth}
			\colorlet{model}{blue!30!black}
\colorlet{model_fill}{model!20!white}

\colorlet{output}{blue!30!black}
\colorlet{output_fill}{model!20!white}

\colorlet{color_simple}{gray!70!green}
\colorlet{color_complex}{gray!70!red}
\colorlet{color_simple_fill}{color_simple!70!white}
\colorlet{color_complex_fill}{color_complex!70!white}
\colorlet{color_simple_back}{color_simple_fill!50!white}
\colorlet{color_complex_back}{color_complex_fill!50!white}

\providecommand{\introDirectIndirectScale}{1.0}

\begin{tikzpicture}[scale=\introDirectIndirectScale]
	\node [shading = axis,rectangle, top color=color_complex_back, bottom color=white, anchor=north west, minimum width=\introDirectIndirectScale*8cm, minimum height=\introDirectIndirectScale*4cm] (background) at (0,5){};
	
	\draw[very thick,->] (0,0) -- (0,5) node[pos=0.5, rotate=90, align=center,anchor=south,yshift=0.3em]
		{knowledge extracted\\about context of data-points};
	\draw[very thick,->] (0,0) -- (8,0) node[pos=0.5, align=center, anchor=north, yshift=-0.2em]
		{knowledge extracted about context-specific graphs};
		
	\fill[model_fill] (1,1) circle(0.2);
	\draw[model] (1,0.8) node[anchor=north] {raw data};
	
	\fill[output_fill] (1,4) circle(0.2);
	\draw[output] (1,4.2) node[anchor=south, xshift=3.5em]{(time-)resolved indicators};
	
	\fill[output_fill] (7,1) circle(0.2);
	\draw[output] (7,0.8) node[anchor=north] {set of graphs};
%	
%	\fill[output_fill] (7,4) circle(0.2);
%	\draw[output] (7.2,4) node[anchor=west, align=left]{set of graphs\\and indicators};
	
	\draw[line width=0.3em,color_simple,-Stealth] (1.3, 1) -- (6.7,1)
		node [pos=0.5, below]{\textbf{direct estimation}};
	\draw[line width=0.3em,color_complex,-Stealth, dashed] (1, 1.3) -- (1,3.7);
	\draw[line width=0.3em,color_complex,-Stealth, dashed] (1.25, 3.88) -- (6.75,1.22)
		node [pos=0.5, above, sloped] {\textbf{indirect via masking by indicator}};
		
%	\draw[line width=0.75em, color_complex,-Stealth] (7,2.9) -- (7.25, 4.9);
%	\draw[line width=1.25em, color_complex] (7,2.9) -- (7.15, 4.1);
	\draw[line width=0.75em, black!40!white,-Stealth] (6.75,2.5) -- (7, 4.8)
		node [pos=0.35, sloped, anchor=north, yshift=-0.7em] {information content};
\end{tikzpicture}
		\end{minipage}
		\hfill
		\begin{minipage}{0.34\textwidth}			
			\caption{Illustration of direct vs.\ indirect graph-discovery.
				The green (solid) arrow requires only the extraction of low-complexity knowledge,
				red (dashed) arrows require the extraction of high-complexity knowledge.
				The information-content of time-resolved indicators will typically increase
				with larger sample-size,
				while the information
				contained in causal graphs does not.
				So the gradient in information-content (gray arrow, wide),
				for larger data-sets, points steeply upwards.}
			\label{fig:illustrate_direct_vs_indirect}
		\end{minipage}
	\end{figure}
	
	To approach these fundamental challenges, we seek to adapt two guiding principles:
	Work locally in the graph (gL) and test directly (D). By testing directly, we mean,
	avoid the intuitive but perilous detour through recovering the (often complex) information
	contained in the regime-structure, but approach the eventually required result directly
	(this will become much clearer with the concrete approach in §\ref{sec:dyn_indep_tests}).
	In reference to these principles we will refer to our approach as a gLD-framework.
	
	\begin{rmk}
		Intuitively, one would first recover the regime-structure, then per-context graphs.
		Testing for graphs directly also means the recovery of regime-structure (now the second step)
		can employ the graphical results, for example to optimize signal-to-noise, for better results.
		Indeed there are multiple simplifications possible, cf.\ §\ref{apdx:indicator_resolution}.
		The present paper focuses on the (direct) graph-discovery step of this scheme.
	\end{rmk}
	
	Besides these statistical challenges, there is also a practical one:
	Non-stationarity encompasses a vast range of phenomena that each may be best approached
	through very different methods. Thus a systematic approach to structure the problem
	and \emph{fit together} a multitude of specific methods and ideas is of great relevance in practice.
	To this end,
	we try to keep the framework modular, making the relation of its separate parts to
	conventional CD methods, independence-testing, CPD and clustering clearer. This modularity also
	applies to the assumptions about regime-structure:
	As the reader may have noted at the list of examples above,
	we emphasized how similar problems arise from patterns in time, in space or in other parameters. Indeed,
	our framework allows for the precise kind of pattern to be specified and leveraged for detection
	almost entirely separate from other aspects of the setup.
		
	\colorlet{boxcolor}{black!30!blue}
	\colorlet{datacolor}{black!30!green}
	\colorlet{discoverycolor}{black!30!orange}
	\colorlet{compositioncolor}{black!30!purple}
	\usetikzlibrary{calc}
	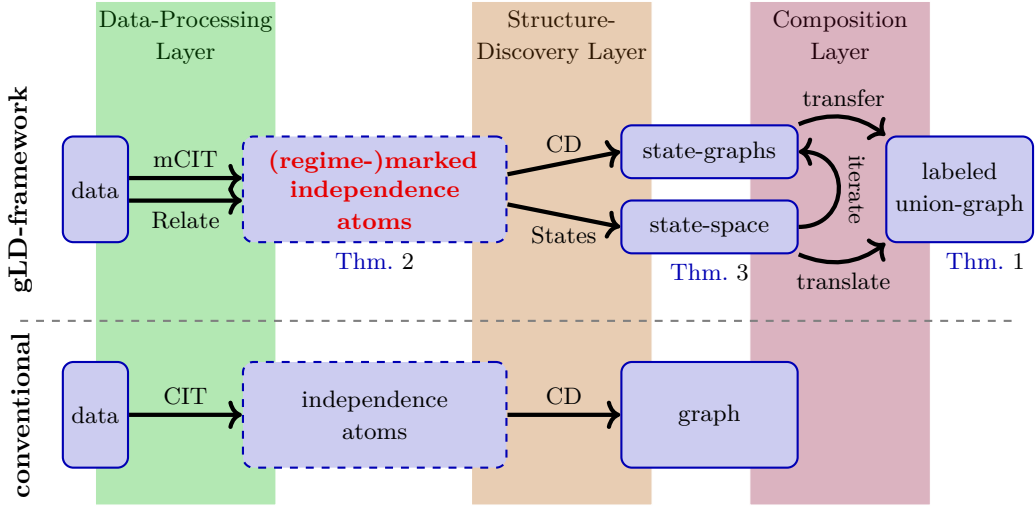
\begin{figure}[ht]
		\colorlet{boxcolor}{black!30!blue}
\colorlet{datacolor}{black!30!green}
\colorlet{discoverycolor}{black!30!orange}
\colorlet{compositioncolor}{black!30!purple}
\usetikzlibrary{calc}

\providecommand{\introFrameworkScale}{1.0}
\providecommand{\introFrameworkIterShift}{2em}
\providecommand{\introFrameworkThmMCIT}{Thm.\ \ref{thm:mCIT}}
\providecommand{\introFrameworkThmStateSpace}{Thm.\ \ref{thm:state_space_construction}}
\providecommand{\introFrameworkThmUnionGraph}{Thm.\ \ref{thm:core_algo}}

\begin{tikzpicture}
	[
	scale=\introFrameworkScale,
	block/.style={
		rectangle,
		draw=boxcolor,
		thick,
		fill=boxcolor!20,
		align=center,
		rounded corners,
		minimum height=4em
	}
	]
	\draw[fill=datacolor!30, draw=datacolor!0] (1.0,-4.2) rectangle (3.4, 2.5);
	\draw (2.2, 2.5) node [anchor=north, align=center, color=datacolor!25!black]
	{\small Data-Processing\\\small Layer};
	\draw[color=datacolor!25!black] (1.05, 1.575) -- (3.35,1.575);
	
	\draw[fill=discoverycolor!30, draw=discoverycolor!0] (6.0,-4.2) rectangle (8.4, 2.5);
	\draw (7.2, 2.5) node [anchor=north, align=center, color=discoverycolor!25!black]
	{\small Structure-\\\small Discovery Layer};
	\draw[color=discoverycolor!25!black] (6.05, 1.575) -- (8.35,1.575);
	
	\draw[fill=compositioncolor!30, draw=compositioncolor!0]
	(9.7,-4.2) rectangle (12.1, 2.5);
	\draw (10.9, 2.5) node [anchor=north, align=center, color=compositioncolor!25!black]
	{\small Composition\\\small Layer};
	\draw[color=compositioncolor!25!black] (9.75, 1.575) -- (12.05,1.575);

	%------------------------------------------------------------------------------------------

	\draw (0,-3) node (hCD) [rotate=90] {\textbf{conventional}};
	
	\draw (1,-3) node (dataCD) [block] {\small data};
	\draw[->, ultra thick] (dataCD.east) --	+(1.5,0)
	node [above, pos=0.5] {\small CIT}
	node (istructCD) [block, anchor=west, dashed, text width=9.3em] {\small independence\\\small atoms};
	\draw[->, ultra thick] (istructCD.east) -- +(1.5,0)
	node [above, pos=0.5]{\small CD}
	node (graphCD) [block, anchor=west, text width=6em] {\small graph};
	
	\draw[color=gray, dashed, thick] (0,-1.75) -- +(13.25,0);
	%------------------------------------------------------------------------------------------
	
	\draw (0,0) node (hGLDF) [rotate=90] {\textbf{gLD-framework}};
	
	\draw (1,0) node (dataGLDF) [block] {\small data};
	\draw (dataGLDF.east)+(1.5,0) node (transientGLDF) [block, anchor=west, dashed, text width=9.3em]
	%{\small marked independences\\\small indicator relations};
	{\small\textbf{\textcolor{red!90!black}{(regime-)marked}}\\
		\small\textbf{\textcolor{red!90!black}{independence}}\\
		\small\textbf{\textcolor{red!90!black}{atoms}}};	
	\draw[boxcolor] (transientGLDF.south) node[anchor=north]
	{\small \introFrameworkThmMCIT};
	\draw[->, ultra thick] (dataGLDF.east)+(0,+0.15) --
	node [above, pos=0.5] {\small mCIT} 
	($(transientGLDF.west)+(0,+0.15)$);		
	\draw[->, ultra thick] (dataGLDF.east)+(0,-0.15) --
	node [below, pos=0.5] {\small Relate}		
	($(transientGLDF.west)+(0,-0.15)$);		
	\draw[->, ultra thick] ($(transientGLDF.east)+(0,0.2)$) -- +(1.5,0.3)
	node [above, pos=0.5]{\small CD}
	node (graphGLDF) [block, minimum height=2em, anchor=west, text width=6em]
	{\small state-graphs};
	%\draw[boxcolor] (graphGLDF.north) node[anchor=south, align=center]
	%{by CD\\consistency};
	\draw[->, ultra thick] ($(transientGLDF.east)+(0,-0.2)$) -- +(1.5,-0.3)
	node [below, pos=0.5, yshift=-0.6em, align=center]{\small construct\\[-0.3em]\small state-space}
	node (statesGLDF) [block, minimum height=2em, anchor=west, text width=6em]
	{\small state-space};
	\draw[boxcolor] (statesGLDF.south) node[anchor=north, xshift=-0.3em]
	{\small \introFrameworkThmStateSpace};
	\draw[->, ultra thick] (statesGLDF.east) to[bend right=90, distance=\introFrameworkIterShift] 
	node [right, pos=0.5, rotate=-90, anchor=south] {\small iterate}
	(graphGLDF.east);
	
	\draw (11.5, 0) node (resultGLDF) [block, anchor=west]
	{\small labeled\\\small union-graph};
	\draw[boxcolor] (resultGLDF.south east) node[anchor=north east]
	{\small \introFrameworkThmUnionGraph};
	\draw[->, ultra thick] (graphGLDF.north east) to[bend left=30]
	node[above, midway]{\small transfer}
	(resultGLDF.north west);
	\draw[->, ultra thick] (statesGLDF.south east) to[bend right=30]
	node[below, midway]{\small translate}
	(resultGLDF.south west);
\end{tikzpicture}
		\caption{
			Framework Architecture.
			Blue boxes are abstract concepts encoding knowledge
			(dashed: commonly lazily evaluated).
			Arrows represent algorithmic components
			converting such knowledge (see §\ref{sec:dynamic_cd});
			these components are highly modular:
			Conventionally one may combine different CD-algorithms
			with different CITs; in our framework individual
			components enjoy similar independence.
			By independence \inquotes{atoms} we refer to individual independencies
			(no relations like d-separations are implied at that stage).
			The core idea is the introduction of a new abstraction of independence atoms
			that additionally encode context-information. Its careful choice enables our framework.
		}\label{fig:architecture}
	\end{figure}
	
	\paragraph{Realization of the guiding principles:}
	Having laid out the problem, its challenges, and conceptual bounds for its realization,
	we can finally discuss how to actually approach the solution of the problem.
	Figure \ref{fig:architecture} illustrates the high-level architecture of our framework.
	First focus on the lower panel conceptualizing the dataflow in conventional CD:
	There are two abstract layers. A CD-logic, which operates not on data directly, but rather
	on independence-statements, and a conditional independence test (CIT) 
	transforming data into independence-statements.
	The core idea of our framework is to modify the CIT to also bear the
	main burden of pattern-discovery. This means we can retain CD almost unchanged
	(both in theory and code-implementation; both IID and time series), and have produced a problem
	local in the graph already.
	All scaling with the regime-resolution's complexity is confined to the
	data-processing layer which already produces bounded (and in practice low) complexity output.
	The direct-testing principle thus will be realized in the data-processing layer.
	
	Evidently the main difficulty is now in actually realizing such a modified CIT that checks for the
	existence of a pattern in (conditional) independence rather than just independence vs.\ dependence.
	This test can avoid scaling with information in the regime-structure, because
	an output deciding between three alternatives -- independence, dependence or existence of
	any regime-structure realizing both --
	turns out to be (almost; §\ref{sec:indicator_translation}) enough information to complete HCCD.
	This main difficulty of producing suitable modified CITs
	can further be subdivided into sub-tasks (see §\ref{sec:dyn_indep_tests}).
	Among other advantages, this simplifies the identification of CPD, CIT and clustering related subtasks,
	and isolates the underlying problem sufficiently to make a direct testing approach not just
	feasible, but makes it manifest itself intuitively.
	Finally, this systematic analysis of the task also reveals two unavoidable limitations
	to what kinds of results are possible for HCCD, irrespective of what methodology is employed.
	These limitations are related to testing for the \emph{existence} of non-IID structure,
	and to the combination of accepting independence while rejecting IIDness for \emph{independent regimes}.
	Knowing of these fundamental restrictions,
	we can then understand theoretically and numerically the incurred trade-offs one has to consider,
	which also helps making informed hyperparameter decisions.
	
	The ideas outlined in the previous paragraph modify what is labeled as \inquotes{data-processing layer}
	in Fig.\ \ref{fig:architecture}.
	However, we also have to systematically evaluate the additional information provided by this
	modified data-processing layer.
	In particular, we have to recover a space of contexts or \inquotes{model-states}
	(even the number of changes in the independence-structure need not be the same as the number of changes in the model;
	arrow labeled \inquotes{construct state-space} in Fig.\ \ref{fig:architecture}).
	Finally, other than for conventional CD, we have to compose a meaningful HCCD output:
	We have to co-ordinate multiple applications of CD (there are evidently multiple graphs to produce).
	Observed changes in the independence-structure can be attributed non-trivially
	to model-properties.
	Some edge-orientation information can be transferred between contexts, other
	information about skeleton and orientations in context(-specific)-graphs is determined by the underlying
	(IID\Slash{}stationary) CD-algorithm; for example application with FCI will output
	partially directed acyclic graphs (PAGs).
	At least in the (suitably) acyclic case, the result can then be represented as a labeled
	union-graph: a single graph where context-specifically changing links are labeled
	(or marked in some other way, see Fig.\ \ref{fig:intro_toy_model}).

	\paragraph{Contributions:}
	The main contributions of this paper are to
	\begin{itemize}
		\item
		analyze the choice of useful inductive biases, improve the understanding of the HCCD problem and its challenges,
		and lay out guiding principles to avoid apparent difficulties on a conceptual level (§\ref{sec:intro}).
		\item
		devise a framework (Fig.\ \ref{fig:architecture}) based around modifying the independence-testing procedure that can realize these
		principles in a highly modular way.
		\item
		analyze in detail the realization of the modified independence-testing (§\ref{sec:dyn_indep_tests})
		theoretically and in (sub-system level) numerical experiments;
		the analysis includes the uncovering of fundamental limitations
		to the feasibility of HCCD
		(Prop.\ \ref{prop:imposs_A}, \ref{prop:imposs_B}, Lemma \ref{lemma:necessity_of_dyn_indep_struct}), practical resolutions via trade-offs (Fig.\ \ref{fig:tradeoffs}) and meaningful
		convergence statements (Thm.\ \ref{thm:mCIT}), as well as potential paths forward (§\ref{apdx:mCIT_future_work}).
		\item 
		analyze in detail the specification and realization of the state-space recovery
		(§\ref{sec:indicator_translation}, Thm.\ \ref{thm:state_space_construction}).
		This includes a meaningful notion of identifiable states (§\ref{sec:identifiable_states})
		and theoretically sound, practically feasible algorithmic implementations (§\ref{sec:indicator_translation}) for the acyclic and modular case.
		We also provide a proof-of-concept implementation of the required indicator-relation
		test (§\ref{apdx:implication_testing}).
		\item 
		analyze in detail the properties of the composition algorithm (§\ref{sec:dynamic_cd},
		Thm.\ \ref{thm:core_algo}).
		This includes a proof
		of its soundness given an oracle for the modified independence-tests
		and sound state-space recovery (previous two points).
	\end{itemize}
	The structure of our framework may, after the fact, appear evident and unsurprising,
	however, there are a great many plausible approaches to the problem, and it seems to be
	primarily the \emph{correct structure} that decides, whether ultimately everything fits
	together and produces useful results on finite data.
	
	To demonstrate the feasibility and validate our ideas about the inductive bias incurred,
	we further
	\begin{itemize}
		\item
		implement, for each module of our framework, a simple baseline version
		and combine them into a reference method.
		\item
		perform extensive numerical experiments and comparisons to other approaches
		to test our expectations concerning the behavior with a number of parameters.
		\item
		discuss the results of these experiments, and some weaknesses unveiled.
	\end{itemize}
	Our numerical experiments focus on simple settings
	(continuous variables, linear models, no non-graphical non-stationarities),
	but we believe that our very modular approach should generalize well to more complex settings.
	We include discussions of potential future work, supporting this idea.
		
	\begin{rmk}[Nomenclature]
		The terms context, state and regime are sometimes assigned rather
		specific meanings in the literature.
		This should not cause confusion here, as the meaning in formal
		results is always fixed and in informal statements sufficiently clear
		from context. Nevertheless, we briefly outline, how and
		why we use multiple expressions for similar concepts:
		We refer to regions of the index-set (for example in time)
		that are described by the same model as regimes, the resolution
		of regimes in the index-set (for example in time) as regime-structure
		and the individual models as context-specific (they are derived from
		a meta-model given the context of its regime). If a link changes
		(is present in some, but not all contexts), then we refer to the presence
		of the link as local context and its local regime-structure as binary context-indicator.
		Alternatively the meta-model may be described as a finite state machine,
		with its state setting the context of for observations.
		Generally the details of this state machine will depend on
		exogenous unobserved factors and cannot be recovered from observations.
		Hence we focus on results about the context-enriched causal system instead.
		However, a causal modularity assumption can be employed to
		structure the space of states (and thus the set of contexts)
		effectively. In cases where such assumptions relating
		contexts to states can be leveraged, we preferentially use the term state
		(rather than context).
	\end{rmk}

	\section{Related Literature}\label{sec:related_lit}
	
	The notion of causality we use is the one developed by Pearl and others, see \eg
	\citep{PearlBook, Elements, Glymour2016}; other approaches to formally capture causation, for example from
	temporal properties \citep{granger1969investigating}, are also employed by some references below.
	\textbf{Causal discovery} algorithms have been a subject of study for decades. Our work builds on constraint-based
	algorithms \citep{spirtes2001causation, SGSalgo, PCalgo, pcmci, pcmci_plus, lpcmci}, but also other approaches like score-based
	methods \citep{Meek1997GraphicalMS, Chickering2002OptimalSI}, especially when using local scores,
	may be compatible with many of the ideas presented.
	Besides the study of different data-setups
	like (discrete) time series \citep{pcmci, granger1969investigating}, 
	also event-based models like \citep{Xu2016LearningGC} or, especially more recently, continuous-time
	models \citep{manten2024signature,boeken2024dynamic} are studied in the literature.
	Below, we focus in particular on variants that take into account multi-context data (like multiple datasets
	or multiple contexts) from a mostly constraint based view-point. We study the case of observational data,
	interventional data from multiple contexts is studied for example by \citep{peters2016causal,Li2023CausalDF}.
	
	Adding non-stationarities to a problem \textbf{involves a large number of choices}:
	The considered functional form \eg in time (like smooth functions, discrete contexts, etc.),
	the affected aspects of the model (in causal terms, for example noises, parent-sets, mechanisms
	and relations between such local choices) and the considered prior knowledge
	about these aspects.
	
	But beyond these aspects of the system, there are further many plausible \textbf{goals} to be considered:
	Non-stationarity affects the quality of results produced by tests
	often designed for stationary (or IID) data -- for example the rates at which CITs report
	false positives or false negatives -- and thus it can be studied as a problem of
	robustness of testing \citep{CD-NOD,Saggioro2020}.
	A related question is that of the \emph{interpretation} of heterogeneous information in terms of
	shared model-information, for example a \inquotes{union-graph} \citep{Saeed2020,strobl2023causal}.
	Multi-context data does contain additional information about the joint (here: stationary)
	aspects of a system \cite{bareinboim2012transportability,CD-NOD,JCI,JPCMCI}.
	An often studied aspect of this is information about edge-orientations\footnote{Additional edge-orientations
	can also be found for example by restricting function-classes of mechanism (\eg assuming linearity) \citep{shimizu2006linear},
	multi-context information provides an additional pathway to such information.},
	for example \inquotes{Constraint-based causal Discovery from heterogeneous/non-stationary Data}
	(CD-NOD) \citep{CD-NOD} (with a recent more time series focused variant \citep{CD-NOTS})
	uses kernel-tests to analyze which mechanisms vary (how strong);
	this information can under a causal modularity assumption be interpreted as orientation information.
	Finally, it is also a quite reasonable goal to extract information about non-stationarities.
	This includes differences between discovered contexts, which we discus below first for
	known context-assignments (where also some further details about extracting joint information
	like orientations, see above, are given).
	Then we give a similar discussion for the case of hidden contexts.
	In this case, additionally the question about the extraction of the form of non-stationary
	behavior, for example as a function of time, arises.
	
	\textbf{Given knowledge about context-assignments}, this knowledge can generally be
	exploited to get orientation information and locate changes in the graph \citep{CD-NOD,JCI,JPCMCI},
	much of which was systematized by the \inquotes{joint causal inference} (JCI) framework \citep{JCI}.
	The basic idea is, to include the context as a variable -- possibly with special properties
	like presumed exogeneity -- into CD, which both localizes the context-dependence in the graph
	and makes new orientation information available from new v-structures \cite{Verma1991}
	(intuitively, if a mechanism $X\rightarrow Y$ is context-dependent, then $Y$ is context-dependent,
	while $X$ is not, so that the orientation $X\rightarrow Y$ can be inferred,
	intuitively from $Y$ changing and by CD from the unshielded collider $X \rightarrow Y \leftarrow C$).
	Context-assignment can be highly flexible, for example
	\citep{Cai2021THPsTH, Zhou2013LearningSI, Xu2016LearningGC} employ a graphical structure also on context-information
	(reflecting for example a device-network-topology), with a formalization that ultimately draws on
	the theory of Hawkes-processes \citep{delattre2016hawkes} (to study event data),
	while \cite{gao2023causal}
	uses periodic structures, but many methods specialize to
	persistent temporal regimes \citep{Saggioro2020} or multiple data-sets \citep{bareinboim2012transportability,SelectionVars,JCI,JPCMCI}.
	CD-NOD \citep{CD-NOD} can treat the second case and additionally the case of quantitative (effect strength)
	temporal variation,
	where it does not require prior information	about precise regime-locations, see below.
	Further also knowledge about \emph{context-specific} graphs (as opposed to union-models) can be exploited
	for example for down-stream tasks like effect-estimation \citep{bareinboim2012transportability, SelectionVars}.
	Indeed this problem has also been studied for example by \citep{EndoMethod} and from the perspective
	of specifically finding differences (for example in causal effects) between contexts by \citep{assaad2024causal}.
	
	\textbf{Not knowing the context-assignments in advance}, some of the above ideas still indirectly apply,
	for example \citep{JPCMCI} is able to deconfound certain context-related (in the sense of \citep{JCI})
	situations. The full recovery of context-specific graphs and context-assignments (the main topic of the present paper)
	also has been studied before. A simple, but very practical, approach is to use some sort of sliding window-procedure.
	The work pursuing the most similar goals to ours is probably \citep{Saggioro2020},
	who develop an algorithm called Regime-PCMCI to discover different regimes and associated graphs.
	The basic idea there being to (starting from an initial guess) iteratively improve regime-assignments by fitting causal models
	per context and optimizing assignments for optimal predictive quality. The requirement to fit a model puts
	some restrictions on possible models (no confounders and edges must be orientable, \citep{Saggioro2020} study
	time series without contemporaneous links and without confounders), but it allows to discover also
	shifting (but stable within each context) causal effects, beyond graphical changes (which we focus on).
	There are also comparable parametric approaches \cite{malinsky2019learning}.
	Albeit for very different types of data \citep{BalsellsRodas2023}, propose a approach which also
	employs model (density) fitting plus a parametric (hidden Markov model-like) assumption on regime-assignments.
	A final remark, especially compared to typical \citep{JCI} settings, should be that we study a local
	and binary (dependent or independent) problem, so that the setup for our purposes can always
	be simplified (without loss of generality) to the study of binary context-variables.
	
	This binary (in particular categorical) nature of the context means we are looking for \textbf{qualitative changes}.
	It is of course also reasonable to investigate quantitative changes (drifts or more complex dependencies on
	a sample-index like time), see for example \citep{mameche2025spacetime}; interestingly their approach seems to
	be local in the graph (this also applies to CD-NOD, see next paragraph).	
	Ultimately one is in such cases often interested in qualitative questions, like \inquotes{is there
	an upward\Slash{}downward trend}, at which point the philosophy behind our framework might be helpful.
	Generally the idea of testing such properties directly is of course not new, see for example
	\citep[§10.13]{VapnikEstimation}, but we are not aware of this idea having been applied to
	independence-testing, CD (as studied in this paper) or even effect-estimation questions (as posed above).
	
	For the temporal case \textbf{CD-NOD} \citep{CD-NOD} specializes precisely in finding quantitative 
	changes. While this formally includes qualitative changes as a special case, no information about such
	a qualitative nature is extracted. Indeed this might also prove extraordinary challenging in a kernel-setting.
	CD-NOD and our approach in many ways supplement each other and it would be an interesting question for future
	work, if they actually synergize well in practice: CD-NOD can in particular extract orientation-information
	that would not be conventionally accessible by constraint-based causal discovery methods, and the kernel-based
	implementations it provides are well-suited for the analysis of quantitative smooth changes;
	our approach on the other hand focuses on extracting qualitative information and wraps CD methods
	in a way that could also incorporate CD-NOD's orientation information in principle.
	CD-NOD introduces new information to what we label the data-processing layer in Fig.\ \ref{fig:architecture};
	while not phrased as a modification of the independence-structure in \citep{CD-NOD}, 
	the orientation-information produced by in their test could probably be systematically included
	in such a way.

	An important conceptual step in our framework is to modify the notion of \textbf{independence-structure}.
	Indeed there is a related class of concepts loosely referred to as \inquotes{local independence}
	(local in the indexing, \eg time, not in a causal graph)
	according to \citep{henning1989meanings}
	\inquotes{the concept of local (conditional) independence has been ascribed at least a dozen definitions in the testing literature}.
	Typically, the assumption is that (in some formal sense)
	there is a latent $L$ such that knowing $L$ we have for example $X\independent Y | L$
	while $X \dependent Y$, where $L$ contains little enough information (in some sense) for this problem to become testable.
	Compare this to our setup, which can instead be formulated as there being an $L$
	such that values of $L$ split the data
	into subsets $X\independent Y | L=0$ while  $X\dependent Y | L=1$.
	Local independence ideas in this sense have played a role in causal methodology for example in
	\citep{JPCMCI}. \citep{rodas2021causal} introduces a related idea of conditional stationarity for quantitative properties.
	There are also modifications of (graphical) independence-models to account for known context-assignments,
	especially (but not only) for categorical variables as \inquotes{stratified graph} \citep{stratified_graphs}
	and (closely related) \inquotes{labeled DAG} (LDAG) \citep{LDAG_logical, LDAG_learning}.
	For context-specific knowledge, the translation of independence information into causal information
	can become quite non-trivial \citep{EndoMethod,EndoTheory}.
	On the practical testing of conditional independence, there is of course likewise a large corpus of
	literature; a comprehensive review is not possible within the scope of this work.
	However, we want to highlight two examples relevant to our discussion:
	Already \citep{Aitkin1985EstimationAH} discuss the testing of independencies on mixtures.
	Further there are known impossibility-results \citep{shah2020hardness} beyond the (pattern-detection related)
	ones we encounter in §\ref{sec:dyn_indep_tests}.
	
	The questions we ask may also be seen as (and are certainly related to) aspects of \textbf{multi-scale} problems.
	Such problems are ubiquitous in nature (\eg \citep{Kannenberg2019})
	and have formal approaches ranging from multi-level statistics \citep{Gelman2006} to
	causal effect-estimation in the Fourier-domain\footnote{
		Also the \inquotes{driving force estimation} of CD-NOD \citep{CD-NOD},
		quantifies context-influence via suitable eigenvectors.
		While not a Fourier\Slash{}Laplace-transform, this also is a basis-change in function-space,
		with focus on an additional relevance-sorting to make projection to a low-dimensional subspace meaningful.}
	\citep{reiter2023causal}.
	Further, we detect patterns in all-variable samples (like regions in time),
	while there is also the orthogonal problem of finding features representing variables in-sample in the
	approaches called causal feature learning (CFL) \citep{CFL,CFL_ElNino} and causal representation learning
	(CRL) \citep{scholkopf2021toward}.
	
	Finally, we will naturally encounter questions from a wide range of \textbf{related fields}.
	These include change-point detection \citep{truong2020selective}
	(or more generally signal processing), also by causal principles \citep{gao2024causal}
	(which could be helpful for the second step of indicator-resolution,
	see also §\ref{apdx:indicator_resolution}), clustering.
	Hyperparameter choices have relations to over-\Slash{}under-fitting questions
	and the problem is inherently related to missing-value problems, so unsurprisingly
	many approaches \citep{Saggioro2020, BalsellsRodas2023} are implicitly or explicitly
	related to the EM-algorithm \citep{dempster1977maximum}
	and it seems plausible that such approaches could improve our current
	implementation of §\ref{sec:weak_regimes}.
	Also ideas from the shape\Slash{}pattern recognition literature \citep{arias2005near}
	are related, even though the typical problem-statements seem to involve a
	signal vs.\ size behavior (\eg keeping the integral over the signal fixed) that is not applicable
	to independence-testing (there is no such thing as a \inquotes{stronger} independence signal, as this is the
	null hypothesis).
	An important connection that perhaps should be investigated further in future work is that to Vapnik--Chervonenkis (VC)
	theory. Many of the properties we encounter in §\ref{sec:dyn_indep_tests} seem to mirror our expectations
	from VC-theory, yet they are surprisingly difficult to make concrete. The fundamental limitation
	of Prop.\ \ref{prop:imposs_A} can be intuitively understood as a maximum amount of asymptotic
	pattern-information per sample, which seems closely related to VC-entropy \citep[§7.A1]{VapnikEstimation}.

	\section{Underlying Model}\label{sec:models}
	
	Causal properties are properties of the underlying generating model, not of the
	underlying distribution \citep[§1.5 (p.\,38)]{PearlBook}.
	We want to study causal properties of non-IID (or non-stationary time series) systems
	and it is thus of great importance,
	to clearly specify what such a model may look like and which of its properties we are interested in.
	We start from structural causal models (SCMs) in the usual sense §\ref{sec:standard_scm} and generalize
	them in a way that formally captures the properties relevant to changing causal graphs §\ref{sec:model_defs}.
	If the graph associated to our model can change, then the underlying model must transition
	between different states §\ref{sec:model_states} each described by a simpler model.
	Not all such model-states will occur in data, and not all of them can necessarily be distinguished by
	a CD-algorithm in practice; to make the second point clearer, we will first discuss CD in a standard
	sense in §\ref{sec:stationary_cd}, before returning to model-states in §\ref{sec:identifiable_states}.
	The discussion of CD in a standard sense §\ref{sec:stationary_cd} also gives a more abstract perspective on
	constraint-based CD as an abstract machinery transforming independence information into information about
	causal graphs, which will simplify arguments for the discussion of HCCD in §\ref{sec:dynamic_cd}.
	Finally, we outline assumptions and goals for causal discovery in this setup §\ref{sec:goals}.
	Throughout this section we focus on the case where the regime-structure occurs at random,
	but is fixed as part of the model; details on the modeling of randomness of the regime-structure
	itself in a meta-model are discussed in the appendix §\ref{apdx:model_details}.

	\subsection{SCMs}\label{sec:standard_scm}
		
	We start from the standard formulation via structural causal models (SCM) \citep{PearlBook,Elements}:
	\begin{notation}\label{def:SCM}
		For some finite index set $I$, fix a set of \defName{endogenous variables}
		$\{X_i\}_{i\in I}$,
		taking values in $\val{X}_i$,
		with $i \in \mathcal{O} \subseteq I$ observed (if $I = \mathcal{O}$ the model is called causally sufficient),
		and mutually independent \defName{exogenous noises} (hidden) $\{\eta_i\}_{i\in I}$,
		taking values in $\val{N}_i$.
		We write $V := \{X_i | i\in I\}$ for the set of all endogenous variables,
		$U := \{\eta_i | i\in I\}$ for the set of all exogenous noises.
		For $A \subseteq V$, denote $X_A = (x_j)_{j\in A} \in \val{X}_A := \prod_{j\in A} \val{X}_j$,
		and similarly for $B \subseteq U$,
		further $\val{X} := \val{X}_V$ and $\val{N} := \val{N}_U$.
		
		For each $i \in I$ there is a set of \defName{parents}
		$\Pa_i \subseteq I \setminus \{i\}$ and a \defName{mechanism}
		$f_i : \val{X}_{\Pa_i} \times \val{N}_i \rightarrow \val{X}_i$
		such that the endogenous variables satisfy the \defName{structural equations}
		$X_i = f_i(X_{\Pa_i}, \eta_i)$ relative to its parents and noise-term.
		Parent-sets are assumed to satisfy a suitable minimality condition,
		\eg \citep[Def.\ 2.6]{BongersCyclic}.
		The model is called additive if each $f_i$ can be written as		
		$f_i(X_{\Pa_i}, \eta_i) = \eta_i + \sum_{j\in\Pa_i} f_{ij}(X_j)$.
		It is called linear if each $f_i$ is (affine-)linear.
		
		In the causally sufficient case,
		the causal graph $G$ is a graph on the vertex-set $\mathcal{O} = I$
		with a directed edge from $i$ to $j$ iff $j \in \Pa_i$.
		In the presence of latent confounders $\mathcal{L} = I \setminus \mathcal{O} \neq \emptyset$,
		bi-directed edges are added between $i$ and $j$ iff
		there is a latent $l \in \mathcal{L}$ and directed paths
		from $X_l$ to $X_i$ and $X_j$ where only the respective endpoints
		($i$ and $j$) are observed.
	\end{notation}

	\subsection{Non-Stationary Models}\label{sec:model_defs}

	Graphical changes are inherently qualitative: Each (directed) link in the graph may be
	present or not. In particular, explicit time-dependence of the causal graph
	can always be described by a collection of \inquotes{indicators} that indicate the presence
	(or absence) of any particular link. Links that never change have a trivial indicator:
	\begin{notation}
		To make the distinction to the variable-indices $I$ clearer,
		we denote the sample-indices by $T$, where \asswlog $T\subseteq \mathbb{Z}$.
		We refer to $t\in T$ as time for the same reason,
		but $T$ can be any index set and may also encode spatial or other patterns
		(see §\ref{sec:persistence_formal}).
	\end{notation}
	\begin{definition}[Indicators]\label{def:indicator}
		An indicator is a mapping $R : T \rightarrow \{0,1\}$.
		An indicator $R$ is	trivial, if $R$ is constant (as a mapping).
	\end{definition}

	We will later put constraints on the form of indicators enforcing a kind of pattern in $T$ that limits
	the information-content of $R$ (§\ref{sec:persistence}) --
	for example perpersistence in time limiting the frequency of change.
	For additive models, it is straightforward to define a non-stationary
	model with changes in the causal graph modeled by indicators:
	
	\begin{definition}[Non-Stationary Additive SCM]\label{def:model_nonstationary_scm}
		Given variables $\{X_i\}_{i\in I}$, (possibly trivial)
		binary indicators $\{R_{ij}\}_{i,j \in I}$ and time-independent
		structural mappings $f_{ij} : \mathbb{R} \rightarrow \mathbb{R}$,
		the non-stationary additive SCM $M$ is given on
		$\{X_i\}_{i\in I}$ by the modified structural equations
		$X^t_i = \eta^t_i + \sum_{j\in I} R_{ij}(t) \times f_{ij}(X_j^t)$.		
		
		\emph{Remark:} We formally sum over all $j \in I$ (we have not yet defined parent-sets).
		Because the $R_{ij}$ are allowed to be trivially equal to zero,
		we will \asswlog assume all $f_{ij}$ are non-constant.		
	\end{definition}
	Parent-sets	may then be defined by non-vanishing of indicators:
	\begin{definition}[Non-Stationary Parents]
		The parents at time $t$ of $X_i$ in the non-stationary additive SCM $M$
		are the elements of the set
		\begin{equation*}
			\Pa_i^t = \{ j \in I | R_{ij}(t) \neq 0 \} \txt.
		\end{equation*}
		We define the non-stationary causal graph $G(t)$ relative to $\Pa_*^t$
		as before (see notations \ref{def:SCM}).
	\end{definition}
	
	\begin{rmk}
		If all $R_{ij}$ are constant on an interval $[a,b] \subseteq T$,
		then for $t\in [a,b]$ the non-stationary graph $G(t)$ is the
		causal graph of the model obtained for the system restricted to $[a,b]$
		in the standard sense.
	\end{rmk}
	
	Non-additive models can be described similarly, with details provided in §\ref{apdx:models_non_additive}.

	\subsubsection{Limitations}
	
	This description does not capture other, potentially practically relevant, non-stationarities like
	parameter-drifts or other explicit time-dependence of mechanisms or noises.
	One might be inclined to start from a very general model-definition (like explicitly time-dependent
	mechanisms), however, one has to be very careful to avoid tautologies:
	To learn something from a model, it must make interesting properties explicit.
	A model with arbitrary time-dependence of mechanisms may implicitly capture the presence of edges,
	for example by (non-)constness in a particular parent at a fixed time $t$, but this property
	(which we are ultimately interested in) itself is not explicit from its structure.
	For this reason, we decided to give a minimal extension beyond standard SCMs that explicitly
	captures what we need; every state (see §\ref{sec:model_states}) then has fixed parent-sets,
	per state models may still be modified to capture other types of non-stationarity, but this
	approach separates graphical and non-graphical changes in a clear hierarchy.
	Indeed it may also be useful for the interpretation of results to further
	consider changes in noise-terms separately (extending to general probability-kernels as mechanisms
	in another step).
	Nevertheless, the other view-point -- to generalize first --
	is of course also useful, especially in cases where graphical
	changes are not of interest.

	\subsection{Model States}\label{sec:model_states}
	
	The above models become SCMs in the standard sense once we fix the values of the indicators.
	There is only a finite number of combinations of values of the non-trivial indicators,
	and many causal properties of the model may be understood from knowing the
	causal graph of each such \inquotes{state}.
	
	\begin{definition}[States]\label{def:states}
		The set of non-trivial indicators is
		\begin{equation*}
			\NonTrivialIndicators\halfquad:= \halfquad\big\{ \halfquad R_{ij}\halfquad |\halfquad
			 i\neq j \text{ and } R_{ij} \not\equiv\const \halfquad\big\}
			\quad
			\txt{with $\kappa = \setelemcount{\NonTrivialIndicators}$ elements.}
		\end{equation*}
		Denoting the $\kappa$ elements of $\NonTrivialIndicators$ by
		$R_1, \ldots, R_{\kappa}$
		we may capture the state of the model by
		\begin{equation*}
			\sigma :
			T \rightarrow \{0,1\}^\kappa,
			t \mapsto \big( R_{1}(t), \ldots, R_{\kappa}(t) \big)
			\txt.
		\end{equation*}
		We call the $2^\kappa$ elements of $S := \{0,1\}^\kappa$
		potential (as opposed to reached, Def.\ \ref{def:reached_states}) system states
		and write $T_s := \sigma^{-1}(\{s\}) \subseteq T$ for the time-indices realizing state $s$.
	\end{definition}
	States capture all qualitative graphical causal non-stationarity (proof in §\ref{apdx:model_detail_proofs}); we will represent graphs $G(t)$ by some class $\GClasses$ (like PAGs), so $G(t)\in\GClasses$
	(see §\ref{sec:stationary_cd}).
	\begin{lemma}[State-Factorization]\label{lemma:graphs_factor_through_states}
		The map $G: T \rightarrow \GClasses, t \mapsto G(t)$ factors through $\sigma:T\rightarrow S$
		(\inquotes{$G$ can be written as a function of $s\in S$}),
		\ie there is a unique mapping $G_{(s)}: S \rightarrow \GClasses$ such that
		$G(t) = (G_{(s)} \circ \sigma)(t) := G_{(s)}(\sigma(t))$.
		We will by slight abuse of notation write $G(s)=G_{(s)}(s)$.
	\end{lemma}
	This means, all qualitative graphical changes are captured by the map $G(s)$
	defined on the finite state-space $S$.
	We will focus on learning $G(s)$ for all contexts $s$ soundly and approximating
	the (global) context-indicator $\sigma(t)$,
	when of interest, in a post-processing step, see §\ref{sec:goals} and §\ref{apdx:indicator_resolution}.
	This notion of states will be further refined in §\ref{sec:identifiable_states}.

	\subsection{Constraint-Based Causal Discovery Revisited}\label{sec:stationary_cd}
	
	Constraint-based causal discovery algorithms employ conditional independence-relations
	of an observed distribution to draw conclusions about the underlying causal graph.
	Indeed, this structure of independencies is the \emph{only} information\footnote{More precisely:
	The only information \emph{contained in the data} which is used. A priori information like
	a causal sufficiency assumption is of course also implicitly employed.} used;
	while this limits the results obtained, for example to Markov equivalence-classes (sets of
	possible graphs rather than individual graphs), it is important for the
	remainder of this paper to note that no other information is needed.
	Thus an abstract constraint-based causal discovery algorithm can be defined
	as a mapping from independence-structures to sets of graphs,
	with certain properties encoding soundness, completeness and others.
	In practice, it is of course of great relevance, how this mapping can be efficiently
	computed. Abstracting this aspect away will later allow us to employ
	standard algorithms and implementations to compute the abstract mapping;
	thereby we will not have to reinvent such efficient realizations but can
	instead leverage existing methodology.
	For simplicity we only consider CD-algorithms that will (given independence-test results)
	execute deterministically.
	This includes to our knowledge most constraint-based causal-discovery algorithms used in practice.
	
	We will start with a simple yet convenient abstraction of independence-structures\footnote{Here denoting
		simple atomic conditional independence-statements, without
		reasoning about logical implications between these statements \citep{LDAG_logical},
		which suffices for our purposes.},
	then continue with a description of what results are produced (\eg Markov equivalence-classes),
	what algorithms formally compute, how their correctness may be defined,
	and how the sparsity of 
	algorithms
	(small number of independencies actually evaluated)
	comes into play.
	For the remainder of this section we fix a $K$-element index-set $I$
	and variables $\{V_i\}_{i\in I}$ to avoid confusion with $X,Y,Z$ denoting individual or sets of $V_i$ used 
	in independence statements.
	\begin{definition}[Independence Atoms]\label{def:independence_structure}
		We define multi-indices (corresponding to conditional in\-de\-pen\-dence-tests, see below)
		of the form\footnote{We use the usual convention
			that for a set $A$, the set $A^0 = \{*\}$ is the one-element set.
			Below, we denote $\IStruct(i_x, i_y, (*)) = 0$ as $X \independent_{\IStruct} Y$.}
		(with $A \sqcup B$ denoting the disjoint union)
		\begin{equation*}
			\Multiindex = (i_x, i_y, (i_z^1, \ldots, i_z^m))
			\quad\in\quad
			\MultiindexSet := I \times I \times \big(I^0 \sqcup I^1 \sqcup \ldots \sqcup I^{K-2}\big)\text.
		\end{equation*}
		An independence-structure is (for our purposes) a mapping
		\begin{equation*}
			\IStruct : \MultiindexSet
			\rightarrow \{0,1\}\txt.
		\end{equation*}	
		The set of all independence-structures is denoted by $\IStructs$.		
	\end{definition}
	\begin{notation}
		Using the shorthands
		$X := V_{i_x}$, $Y := V_{i_y}$ and $Z_j := V_{i_z^j}$,
		we will denote
		\begin{align*}
			\IStruct(i_x, i_y, (i_z^1, \ldots i_z^m)) = 0
			\halfquad&\txt{as}\halfquad
			X \independent_{\IStruct} Y | Z_1, \ldots, Z_m
			\halfquad\txt{and}
			\\
			\IStruct(i_x, i_y, (i_z^1, \ldots i_z^m)) = 1
			\halfquad&\txt{as}\halfquad
			X \dependent_{\IStruct} Y | Z_1, \ldots, Z_m
			\txt.
		\end{align*}
	\end{notation}
	\begin{example}[Independence Oracle]\label{example:oracle}
		The independence oracle $\IStructOracle(M)$ for an SCM $M$ 
		is the (unique) independence-structure with the property
		\begin{align*}
			X \independent_{\IStructOracle(M)} Y | Z_1, \ldots, Z_m
			\quad\Leftrightarrow\quad
			&X \independent_{M} Y | Z_1, \ldots, Z_m \\
			\overset{\txt{faithful, Markov}}{\Leftrightarrow}\quad
			&X \independent_{G} Y | Z_1, \ldots, Z_m
		\end{align*}
		where $X \independent_{M} Y | Z_1, \ldots, Z_m$ denotes independence
		in the distribution induced by $M$.
		In the faithful and Markov case (usually assumed),
		the right-hand-side (and thus the left-hand-side)
		is further equivalent to d-separation in the causal graph $G$ of $M$
		(second line).		
	\end{example}
	CD algorithms output sets of compatible graphs rather than individual
	graphs:
	\begin{definition}[Resolved Graphs]\label{def:cd_alg_abstract_target}
		Let $\GClassesGeneral$ be a set whose elements are the graphs describing
		the possible ground-truths, \eg directed acyclic graphs (DAGs) or ancestral graphs (AGs).
		We call an equivalence-relation $\sim$ on $\GClassesGeneral$
		(\eg Markov-equivalence)
		a graphical kernel (it will capture which graphs end up in the same output class)
		and its equivalence-classes $\GClasses = \sfrac{\GClassesGeneral}{\sim}$
		the resolved graphical structure (its elements will be distinguishable).
	\end{definition}
	\begin{example}\label{example:target_graph_classes}
		Standard algorithms have resolved graphical structures in this sense:
		\begin{enumerate}[label=(\alph*)]
			\item PC-Algorithm \citep{PCalgo}:
			Let $\GClassesGeneral^{\txt{DAG}}$ be the set of directed acyclic graphs (DAGs).
			Define $\sim_{\txt{Markov}}$ as Markov equivalence and $\GClasses_{\txt{Markov}}$ as the set of Markov equivalence-classes of DAGs.
			\item FCI-Algorithm \citep{spirtes2001causation}:
			Let $\GClassesGeneral^{\txt{AG}}$ be the set of ancestral graphs (AGs).
			Define $\GClasses_{\txt{PAG}}$ as the set of inducing-path (maximal, MAG) graphs
			with partial orientations in the sense of FCI
			(the equivalence relation here also takes the potential
			existence of inducing paths into account).
			\item 
			PC-Skeleton:
			Let $\GClassesGeneral^{\txt{skeleton}}$ be the set of undirected graphs.
			The Markov equivalence-classes are identical to
			$\GClasses_{\txt{skeleton}} = \GClassesGeneral^{\txt{skeleton}}$ if orientations are ignored
			\citep{Verma1991}.
		\end{enumerate}
	\end{example}
	These together already allow for an abstraction of constraint-based causal
	discovery algorithms as follows:
	\begin{definition}\label{def:cd_alg_abstract}			
		An (abstract) constraint-based causal discovery algorithm
		is a mapping
		\begin{equation*}
			\CDAlg : \IStructs \rightarrow \GClasses_{\CDAlg}
		\end{equation*}
		into an algorithm-dependent resolved graphical structure $\GClasses_{\CDAlg}$,
		induced by $\sim_{\CDAlg}$.
	\end{definition}
	\begin{example}
		Standard algorithms are abstract algorithms in this sense:
		\begin{enumerate}[label=(\alph*)]
			\item PC-Algorithm:
			Use $\GClasses = \GClasses_{\txt{Markov}}$
			(cf.\ example \ref{example:target_graph_classes}a)
			and define $\CDAlgPC(\IStruct)$ as the output of the PC-algorithm when given the
			independencies in $\IStruct$.
			\item FCI-Algorithm:
			Use $\GClasses = \GClasses_{\txt{PAG}}$
			(cf.\ example \ref{example:target_graph_classes}b)
			and define $\CDAlgFCI(\IStruct)$ as the output of the FCI-algorithm when given the
			independencies in $\IStruct$.
			\item 
			PC-Skeleton:
			Use $\GClasses = \GClasses_{\txt{skeleton}}$
			(cf.\ example \ref{example:target_graph_classes}c)
			and define $\CDAlgPC_\text{skeleton}(\IStruct)$
			as the output of the skeleton phase of the PC-algorithm given the
			independencies in $\IStruct$.
		\end{enumerate}
	\end{example}
	We also need a notion of correctness.
	So far, we only captured what algorithms do, but not under which assumptions
	they are sound and complete.
	Up to finite-sample errors, this is the question about consistency in the oracle case:
	\begin{definition}[CD Correctness]\label{def:cd_alg_abstract_consistent}
		Given a class $\MClasses$ of SCMs on the variables $\{X_i\}_{i\in I}$ indexed by $I$,
		where $M\in \MClasses$ has the (true) causal graph $g(M) \in \GClassesGeneral$
		(using $g$ for
		individual graphs as opposed to equivalence-classes)
		a causal discovery-algorithm $(\CDAlg, \GClasses_{\CDAlg})$ is
		$\MClasses$-consistent, if given the independence oracle $\IStructOracle(M)$
		(example \ref{example:oracle}),
		it satisfies $\forall M\in\MClasses$
		\begin{equation*}
			g(M) \in \quad \CDAlg(\IStructOracle(M))\text.
		\end{equation*}
	\end{definition}
	At this point, standard consistency proofs under assumptions restricting
	models to a specific $\MClasses$ translate to:
	\begin{example}
		For the algorithms above, it holds that \citep{spirtes2001causation, ZhangFCIorientationrules}:
		\begin{enumerate}[label=(\alph*)]
			\item PC-Algorithm:
			Let $\MClasses^{\text{ac}}_{\text{ff,\,cs}}$ be the class of faithful, causally sufficient SCMs with acyclic graphs.
			Then $(\CDAlgPC, \GClasses_{\CDAlgPC})$ is
			$\MClasses^{\text{ac}}_{\text{ff,\,cs}}$-oracle-consistent.
			\item FCI-Algorithm:
			Define $\MClasses^{\text{aac}}_{\text{ff}}$ as the class of faithful SCMs with almost acyclic graphs.
			Then $(\CDAlgFCI, \GClasses_{\CDAlgFCI})$ is $\MClasses^{\text{aac}}_{\text{ff}}$-oracle-consistent.
			\item PC-Skeleton:
			Let $\MClasses^{\text{ac}}_{\text{ff,\,cs}}$ be the class of faithful, causally sufficient SCMs with acyclic graphs.
			Then $(\CDAlgPC_\text{skeleton}, \GClasses_{\txt{skeleton}})$ is
			$\MClasses^{\text{ac}}_{\text{ff,\,cs}}$-oracle-consistent.
		\end{enumerate}
	\end{example}
	\begin{rmk}
		The strength of a statement about consistency depends on
		\emph{both} the size of the class $\MClasses$ (larger is more general)
		\emph{and} on the amount of information provided by knowledge of an element
		$G\in \GClasses_{\CDAlg}$ (smaller sets $G$ contain more detail).
		So, for example, there is no strict hierarchy between PC and FCI:
		FCI applies more generally, but may provide less information in cases where
		PC does apply.
	\end{rmk}
	Finally, of particular practical relevance is the sparsity of lookups:
	There are a lot of statements in the independence-structure, but only
	few of them are actually \inquotes{consumed} by most algorithms.	
	This reduces runtime(-complexity) and helps to keep error-rates in check.
	While the general treatment is substantially simplified by abstracting
	these implementation-details away, knowledge of such properties
	\emph{can} be helpful (and can be formalized), a simple example will
	be seen in §\ref{sec:indicator_translation}, more details are provided in 
	§\ref{apdx:cd_lookup_regions}.

	\subsection{Identifiable State-Space}\label{sec:identifiable_states}
	
	Identifiability considerations concerning the state-space will require
	a formal account to three similar but conceptually disparate problems:
	Deterministic relations between indicators may make states \emph{unreachable};
	for finite observation-time, it may simply happen that some state never occurs
	making them \emph{not reached}; and states $s\neq s'$ with $g_s \sim_{\CDAlg} g_{s'}$ are not $\CDAlg$-\emph{distinguishable}.
	
	\paragraph{Reached State-Space:}
	
	A more formal definition requires some clarifications concerning
	the model-structure given in §\ref{apdx:model_details};
	the question whether a state is \emph{sufficiently} reached to be detected in
	a finite-sample sense can be treated similarly.
	This is a very intuitive problem, and the following simple definition
	will suffice to read the remainder of this paper:	
	\begin{definition}[Reached States]\label{def:reached_states}
		A state $s\in S$ (Def.\ \ref{def:states}) is $T$-reached, if $T_s = \sigma^{-1}(\{s\}) \neq \emptyset$.
		Denote by $S_T \subseteq S$ the set of $T$-reached states in $S$.
	\end{definition}
	It should be evident that, given data for $T$, we will not be able to learn anything
	about specifics of non-$T$-reached states.
	
	\paragraph{Resolved State-Space}
	Our approach inspects the (multi-valued) independence-structure only, so
	states $s$, $s'$ featuring
	the same independence-structure are indistinguishable.
	Consider the following example:
	\begin{example}[Non-Identifiable Regions of the State-Space]
		We are given a model with two states, $A$ and $B$.
		In state $A$, the causal graph of the model is $X \leftarrow Z \rightarrow Y$,
		in state $B$ the causal graph is $X\rightarrow Z \rightarrow Y$.
		In both states, the independence-structure is identical.\\
		\emph{Remark:} This model is neither union-acyclic nor modular (Def.\ \ref{def:causal_modularity_order}).
	\end{example}
	Generally, which states are actually distinguishable
	could depend on the CD-algorithm used, so we specify which states are distinguishable
	as follows:
	\begin{definition}[Identifiable States]\label{def:identifiable_states}
		Given an abstract causal discovery algorithm (Def.\ \ref{def:cd_alg_abstract})
		$\CDAlg$,
		we call two states $\CDAlg$-indistinguishable $s \sim_{\CDAlg} s'$, if
		\begin{equation*}
			s \sim_{\CDAlg} s' :\Leftrightarrow g_s \sim_{\CDAlg} g_{s'}\txt,
		\end{equation*}
		where $g_s$ and $g_{s'}$ are the ground-truth graphs for $s$ and $s'$ and
		$\sim_{\CDAlg}$ on the right-hand side (as a relation on graphs) denotes
		the equivalence-relation of Def.\ \ref{def:cd_alg_abstract_target} (\eg Markov-equivalence
		if $\CDAlg = \CDAlgPC$).
		This defines an equivalence-relation on the model states $S$.
		We call the equivalence-classes the ($T$-)identifiable (by $\CDAlg$) state-space
		\begin{equation*}
			S^{\CDAlg} := \sfrac{S}{\sim_{\CDAlg}}
			\quad\txt{and}\quad
			S^{\CDAlg}_T := \sfrac{S_T}{\sim_{\CDAlg}}
			\txt.
		\end{equation*}		
	\end{definition}
	Non-equivalent (in this sense) states are thus distinguishable in the oracle case,
	if the output of the used causal discovery algorithm is consistent
	(Def.\ \ref{def:cd_alg_abstract_consistent}).
	On the other hand, there is little hope to detect (by constraint-based causal discovery alone)
	a more fine-grained state-structure than $S^{\CDAlg}_T$ on the system.

	\subsection{Assumptions and Goals}\label{sec:goals}
	
	We focus on patterns in exogenous variables, for example persistence in \emph{time};
	possible extensions to patterns in endogenous variables are briefly discussed in §\ref{apdx:endogenous}.
	Our approach is tailored towards modular (independently changing) mechanisms
	(see §\ref{sec:num_experiments}; Def.\ \ref{def:causal_modularity_order}).
	Since, on each non-empty $T_s$, the associated graph $G(s)$ (lemma \ref{lemma:graphs_factor_through_states})
	is the causal graph in the standard sense of the model in state $s$, it makes sense to assume:
	\begin{assumption}[Statewise Faithfulness Property]\label{ass:momentary_faithful}
		Given $s \in S_T^{\CDAlg}$, then
		independence $X \independent_{P_s} Y | Z$ in the distribution $P_s(\ldots):=P(\ldots|t\in T_s)$
		implies d-separation $X \independent_{G(s)} Y | Z$ in $G(s)$.
	\end{assumption}
	And we also immediately obtain (cf.\ \citep[Prop.\,6.31 (p.\,105)]{Elements}):	
	\begin{lemma}[Statewise Markov Property]\label{lemma:momentary_markov}
		Given $s \in S_T^{\CDAlg}$, assume $G(s)$ is acyclic,
		then d-separation $X \independent_{G(s)} Y | Z$ in $G(s)$ implies
		the corresponding
		independence $X \independent_{P_s} Y | Z$ in the distribution $P_s(\ldots):=P(\ldots|t\in T_s)$.
	\end{lemma}
	
	The information we reconstruct from data concerns
	the following qualitative properties.
	Their complexity increases with model-complexity,
	but not with sample-size:	
	\begin{itemize}
		\item Existence of states: Which model-states exist? Which states are reached?
		\item Causal graphs of identifiable states: For each identifiable $s\in S_T^{\CDAlg}$,
		what is $G(s)$?
		\item Attribution to model properties: Which indicators $R_{ij}$ \emph{in the model} are non-trivial?
			(For example in the presence of hidden confounding and inducing paths, this is \emph{not} implicit
			in the graphs.)
	\end{itemize}
	These are supplemented with quantitative information about the form of non-trivial $R_{ij}$.
	Here complexity necessarily increases with sample-size, so only approximations are possible
	(even in asymptotic limits):
	\begin{itemize}
		\item Temporal location of states: When is the system in state $s$? 
		\Ie approximate $T_s$. Or for all states together, approximate
		$\sigma(t)$.
		\item Uncertainty quantification: Estimate and represent expected errors.
	\end{itemize}
	Giving an estimate for the (temporal) location of regimes is
	a post-processing step in our framework
	(see also §\ref{apdx:indicator_resolution}), we focus on the HCCD problem.
	Especially giving good uncertainty statements on time-resolutions seems relevant.
	Due to the length of this paper, we leave a detailed study of this conceptually
	separate problem to future work.

	\section{Hidden Context Causal Discovery}\label{sec:dynamic_cd}
	
	We next analyze our primary task: Given data from a non-stationary model, as described in
	§\ref{sec:model_defs}, and an abstract stationary CD-algorithm, as described in
	§\ref{sec:stationary_cd} (\eg PC or FCI), how can the collection of context
	specific graphs be extracted from a 
	modified notion of conditional independency relations?

	We extend the standard notion of an independence-structure in two steps.
	This extension together with the viewpoint of abstract causal discovery
	algorithms will in §\ref{sec:core_algo} allow
	the direct application of both theoretical results for,
	and practical implementations of, standard CD-algorithms to non-stationary settings.
	We describe the IID-case, the stationary case works analogously.
	
	\subsection{Regime-Marked Independence Structure}
	\label{sec:three_way_testing}
	
	Conventionally, a conditional independence takes one of two possible values:
	Independence $X \independent Y | Z$, or dependence $X \dependent Y | Z$.
	Here, $Z$ denotes a set of variables.
	We will add a third outcome: The presence of a
	(true) regime.
	We assume some ground-truth structure on the data.
	This structure automatically exists for example for the models discussed in §\ref{sec:models},
	but it will typically not be identifiable without linking it
	to some a priori known pattern-information like persistence of 
	temporal regimes in §\ref{sec:persistence}.
	We fix a single test\Slash{}multi-index (\ie the variables $X$, $Y$ and the set $Z$ in $X \independent Y | Z$),
	definitions below are to be understood to exist for any particular such choice.
	
	\begin{assumption}[Regime Structure]\label{ass:segment_structure}
		We assume, the data indexed by $T$ can be divided into (local) regimes,
		that is $n$ (possibly $n=1$) disjoint ($i\neq j \Rightarrow T_i \cap T_j = \emptyset$),
		non-empty 
		subsets $T_1, \ldots, T_n$ with $T = T_1 \cup \ldots \cup T_n$
		and such that for each regime $T_i$ the data points 
		with $t \in T_i$ in that regime are IID
		distributed according to a (jointly for all variables $X$, $Y$ and $W \in Z$) distribution $P_i$.
		We assume these subsets are maximal with this property, \ie $i\neq j \Rightarrow P_i \neq P_j$.
		In the one-dimensional case (\eg if $T$ encodes time)
		we call (maximal) connected intervals within regimes segments (as is typically
		done in CPD literature). We call the random variable $L$ measuring the
		number of data-points per segment the segment-length.
	\end{assumption}
	\begin{example}\label{example:model_segmentation}
		A non-stationary model as described in §\ref{sec:models} produces data
		which is globally (for all variables) structured by times $T_s$ spend in state $s$.
		For any specific test $X \independent Y | Z$,
		changes in mechanisms at non-ancestors of $Z \cup \{X, Y\}$
		will not change $P_i$, thus regimes
		are unions of $T_s$ over all $s\in S$ with relative changes only in such non-ancestors.
	\end{example}

	We are interested in three different ground-truth configurations:
	\begin{definition}[Regime Marked Dependence]\label{def:marked_independence}
		To be revisited in Def.\ \ref{def:marked_independence_rel_d}.
		The statistical testing of these hypotheses will be discussed in §\ref{sec:dyn_indep_tests}.
		\begin{enumerate}
			\item[$0:$]
			We call $X$ globally independent of $Y$ given $Z$,
			denoted $X \independent Y | Z$, if
			$\forall i: X \independent_{P_i} Y | Z$.
			\item[$1:$]
			We call $X$ globally dependent on $Y$ given $Z$,
			denoted $X \dependent Y | Z$, if
			$\forall i: X \dependent_{P_i} Y | Z$.
			\item[$\mCIToutR:$]
			We say there is a true regime between $X$ and $Y$ given $Z$,
			denoted $X \independent_R Y | Z$, if $\exists i_0 : X \independent_{P_{i_0}} Y | Z$
			and $\exists i_1 : X \dependent_{P_{i_1}} Y | Z$.
		\end{enumerate}
	\end{definition}
	\begin{definition}[Detected Indicator]\label{def:detected_indicator}		
		There is an associated indicator $R_{XY|Z}$, which we call the detected indicator, given by
		$R_{XY|Z}(t) = 0$ for $t\in T_i$ with $X \independent_{P_i} Y | Z$
		and $R_{XY|Z}(t) = 1$ otherwise. This indicator is trivially $\equiv 0$ in case $0$, trivially $\equiv 1$
		in case $1$ and non-trivial if and only if we are in case $\mCIToutR$.
	\end{definition}
	A regime-marked independence-structure is then (in analogy to Def.\ \ref{def:independence_structure}):
	\begin{definition}[Regime Marked Independence-Atoms]
		A regime-marked independence-structure is a mapping (where $\mCIToutR$ on the rhs is just some special symbolic value)
		\begin{equation*}
			\IStructM : \MultiindexSet
			\rightarrow \{0, 1, \mCIToutR\}\txt.
		\end{equation*}	
		We denote the set of all regime-marked independence-structures by $\IStructsM$
		and
		\begin{align*}
			\IStructM(i_x, i_y, (i_z^1,\ldots,i_z^m))=0
			\halfquad&\txt{by}\halfquad
			X \independent^{\IStructM} Y | Z_1, \ldots, Z_m
			\txt,
			\\
			\IStructM(i_x, i_y, (i_z^1,\ldots,i_z^m))=1
			\halfquad&\txt{by}\halfquad
			X \dependent^{\IStructM} Y | Z_1, \ldots, Z_m
			\txt{ and}
			\\
			\IStructM(i_x, i_y, (i_z^1,\ldots,i_z^m))=\mCIToutR
			\halfquad&\txt{by}\halfquad
			X \independent^{\IStructM}_R Y | Z_1, \ldots, Z_m
			\txt.
		\end{align*}
	\end{definition}
	The formal connection to the models of §\ref{sec:models} is again (cf.\ example \ref{example:oracle})
	made via an	oracle structure:
	\begin{example}[Marked Independence Oracle]\label{example:oracle_marked}
		The regime-marked independence oracle $\IStructOracleM(M)$ for a non-stationary
		SCM $M$ 
		is the regime-marked independence-structure with the property
		\begin{align*}
			X \independent^{\IStructOracleM(M)} Y | Z_1, \ldots, Z_m
			\quad\Leftrightarrow\quad
			&X \independent^{M} Y | Z_1, \ldots, Z_m \\			
			X \dependent^{\IStructOracleM(M)} Y | Z_1, \ldots, Z_m
			\quad\Leftrightarrow\quad
			&X \dependent^{M} Y | Z_1, \ldots, Z_m \\			
			X \independent_R^{\IStructOracleM(M)} Y | Z_1, \ldots, Z_m
			\quad\Leftrightarrow\quad
			&X \independent^{M}_R Y | Z_1, \ldots, Z_m
		\end{align*}
		where $X \independent^{M}_* Y | Z_1, \ldots, Z_m$ denotes
		the respective statement (Def.\ \ref{def:marked_independence}) for
		the distribution induced by $M$; regimes are fixed according
		to example \ref{example:model_segmentation}.
	\end{example}

	It is of course of pivotal importance if and how this structure can be
	discovered from data in theory and in practice.
	In §\ref{sec:dyn_indep_tests} we propose statistical testing strategies to approach this problem,
	and show that these have good theoretical and practical properties.

	\subsection{Multi-Valued Independence Structure}\label{sec:extended_indep_struct_subsection}
	If there are multiple independence-statements with true regimes,
	then two questions remain (illustrated by example \ref{example:induced_indicators} below):
	Which combinations of values actually appear (see below)?
	And how are these related to model states (§\ref{sec:model_states})?
	Both problems will be systematically addressed in §\ref{sec:indicator_translation}.
	\begin{notation}[Indicator Naming]
		Given $X$ and $Y$ corresponding to variables $V_{i_x}, V_{i_y}$ with
		indices $i_x, i_y \in I$, we denote the indicator $R_{i_xi_y}$ appearing
		in the definition of models \ref{def:model_nonstationary_scm}
		by $\Rmodel_{XY}(t) := R_{i_xi_y}(t)$.
		Additionally, we use the detected indicators $R_{XY|Z}$ (possibly $R_{XY|\emptyset}$)
		of Def.\ \ref{def:detected_indicator} if there is a true regime between $X$ and $Y$ given $Z$.
	\end{notation}	
	\begin{example}[Induced Indicators]\label{example:induced_indicators}
		Consider $X \leftarrow Y \rightarrow Z$ with only $R^{\text{model}}_{XY}$
		non-trivial (the link $X \leftarrow Y$, and only the link $X \leftarrow Y$,
		is context-dependent). This will, besides $R_{XY|\emptyset}$ non-trivial, also
		lead to $R_{XZ|\emptyset}$ non-trivial,
		even though $\Rmodel_{XZ}$ is trivial.
		Further, while both are non-trivial,
		$R_{XZ|\emptyset} \equiv R_{XY|\emptyset}$; the underlying
		model only has two states (as opposed to the four potential values the pairs
		$(R_{XY|\emptyset}(t), R_{XZ|\emptyset}(t))$ could take): One with the link $X \leftarrow Y$ active,
		one with the link $X \leftarrow Y$ not active.
	\end{example}
	We capture what we need to know about co-occurrence of true regimes as follows:
	\begin{definition}[Multi-Valued Independence Structure]\label{def:dynamic_independence}
		A multi-valued independence-structure $\IStructExt$
		is a set of independence structures.
		We call elements of $\IStructExt$ detected states.
	\end{definition}
	\begin{rmk}\label{rmk:induced_marked_structure}
		A multi-valued independence structure $\IStructExt$ induces
		a regime-marked independence structure $\IStructM$ by
		mapping a multi-index $\vec{i}=(i_x, i_y, i_{\vec{z}})$ to
		$0$ (or $1$) if all $\IStruct \in \IStructExt$ take the same value $0$ (or $1$)
		on the argument $\vec{i}$,
		and to $R$ if there are $\IStruct, \IStruct' \in \IStructExt$
		with $\IStruct(\vec{i}) \neq \IStruct'(\vec{i})$.
	\end{rmk}
	\begin{example}[Multi-Valued Independence Oracle]\label{example:oracle_extended}
		The multi-valued independence oracle $\IStructOracleX(M)$ for a non-stationary SCM $M$
		is the multi-valued independence-structure
		\begin{equation*}
			\IStructOracleX(M)
			:=
			\big\{
			\IStructOracle(M_s) \big| s\in S(M)
			\big\}
		\end{equation*}
		where $S(M)$ is the set of states of $M$ (Def.\ \ref{def:states}),
		and $\IStructOracle(M_s)$ is the independence oracle (Example \ref{example:oracle})
		of the model $M$ in state $s$.
	\end{example}
	\begin{example}[Implicit Marked Oracle]
		The multi-valued independence oracle $\IStructOracleX(M)$
		induces by Rmk.\ \ref{rmk:induced_marked_structure}
		a regime-marked independence structure $\IStructM$.
		This regime-marked independence structure is the
		regime-marked oracle $\IStructM = \IStructOracleM(M)$ of example \ref{example:oracle_marked}.
	\end{example}
	
	We will in practice divide the discovery of multi-valued independence-structures
	from data into three simpler sub-problems:
	\begin{enumerate}[label=(\roman*)]
		\item
		Discover the regime-marked structure $\IStructM$ induced by $\IStructX$.
		This is fully local in the graph and can be decided per independence-statement;
		see §\ref{sec:dyn_indep_tests}.
		\item 
		Discover simple binary relations between (small) sets of indicators.
		This requires only tests with binary results
		(see §\ref{sec:indicator_translation}, §\ref{apdx:required_implication_tests},
		§\ref{apdx:implication_testing}).
		\item 
		In §\ref{sec:indicator_translation} we will show in detail how these entail the
		full multi-valued structure from constraints arising in causal modeling.
	\end{enumerate}
	In the schematic picture of our framework given in the introduction (Fig.\ \ref{fig:architecture})
	the first two points contribute the data-processing layer, and the box labeled \inquotes{marked independence atoms}.
	The third point realizes the arrow labeled \inquotes{states} in the structural layer,
	the multi-valued independence structure as defined here combines the states (labeled
	\inquotes{state space} in the figure) and the associated independence-structure(s).

	\subsection{Necessity}\label{sec:dyn_structure_necessity}
	
	Knowledge of the (qualitative) hidden-context context-specific causal structure
	$\{G_s\}_{s\in S}$ of a model $M$ is stronger than knowledge of its 
	multi-valued independence-structure $\IStructOracleX(M)$.
	\Ie it is necessary to implicitly or explicitly
	discover the multi-valued independence-structure to solve the HCCD
	problem:
	\begin{lemma}[Multi-Valued d-Connectivity]\label{lemma:necessity_of_dyn_indep_struct}
		Given a non-stationary SCM $M$ with states $M_s$ faithful and Markov to
		their respective causal graphs $G_s$, then
		$\{G_s\}_{s\in S}$ determines $\IStructOracleX(M)$.
	\end{lemma}
	\begin{proof}
		By hypothesis, for each $s$
		d-connectivity is equivalent to independence, and $G_s$ determines
		$\IStructOracle(M_s)$ by d-connections.
		By definition, $\IStructOracleX(M) = \{\IStructOracle(M_s)|s\in S\}$.
	\end{proof}
	\begin{rmk}\label{rmk:necessity_of_dyn_indep_struct}
		$\IStructOracleX(M)$ further determines $\IStructOracleM(M)$ by
		Rmk.\ \ref{rmk:induced_marked_structure} and all indicator implications by
		Rmk.\ \ref{rmk:indicator_implciation}.
	\end{rmk}
	
	In the next section §\ref{sec:core_algo},
	we show that the reverse direction,
	discovering hidden-context context-specific causal structure $\{G_s\}_{s\in S}$
	from the extended independence-structure, can be reduced to the problem
	of causal discovery in the IID or stationary case.
	For the IID\Slash{}stationary case, there is of course a large
	range of algorithms available; these can easily be integrated with our approach.

	\subsection{The Core Algorithm}\label{sec:core_algo}
	
	In this section, we explain the
	pseudo-code given in algorithm \ref{algo:core}.
	Formal properties are then described in Thm.\ \ref{thm:core_algo}.	
	The general outline is as follows:
	We are given a multi-valued independence structure whose discovery from data
	will be described in §\ref{sec:dyn_indep_tests} and §\ref{sec:indicator_translation}.
	More precisely, we are given two functions encoding multi-valued independence information
	by implementing
	\begin{itemize}
		\item
		the discovery of marked independence (Def.\ \ref{def:marked_independence}):
		Given a multi-index $\vec{j} = (i_x, i_y, i_{\vec{z}})$ corresponding to
		a conditional independence query $X \independent Y | Z$,
		this algorithm called \algMarkedIndependence{}
		returns one of three possible outcomes: $0$ (global independence),
		$1$ (global dependence) or $R$ (true regime), cf.\ example \ref{example:oracle_marked}.
		The implementation is described in §\ref{sec:dyn_indep_tests}.
		\item
		partial state-space construction:
		Let $J$ be a (possibly incomplete) set of marked independencies
		(\ie a set of multi-indices with $\vec{j}\in J$ $\Rightarrow$
		\algMarkedIndependence{}$(\vec{j}) = \mCIToutR$).
		Then independence-statements in $J$, by definition, do \emph{not} take a
		definite (global) value dependent or independent (or $1$ and $0$),
		but rather a state-dependent value.
		The possible combinations of values of elements in $J$
		are precisely the mappings from $J$ to $0$ or $1$, denoted $\Map(J, \{0,1\})$.
		As already the simple example \ref{example:induced_indicators} above shows
		not all elements of $\Map(J, \{0,1\})$ are realized (in the example, $|J|=2$,
		thus $|\Map(J, \{0,1\})|=4$, while there are only two states, namely
		the state that maps both elements of $J$ to $0$ and the state which
		maps both elements of $J$ to $1$).
		In §\ref{sec:indicator_translation}
		we will discuss in detail why this is the case, and how it can be resolved;
		to understand the core algorithm,
		it suffices to know that §\ref{sec:indicator_translation} will produce
		an algorithm \algConstructStateSpace{}, which
		recovers the \emph{reached} states $S^J \subseteq \Map(J, \{0,1\})$.
		The implementation is described in §\ref{sec:indicator_translation}.
	\end{itemize}
	From this information encoding the multi-valued independence structure,
	we want to find the 
	set of identifiable states $S_T^{\CDAlg}$ (Def.\ \ref{def:identifiable_states})
	and associated (equivalence-classes of) graphs $\{G_s\}_{s\in S_T^{\CDAlg}}$.
	Indeed we will focus on producing the correct set of graphs $\{G_s | s\in S_T^{\CDAlg}\}$,
	the association of model-states (translation into model properties) is
	discussed in §\ref{apdx:state_translation}.
	
	Details on the formal description of the core algorithm \ref{algo:core}
	and a proof of the correctness
	(in the oracle case) of its output is given in §\ref{apdx:core_algorithm}.
	Indeed, there are some important aspects, like the completeness of the recovered
	set of graphs that require a careful analysis.
	Nevertheless the general behavior of algorithm \ref{algo:core} can be understood
	informally as follows:
	
	\begin{algorithm}
		\caption{\texttt{run\_hccd} (Multiple Causal Discovery)}
		\label{algo:core}
		\begin{algorithmic}[1]
			\Function{pseudo\_cit[$s$,$J$,$J'_s$, \normalfont\texttt{data}]}
			{$\vec{j}$}\Comment{capture $J'_s$ by reference (in\Slash{}out)}
				\If{$\vec{j} \in J$}\Comment{value depends on $J$-state}
					\State\Return 'independent' \textbf{if} $s(\vec{j})=0$ \textbf{else} 'dependent'
				\Else\Comment{not (yet) known to feature a regime}
					\If{\algMarkedIndependence{}(\texttt{data}, $\vec{j}$) = 0
						 (globally independent)}
						\State \Return 'independent'						
					\ElsIf{\algMarkedIndependence{}(\texttt{data}, $\vec{j}$) = 1 (globally dependent)}					
						\State \Return 'dependent'
					\ElsIf{\algMarkedIndependence{}(\texttt{data}, $\vec{j}$) =
						$\mCIToutR$	(true regime)}
						\State $J'_s := J'_s \cup \{\vec{j}\}$\Comment{we were not previously aware of this change}
						\State \Return 'dependent'\Comment{default to union model}
					\EndIf\Comment{Only these three cases exist.}
				\EndIf
			\EndFunction
			\vspace*{0.5em}\hrule\vspace*{0.5em}
			\State \textbf{Input:} The multivariate \texttt{data}.
			\State \textbf{Output:}
			A set of graphs $\{ G_s | s\in S \}$.
			\State $J := \emptyset$
			\Comment{will record marked independencies $\vec{j}$ encountered so far}
			\Repeat
				\State
				$S := $ \algConstructStateSpace{}(\texttt{data}, $J$)
				\For{$s \in S$}
					\State $J'_s := \emptyset$					
					\State $G_s := \texttt{CD}($ \algPseudoCIT{}[$s$, $J$,
					\textbf{in/out} $J_s'$, \texttt{data}] $)$
				\EndFor
				\State $J' := \cup_s J'_s$
				\State $J = J \cup J'$
			\Until{$J'=\emptyset$}
			\State\Return $\{ G_s \}_{s\in S}$
		\end{algorithmic}
		\hrule\vspace{0.5em}
		States $s\in S$ are encoded as mappings $s : J \rightarrow \{0,1\}$
		and constructed by \algConstructStateSpace{}
		as detailed in §\ref{sec:indicator_translation}
		as Algo.\ \ref{algo:constr_state_space}.
		The underlying CD algorithm '\texttt{CD}' (\eg PC)
		is treated as a function taking another function (a lazily evaluated
		independence-structure, here: \algPseudoCIT{})
		as an input, cf.\ §\ref{sec:stationary_cd}.
		This \algPseudoCIT{} uses the values of $s$, $J$, $J_s'$ and \texttt{data}
		from the local context (inner scope of the for loop; besides its
		return-value seen by \texttt{CD} it also writes to $J_s'$).
		It makes heavy use of \algMarkedIndependence{} described in §\ref{sec:dyn_indep_tests}
		as Algo.\ \ref{algo:mCIT}. In §\ref{apdx:illustrate_algo_evaluation}, the algorithm
		is explicitly evaluated step by step on an example (in the oracle case).
	\end{algorithm}
	
	Initially we know of no state-dependent independence-statements $J_0 = \emptyset$,
	and using \algConstructStateSpace{} will in this case always return the one
	trivial map $\emptyset \rightarrow \{0,1\}$, that is, $S_0=\{*\}$ has a single element.
	We thus invoke \algRunCD{} once, with $J=\emptyset$.
	This executes a standard causal discovery-algorithm \texttt{CD},
	for example $\CDAlg=\CDAlgPC$. Whenever a CIT would be executed by CD,
	it instead runs \algPseudoCIT{}, which for
	$J=\emptyset$ simply runs \algMarkedIndependence{}.
	If the outcome is \inquotes{true regime} we
	remember this test, by storing it in $J'$, and continue as if the result had
	been dependence (in the initial iteration this convention yields the union-graph,
	Lemma \ref{lemma:states_and_union_graph}).
	We end up with a single graph $G$ and a set of context-specific independence-statements $J'$,
	which we remember as $J_1$.
	
	In the $i^{\txt{th}}$ iteration, $S_i$, by specification of \algConstructStateSpace{},
	contains those combinations of known state-dependent independence-statements in $J_i$
	that actually appear (the reached $J_i$-states, see above).
	For each such state $s$, we invoke \algRunCD{} once with this state $s$.
	This again executes a standard causal discovery-algorithm \texttt{CD},
	and whenever a CIT would be executed, it instead runs \algPseudoCIT{}.
	For the state-dependent independence-statements in $\vec{j}\in J_i$, we fix values
	to $s(\vec{j})$. Effectively \algRunCD{} runs once for each actual state (combination
	of values taken by the state-dependent statements $J_i$ fixed).
	Any previously unknown state-dependent statements (across all states) are recorded into $J'_i$
	and included additionally for the next iteration $J_{i+1} = J_i \cup J'_i$.
	
	At some point, no new
	state-dependencies are found $J'= \emptyset$ and the algorithm terminates.
	It returns all graphs discovered in the last iteration.
	In the last iteration (since $J' = \emptyset$), for each state $s$, the
	result of \algPseudoCIT{} is always determined either by a \inquotes{standard}
	independence result ('dependent' or 'independent') or by the state $s$.
	Thus \algRunCD{}, for each actual state $s$, produces the graph discovered by $\CDAlg$
	if state-dependent independencies take the values prescribed by the state,
	while non-state-dependent independencies take their global values.

	\paragraph{Properties:}
	The above algorithm potentially has to rerun the underlying (stationary) CD-algorithm
	often. This is surprisingly not a problem in practice:
	The limiting factor both in terms of runtime-requirements and in terms of
	statistical sample-efficiency is in the number of actually executed tests.
	But the output of \algMarkedIndependence{} is deterministic (given data)
	and can be cached, so there is little increase in the required number of tests
	compared to the union-graph.
	In fact, the union-graph is denser (contains more edges)
	than any context-specific graph, thus for example for $\CDAlg=\CDAlgPC$
	already the first iteration
	will run all tests required for (the skeleton-phase of) any
	of the reruns.
	This also means, typically all marked independencies that will be found
	are found in the initial iteration, so that the algorithm in practice
	usually converges after only two iterations.

	The output of Algo.\ \ref{algo:core} is sound and complete in the following oracle-case (for details and proofs,
	see §\ref{apdx:core_algorithm}):	
	\begin{thm}[Core Algorithm]\label{thm:core_algo}
			Given
		a non-stationary model $M$ (Def.\ \ref{def:model_nonstationary_scm}\Slash{}\ref{def:model_nonstationary_scm_nonadd}),
		a marked independence oracle $\IStructOracleM(M)$ (Def.\ \ref{example:oracle_marked}; invoked as
		{\normalfont{\algMarkedIndependence{}}} in algorithm \ref{algo:core}),
		an abstract CD-algorithm $\CDAlg$ (Def.\ \ref{def:cd_alg_abstract}; invoked as
		{\normalfont{\texttt{CD}}} in algorithm \ref{algo:core}),
		which is consistent (Def.\ \ref{def:cd_alg_abstract_consistent})
		on a set of models $\mathcal{M}$ including all reached states (Def.\ \ref{def:reached_states})
		$\forall s\in S_T: M_s \in \mathcal{M}$
		and
		a sound and complete, on a model-class containing $M$,
		state-space construction (Def.\ \ref{def:state_space_construction_specification}; invoked as
		{\normalfont{\algConstructStateSpace{}}} in algorithm \ref{algo:core}).
		
		Then algorithm \ref{algo:core} terminates after a finite number of iterations
		and its output is a set (\ie no ordering or characterization of states by model properties
		is implied) which consists of precisely one graph (a $\CDAlg$ equivalence class, Def.\ \ref{def:cd_alg_abstract_target}) per $\CDAlg$-identifiable state $s\in S_T^{\CDAlg}$
		\begin{equation*}
			\Big\{ G_s \big| s\in S_T^{\CDAlg} \Big\}
			\quad\text{such that}\quad
			\forall \tilde{s} \in s:
			g_{\tilde{s}} \in G_s
			\text,
		\end{equation*}
		where $g_{\tilde{s}}$ is the ground-truth graph in state $\tilde{s}$
		(which is a representative of the $\CDAlg$-identifiable state $s$).
	\end{thm}
	\begin{example}
		If $\CDAlg = \CDAlgPC$,
		all states $S$ are in different Markov-equivalence-classes,
		reached (occur in data),
		and the models $M_s$ of the individual states are causally sufficient,
		faithful with acyclic true graph $g_s$ and acyclic union-graph,
		then the identifiable state-space is $S_T^{\CDAlgPC}=S$ and
		in the oracle case algorithm \ref{algo:core} returns the set
		of Markov equivalence-classes $\{G_s|s\in S\}$
		representing the true graphs $g_s$, \ie $g_s \in G_s$
		(by $S = S_T^{\CDAlgPC}$, each $s$ has a single representative \inquotes{$\tilde{s}=s$}).
	\end{example}

\section{Marked Independence Tests}
\label{sec:dyn_indep_tests}

Our framework relies on knowledge of the multi-valued independence structure described in §\ref{sec:extended_indep_struct_subsection}.
Indeed, as shown in §\ref{sec:dyn_structure_necessity} \emph{any} solution to the
HCCD problem (learning the set of context-specific graphs)
has to implicitly or explicitly extract this knowledge.
Thus there is a profound interest in understanding its recovery from data.
However, this recovery from data turns out to be a quite challenging problem:
We show two impossibility-results that put limits to what can be
identified from data even in principle. These impossibilities are inherent to the
problem, their occurrence is general, not specific to our approach.

A major strength of our approach is that it can make these limitations
explicit and allows to understand and navigate them through hyperparameter controlled trade-offs.
Indeed our main goal is to \emph{understand} the problem and the behavior and
implied inductive biases of different solutions.
To this end, we divide the problem further into conceptually separate sub-tasks.
This allows for the theoretical and numerical
analysis of assumptions, null-distributions and power on different alternatives.
Still, with the primary goal of gaining insights, we focus on very simple
and transparent solutions for the separate tasks.
We do, however, outline how they are interconnected with well-known
standard tasks and believe that leveraging more sophisticated
state-of-the-art methodology on these tasks will allow for substantial
improvements on finite sample performance in the future; these connections
will be summarized in §\ref{sec:dyn_indep_test_summary} at the end of this section.
Both the subdivision into simpler tasks and the use of simple and transparent 
solutions also aids numerical verification of the theoretical findings:
Benchmarking simple modular components at
sub-millisecond runtimes allows for otherwise impossible exploration
of the behavior with large numbers of model and algorithmic parameters.

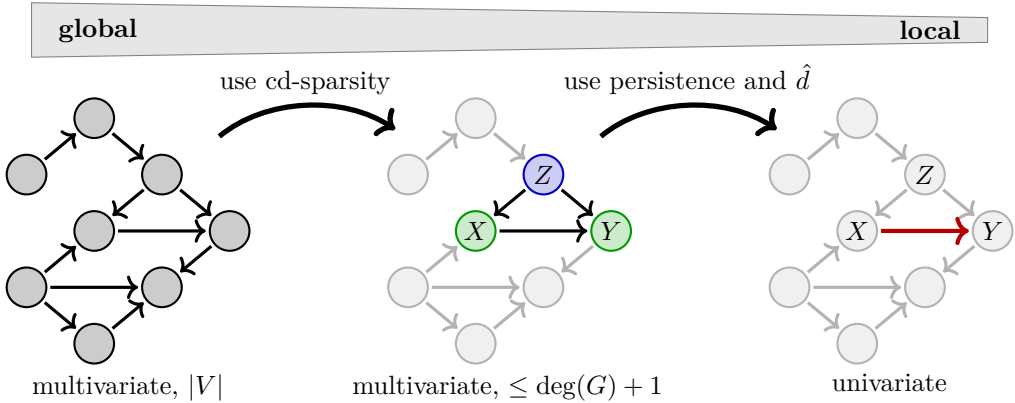
\begin{figure}[ht]	
	\begin{minipage}{\textwidth}
		\centering
		\colorlet{nodecolor}{black}

\providecommand{\mcitGlobalToLocalScale}{1.0}

\providecommand{\mcitGlobalToLocalEmToCentimeter}{0.35138}
%\makeatletter
%\def\convertto#1#2{\strip@pt\dimexpr #2*65536/\number\dimexpr 1#1}
%\makeatother
%
%\newdimen\mylength
%\mylength=1em
%\convertto{cm}{\the\mylength} cm

\begin{tikzpicture}
	[
	xscale=\mcitGlobalToLocalScale,
	yscale=0.75*\mcitGlobalToLocalScale,
	subpicture/.style={
		inner sep=0em
	}
	]
	
	\draw (0,0) node [subpicture] (A) {
		\begin{tikzpicture}
			[
			xscale=0.9*\mcitGlobalToLocalScale,
			yscale=0.75*\mcitGlobalToLocalScale,
			inner sep=0,
			outer sep=0.15cm*\mcitGlobalToLocalEmToCentimeter,
			varnode/.style={
				circle,
				draw=nodecolor,
				thick,
				fill=nodecolor!20,
				align=center,
				minimum height=1.5em
			},
			mechedge/.style={
				->,
				very thick
			}
			]	
			\draw (0,0) node(a)[varnode]{};
			\draw (1,1) node(b)[varnode]{};
			\draw (2,0) node(c)[varnode]{};
			\draw (1,-1) node(d)[varnode]{};
			\draw (3,-1) node(e)[varnode]{};
			\draw (0,-2) node(f)[varnode]{};
			\draw (2,-2) node(g)[varnode]{};
			\draw (1,-3) node(h)[varnode]{};
			
			\draw[mechedge] (a) -- (b);
			\draw[mechedge] (b) -- (c);
			\draw[mechedge] (c) -- (e);
			\draw[mechedge] (e) -- (g);
			
			\draw[mechedge] (c) -- (d);
			\draw[mechedge] (d) -- (e);
			
			\draw[mechedge] (f) -- (d);
			\draw[mechedge] (f) -- (g);
			\draw[mechedge] (f) -- (h);
			\draw[mechedge] (h) -- (g);
	\end{tikzpicture}};
	
	\colorlet{nodecolor_xy}{black!40!green}
	\colorlet{nodecolor_z}{black!20!blue}
	\colorlet{nodecolor}{black!30!white}
	\colorlet{edgecolor}{black!30!white}
	\colorlet{edgecolor_xyz}{black}
	
	\draw (A.east)+(5cm*\mcitGlobalToLocalEmToCentimeter,0) node [anchor=west,subpicture] (B) {
		\begin{tikzpicture}
			[
			xscale=0.9*\mcitGlobalToLocalScale,
			yscale=0.75*\mcitGlobalToLocalScale,
			inner sep=0,
			outer sep=0.15cm*\mcitGlobalToLocalEmToCentimeter,
			varnode/.style={
				circle,
				draw=nodecolor,
				thick,
				fill=nodecolor!20,
				align=center,
				minimum height=1.5em
			},
			mechedge/.style={
				->,
				very thick,
				edgecolor
			}
			]	
			\draw (0,0) node(a)[varnode]{};
			\draw (1,1) node(b)[varnode]{};
			\draw (2,0) node(c)[varnode, draw=nodecolor_z, fill=nodecolor_z!20]{$Z$};
			\draw (1,-1) node(d)[varnode, draw=nodecolor_xy, fill=nodecolor_xy!20]{$X$};
			\draw (3,-1) node(e)[varnode, draw=nodecolor_xy, fill=nodecolor_xy!20]{$Y$};
			\draw (0,-2) node(f)[varnode]{};
			\draw (2,-2) node(g)[varnode]{};
			\draw (1,-3) node(h)[varnode]{};
			
			\draw[mechedge] (a) -- (b);
			\draw[mechedge] (b) -- (c);
			\draw[mechedge,edgecolor_xyz] (c) -- (e);
			\draw[mechedge] (e) -- (g);
			
			\draw[mechedge,edgecolor_xyz] (c) -- (d);
			\draw[mechedge,edgecolor_xyz] (d) -- (e);
			
			\draw[mechedge] (f) -- (d);
			\draw[mechedge] (f) -- (g);
			\draw[mechedge] (f) -- (h);
			\draw[mechedge] (h) -- (g);
	\end{tikzpicture}};
	
	\colorlet{nodecolor}{black!30!white}	
	\colorlet{edgecolor_xy}{black!30!red}
	\colorlet{edgecolor}{black!30!white}
	
	\draw (B.east)+(5cm*\mcitGlobalToLocalEmToCentimeter,0) node [anchor=west,subpicture] (C) {
		\begin{tikzpicture}
			[
			xscale=0.9*\mcitGlobalToLocalScale,
			yscale=0.75*\mcitGlobalToLocalScale,
			inner sep=0,
			outer sep=0.15cm*\mcitGlobalToLocalEmToCentimeter,
			varnode/.style={
				circle,
				draw=nodecolor,
				thick,
				fill=nodecolor!20,
				align=center,
				minimum height=1.5em
			},
			mechedge/.style={
				->,
				very thick,
				edgecolor
			}
			]	
			\draw (0,0) node(a)[varnode]{};
			\draw (1,1) node(b)[varnode]{};
			\draw (2,0) node(c)[varnode]{$Z$};
			\draw (1,-1) node(d)[varnode]{$X$};
			\draw (3,-1) node(e)[varnode]{$Y$};
			\draw (0,-2) node(f)[varnode]{};
			\draw (2,-2) node(g)[varnode]{};
			\draw (1,-3) node(h)[varnode]{};
			
			\draw[mechedge] (a) -- (b);
			\draw[mechedge] (b) -- (c);
			\draw[mechedge] (c) -- (e);
			\draw[mechedge] (e) -- (g);
			
			\draw[mechedge] (c) -- (d);
			\draw[mechedge, color=edgecolor_xy, ultra thick] (d) -- (e);
			
			\draw[mechedge] (f) -- (d);
			\draw[mechedge] (f) -- (g);
			\draw[mechedge] (f) -- (h);
			\draw[mechedge] (h) -- (g);
	\end{tikzpicture}};
	
	\draw (B.north west)+(0.8cm*\mcitGlobalToLocalEmToCentimeter,-2cm*\mcitGlobalToLocalEmToCentimeter) node (targetB){};
	\draw[->,line width=0.2cm*\mcitGlobalToLocalEmToCentimeter] (A.north east)+(-1.2cm*\mcitGlobalToLocalEmToCentimeter,-2cm*\mcitGlobalToLocalEmToCentimeter) to[bend left=45]
	node[midway, above]{use cd-sparsity}
	(targetB);
	
	\draw (C.north west)+(0.8cm*\mcitGlobalToLocalEmToCentimeter,-2cm*\mcitGlobalToLocalEmToCentimeter) node (targetC){};
	\draw[->,line width=0.2cm*\mcitGlobalToLocalEmToCentimeter] (B.north east)+(-1.2cm*\mcitGlobalToLocalEmToCentimeter,-2cm*\mcitGlobalToLocalEmToCentimeter) to[bend left=45]
	node[midway, above]{use persistence and $\hat{d}$}
	(targetC);
	
	\draw (A.south) node[anchor=north]{multivariate, $|V|$};
	\draw (B.south) node[anchor=north]{multivariate, $\leq \deg(G)+1$};
	\draw (C.south) node[anchor=north]{univariate};
	
	\draw (A.north west)+(1cm*\mcitGlobalToLocalEmToCentimeter,2cm*\mcitGlobalToLocalEmToCentimeter) node (topleftbottom){};
	\draw (A.north west)+(1cm*\mcitGlobalToLocalEmToCentimeter,4.75cm*\mcitGlobalToLocalEmToCentimeter) node (toplefttop){};
	\draw (C.north east)+(-1cm*\mcitGlobalToLocalEmToCentimeter,2.75cm*\mcitGlobalToLocalEmToCentimeter) node (toprightbottom){};
	\draw (C.north east)+(-1cm*\mcitGlobalToLocalEmToCentimeter,4.0cm*\mcitGlobalToLocalEmToCentimeter) node (toprighttop){};
	
	\draw[color=gray, fill=gray!20!white] (topleftbottom.center) -- (toprightbottom.center)
	-- node[left, midway, color=black, inner sep=1cm*\mcitGlobalToLocalEmToCentimeter] {\textbf{local}}
	(toprighttop.center) -- (toplefttop.center) --
	node[right, midway, color=black, inner sep=1cm*\mcitGlobalToLocalEmToCentimeter] {\textbf{global}}
	cycle;
\end{tikzpicture}
	\end{minipage}
	
	\caption{Our framework turns a (graph-)global clustering problem into (graph-)local ones.}
	\label{fig:graph_global_to_graph_local}
\end{figure}

As already outlined in the introduction §\ref{sec:intro} to this paper, our approach
is guided by the principles of (graph-)locality and directness of testing; further, our
architecture (Fig.\ \ref{fig:architecture}) confines the associated issues to
the data-processing layer discussed here. Therefore we emphasize throughout
this section the role and realization of these principles.
Locality is achieved, as illustrated in Fig.\ \ref{fig:graph_global_to_graph_local},
in two steps: The first one is to reduce the problem from a global question to
a question about few variables involved in an independence test; this step
is realized by design and factoring through the multi-valued independence structure.
The second step focuses on the causal mechanism relating two variables.
For optimal signal to noise, one should ask:
How does one isolate, from the information contained in these few variables $X$, $Y$, 
$Z_1, \ldots, Z_k$, precisely that information describing the (possible)
dependence between the pair $X$ and $Y$ (given $Z$)?
This question has of course a very standard answer: dependence scores $\hat{d}$
(for example partial correlation),
as used in conditional independence tests (CITs) are designed precisely to do so.
Even more conveniently these scores are typically univariate,
so that we obtain not only good signal to noise, but also a dramatically simplified problem.
This simplified problem will be described and studied in §\ref{sec:univar_problem},
where it will also become apparent, why it can be efficiently tested in a \inquotes{direct} way.

\begin{rmk}
	Going from one global problem with poor signal-to-noise to many local problems
	with good signal-to-noise, one should ask oneself: How many is \inquotes{many}, and is
	this better after accounting for multiple testing?
	Given a suitable, sparse testing strategy (in our case
	provided by the CD-algorithm), testing locally realizes an inductive bias favoring sparse graphs.
	It will typically perform better than global methods on sparse enough problems.
	Sparsity of the underlying graph is a pivotal assumption underlying
	constraint-based causal discovery in general, so that we require it here,
	should not lead to substantial \emph{additional} restrictions.
\end{rmk}
\begin{rmk}
	Typically, there are multiple different conditioning sets $Z$ tested for each pair $X$, $Y$.
	At this point, there is a subtle but very powerful conceptual advantage of performing graph-discovery before any other regime-detection: As long as we maintain the
	faithfulness assumption (not erroneously conclude independence or a true regime),
	the result of graph-discovery will be consistent if we can correctly accept
	independence or a true regime for a \emph{valid separating set} $Z$
	(\eg the other parents of the target in PC). For CD
	our FPR-control only has to work correctly if $Z$ is a valid separating set.
	In this sense graph-discovery can be \emph{truly local}\Slash{}occur on a single test without prior knowledge of the multiple causal graph(s).
\end{rmk}

This section starts with the introduction of persistence assumptions and their
necessity due to a first impossibility-result §\ref{sec:persistence}.
Then we discuss the details of combining persistence with
dependence scores to reduce the problem to
a univariate one §\ref{sec:univar_problem}.
Next, we discuss the structure of this univariate problem, with particular focus
on which inductive biases are meaningful and should be realized and how these
considerations affect the precise formulation of sub-tasks and their order of execution §\ref{sec:testing_order}.
One scenario, which we will refer to as a \inquotes{weak-regime} §\ref{sec:weak_regimes} turns out to be particularly
challenging; this can be explained by a second impossibility-result concerning finite-sample error-control.
We propose assumptions, trade-offs and potential future paths to handle this fundamental limitation.
Afterwards, we briefly discuss details specific to conditional ($Z\neq\emptyset$)
independence-tests §\ref{sec:testing_conditional}.
In §\ref{sec:mCIT_scaling} we show that our direct testing philosophy does
lead to good asymptotic behavior in theory (and also in numerical experiments, §\ref{sec:num_experiments}).
Finally we
summarize our findings and discuss future work §\ref{sec:dyn_indep_test_summary}.
Further details and proofs can be found in the appendix §\ref{apdx:data_processing_layer}.

\subsection{Pattern Assumptions}\label{sec:persistence}

We analyze the existence of statistical tests for
the multi-valued independence-structure, finding the necessity of assumptions.
The underlying problem is closely related to randomness\Slash{}IIDness testing
\citep{WaldWolfowitzRuntest, RandomnessTestsOnline} (see also §\ref{sec:testing_homogeneity}).
We focus on the assumption of persistence in the following sense:
Real world regimes, for example, often persist for an extended period of time, \ie context-switches
occur on a time-scale larger than the primary dynamics under study.
This kind of prior knowledge about persistence, more generally about
patterns in the assignment of data points to
regimes is not itself in any way specific to time series.
Rather, for example spatially resolved data may be available where regimes are
oftentimes likely to
\inquotes{persist} to neighboring (in space) sites.
This kind of persistence (or pattern) assumption can be seen to allow for the resolution of
the initial impossibility result.
Persistence and more generally patterns are formalized in §\ref{sec:persistence_formal} in a way suitable
for the remainder of §\ref{sec:dyn_indep_tests};
there is not one objectively best way to leverage these patterns,
and we briefly explain the reasoning behind choices made.

\subsubsection{Necessity}

\begin{figure}[ht]
	\begin{minipage}{\textwidth}
		\centering
		\providecommand{\mcitPersistenceScale}{1.0}
\providecommand{\mcitPersistenceImgScale}{0.2}

\begin{tikzpicture}[scale=\mcitPersistenceScale]
	\draw (-1,0) node
	{\mayincludegraphics[scale=\mcitPersistenceImgScale,trim={1cm 1cm 1cm 1cm},clip]
		{figures/persistence/densities}};
	\draw (-1,1)	node[align=center, anchor=south]{(possibly multimodal)\\\textbf{densities}};
	
	\draw[->, very thick] (1,0.5) -- (2,1.0);
	\draw[->, very thick] (1,-0.5) -- (2,-1.0);
	\draw (1.5, 0) node[align=center] {\emph{different}\\mixtures};
	
	\draw (4,1.25) node
	{\mayincludegraphics[scale=\mcitPersistenceImgScale,trim={1cm 1cm 1cm 1cm},clip]
		{figures/persistence/time_resolved_3}};
	\draw (4,-1.25) node
	{\mayincludegraphics[scale=\mcitPersistenceImgScale,trim={1cm 1cm 1cm 1cm},clip]
		{figures/persistence/time_resolved_2}};
	\draw (4, 0) node[align=center] {\textbf{time-resolved}};
	
	\draw (6, 0) node[align=center] {vs.};
	
	\draw (8,1.25) node
	{\mayincludegraphics[scale=\mcitPersistenceImgScale,trim={1cm 1cm 1cm 1cm},clip]
		{figures/persistence/time_disregarded_3}};
	\draw (8,-1.25) node
	{\mayincludegraphics[scale=\mcitPersistenceImgScale,trim={1cm 1cm 1cm 1cm},clip]
		{figures/persistence/time_disregarded_2}};
	\draw (8, 0) node[align=center] {\textbf{time-agnostic}};			
\end{tikzpicture}
	\end{minipage}
	\caption{Without strong parametric restrictions on the form of densities (lhs), mixtures cannot
		be uniquely decomposed (rhs). Persistence can help resolving this ambiguity (middle).}
	\label{fig:persistence_illustrate}
\end{figure}
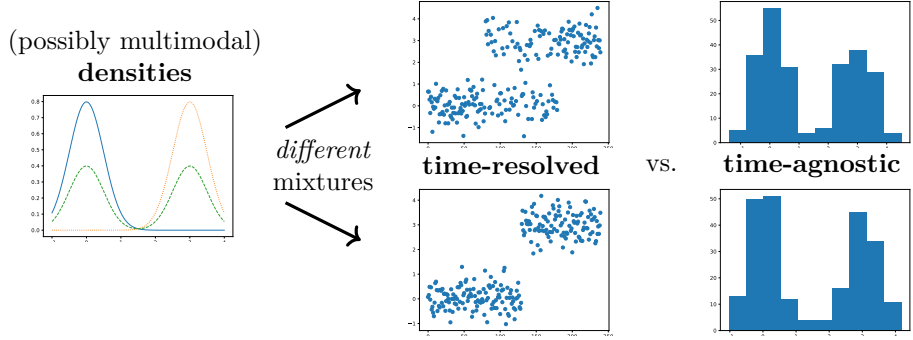

As show in Fig.\ \ref{fig:persistence_illustrate}, mixtures generally cannot
be easily decomposed into contributing densities: The histograms on the right-hand-side
differ for the two mixtures (top and bottom row) only by finite sample effects.
On time-resolved data with persistent structure (middle panel) the difference is
clearly visible.
However, in the IID case only the information in the histograms (total density) is considered.
While easy to illustrate on multi-modal densities, this problem is not specific to
multi-modality. Indeed it occurs, if the density candidates (families from which to pick
representatives) are \emph{not} linearly independent
in a very specific way \citep{yakowitz1968identifiability, BalsellsRodas2023},
and even in such cases, one would in the finite sample case certainly want to leverage
persistence as well.
There are many formal statements that capture this general behavior. In §\ref{apdx:necessity_persistence}
we show the following result:
\begin{prop}[Impossibility Result A]\label{prop:imposs_A}
	Given data as a set (without non-IID structure on the index-set),
	then the marked independence statement is not in general identifiable.
	By Lemma \ref{lemma:necessity_of_dyn_indep_struct}
	(see also Rmk.\ \ref{lemma:necessity_of_dyn_indep_struct})
	also the HCCD-problem is not in general identifiable.
\end{prop}
While Fig.\ \ref{fig:persistence_illustrate} may provide some intuition \emph{that}
persistence can help with this problem, we provide in §\ref{apdx:persistence_helps_necessity}
additionally some formal arguments to demonstrate \emph{how} it helps.

\subsubsection{Formulation}\label{sec:persistence_formal}

We start by two simple definitions that can be used to formalize a persistence requirement
at great flexibility, making it applicable not only to persistent-in-time regimes
but also to more general patterns in data.
The general form of this pattern (for example persistent in time) must of course
be known a priori.

\begin{definition}[Patterns: Blocks, Alignment and Validity]\label{def:blocks}
	Let $I$ be a set indexing our data.
	Given an integer $B\geq 1$,
	we call a subdivision of $I$ (and thus of the data)
	into disjoint subsets, containing $B$ elements each,
	a subdivision into blocks of size $B$.
	We will assume that the number $\Theta$ of blocks is maximal, \ie
	$\Theta = \lfloor \sfrac{N}{B} \rfloor$.
	Let $\mathcal{B}_I$ be the set of block subdivisions of $I$.
	A pattern is a fixed a priori structure that, given a block-size
	$B$, determines a subdivision of $I$ into blocks of size $B$.
	Formally it is a mapping:
	\begin{equation*}
		\Pattern : \mathbb{N} \rightarrow \mathcal{B}_I, B \mapsto T_B\txt,
	\end{equation*}
	such that $T_B$ is a subdivision into blocks of size $B$.
	We call a block $b$ ($\Pattern$-)aligned if $b \in \Pattern(\setelemcount{b})$.
	Given a regime structure (Ass.\ \ref{ass:segment_structure})
	we call a block $b$ valid, if all indices $i \in b$ are in the same context,
	otherwise we call $b$ invalid.
\end{definition}

\begin{example}[Time-Aligned Blocks]\label{example:persistent_in_time_pattern}
	The time-aligned pattern is the subdivision of $I=T=[0,N)\cap\mathbb{Z}$
	into size $B$ blocks of the form $b_\tau = [\tau B,(\tau+1)B) \subseteq T$ with
	$\tau \in [0,\Theta)\cap\mathbb{Z}$.
\end{example}

For more examples, see §\ref{apdx:patterns}.
To study convergence-behavior, a careful analysis of the asymptotic behavior of a
persistence assumption is necessary and carried out in §\ref{apdx:model_details}.
For finite sample and practical considerations the following definition
conveys the general idea.

\begin{definition}[Persistence]\label{def:persistence}
	We say a random indicator (cf.\ §\ref{apdx:model_details})
	is $(L, \chi)$-$\Pattern$-persistent, if a $\Pattern$-aligned block of size $L$
	is invalid with probability less than $\chi$.
\end{definition}

In practice, it is relevant to have $(L,\chi)$-persistence with $L$ large enough, yet
$\chi$ small enough.
This is one possible formulation of persistence (see §\ref{apdx:persistence_helps_necessity}).
More creative formulations like requiring the applicability of a particular
CPD method, leveraging prior knowledge about CP-candidates (for example a technological system
may be known to transition its state only, but not always, at specific points)
or parametric models for the indicator itself can also leverage
the compatibility with a (known) pattern.
In principle, our framework can be used with any such approach to persistence-based
aggregation of data (§\ref{apdx:leveraging_cpd_clustering}).
For further motivation of the choice employed here see §\ref{apdx:persistence_by_blocks}.

\subsection{The Univariate Problem}\label{sec:univar_problem}

Here we discuss the emerging univariate problem, and thus the main practical challenges
associated to testing marked independence.
We start by clarifying the reduction to a univariate problem already indicated at the beginning
of §\ref{sec:dyn_indep_tests}. Of the two steps shown in
Fig.\ \ref{fig:graph_global_to_graph_local}
we focus on the second one: turning data for $X$, $Y$ and (multiple) $Z_j$ into
a local score associated to an edge via an \emph{underlying} dependence-score $\hat{d}$.
Our numerical experiments
focus on partial correlation, but in principle the precise nature of $\hat{d}$ is not important,
we need only properties typically satisfied by scores used for independence-testing (in
a standard sense):
\begin{assumption}[Underlying CIT]\label{ass:underlying_d}
	On IID\Slash{}stationary data,
	the underlying dependence-score $d$ is estimated
	by $\hat{d}$ in an unbiased and consistent (see §\ref{apdx:simplify_d_for_power}) way
	and equals $0$ on independent data.
	The $\alpha$-confidence interval around zero
	for a given sample-count is known or can be estimated.
\end{assumption}
\begin{rmk}\label{rmk:relation_to_general_cits}
	Conditional independence testing is a hard problem \citep{shah2020hardness}.
	The additional information reported by a mCIT will certainly not make this problem easier.
	We do, however, want to focus on which \emph{new} problems arise for mCITs compared
	to conventional CITs.
	These new problems can be studied in isolation by formally searching not for
	independent and dependent regimes,
	but rather for regimes with dependencies
	$d_0 \neq d_1$ and checking if $d_0 =0$ while assuming reasonable convergence behavior
	(this assumption excludes\Slash{}detaches general CI-testing problems;
	see §\ref{apdx:simplify_d_for_power}).
	General CIT-testing related problems can then be accounted for
	either by assumptions on models (for example linearity)
	or by suitable interpretation of results (for example interpret \inquotes{independent}
	regimes as uncorrelated instead).
\end{rmk}
\begin{definition}[Marked Score-Independence Atoms]
	\label{def:marked_independence_rel_d}
	In the light of Rmk.\ \ref{rmk:relation_to_general_cits},
	we reformulate Def.\ \ref{def:marked_independence} of marked independence
	relative to $d$. Denoting the dependence-score estimate for the test $X \independent Y | Z$
	by $d(X\independent Y|Z)$ we distinguish the following hypotheses:
	\begin{enumerate}
		\item[$0:$]
			We call $X$ globally independent of $Y$ given $Z$,
			denoted $X \independent Y | Z$, if
			\begin{equation*}
				\forall i: E_{P_i}[d(X\independent Y|Z)] = 0
				\text.
			\end{equation*}
		\item[$1:$]
			We call $X$ globally dependent on $Y$ given $Z$,
			denoted $X \dependent Y | Z$, if
			\begin{equation*}
				\forall i: E_{P_i}[d(X\independent Y|Z)] \neq 0
				\text.
			\end{equation*}
		\item[$\mCIToutR:$]
			We say there is a true regime between $X$ and $Y$ given $Z$,
			denoted $X \independent_R Y | Z$, if
			\begin{equation*}
				\exists i_0 : E_{P_{i_0}}[d(X\independent Y|Z)] = 0
				\halfquad\txt{and}\halfquad
				\exists i_1 : E_{P_{i_1}}[d(X\independent Y|Z)] \neq 0
				\text.
			\end{equation*}
	\end{enumerate}
\end{definition}
For the sake of concreteness and simplicity, we restricted our numerical results and implementation
to a partial-correlation test. An efficient (both in statistical and runtime terms)
implementation of non-parametric scores can probably benefit substantially from
more specific optimizations and is left to future work.

The estimate $\hat{d}$ is statistical in nature and requires a collection of
data points to act upon. We take the simplest possible approach and apply $\hat{d}$ 
to the collections of data points obtained by exploiting a persistence assumption
as described in §\ref{sec:persistence}. Concretely, this means for the approach by
data-blocks §\ref{sec:persistence_formal} we apply $\hat{d}$ to each block of data individually.
This defines our univariate problem: analyze the resulting collection of dependence-values
to decide the marked independence query.

We focus on the question of marked independence. In §\ref{sec:indicator_translation},
when studying state-space constructions, we
will encounter another data-processing-layer test which we will refer to as
\inquotes{implication-test}, which behaves similar
and is discussed in §\ref{apdx:implication_testing}.
We distinguish three outcomes (dependence, global independence, existence of
an independent regime) on four possible ground-truths, see Fig.\ \ref{fig:univariate_problem}.
On the collection of dependence values this can be formulated as a
clustering-like\footnote{We have specific prior knowledge, like the value $d=0$ being special
with known $\alpha$-confidence region,
and are interested in a weaker notion of solution as is typically discussed in clustering
literature.} problem:
On valid blocks (those blocks contained entirely in a single regime) we could
simply ask, \inquotes{Is there more than one cluster?} and \inquotes{Is one (possibly the only) cluster at $d=0$?}
to decide the problem. The presence of invalid blocks asks for a robust assessment of these
questions, but does not fundamentally change the logic.

\begin{figure}[ht]
	\begin{minipage}{\textwidth}
		\centering
		\providecommand{\mcitUnivariateProblemScale}{1.0}
\providecommand{\mcitUnivariateProblemImgScale}{0.2}

\begin{tikzpicture}[scale=\mcitUnivariateProblemScale]
	\draw[->] (0,-1) -- (0,1);
	\draw[->] (-1.2,-0.95) -- (1.3,-0.95) node[midway, below]{$\hat{d}$};
	% left 25, right 22, bot 18, top 15?
	\draw (0,0) node {\mayincludegraphics[scale=\mcitUnivariateProblemImgScale,trim={2.4cm 1.7cm 2.1cm 1.4cm},clip]
		{figures/univariate_problem/independent}};
	
	\pgfmathsetmacro{\depx}{3}
	\pgfmathsetmacro{\adjustzero}{0.375}
	\draw[->] (\depx-\adjustzero,-1) -- (\depx-\adjustzero,1);
	\draw[->] (\depx-1.2,-0.95) -- (\depx+1.3,-0.95) node[midway, below]{$\hat{d}$};
	\draw (\depx,0) node {\mayincludegraphics[scale=\mcitUnivariateProblemImgScale,trim={2.4cm 1.7cm 2.1cm 1.4cm},clip]
		{figures/univariate_problem/dependent}};
	
	\pgfmathsetmacro{\truerx}{6}
	\draw[->] (\truerx-\adjustzero,-1) -- (\truerx-\adjustzero,1);
	\draw[->] (\truerx-1.2,-0.95) -- (\truerx+1.3,-0.95) node[midway, below]{$\hat{d}$};
	\draw (\truerx,0) node {\mayincludegraphics[scale=\mcitUnivariateProblemImgScale,trim={2.4cm 1.7cm 2.1cm 1.4cm},clip]
		{figures/univariate_problem/true_regime}};
	
	\pgfmathsetmacro{\weakrx}{9}
	\draw[->] (\weakrx-\adjustzero,-1) -- (\weakrx-\adjustzero,1);
	\draw[->] (\weakrx-1.2,-0.95) -- (\weakrx+1.3,-0.95) node[midway, below]{$\hat{d}$};
	\draw (\weakrx,0) node {\mayincludegraphics[scale=\mcitUnivariateProblemImgScale,trim={2.4cm 1.7cm 2.1cm 1.4cm},clip]
		{figures/univariate_problem/weak_regime}};

	\fill[color=gray!30] (-1.4, 1.1) rectangle (1.4, 1.7);
	\draw (0, 1.1) node [anchor=south]{independent};
	
	\fill[color=blue!30] (\depx-1.4, 1.1) rectangle (\depx+1.4, 1.7);
	\draw (\depx, 1.1) node [anchor=south]{dependent};
	
	\fill[color=green!30] (\truerx-1.4, 1.1) rectangle (\truerx+1.4, 1.7);
	\draw (\truerx, 1.1) node [anchor=south]{true regime};
	
	\fill[color=blue!30] (\weakrx-1.4, 1.1) rectangle (\weakrx+1.4, 1.7);
	\draw (\weakrx, 1.1) node [anchor=south]{weak regime};
	
	\draw[very thick] (\weakrx-1.5, 0) -- (\weakrx-1.5, -2.2) -- (\weakrx + 1.5, -2.2)
	node[midway, above]{each dependent};
	
	\draw[very thick] (\truerx-1.5, 0) -- (\truerx-1.5, -2.4) -- (\weakrx + 1.5, -2.4)
	node[pos=0.25, above]{non-homogenous};
	
	\draw[very thick] (\depx-1.5, 0) -- (\depx-1.5, -2.6) -- (\weakrx + 1.5, -2.6)
	node[pos=1/6, above]{dependent};		
\end{tikzpicture}
	\end{minipage}
	\caption{Different ground-truths for the univariate problem, expected outputs
		(color of headings),
		examples for density of per-block dependence-scores $\hat{d}$,
		and grouping by tests (braces below plots; reject independence, reject homogeneity, reject independence in low-dependence regime).}
	\label{fig:univariate_problem}
\end{figure}
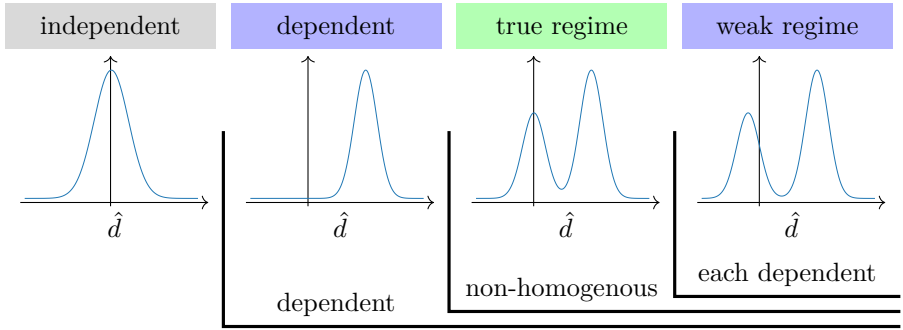

\textbf{It is of great significance that this is a direct testing formulation:}
We do not insist on any knowledge concerning the assignment of blocks (or data points)
to regimes. Rather, if multiple peaks overlap, we only ask, is one of them is centered at zero?
We can thereby accept
that for points in such an overlap we do not have substantial information about
their regime-assignments; we do not need such information to solve our task.

We formulated the problem above by asking two yes\Slash{}no questions. Indeed statistical
testing of binary hypotheses is more common and seems simpler in the present case.
Deciding between four ground-truth scenarios (three potential results) will require
at least two binary tests. It turned out that the second question
(\inquotes{is there a cluster at zero?}) on homogeneous data is rather standard (it is an
independence test), while in scenarios with multiple clusters it is very difficult.
So we treat both cases separately, thus will discuss a total of three binary tests.
This also leads to rather intuitive tests with simpler failure-modes:
Is the data dependent? Is it non-homogeneous? Is the lowest-dependence cluster centered at $d\neq 0$?
As shown in Fig. \ref{fig:univariate_problem}, these tests distinguish the four
ground-truth scenarios.

Generally, in an oracle case (no finite sample errors), the order of execution of
these tests is irrelevant. Nevertheless, there are at least two evident finite-sample
effects to consider:
\begin{enumerate}[label=(\alph*)]
	\item 
	Not all three problems are equally difficult, so unsurprisingly in practice the associated
	tests have different statistical properties. More concretely,
	the (global) dependence test and the homogeneity (IIDness\Slash{}randomness)
	test are more reliable
	than the distinction between weak and true regimes (deeper reasons
	will be discussed in §\ref{sec:weak_regimes} and §\ref{apdx:why_weak_is_harder}).
	Using three binary tests implies a total of eight possible combinations of results,
	which is more than needed, so some will collapse to the same output. This
	has the convenient side-effect that one will not have to execute
	all three tests on all scenarios. Thus, we can \eg on any homogeneous
	scenario entirely avoid the more difficult test.
	\item 
	There is an inductive bias in the order of test execution that
	encodes our belief about the relative frequency of different ground-truth scenarios
	and the severity of different error-types (for example the weak regime and global
	dependence are assigned the same outcome, so errors confusing these two scenarios are
	much less harmful than others).
\end{enumerate}

Next, we lay out the details of ordering the dependence
and homogeneity tests (the weak-regime test will be executed last, see (a)).
Afterwards we discuss the realization of the individual tests.

\subsubsection{Testing Order}\label{sec:testing_order}

Inspecting Fig.\ \ref{fig:univariate_problem}, we could apply tests in order
from left to right (dependence, homogeneity, weak regime), executing the next one
whenever the previous one is positive. But since all inhomogeneous results are
dependent, we could also test for homogeneity (as will be explained in §\ref{sec:testing_homogeneity})
first, and depending on the outcome
test either (global) dependence or weak regimes afterwards.
In the oracle-case this does not change the outcome, but in the finite sample case it may
(see points (a) and (b) in the introduction to the univariate problem above).

Applying an underlying score $\hat{d}$ (and associated independence-test)
designed for IID data to a non-IID sample and evaluating it as a CIT does \emph{not} affect FPR-control:
Inhomogeneity of $d$-values implies they are non-zero somewhere, thus
true negatives (the null) are independent \emph{and homogeneous}\footnote{This is slightly oversimplified:
Different models satisfying the null hypothesis need not produce the same variance of $\hat{d}$.
For practical considerations focusing on expectations is not a substantial restriction.}
(see also Fig.\ \ref{fig:univariate_problem}).
Thus type I errors remain controlled.
Concerning type II errors (statistical power) however,
an estimator which is well-suited for rejecting IID alternatives,
may no longer perform well on heterogeneous alternatives.
Finite sample performance in this case depends on the composition (rate of IID-alternatives
vs.\ rate of non-IID alternatives and their form) of the model-prior;
so there is typically no uniformly best estimator, rather we should seek to understand
the inductive biases of different approaches.

We restrict our discussion of this power-optimization to the comparison of
the two possible orders of independence- and homogeneity-test.
We expect the following:
Applying the independence test first ensures optimal power on IID data
and will perform well if the typical alternative encountered is close to IID.
Applying the homogeneity-test first avoids type II errors on certain types of
inhomogeneous data. A simple example for a problem where IID-approaches may have low
power for the case of correlation-based independence-testing is
the case of two regimes with similar size and similar correlation but of different sign
(which is also given as a motivating example in \citep[2.D]{Saggioro2020}).

Testing homogeneity first, on (homogeneous) independent IID data produces a 
false positive, if \emph{either} the homogeneity test produces a false positive
or if -- after a true negative from the homogeneity test -- the dependence
test produces a false positive. So if both control errors at a rate $\alpha$,
the total false positive rate is $\alpha + (1-\alpha) \alpha \approx 2\alpha$.
We are ultimately interested in the full marked independence result,
and there is no objectively fair comparison; luckily this effect is small
enough to not hinder the verification of our understanding of the associated inductive biases.
We include in Fig.\ \ref{fig:global_dependence} the global dependence first
approach twice with $\alpha = 2.5\%$ and with $\alpha = 5.0\%$ on its single stage
(to illustrate the above mentioned effect)
and the homogeneity first approach with $\alpha = 2.5\%$ on both of its stages.

\begin{figure}[ht]
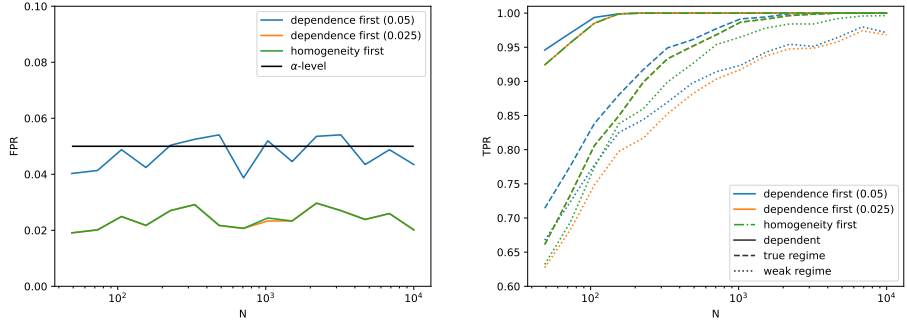

	\begin{minipage}{0.9\textwidth}
		\hfill
		\mayincludegraphics[width=0.45\textwidth,trim={0.5cm 0.2cm 1.5cm 1.25cm},clip]
		{global_dependence/fpr}
		\hfill
		\mayincludegraphics[width=0.45\textwidth,trim={0.5cm 0.2cm 1.5cm 1.25cm},clip]
		{global_dependence/tpr}
		\hfill
	\end{minipage}
	\caption[Global Dependence (dependence first vs.\ homogeneity first)]{
		Global Dependence (dependence first vs.\ homogeneity first):
		FPR (lhs) targeting 5\% (blue, green) and 2.5\% (orange),
		and TPR (rhs) for different ground-truth scenarios (see main text).
		Details on the benchmarking setup are given in §\ref{apdx:model_details}.
		Solid lines are dashdot if on top of each other.
	}
	\label{fig:global_dependence}
\end{figure}

Numerical experiments (see Fig.\ \ref{fig:global_dependence}) support the qualitative understanding outlined above: Error control works in both cases; the homogeneity test
used tends to overcontrol false-positives for larger $N$, this is discussed in the next section §\ref{sec:testing_homogeneity}.
Power against IID alternatives and against the true-regime case
differs mostly due to the ambiguous choice of
$\alpha$ (see above; in Fig.\ \ref{fig:global_dependence} for solid
lines and dashed lines respectively, green is between orange and blue).
The most evident difference is visible for the weak regime case
(see §\ref{sec:weak_regimes}, dotted lines):
Here the homogeneity first approach substantially outperforms
the global dependence first approach.

Our choice of inductive bias for subsequent numerical experiments
and recommendation for application of our method is to use the homogeneity first approach.
The reasons for this conclusion are:
\begin{itemize}
	\item The loss in power against IID-alternatives by applying homogeneity first is small.
	\item There is a substantial gain on a particularly challenging case
	(weak regimes, §\ref{sec:weak_regimes}); an approach that
	performs well on simple cases will eventually produce diminishing returns,
	a focus on challenging cases seems to make sense.
	\item We expect our framework and method to be applied in cases where there is
	substantial interest in, and a plausible presence of, non-IID effects;
	a focus primarily on the IID case seems like an inappropriate choice of inductive bias.
\end{itemize}

\subsubsection{Homogeneity}\label{sec:testing_homogeneity}

We want to decide whether our data is IID (random) or not.
This is a longtime and extensively studied problem
\citep{WaldWolfowitzRuntest, RandomnessTestsOnline},
so we again focus on simplicity and insights about the problem, rather than giving an extensive
review of related literature.
Future implementations of our framework can easily switch to more sophisticated approaches.
Generally it is of course not possible to test if data is IID considering all possible alternatives
(see also prop.\ \ref{prop:imposs_A}).
However, by the persistence-assumption, we already singled out a specific class of
alternatives. So we study testing of
IIDness vs.\ persistent alternatives. Our goal is to remain close to
our intuition about persistence, while still providing
a sound statistical test. This also includes a solid understanding of
the relation of hyperparameter-choices to persistence assumptions.

Indeed this intuition is very simple:
Under the null (trivial regime-structure) all blocks (aligned or not) are equal.
Pretend for the moment that we knew  the true value of $d$ and the $\beta$-confidence
intervals for $\hat{d}$ per block around this $d$.
Then the number of blocks with $\hat{d}$ outside
of the $\beta$-confidence region are binomially distributed; thus confidence bounds
can be obtained.
Under the alternative (persistent regime-structure) on the other hand,
valid blocks are distributed as if drawn
from one of multiple contexts exclusively, say with ground-truth values $d_0 \neq d_1$
(for gaining power given more than two regimes, similar arguments apply).
So the valid blocks lead to a multi-modal distribution spaced by
$|d_1 - d_0|$. If each regime contains at least a ratio of $\beta$ many valid blocks,
and $|d_1 - d_0|$ is large compared to the size of the $\beta$-confidence-region
(which depends on block-size, cf.\ also Fig.\ \ref{fig:tradeoffs} and discussion in §\ref{sec:weak_regimes}), we can reject IIDness. More formally:
\begin{lemma}[Binomial Test]\label{lemma:binomial_test}
	Fix a hyperparameter $\beta \in (0,1)$. Given $\alpha \in (0,1)$ and
	a sound estimator $\hat{q}_\beta^\leq$ of a lower bound of the $\beta$-quantile of the distribution
	in the $d_1$-regime (Def.\ \ref{def:quantile_est}), then we can reject data-homogeneity
	with type I errors controlled at $\alpha$ as follows:
	
	Let $k := \setelemcount{\{\tau | d_\tau < q_\beta \}}$ be the number of blocks
	with $d_\tau$ less than $q_\beta$.
	Denote by $\phi^{\txt{binom}}_{\Theta, \beta}$ the cumulative distribution-function
	of the binomial distribution $\BinomDist(n=\Theta, p=\beta)$.
	Then $p_0 := 1-\phi^{\txt{binom}}_{\Theta, \beta}(k)$ is a
	valid p-value under the null:
	\begin{equation*}
		P_{\txt{homogenous}}( p_0 < \alpha ) \leq \alpha\txt.
	\end{equation*}
	In particular rejecting homogeneity iff $p_0<\alpha$ leads to a valid test
	(Def.\ \ref{def:error_control_power_meta}).
\end{lemma}
\begin{rmk}\label{rmk:binomial_rounding_control}
	Controlling via a binomial, thus via an \emph{integer} count, will typically not
	control at $\alpha$ but rather below $\alpha$ (at the next larger count $k$ such that
	the cdf at $k$ is greater than $1-\alpha$);
	this would be true even if all approximations were exact.
\end{rmk}

An appropriate definition of soundness for the estimator $\hat{q}_\beta^\leq$ 
is discussed in §\ref{apdx:homogeneity_quantile_est}.
In §\ref{apdx:homogeneity_quantile_est} we also provide such sound estimators based
on analytical results for partial correlation and based on bootstrapping\Slash{}shuffling
suitable for general estimators and analyze them numerically on partial correlation.
Subsequent numerical results employ
the analytical estimate, primarily for low run-time at small loss of power;
applications and specialized versions for time series (beyond MCI, cf.\ §\ref{apdx:MCI})
may want to reconsider this choice.
Proofs and a discussion of power-results can be found in §\ref{apdx:testing_homogeneity}.

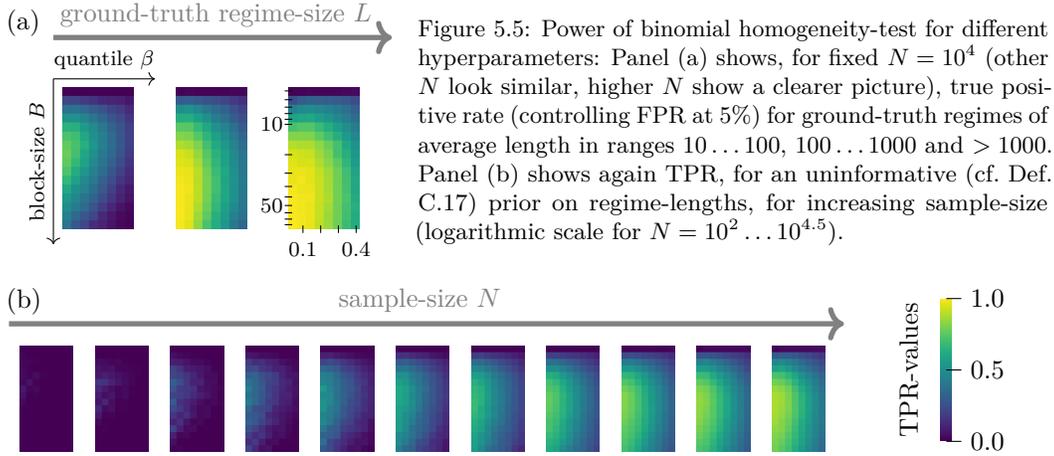
\begin{figure}[ht]
	\begin{minipage}{0.38\textwidth}
		\begin{tikzpicture}
			\draw (-1, 1.8) node {(a)};
			\draw[->, line width=.2em, color=gray] (-0.6,1.6) -- (3.9,1.6)
			node[midway, above]{ground-truth regime-size $L$\hspace*{0.5em}};
			\draw (0,0) node
			{\mayincludegraphics[scale=0.2]{homogeneity/hyperparameters_individual/by_length_region_10000_1}};
			\draw (1.5,0) node
			{\mayincludegraphics[scale=0.2]{homogeneity/hyperparameters_individual/by_length_region_10000_2}};
			\draw (3,0) node
			{\mayincludegraphics[scale=0.2]{homogeneity/hyperparameters_individual/by_length_region_10000_3}};
			
			\draw[<-] (-0.6, -1.15) -- (-0.6, 1.05) node [midway, left, anchor=south, rotate=90]
			{\footnotesize{block-size $B$}};
			\draw[->] (-0.6, 1.05) -- (0.75, 1.05) node [midway, below, anchor=south]
			{\footnotesize{quantile $\beta$}};

			\foreach \offset [count=\i from 0]
			in {0.8900, 0.7729, 0.6740, 0.5883, 0.5126, 0.4450, 0.0414, -0.2044, -0.3818, -0.5206, -0.6347, -0.7315, -0.8156, -0.8900} {
				\draw (2.575,\offset) -- (2.475,\offset) node (ytick\i){};
			}
			\draw (ytick5.west)+(-0.05,0) node {\footnotesize{$10$}};
			\draw (ytick10.west)+(-0.05,0) node {\footnotesize{$50$}};

			\draw (2.6 + 0.05 * 2.3571, -0.9) -- (2.6 + 0.05 * 2.3571, -1)
			node[below]{\footnotesize $0.1$};
			\draw (2.6 + 0.15 * 2.3571, -0.9) -- (2.6 + 0.15 * 2.3571, -1);
			\draw (2.6 + 0.25 * 2.3571, -0.9) -- (2.6 + 0.25 * 2.3571, -1);
			\draw (2.6 + 0.35 * 2.3571, -0.9) -- (2.6 + 0.35 * 2.3571, -1)
			node[below]{\footnotesize $0.4$};
		\end{tikzpicture}
		\vspace*{-0.75em}
	\end{minipage}
	\begin{minipage}{0.6\textwidth}
		\caption{
			Power of binomial homogeneity-test for different hyperparameters:
			Panel (a) shows, for fixed $N=10^4$ (other $N$ look similar, higher $N$ show a clearer picture), true positive rate (controlling FPR at 5\%) for ground-truth regimes
			of average length in ranges $10\ldots100$, $100\ldots1000$ and $>1000$.
			Panel (b) shows again TPR,
			for an uninformative (cf.\ Def.\ \ref{def:uninformative_prior_temporal})
			prior on regime-lengths, for increasing sample-size
			(logarithmic scale for $N=10^2\ldots10^{4.5}$).
		}
		\label{fig:homogeneity_power}
		\vspace*{-0.75em}
	\end{minipage}
	\begin{minipage}{0.8\textwidth}
		\begin{tikzpicture}		
			\draw (4.7, 1.3) node {(b)};
			\draw[->, line width=.2em, color=gray] (4.5,1) -- (15.6,1)
			node[midway, above]{sample-size $N$\hspace*{0.5em}};
			\foreach \Nidx in {0,...,10}{
				\draw (5+\Nidx, 0) node 
				{\mayincludegraphics[scale=0.15]{homogeneity/hyperparameters_individual/true_regime_\Nidx}};
			}
		\end{tikzpicture}
	\end{minipage}\hfill
	\begin{minipage}{0.15\textwidth}
		\begin{tikzpicture}
			\draw (-0.5,0) node[rotate=90] {TPR-values};
			\draw (0,0) node {				
				\mayincludegraphics[scale=0.15, angle=90]{color_map_power_default}
			};
			\draw (0.0,0.95) -- (0.2,0.95) node[anchor=west] {1.0};
			\draw (0.0,0) -- (0.2,0) node[anchor=west] {0.5};
			\draw (0.0,-0.95) -- (0.2,-0.95) node[anchor=west] {0.0};
		\end{tikzpicture}		
	\end{minipage}
\end{figure}

\paragraph{Choice of hyperparameters:}
Scans over the hyperparameters block-size $B$ and quantile $\beta$
on uninformative priors (see Fig.\ \ref{fig:homogeneity_power})
can be used to pick reasonable hyperparameters. Unsurprisingly, on large
(ground-truth) regimes, larger blocks perform better, thus we propose, additionally to the
\inquotes{uninformed} hyperparameters, a second set of parameters for
the case where background-knowledge suggests large ($\gtrsim 100$ data points per segment)
regimes. We pick hyperparameters based on sample-size $N$ and size of the conditioning-set
$|Z|$ based on a simple semi-heuristic form. A detailed analysis is given
in §\ref{apdx:homogeneity_hyper_heuristic}.
If substantial prior knowledge on ground-truth regime-sizes is available it may make sense
to consult the figures in \ref{apdx:homogeneity_hyper_heuristic} for better choices.
A summary will be given in Rmk.\ \ref{rmk:hyperparam_meaning}.

\begin{rmk}
	Note that the argument (and test) given in this section is making a
	\emph{direct} decision about the presence of non-IID structure,
	without requiring an explicit \mbox{(time-)}\allowbreak{}resolution of regimes.
	Indeed, given an (approximate) resolution into multiple (say two)
	regimes differing by $\Delta d = |d_1-d_0|$, testing IIDness
	must avoid (take into account) over-fitting the regime-structure:
	Allowing too many choices (high temporal resolution)
	for regime-resolutions will inflate the null (the $\alpha$-plausible
	value-range for $\Delta d$ given IID data), so an indirect approach
	must control overfitting (related to the choice of block-sizes in our
	approach) \emph{and} estimate the corrected null-distribution (which
	can be complicated).
\end{rmk}

\subsubsection{Weak Regimes}\label{sec:weak_regimes}

By the previous testing-stages, we have
(almost, we focus on the case of at most two regimes,
which primarily affects power-results;
see §\ref{apdx:simplify_regime_structure} and §\ref{apdx:weak_methods})
reduced the problem to the following situation:
We are given data with precisely two different persistent
regimes of dependence $d_1$ and $d_0$ respectively, where $|d_1| > |d_0|$.
The question we seek to answer is whether $d_0 = 0$.
Unfortunately this cannot be tested statistically in a standard sense,
a problem resulting essentially from the combination of
the \emph{rejection} of homogeneity with the \emph{acceptance} of independence:
\begin{prop}[Impossibility Result B]\label{prop:imposs_B}
	Assume the model-class considered has well-defined model-properties
	$d_0 \in\mathbb{R}_{\geq 0}$ and $\Delta d \in\mathbb{R}_{\geq 0}$
	such that a model $M$ has an independent regime iff $d_0(M) = 0$ and
	is homogeneous iff $\Delta d(M) = 0$.
	Further assume that we are given an mCIT $\hat{T} \in \{0,1,\mCIToutR\}$
	such that $\Pr(\hat{T}=\mCIToutR|d_0,\Delta d)$ is continuous in $d_0$ and $\Delta d$,
	and that there is no a priori known gap in realized parameters $d_0$ and $\Delta d$,
	that is $P(d_0,\Delta d) > 0$ near $d_0 = 0$,
	and there is no a priori known gap in $\Delta d$ on weak regimes,
	that is $P(\Delta d|d_0=0) > 0$ near $\Delta d=0$.
	
	Then, given $\alpha > 0$, no such mCIT can for any finite sample-size $N$
	\begin{enumerate}[label=(\alph*)]
		\item on true regimes uniformly control
		FPR at $<\alpha$ and have non-trivial power $>\alpha$ on any global independence alternative or
		\item on global dependence or the union of global hypotheses control
		FPR at $<\alpha$ and have non-trivial power $>\alpha$ on any true regime alternative.
	\end{enumerate}
\end{prop}
A more detailed discussion is provided in §\ref{apdx:why_weak_is_harder};
this is an important conceptual reason why the weak-regime test is more difficult
to realize than the homogeneity test.
One may of course ask about convergence of a specific decision-procedure
(without finite-sample error control) in some appropriate limit.
We do provide such an argument in §\ref{apdx:homogeneity_scaling}.
However, the above result shows that even if a procedure yields the correct decision
asymptotically, given any finite sample-count $N$,
there are always regime-structures on which it will not produce meaningful results.
So the more urgent question is: \emph{How} should a procedure look like to
make it transparent \emph{when} it will work and what are the \emph{trade-offs}
we necessarily have to make?

This section focuses on the concrete challenges arising for this problem.
The viewpoint we take is a rather practical one, focusing on assumptions
on model\Slash{}regime structures that are intuitive, have a clear relation to
the hyperparameters of the method we introduce and lead to reasonable
finite-sample performance on uninformative model-priors, where they restore
FPR-control in a Bayesian (averaging) sense.
A potential path forward to a theoretically more satisfying solution
is laid out in §\ref{apdx:assess_validity} based on the idea
to first assess whether the model\Slash{}regime structures satisfy given
assumptions, and insist on a decision \emph{only} then.

There are also some further subtleties that arise from the 
integration of the test with the other stages:
If we understand testing $d_0$ as an independence-statement on a subset of data,
than a (false) positive leads to a (false) rejection of the existence of
a true regime. From the perspective of detecting regimes, it thus leads to a
(false) negative. To avoid confusion, we show below the impact
of hyperparameter choices on recall and precision \emph{of regime-detection},
as this is more clearly interpretable in the greater context of our framework.
This also affects some practical decisions, for example:
\begin{rmk}\label{rmk:low_sample_count_as_validity}
	In cases where too few data points
	of low dependence are available for reliable testing,
	we do \emph{not} interpret the result as a regime,
	which makes sense from the regime-detection perspective, but is unnatural from
	the independence-testing perspective.
\end{rmk}
By testing in the wake of the homogeneity test,
the prior over parameters which the weak-regime test sees (even for an uninformative model-prior)
is rather specific. For example it encounters up to $\alpha$ only (persistence-)non-homogeneous data,
because the homogeneity test must have had power to reject.

We return to the concrete test-implementation.
The weak-regime test could be approached by clustering the $\hat{d}$-scores obtained
on blocks. For example for partial correlation the score (z-values) is approximately normal
so that a multi-modal normal approach could yield good results.
In practice there are at least two difficulties in making this work:
One requires a good uncertainty estimate on the cluster-means, and one requires robustness
against invalid blocks. We believe that it is possible to substantially improve
performance in this step by using an appropriate clustering method, but because
of these difficulties, we opt for a modified independence-test instead.
This also makes the applicability assumptions and trade-offs, which we laid out as
our principle goals before, more accessible.

Concretely we study the statistical behavior of those blocks with
$\hat{d}$-scores below a certain cutoff $c$;
we will assume \asswlog $d_1>0$.
We face multiple challenges:
\begin{enumerate}[label=(\roman*)]
	\item There are impurities below the cutoff, coming from the dependent regime.
	\item There are	data points above the cutoff lost from the low-dependence-regime,
		this incurs a bias on those remaining.
	\item Even in the persistent case, not \emph{all} blocks are valid,
		and we have to account for the impact of invalid blocks.
\end{enumerate}
Importantly, by our direct approach, we can treat these contributions statistically,
without the need to reason about the identity of individual blocks:
It is enough to assess the rate of impurities and their impact,
but we do not need to know which particular blocks are impurities.
The choices of the cutoff $c$ and block-size $B$ allow to navigate
the trade-offs between these problems as illustrated in 
Fig.\ \ref{fig:tradeoffs}.
Finally, the choice of cutoff also affects the available sample-size (below $c$).

\begin{figure}[ht]
	\begin{minipage}{0.6\textwidth}
		\begin{tikzpicture}[scale=0.5]
			\draw (0,0.2) node {\mayincludegraphics
				[scale=0.2,trim={1cm 1cm 1cm 1cm},clip]{figures/tradeoffs/B0_c0}};
			\draw (5,0.2) node {\mayincludegraphics
				[scale=0.2,trim={1cm 1cm 1cm 1cm},clip]{figures/tradeoffs/B1_c0}};
			\draw (10,0.2) node {\mayincludegraphics
				[scale=0.2,trim={1cm 1cm 1cm 1cm},clip]{figures/tradeoffs/B2_c0}};
			
			\draw (0,4) node {\mayincludegraphics
				[scale=0.2,trim={1cm 1cm 1cm 1cm},clip]{figures/tradeoffs/B0_c1}};
			\draw (5,4) node {\mayincludegraphics
				[scale=0.2,trim={1cm 1cm 1cm 1cm},clip]{figures/tradeoffs/B1_c1}};
			\draw (10,4) node {\mayincludegraphics
				[scale=0.2,trim={1cm 1cm 1cm 1cm},clip]{figures/tradeoffs/B2_c1}};
			
			\draw (0,-3.2) node {\mayincludegraphics[scale=0.2]{figures/invalid_fraction/low}};
			\draw (5,-3.2) node {\mayincludegraphics[scale=0.2]{figures/invalid_fraction/mid}};
			\draw (10,-3.2) node {\mayincludegraphics[scale=0.2]{figures/invalid_fraction/high}};
			
			\draw[line width=0.1em, ->] (-2.5,-1.8) -- (-2.5, 6) node [midway, left, anchor=south,
			rotate=90] {\small cutoff};
			\draw[line width=0.1em, ->] (-2.5,-1.8) -- (13, -1.8) node [midway, below] {\small
				block size};
		\end{tikzpicture}
	\end{minipage}
	\begin{minipage}{0.38\textwidth}
		\caption{Trade-offs in hyperparameter choices.
			The two rows of \inquotes{density}-plots have different cutoffs (vertical black line).
			Illustrated for both
			are densities for independent (blue) and dependent (orange) regime, as well
			as rates of:
			impurities (orange shaded); lost data points (blue shaded) and bias induced
			by not using them (hatched blue, vertical blue lines are
			true and biased mean);
			invalid blocks (hatched red).
			The black dash--dot line is the (actually observed) sum of all
			contributions.
			The bottom panel shows the increase in the amount of data in invalid
			blocks with larger block-size.
		}\label{fig:tradeoffs}
	\end{minipage}
\end{figure}
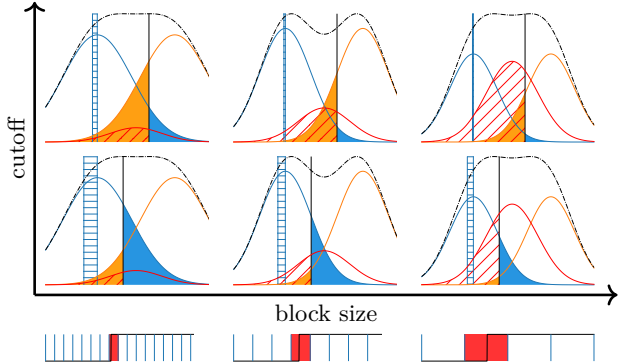

How can one account for these effects?
The effect of the lost samples can be well approximated by considering a
cutoff normal density; the incurred shift $\mu^B_c<0$
of the mean-value is shown in  Fig.\ \ref{fig:tradeoffs}
as blue-hatched region.
For the effect of impurities one may give a simple upper-bound as long as
there are less impurities, say $k_I$ than lost data points, say $k_L$:
Impurities individually contribute less than $c$ to the mean, thus at most $\leq \frac{k_I}{\Theta} c$
in total, while lost data points loose at least $c$, thus at least $\geq \frac{k_L}{\Theta} c$,
so that the mean changes by $\frac{k_I-k_L}{\Theta} c < |\mu_{0,\sigma_B}^c|$.
This argument required less impurities than lost data points ($k_I \leq k_L$).
Thus, as long as there are few invalid blocks, the mean-value of 
block-scores below $c$ is in the range $[\mu_{0,\sigma_B}^c, 0]$ if $d_0=0$
up to the expected finite-sample confidence of $\hat{d}$.
In summary we have just constructed the following method:
\begin{assumption}[Acceptance Interval Applicability]\label{ass:interval}
	Require controlled impurities $c < q_a$ where $q_a(B)$ is the $a$-th quantile of $\hat{d}$ on all data,
	and few invalid blocks $\chi \approx 0$:
	\begin{equation*}
		c < q_a
		\quad\text{and}\quad 
		B \ll L
	\end{equation*}
\end{assumption}
\begin{lemma}[Acceptance Interval]\label{lemma:interval_method}
	Some additional but rather weak assumptions are required,
	these are discussed in §\ref{apdx:simplify_d_for_power},
	see also the more formal statement Lemma \ref{lemma:interval_method_apdx}
	in the appendix.
	
	Choose a cutoff $c>0$
	and a minimal valid size $n_0^{\txt{min}}$,
	compute the mean $\hat{d}_{\txt{weak}}$
	of blocks of dependence $\hat{d}$
	with $\hat{d} < c$ 
	where the number of such blocks is $n_c$.
	\begin{equation*}
		\hat{d}_{\txt{weak}} = n_c^{-1} \sum_{\tau|\hat{d}_\tau\leq c} \hat{d}_\tau
	\end{equation*}
	We assume that the variance $\sigma_B^2$ of
	$\hat{d}$ estimated on $B$ samples under the null is known.
	
	Given an error-control target $\alpha > 0$,
	let $\sigma_{n_c}^\alpha = n_c^{-\frac{1}{2}}\sigma_B\phi(1-\frac{\alpha}{2})$
	be the (two-sided) $\alpha$-interval around $0$.
	Define an acceptance interval $I := [\mu_{0,\sigma_B}^c-\sigma_{n_c}^\alpha, \sigma_{n_c}^\alpha]$.
	Accept the null hypothesis (independence of the weak regime)
	if $n_c \geq n_0^{\txt{min}}$ and $\hat{d}_{\txt{weak}} \in I$,
	otherwise reject it.	
	Under Ass.\ \ref{ass:interval}, this test controls false positives.
\end{lemma}
For proofs, a more detailed discussion of assumptions and a formal interpretation
of the assumption $\chi \approx 0$ see §\ref{apdx:testing_weak_regime},
where also a statistical power analysis is given.
The assumption \ref{ass:interval} and Fig.\ \ref{fig:tradeoffs} inform us about the meaning of our hyperparameters:
\begin{rmk}[Relations of Hyper-Parameters and Model]\label{rmk:hyperparam_meaning}
	\begin{align*}
		B \leftrightarrow L
			\quad:\quad& \txt{$B$ has to be chosen small enough relative to $L$}\\
		c \leftrightarrow d_1 \approx \Delta d
			\quad:\quad& \txt{(weak-test only) $c$ has to be chosen small compared to}\\
			&\txt{regime-separation	$\Delta d = |d_1 - d_0|$ (measured by $d$)}\\
		\alpha, |I| \leftrightarrow d_0
			\quad:\quad& \txt{$\alpha$ controls the sensitivity for small $d_0$,}\\
			&\txt{but the size
			of the interval $I$ (thus $B$, $c$) also contributes,}\\
		\beta \leftrightarrow a
		\quad:\quad& \txt{(homogeneity-test only) $\beta$ controls the sensitivity to small regime-}\\
			&\txt{fractions $a$ (fraction of data-points in the independent regime).}\\
	\end{align*}
\end{rmk}

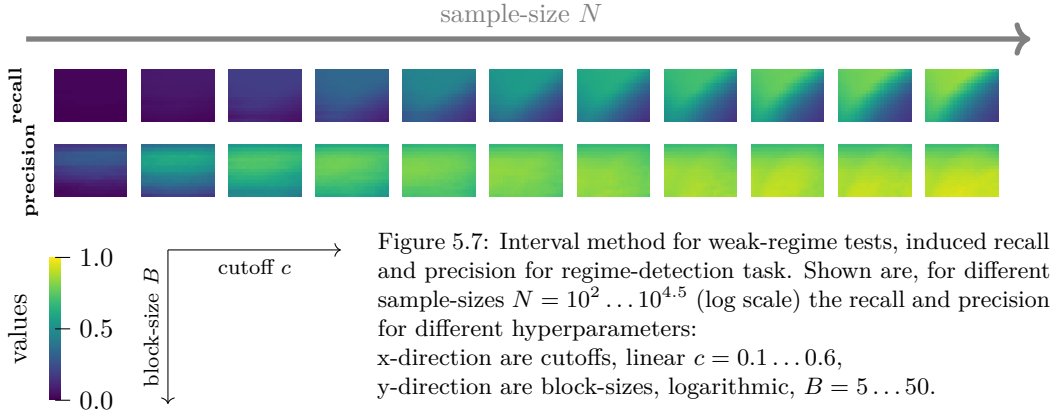
\begin{figure}[ht]
	\begin{minipage}{1.0\textwidth}
		\begin{tikzpicture}
			\draw[->, line width=.2em, color=gray] (-0.75,0.75) -- (12.6,0.75)
			node[midway, above]{sample-size $N$\hspace*{0.5em}};

			\draw (-0.5, 0) node[anchor=west]
			{\mayincludegraphics[scale=0.425]{weak/interval_recall}};

			\draw (-0.5, -1.0) node[anchor=west] {\mayincludegraphics[scale=0.425]{weak/interval_precision}};
			
			\draw (-0.9, 0) node[rotate=90]{\scriptsize\textbf{recall}};		
			\draw (-0.7, -1.) node[rotate=90]{\scriptsize\textbf{precision}};
		\end{tikzpicture}
	\end{minipage}
	\begin{minipage}{0.35\textwidth}
		\begin{tikzpicture}
			\draw (-0.5,0) node[rotate=90] {values};
			\draw (0,0) node {				
				\mayincludegraphics[scale=0.15, angle=90]{color_map_power_default}
			};
			\draw (0.0,0.95) -- (0.2,0.95) node[anchor=west] {1.0};
			\draw (0.0,0) -- (0.2,0) node[anchor=west] {0.5};
			\draw (0.0,-0.95) -- (0.2,-0.95) node[anchor=west] {0.0};
			
			\draw[<-] (1.5, -1) -- (1.5, 1.05) node [midway, left, anchor=south, rotate=90]
			{\footnotesize{block-size $B$}};
			\draw[->] (1.5, 1.05) -- (3.8, 1.05) node [midway, below, anchor=north]
			{\footnotesize{cutoff $c$}};
		\end{tikzpicture}		
	\end{minipage}\hfill
	\begin{minipage}{0.64\textwidth}
		\caption{Interval method for weak-regime tests, induced recall and precision
			for regime-detection task. Shown are, for different sample-sizes
			$N=10^2\ldots10^{4.5}$ (log scale) the recall and precision for different
			hyperparameters:\\
			x-direction are cutoffs, linear $c=0.1\dots0.6$,\\
			y-direction are block-sizes, logarithmic, $B=5\ldots50$.}
		\label{fig:weak_regime_recall_precision}
	\end{minipage}
\end{figure}
The average recall and precision resulting for regime-detection on the
(uninformative prior) benchmark are shown in Fig.\ \ref{fig:weak_regime_recall_precision}.
The observed structure and the choice of hyperparameters is discussed 
in detail in §\ref{apdx:weak_hyper_choices}.
The main observation is that the resulting precision is comparatively insensitive to
hyperparameter-choices, at least for larger $N \gtrsim 10^3$ samples, while
the recall breaks down for poor choices.
Other than that, similar remarks as for the homogeneity-test are in place.
We again provide two simple recommendations for hyperparameters for
generic cases and for cases for large ($L \gtrsim 100$) ground-truth regimes
in §\ref{apdx:weak_hyper_choices}.

\subsection{Conditional Tests}\label{sec:testing_conditional}

Conditional independence tests are of great importance for constraint-based casual discovery.
In principle the above discussion applies to arbitrary dependence-scores $\hat{d}$,
in particular they can be conditional ones.
Clearly, there are some optimizations possible.
For example if conditional tests are evaluated by regressing out the
conditions and testing on residuals (as it is the case for partial correlation),
then the regressors, if associated to homogeneous links themselves, could be shared
between blocks. We find that such sharing for linear regressors (which converge
comparatively fast) does not lead to tremendous improvements.
Interestingly, not sharing them does make the conditional testing
more robust against other types of non-IIDness, like drifting coefficients
(cf.\ §\ref{apdx:conditional_tests}).
From this perspective, one may want to choose a suitable middle-way
using multivariate (additionally time-dependent) regressors.
The appendices §\ref{apdx:homogeneity_hyper_heuristic} and
§\ref{apdx:weak_hyper_choices} study hyperparameter choices
for conditional tests.

A noteworthy finite-sample effect encountered concerns mediators:
Our method has difficulties with near-deterministic mediators (mediators that
are very strongly correlated to the source-node). We consider this to be a
known challenge to causal discovery in general, it seems to appear more generally
for CITs not just our mCIT (see also §\ref{apdx:datagen_mechanisms}).

\subsection{Full Test and Asymptotics}\label{sec:mCIT_scaling}

We show in the appendix §\ref{apdx:data_processing_layer} that
our method has good asymptotic properties.
While finite-sample guarantees, point-wise convergence and a convergence for
limits where typical regime-sizes are constant (of order $N^0$),
does not seem possible (Prop.\ \ref{prop:imposs_B}),
we can show that our method, under weak assumptions,
satisfies the \inquotes{next best statement}
of convergence given any (arbitrarily slow) increase
of typical regime-sizes.
Typical regimes-size here takes into account how
well the chosen pattern $\Pattern$ matches the actual pattern in data.

\begin{algorithm}
	\caption{\algMarkedIndependence{}}
	\label{algo:mCIT}
	\begin{algorithmic}[1]
		\State \textbf{Input:}
		Data '\texttt{data}' for the variables $X$, $Y$ and $Z_1,\ldots,Z_k$
		in $X \independent Y | Z_1, \ldots, Z_k$.
		\State \textbf{Output:}
		One of the three values defined in Def.\ \ref{def:marked_independence_rel_d},
		these are:\\\quad $0$ (global independence), $1$ (global dependence) or
		$\mCIToutR$ (true regime).
		\If{\texttt{is\_homogeneous}(\texttt{data})}
			\State \Return $1$ \textbf{if} \texttt{is\_dependent}(\texttt{data})
				\textbf{else} $0$
		\Else
			\State \Return $1$ \textbf{if} \texttt{is\_weak\_regime\_dependent}(\texttt{data})
				\textbf{else} $\mCIToutR$
		\EndIf
	\end{algorithmic}
	\hrule\vspace{0.5em}
	Here '\texttt{is\_homogeneous}' and '\texttt{is\_weak\_regime\_dependent}' perform the
	tests described in §\ref{sec:testing_homogeneity} and §\ref{sec:weak_regimes};
	'\texttt{is\_dependent}' performs a standard CIT or a robust variant (averaging per-block
	dependencies, see §\ref{apdx:conditional_tests}).
	These stages implicitly employ a dependence estimator $\hat{d}$ (in numerical experiments:
	partial correlation) and a pattern (like temporal persistence).
\end{algorithm}

\begin{thm}[mCIT Convergence]\label{thm:mCIT}
	For a formal statement, see §\ref{apdx:full_mCIT}.
	Under weak assumptions discussed in §\ref{apdx:simplify_d_for_power},
	given a pattern $\Pattern$ (Def.\ \ref{def:blocks}),
	in an asymptotic limit such that the system is eventually
	$(L, \chi)$-$\Pattern$-persistent for every $L$, $\chi > 0$,
	the mCIT described in Algo.\ \ref{algo:mCIT} 
	is asymptotically correct in a Bayesian sense (on average, but not point-wise).
\end{thm}
The scaling assumption is extremely weak, for example
it includes all \inquotes{trivial} scaling results
$\mathcal{R}_N(t) := \bar{\mathcal{R}}(t/\Lambda(N))$
with $\Lambda(N)\rightarrow \infty$
(Cor.\ \ref{cor:mCIT_on_trivial_scaling})
and even meaningful uninformative priors for
the case of persistent in time regimes (Cor.\ \ref{cor:mCIT_on_uninformative}).
Numerical experiments
(§\ref{sec:num_ex_scenarios}, paragraph \inquotes{Scaling behavior with $N$})
strongly support this claim.

\subsection{Summary}\label{sec:dyn_indep_test_summary}

Our approach to HCCD reduces the data-processing layer within the pipeline
Fig.\ \ref{fig:architecture} to comparatively simple \inquotes{marked independence}
statements (and indicator-relations §\ref{sec:indicator_translation}).
The testing of such statements was in this section further reduced
to a series of what might be considered standard problems:
\begin{itemize}
	\item The segmentation of data, for example by simple division into blocks,
		or CPD (§\ref{apdx:leveraging_cpd_clustering}).
	\item The focus on a single link's strength by dependence-scores.
	\item A univariate clustering problem, for which additional constraints are known,
		that allow for further substructure into:
	\begin{itemize}
		\item Independence Testing.
		\item IIDness\Slash{}randomness testing.
		\item Selection bias and independence testing on selected datasets.
	\end{itemize}
\end{itemize}
A comprehensive review of the extensive literature on these topics
is beyond the scope of this paper.
Instead, we focused on a well-defined structure of the full problem,
its reduction to these well-studied topics and a detailed analysis
of the difficulties encountered.
We provide a simple baseline method realizing the structure and
principles of our framework. Its primary goal is to demonstrate
feasibility and gain insight into finite-sample challenges.

While we expect that future work on these subproblems should be able to substantially
improve finite-sample properties, our baseline already turns out to perform very competitive to
state-of-the-art methods §\ref{sec:num_experiments}.
We attribute this to the systematic addressing of the two deep underlying challenges
outlined in the introduction, by following the guiding principles of locality and direct testing.
The qualitative trends observed in the numerical experiments §\ref{sec:num_experiments}
support this interpretation. In §\ref{apdx:mCIT_future_work} more ideas concerning future work are outlined.

	\section{State-Space Construction}\label{sec:indicator_translation}

	Induced indicators as encountered in example \ref{example:induced_indicators}
	make the extraction of the multi-valued independence-atoms $\IStructX$
	from the knowledge of a marked independence-atoms $\IStructM$ a non-trivial task.
	This section describes, how, despite these difficulties,
	this goal can be achieved with few additional tests.
	We provide an implementation
	of \algConstructStateSpace{} as required by the core-algorithm
	which under reasonable assumptions (see below) is theoretically sound
	and comparatively easy to implement; details and proofs can be found in
	§\ref{apdx:state_space_construction}.
	The approach requires an additional type of test (\ie the data-processing layer
	requires slight extension), to inter-relate detected indicators.
	This test can be realized similar to weak-regime tests (see §\ref{apdx:implication_testing}).
	While already the comparatively simple analysis provided in §\ref{apdx:state_space_construction}
	can substantially simplify the type of additional tests needed and reduce their number,
	we strongly suspect that considerable further improvement should
	be possible (see §\ref{sec:ind_locality}).
	
	We employ assumptions (see §\ref{apdx:state_constr_ass} for technical details) about
	\begin{enumerate}[label=(\alph*)]
		\item 
		acyclicity to simplify the problem and representation of solutions.
		Our approach should be extensible to cyclic models,
		however, this will require a more complex state-space construction
		and an additional type of (data-processing layer) test (§\ref{apdx:state_space_cyclic_models}).
		\item 
		modularity of changes and reachedness of states, to avoid the complexity of an additional
		post-processing phase §\ref{apdx:non_modular_changes}.
		Again, our approach should be extensible.
		\item
		skeleton discovery in the used CD-algorithm.
		While our framework is almost entirely agnostic of the
		implementation details of the CD-algorithm employed,
		we use a simple assumption about the skeleton
		discovery (that seems to be satisfied for all algorithms
		used in practice) to simplify proofs.
	\end{enumerate}
	Under these assumptions,
	the state-space construction can be realized in two phases.
	In the first phase (Algo. \ref{algo:discover_model_indicators}) model-indicators are identified as those links $X$--$Y$,
	for which a regime-dependence is found (the mCIT returns $R$),
	but no $Z$ with $X\independent Y|Z$ was encountered.
	Each model-indicator is then represented by a test of the form
	$X \independent Y|Z$
	such that $R_{XY|Z}\equiv \Rmodel_{XY}$.
	A suitable $Z$ is found as a minimum under an ordering-relation on $Z$
	that corresponds to a subset-relation of regions of independence on
	detected indicators $R_{XY|Z}$. The evaluation of this ordering-relation
	is the first encounter of the new type of data-processing layer test.
	Under the assumption of modular changes, these model-indicators span
	the state-space.
	
	In the second phase (Algo. \ref{algo:indicator_resolution}),
	the value of detected indicators in each state
	is deduced. This will again require the new type of test.
	In general the space of detected indicators grows very fast in the number
	$k$ of model-indicators (as $2^{2^k}$), there are, however,
	some restrictions on the form of detected indicators that can be
	leveraged to obtain a much smaller search-space. A particular
	helpful observation is that the removal of a link (vanishing
	of a model-indicator) from one state to another can close a
	causal path (in the sense of d-connectivity), but never open a new one.
	
	The results of both phases allow for the realization of the
	\algConstructStateSpace{} sub-algorithm required for the core-algorithm
	in §\ref{sec:core_algo}. We show in §\ref{apdx:state_space_construction}:

\begin{algorithm}
	\caption{\algConstructStateSpace{}}
	\label{algo:constr_state_space}
	\begin{algorithmic}[1]
		\State \textbf{Input:}
			The multivariate \texttt{data}, independencies tested $I$,
			marked tests $J \subseteq I$.
		\State \textbf{Output:}
			A set of mappings $S \subseteq \Map(J, \{0,1\})$.
		\Statex \emph{Phase I: Find changing links and represent their model-indicators
			by elements in $J$.}
		\State $\mathcal{M}$ := \texttt{discover\_model\_indicators}($I$, $J$)
		\State $S := \{ s_m | m \in \linspan_{\sfrac{\mathbb{Z}}{2\mathbb{Z}}}(\mathcal{M})^* \}$
		\Statex \emph{Phase II: Represent changing tests relative to model indicators.}
		\For{$\vec{j}\in J$}
			\State $s_{\vec{j}}$
				:= \texttt{represent\_indicator}($\mathcal{M}$, $\vec{j}$)
				\Comment{Represent}
			\State $s_m(\vec{j})$
				:= $s_{\vec{j}}$(m)
				\Comment{Transpose}
		\EndFor
		\State\Return $S$
		\Comment{\inquotes{Forget}\Slash{}discard indexing by $m$}
	\end{algorithmic}
	\hrule\vspace{0.5em}
	Details on \texttt{discover\_model\_indicators} are in §\ref{apdx:model_indicators}
	as Algo.\ \ref{algo:discover_model_indicators}; it returns a subset of $J$
	uniquely representing model indicators (in the oracle case).
	In the modular case, the state-space is spanned by the model-indicators.
	A co-vector $m \in \linspan_{\sfrac{\mathbb{Z}}{2\mathbb{Z}}}(\mathcal{M})^*$
	by definition assigns a value $0$ or $1$ to each basis-vector
	$R^{\text{model}}\in \mathcal{M}$.
	The algorithm \texttt{represent\_indicator}
	(see §\ref{apdx:detected_indicator_representation},
	Algo.\ \ref{algo:discover_model_indicators}) 
	expresses marked independencies $\vec{j}\in J$ relative to values of model-indicators;
	thus they can now be evaluated on co-vectors $m$ (see above).
\end{algorithm}

	\begin{thm}[State-Space Construction]\label{thm:state_space_construction}
		A formal statement is given as Thm.\ \ref{apdx:thm:state_space_construction} in §\ref{apdx:state_space_construction_J}.
		Pseudo-code for the two phases of the algorithm is provided in Algo.\ \ref{algo:discover_model_indicators} and Algo.\ \ref{algo:indicator_resolution}.
		
		Under the assumptions (a-c) outlined above,
		the described two-phase algorithm, in the oracle-case, implements
		a sound and complete state-space construction (Def.\ \ref{def:state_space_construction_specification}),
		satisfying the requirements of Thm.\ \ref{thm:core_algo} (soundness of the core algorithm).
	\end{thm}
	
	Finally, the interpretation of the discovered state-space
	in terms of model-properties can also be non-trivial.
	For example in presence of hidden confounders
	two states may differ by one or more inducing paths being present in
	one of them. This is excluded by a suitable acyclicity assumption,
	but for the general case jeopardizes both the independence of
	model indicators (even in the case of modular changes)
	and their interpretation as changing causal mechanism at that
	location in the graph. While for the discovery of states
	there is a rather clear path forward, their interpretation
	in terms of model-properties appears to be non-identifiable
	in some cases (see §\ref{apdx:state_space_inducing_paths});
	non-identifiable here refers to identifiability from the extended
	independence-structure in principle, not specifically to our
	proof-of-concept method. Such (possibly only partial) translations,
	together with knowledge-transfer results (§\ref{sec:edge_orientation_transfer})
	allow for an easy-to-read representation of the full results by a labeled (or colored) union-graph
	(Fig.\ \ref{fig:intro_toy_model}a).
	
	Our current approach does not take into account constraints on
	the potential representation of detected indicators that arise
	from (partial) knowledge of discovered graphs from previous iterations.
	We believe that in future work it should be possible to considerably
	improve the scaling to much larger models with many changing model-indicators
	by systematically exploiting such constraints (see §\ref{sec:ind_locality}).

	\section{Numerical Experiments}\label{sec:num_experiments}
	
	We evaluate the full method on randomly generated non-stationary SCMs
	(for details see §\ref{apdx:model_details}).
	To allow for a comparison with \citep{Saggioro2020}, we focus on the
	case of time series data without contemporaneous links and without hidden confounders.
	Further numerical experiments, especially for non-time series data, are included
	in §\ref{apdx:num_experiments}.
	We first briefly describe the methods we compare, then discuss results for different scenarios
	obtained via these methods, with details about evaluation metrics deferred to the end of this section.
	
	\subsection{Methods Applied}\label{sec:num_exp_methods}

	We compare multiple methods representing different classes of approaches.
	These representatives may not accurately capture, in terms of absolute numbers,
	what finite-sample performance \emph{could} potentially be achieved by each respective class.
	We will focus on the trends observed for individual approaches when
	changing the model behavior in different ways\Slash{}\inquotes{scenarios} (see next subsection).

	\paragraph{Regime-PCMCI} \citep{Saggioro2020} in simple terms can be described of as follows:
	Divide the data into regimes, run causal discovery (via PCMCI) and fit a model per regime,
	then find regime-assignments of data points that minimize the prediction-error
	from these models. Then iterate these (EM-)steps.
	Here persistence is leveraged by constraining the maximum number of
	regime-switches. The optimization-procedure may not always converge, we only took converged results
	into account.
	Fitting a model is not practically feasible with confounders or
	with unorientable links, which restricts the applicability of 
	such a method to time series models without contemporaneous links and without confounders.
	Model fitting and regime-assignment is \emph{global} in the causal graph,
	and causal discovery is run \emph{indirectly} via the discovered indicators.
	Regime-PCMCI requires\footnote{As \citep{Saggioro2020} remark, it is in principle
	possible (as for any other method, cf.\ \eg clustering below) to find an optimal number of clusters
	by optimizing some score (like minimal description length).
	In practice this would certainly not produce \emph{better} results than providing the ground
	truth; both runtime cost and stability of convergence (rate of successful runs) are also
	prohibitive to numerical experiments with this approach.}
	prior knowledge of the number of regimes and an upper bound
	for the number of switches; we provide it with ground-truth for both,
	as well as minimum and maximum time-lags.
	We further analyze three different hyperparameter-sets, two are recommended by
	the original paper \citep{Saggioro2020} for simple\Slash{}complex models,
	the third is the default parameter-set proposed by the implementation in the
	\texttt{tigramite}\footnote{\label{footnote:tigramite_repo}see \url{https://github.com/jakobrunge/tigramite}}
	python package. In the plots these are marked as: diamond (tigramite),
	plus (simple), X-symbol (complex).
	
	\paragraph{Clustering:}	
	We compare to a simple clustering approach.
	This approach clusters data points \emph{globally}
	(in the causal graph), then applies causal discovery (via PCMCI) \emph{indirectly} per cluster.
	Here the absolute values are not of great significance,
	as the use of other clustering-algorithms could change it (we use a simple
	kmeans clustering), but the qualitative behavior should be sufficiently
	representative for such	global clustering approaches.
	The method is provided with ground-truth information about the number of clusters,
	as well as minimum and maximum time-lags.
	To allow the approach to take advantage of persistence, the clustering is also applied to aggregations of
	data points over different window-sizes.
	
	\paragraph{Sliding-Window:}
	A sliding-window approach is easy to implement: Apply causal discovery (via PCMCI) to each \inquotes{sliding window} \emph{globally} (in the graph).
	For practical reasons our \inquotes{sliding} windows do not slide, but rather the data is again
	divided into fixed size non-overlapping windows.
	It is, however, difficult to evaluate. We start by counting for each link
	the number of windows on which it is present in the specific window's causal graph.
	Then pick two cutoffs $a_\pm$ and consider links absent if they appear less often than $a_-$,
	changing (regime-dependent) if they appear at least $a_-$ but less than $a_+$ often,
	(non-regime-dependent) present in the graph otherwise.
	The authors are not aware of any detailed investigations into the practical choice of $a_\pm$;
	generally this seems to be a difficult problem. The study of
	the practical choice of $a_\pm$ is beyond the scope of this paper.
	Instead, our evaluation fixes $a_\pm$ \emph{a posteriori} to the values which \emph{would have}
	provided the best evaluation metric.
	This is not a method applicable in practice, but it does provide an estimate of an \emph{upper bound}
	on the possible performance of sliding-window methods; to emphasize this hypothetical nature,
	the corresponding results will be plotted as dashed (as opposed to solid) lines.
	Again, also ground-truth about minimum and maximum time-lags is provided.
	The approach takes advantage of persistence by the choice of time-aligned windows.
	We apply this approach with different window-sizes.
	In practice the union-graph should in this case be more reliably obtained
	from a standard (stationary) method applied to all data, and we do not include
	the union of regime-specific edge-sets for comparison in the union-graph results.
	See also §\ref{apdx:sliding_window}.

	\paragraph{gLD-PCMCI (ours):} We will, of course, also include results for 
	the baseline method as described above (using PCMCI as CD algorithm),
	realizing the \emph{graph-local and direct}
	framework described in this paper in a simple but transparent way.
	We compare two sets of hyperparameters that were chosen respectively for an uninformative
	prior and with an inductive bias slightly favoring larger regimes in §\ref{sec:dyn_indep_tests} for IID-data
	(see also §\ref{apdx:prelim_num_results} and §\ref{apdx:hyper_param_policy}
	concerning relevance of and policy of choice of hyperparameters).
	Again, ground-truth about minimum and maximum time-lags is provided,
	\emph{no} ground-truth about the number of contexts or number of transitions is needed.
	Since the other methods here use PCMCI, we show results for gLD-PCMCI in the main text.
	Applications may consider using gLD-PCMCI+ for time series, as this combination
	seems to often perform more robustly (§\ref{apdx:num_exp_compare}, §\ref{apdx:MCI}).
	Our framework is slightly adapted to the time-series case based on the heuristic momentary
	conditional independence (MCI) principle \citep{pcmci}: In a first phase supersets of lagged parents
	are discovered, and in a second phase these are employed to remove lagged correlations
	to restore test-statistics; full mCITs are deployed only in the MCI phase (see §\ref{apdx:MCI}).

	\paragraph{Stationary PCMCI:}
	We additionally compare to the results one obtains by applying
	the PCMCI-algorithm (also underlying the other methods) without modifications to the data.
	This produces only a union-graph. Due to its special role as a reference, this
	line will be drawn in a dash-dot style in the plots below.

	\subsection{Scenarios for Numerical Experiments}\label{sec:num_ex_scenarios}
	
	The problem of non-stationary data-generation has by far too many
	degrees of freedom to define a universal benchmark.
	So we focus instead on a comparison of the qualitative behavior of different
	approaches given different variations of the benchmark data-set, for example the
	behavior with increasing sample-count $N$, with different model or regime-structure
	complexity or with different typical segment lengths.
	Thus, our primary goal of this numerical experiments section is
	to understand and verify our understanding of
	the inductive biases of the compared classes of methods.
	To this end, we study a sequence of \inquotes{scans} over model- (and data-)parameters 
	and analyze the observed behavior.
		
	\paragraph{Scaling behavior with $N$:}
	
	\begin{figure}[ht]
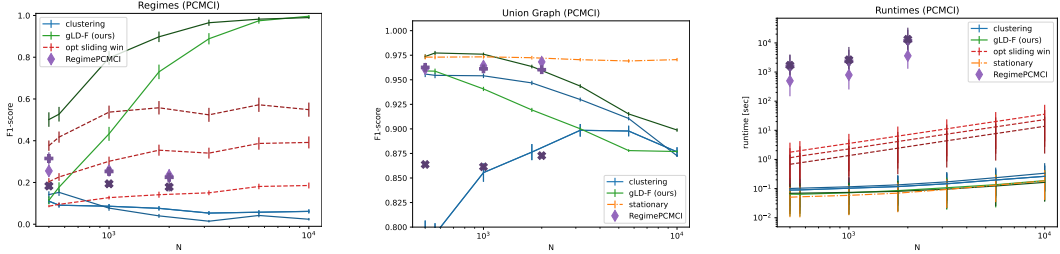

		\mayincludegraphics[width=0.3\textwidth]
			{cd_ts/sample_size/PCMCI_F1_regime}
		\hfill
		\mayincludegraphics[width=0.3\textwidth]
			{cd_ts/sample_size/PCMCI_F1_union_graph}
		\hfill
		\mayincludegraphics[width=0.3\textwidth]
			{cd_ts/sample_size/PCMCI_runtimes}
		\caption{Scaling with $N$:
			Shown are F1-score of regime-detection, union-graph and
			runtime in seconds (log scale)
			vs.\ sample-size $N$. Darker lines represent larger window-sizes or
			the hyperparameter choice optimized for large regimes.
			All approaches used PCMCI as basic CD-algorithm.}
		\label{fig:dcd_ts_scaling}
	\end{figure}
		
	An important aspect of any statistical analysis is convergence speed or more generally the behavior under
	the increase of the sample-size. With regimes in the data, a non-trivial choice with regard to
	the behavior of	segment-sizes with increasing $N$ has to be made.
	We decided for a plausible uninformative prior: For each indicator,
	a typical length-scale is picked first, then all segments of that indicator
	are drawn at random, but similar to this scale.
	An extensive discussion of the rational behind this choice is given in §\ref{apdx:model_details}.
	Figure \ref{fig:dcd_ts_scaling} shows results for the detection of
	changing links (local contexts)
	and the (summary\mbox{-)}\allowbreak{}union-graph skeleton (details on the precise 
	evaluation-metrics are given below in §\ref{sec:num_exp_eval_metrics}).
	The most apparent feature is that both clustering and the Regime-PCMCI approach actually become \emph{worse}
	for this choice of limit. This makes theoretical sense. Both methods are \emph{indirect} and essentially
	race against themselves: There are abstractly two tasks performed. First the assignment of data to regimes,
	second the independence-testing (via causal discovery) on the imperfectly reconstructed per-context data.
	Both tasks benefit from more data, however, for the second task, testing on imperfectly reconstructed per-context data,
	this also means that the sensitivity for these imperfections in the reconstruction increases.
	For example, the reconstructed independent regime, erroneously containing some dependent data points, will have
	non-zero correlation of say $E_N[\hat{d}] = d_0(N)\neq 0$. With growing $N$, the sensitivity of a correlation-test typically
	improves with $N^{-\onehalf}$, so if the reconstruction does not converge fast enough to keep $d_0(N)$ 
	dropping at a rate of at least $N^{-\onehalf}$, the correlation-test will eventually be able to detect the
	spurious dependence. In practice this rate of convergence is plausible for
	limits where $L = \mathcal{O}(N^1)$: Then the information contained in the pattern
	remains constant, so effectively the problem behaves as if the context were known
	(it is then always possible to divert an infinite amount, yet arbitrary small fraction, of the
	data to learn each bit of information about the pattern).	
	While the Regime-PCMCI and clustering fall off with larger sample-count,
	their performance on low sample-counts can be quite good.
	So especially for small sample-count, these approaches can be a good alternative.
	Additionally they can detect non-graphical changes; for example if mechanism-changes
	co-occur with distribution-shifts (especially with mean-shifts) in noise terms,
	then clustering should become a versatile and viable alternative.
	
	Our method is testing in a very simple, but direct way (cf.\ §\ref{sec:dyn_indep_tests}), therefore does not suffer from this
	convergence-race problem.
	The regime-sizes are bounded below to allow for a meaningful comparison, thus it is not surprising that
	the hyperparameter set for large regimes performs better.
	Also for clustering and sliding-windows, the larger window-sizes seem to perform better, probably for the same reason.
	For the sliding-window approach with optimal (a posteriori) cutoffs, there is also no indirection
	and therefore it improves with larger $N$ in this limit. However, the convergence is slower than for our method,
	likely due to propagating errors in CD first, before aggregating across windows. Similarly
	sliding-window variants seems to cap out at
	a lower value; given the fundamental limitations\Slash{}impossibility results found in §\ref{sec:dyn_indep_tests}
	such a non-trivial soft\footnote{The cap is \inquotes{soft}, because under the used model-prior,
	low typical regime-sizes eventually do become less common, albeit only logarithmically fast,
	cf.\ Lemma \ref{lemma:uninformative_prior_scaling_temporal}.} cap is expected for any method
	(it becomes more evident for our method with more challenging setups).
		
	Another very apparent effect is a decrease in the quality of the recovered union-graph (for our method and clustering).
	This effect seems to be present only for time series models (see §\ref{apdx:iid_sample_size} for
	IID-results), and seems to arise from edge-precision (see Fig.\ \ref{fig:pc1_phase_union_graph}).
	This observation is clearly worrisome. A practical \inquotes{fix} is to run a stationary algorithm separately
	and use its output to validate the union-graph.
	However, this means that it unfortunately is not easily possible to
	benefit from potential improvements for the union-graph discovery (in the time series case),
	that our method could provide (see \eg Fig.\ \ref{fig:dcd_ts_regime_complexity} middle panel, Fig.\ \ref{fig:cd_iid_noise_mixed} and in comparable scenarios does provide in the IID case
	Fig.\ \ref{fig:cd_iid_node_count}).
	There are at least two plausible explanations for this issue.
	Our homogeneity-first approach is designed for independent data
	and may not work well as-is with the MCI idea,
	we believe this can be fixed and might also help improve independence-tests for MCI-approaches
	(see also §\ref{apdx:MCI}).
	Using PCMCI+ and skipping the homogeneity-test in the PC$_1$-phase seems to
	improve the behavior of the union-graph supporting this idea, see §\ref{apdx:MCI},
	even though it is not entirely clear, why this would require the use of the PCMCI+ algorithm.
	The second explanation is particular intriguing, because it might simultaneously explain the fall-off observed for the
	clustering-method: Having $k$ time series of length $n$ is not the same as having a single time series of length $N=k\times n$.
	For multiple time series, the stationary distribution is sampled $k$ times, as opposed to once, for initial
	states, injecting additional
	information into the system, which should affect null distributions (of independence tests)
	thus test-statistics and error-rates.
	While we do account for changes in sample-size and batch-aggregation, we do not account for this
	separate effect, and we are not aware of other methods accounting for this effect.
	A broadening of the null, which is not accounted for, would also explain the loss in edge-precision
	specifically (see Fig.\ \ref{fig:pc1_phase_union_graph}).
	
	Finally, the problem might of course be due to a bug in our code.
	However, clustering and our method use different implementations of the partial correlation test
	and share very little to no code,
	so while we cannot exclude this possibility, we believe the issue is \inquotes{real}, due to a statistical effect,
	not an implementation issue.

	\paragraph{Model Complexity by Node-Count:}
	
	\begin{figure}[ht]
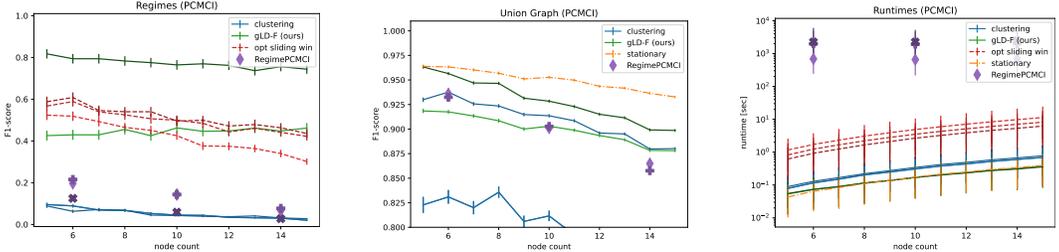

		\mayincludegraphics[width=0.3\textwidth]
		{cd_ts/node_count/PCMCI_F1_regime}			
		\hfill	
		\mayincludegraphics[width=0.3\textwidth]
		{cd_ts/node_count/PCMCI_F1_union_graph}		
		\hfill	
		\mayincludegraphics[width=0.3\textwidth]
		{cd_ts/node_count/PCMCI_runtimes}
		\caption{Scaling with node-count: The number of nodes in the system can serve
			as a proxy for model-complexity.}\label{fig:dcd_ts_complexity}
	\end{figure}

	We use the number of nodes (at fixed link-density per node) as a proxy for model complexity
	(for results on changing link-density, see §\ref{apdx:num_exp_link_density}).
	Figure \ref{fig:dcd_ts_complexity} shows the results, again for regimes and union-graph.
	Concerning regimes, our method, being local in the graph, does not considerably suffer from an increase in
	model complexity (in the sense of node-count).
	All other methods show a downward-trend, as would be expected for global methods.
	The effect is comparatively weak however.
	The recovery of the (summary-)union-graph suffers for all methods from an increase in complexity,
	which is not surprising.
	
	\FloatBarrier
	
	\paragraph{Regime-Structure Complexity by Indicator Count}
	
	\begin{figure}[ht]
		\mayincludegraphics[width=0.3\textwidth]
		{cd_ts/indicator_count/PCMCI_F1_regime}
		\hfill				
		\mayincludegraphics[width=0.3\textwidth]
		{cd_ts/indicator_count/PCMCI_F1_union_graph}
		\hfill				
		\mayincludegraphics[width=0.3\textwidth]
		{cd_ts/indicator_count/PCMCI_runtimes}
		\hfill			
		\mayincludegraphics[width=0.3\textwidth]
		{cd_ts/indicator_count_sync/PCMCI_F1_regime}				
		\hfill				
		\mayincludegraphics[width=0.3\textwidth]
		{cd_ts/indicator_count_sync/PCMCI_F1_union_graph}		
		\hfill						
		\mayincludegraphics[width=0.3\textwidth]
		{cd_ts/indicator_count_sync/PCMCI_runtimes}	
		\caption{Scaling with indicator-count: The number of changing indicators
			as a proxy for regime-structure complexity.
			Top row: modular/independent changes, bottom row: synchronized changes.}
		\label{fig:dcd_ts_regime_complexity}
	\end{figure}
	
	We use the number of changing links (at otherwise fixed graph-parameters) as a proxy for regime-structure
	complexity.
	We consider two cases: Independently changing links (which our method is optimized for)
	and synchronized changes in multiple links (which should favor global methods).
	Some care should be taken in the interpretation of F$_1$-scores because the
	increase in the number of changing indicators changes the actual (ground-truth)
	positive rates (see §\ref{apdx:ts_indicator_count}).
	See Fig.\ \ref{fig:dcd_ts_regime_complexity}.
	At least for Regime-PCMCI, this understanding of biases can be confirmed (for independent-indicators, it has very severe
	convergence-issues, fully breaking down at around three indicators, precluding results for
	larger counts).
	Generally, clustering would probably also benefit from synchronized changes more substantially,
	if these changes would for example change noise-offsets (mean-values of additive noise) at
	the same change-points.
	So, one learning from this experiment is that the model-fitting and evaluation approach of Regime-PCMCI
	is better at finding (global) model changes in the fitted \emph{effects} than for example clustering.
	
	It is noteworthy that our approach, likely due to its built-in robustness towards non-stationarities
	(cf.\ §\ref{apdx:conditional_tests}), can almost maintain the quality of the recovered union-graph
	even under the addition of more non-stationarities. This also applies to Regime-PCMCI in the case of
	synchronized changes.
	
	\paragraph{Regime-Sizes}

	\begin{figure}[ht]
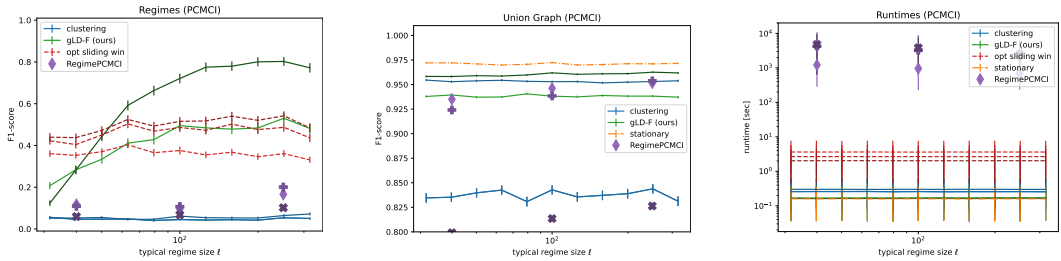

		\mayincludegraphics[width=0.3\textwidth]
		{cd_ts/regime_size/PCMCI_F1_regime}				
		\hfill
		\mayincludegraphics[width=0.3\textwidth]
		{cd_ts/regime_size/PCMCI_F1_union_graph}
		\hfill
		\mayincludegraphics[width=0.3\textwidth]
		{cd_ts/regime_size/PCMCI_runtimes}
		\caption{Scaling with regime-size: Ground-truth regime-sizes should
			affect relative effectiveness of hyperparameters.}
		\label{fig:dcd_ts_regime_sizes}
	\end{figure}
	
	We use different window-sizes and a hyperparameter-set for large regimes vs.\ one for the general case.
	One would expect that starting at small ground truth segment-lengths, then increasing them
	would demonstrate a clear hand-off\Slash{}cross-over at some point in the middle.
	Interestingly, while there is some crossover-behavior, it is not very pronounced for time series
	models (see Fig.\ \ref{fig:dcd_ts_regime_sizes}). The behavior is closer to this expectation in the IID-case
	(§\ref{apdx:num_exp_regime_size}),
	so we believe that this is a time series effect (at least for our method; considering that we use
	hyperparameters configured for the IID case, this is not too surprising).
	From this observation, we recommend to typically use the hyperparameter set for larger regimes in the
	time series case for now.
	It might make sense to investigate good hyperparameter choices for the time series case in future work.

	\subsection{Evaluation Metrics}\label{sec:num_exp_eval_metrics}
	
	We usually compare F$_1$ scores on two problems:
	\begin{itemize}
		\item Which links are found to change between contexts? -- Each method outputs
		(potentially among other information) a collection
		of per-context graphs $G_s$. We compare these, to see which links occur for only some
		$G_s$, thus \emph{change} (for the sliding window case, see §\ref{sec:num_exp_methods}).
		Such a changing link is a 'positive' (all other possible -- see below -- links are 'negatives'),
		it is a 'true positive' ('false negative' if a negative) if the model
		on this link has a non-trivial model indicator, and a 'false positive'
		(or 'true negative' if negative) if the associated
		true model-indicator is trivial (if the link does not change\Slash{}is always present or always absent).
		For computing F$_1$, a negative is any possible link
		which is not included in the set of changing links, but excluding auto-lags\footnote{Our data-generation
		generates no non-trivial indicators on auto-lags, see §\ref{apdx:model_details}.
		All methods can report changes on auto-lags,
		and those are counted as false positives; this seems to happen very rarely.}.
		\item How well is the summary-union-graph recovered? -- The union-graph is the edge-union
		over the individual $G_s$ returned by the method. Since we use PCMCI with $\tau_{\text{min}}=1$
		all links are lagged and oriented forward in time, so there is no ambiguity on orientations at this point. The summary-graph is then the edge-union over all time-lags.
		A 'positive' is a link in this summary-union-graph, it is a 'true positive' if the true
		summary-union-graph also contains this edge, and a 'false positive' if the true
		summary-union-graph does not contain this edge.
	\end{itemize}
	We are particularly interested in the first question, which captures the actual detection
	of the presence of regimes in a way that can be compared for the different methods providing
	different output-information in a simple and fair way.
	Regime-PCMCI (see below) is based on the PCMCI causal discovery algorithm for time series;
	the other methods can be applied with other methods, we compare time series results for
	PCMCI \citep{pcmci}, PCMCI+ \citep{pcmci_plus}, LPCMCI \citep{lpcmci}
	(via implementations in the \texttt{tigramite}\textsuperscript{\ref{footnote:tigramite_repo}}
	python package)
	and IID-data results for PC \citep{PCalgo}, PC-stable \citep{PCstable}
	and FCI \citep{spirtes2001causation}
	(via implementations in the \texttt{causal learn} \citep{causal_learn} python package)
	in more detail in the appendix §\ref{apdx:num_experiments}.
	
	F$_1$-score favors recall over precision;
	for example the best (\wrt F$_1$-score) purely random method is always the trivial-true
	algorithm (always return true). Our method particularly features a very high precision,
	so this puts it at a slight disadvantage; for the most part this is not a problem,
	only the results on increasing indicator count should be interpreted with particular care:
	Increasing the number of non-trivial model-indicators changes the number of true positives
	relative to the number of true negatives, which means the behavior across different
	counts, even the qualitative behavior, is only easily interpreted for methods with
	comparable precision and recall. For details see §\ref{apdx:ts_indicator_count}.
	F$_1$-score is, however, simple and in very wide-spread use which is the reason
	why we employ it here.

	\FloatBarrier
	
	\section{Conclusion}
	
	The ideas behind our framework allow to break the HCCD problem down
	into much more attainable and well-defined sub-problems:
	A (stationary) CD task, a state-space reconstruction task, a marked independence testing task
	and the composition task implemented by the core-algorithm itself to combine
	solutions to these sub-tasks into a useful framework.
	This gives a structure to the problem which allows for a detailed and insightful
	study of the many challenges implicit in the original problem.
	These insights reach from limitations on identifiability and the construction of
	meaningful solutions and their interpretation to systematic improvements
	in the scaling behavior with sample-size (by directness of testing)
	and with complexity of model and regime-behavior (by locality).
	
	\paragraph{Limitations:}
	Our approach so far does \emph{only} discover changes in the graph. It will neither inform about
	nor leverage changes in effect-strengths or noise-distributions.
	It can also not benefit from co-occurrence of multiple changes; it can only
	account for such synchronized or related changes in principle (§\ref{apdx:non_modular_changes}).
	It seems that co-occurrence of graphical (and likely regime-form effect-strength) changes
	favors model-fitting approaches like Regime-PCMCI \citep{Saggioro2020},
	while co-occurrence with changes of noise-distributions might favor clustering approaches.
	Finally, our approach requires comparatively large data-sets, its major
	benefits seem to show primarily if $\gtrsim 1000$ data points are available,
	but this number of course also depends on model-complexity and effect-sizes.
	
	We currently focus on partial correlation tests, thus on linear models.
	It should be possible to support other independence tests, but doing so efficiently
	and at acceptable runtime could benefit from further work.
	For indicator-relation tests as used in §\ref{sec:indicator_translation},
	we provide a proof of concept (§\ref{apdx:implication_testing}), but we believe
	that a careful analysis, especially of the observations in §\ref{sec:ind_locality},
	has the potential to provide additional insight and improved orientation-correctness.
	
	On time series data our baseline implementation encounters unexpected difficulties
	in the quality of the recovered union-graph. The most likely cause of this problem
	is our simple MCI-based time series adaptation §\ref{apdx:MCI}.
	This problem can likely be fixed with tests more appropriate for the time series case, at least
	for the PC$_1$ phase.
	For the moment, we recommend the additional execution of the underlying stationary CD-algorithm
	directly on the data to obtain a reference union-graph to assess the stability of (regions of)
	the union-graph.
	Using PCMCI+ and skipping the homogeneity-test in the PC$_1$-phase seems to
	improve the behavior of the union-graph as well, see §\ref{apdx:MCI}.
	
	We did not include an application to real-world data. A thorough analysis of
	a result on real-world data would doubtless be interesting, but it would also
	require an extensive discussion which considering the length of this paper
	seems more adequate for a dedicated write-up.
	Our implementation in code and numerical experiments are currently limited to
	linear, acyclic models and we focus in the evaluation on causally sufficient (no hidden confounders)
	examples and (for the IID case) on graph skeletons.
	
	\paragraph{Strengths:}
	Our approach can benefit substantially from large data-sets; it demonstrates
	both in theory and numerical experiments good statistical convergence properties
	and maintains low run-times.
	Similarly it scales well to large systems and complex regime-structures with many
	non-trivial model-indicators, again at comparatively low run-times.
	It behaves, except for the union-graph issue on time series mentioned above,
	as expected from the inductive biases realized by a local and direct approach.
	This makes its strengths, weaknesses and applicability to a given problem comparatively easy to assess.
	Our implementation is also compatible with the presence of latent confounders
	and contemporaneous links (thus IID problems) if the underlying CD-algorithm is chosen
	accordingly. It can also not only be used for persistent-in-time regimes but also
	for other types of patterns in data.
	Finally, from a practical perspective, our method does not face numerical-convergence
	problems (like Regime-PCMCI does on some data-sets) and produces succinct and easy
	to interpret output (see also §\ref{sec:edge_orientation_transfer}) at low run-time,
	making it an easy to use potential first extension beyond the stationary case for many applications.
	
	A good framework should be extensible
	and be able to leverage existing technology.
	The use of standard CD-algorithms is already possible in our baseline implementation,
	while we only used partial correlation for our numerical experiments
	our framework was designed to easily accommodate other independence test scores.
	Also the integration of existing CPD or clustering technology seems plausible
	(§\ref{apdx:leveraging_cpd_clustering}).
	Similar remarks apply to the combination with JCI-like \citep{JCI} ideas
	(see §\ref{apdx:JCI}).
	An extension to account additionally for changes in effects or noise-terms, and
	a combinations with methodology like CD-NOD \citep{CD-NOD} for other (non-regime-like)
	types of non-stationarities should be possible at multiple
	points of the framework, by its highly modular design.

\section*{Author's Statements}

\paragraph{Funding Information:}
J.\,R.\ has received funding from the European Research Council
(ERC) Starting Grant CausalEarth under the European Union’s Horizon 2020 research and innovation
program (Grant Agreement No. 948112). J.\,R.\ and M.\,R.\ have received funding from the
European Union’s Horizon 2020 research and innovation programme under grant agreement No
101003469 (XAIDA). This work was supported by the German Federal Ministry of Education and Research (BMBF, SCADS22B) and the Saxon State Ministry for Science, Culture and Tourism (SMWK) by funding the competence center for Big Data and AI "ScaDS.AI Dresden/Leipzig".
The authors gratefully acknowledge the GWK support for funding this project by providing
computing time through the Center for Information Services and HPC (ZIH) at TU Dresden.
This work used resources of the Deutsches Klimarechenzentrum (DKRZ) granted by its Scientific
Steering Committee (WLA) under project ID 1083.

\paragraph{Author Contribution:}
All authors have accepted responsibility for the entire content of this manuscript
and consented to its submission to the journal,
reviewed all the results and approved the final version of the manuscript.

M.\,R.\ has contributed problem analysis, structure and details of formal and
numerical analyses; J.\,R.\ has contributed to the choice of research questions,
the placement relative to existing results and methods
and the accessible presentation of results.

\paragraph{Conflict of Interest:}
Authors state no conflict of interest.

\paragraph{Code Availability:}
An open-source implementation of our framework in python code
is available at \CodeRepo{}.

	\bibliography{./BibTex}
	
	\appendix
	
\section{Further Numerical Experiments}\label{apdx:num_experiments}

The experiments in the main text §\ref{sec:num_experiments} focused on the time series case.
The rational behind this decision was, to include comparisons to existing methods into the main-text;
with Regime-PCMCI being applicable to the time series case only, this lead to the corresponding restriction
in §\ref{sec:num_experiments}. More precisely, in order to have only (time-lag-)orientable links for
fitting models, Regime-PCMCI currently requires time series without contemporaneous links, which means
it does \emph{not include} the IID-case as a special case; model-fitting may be possible
with suitable representatives of graph equivalence-classes in principle.

Since our framework is not limited to time series, it makes sense to demonstrate its capabilities
with IID-data algorithms as well. We study the PC \citep{PCalgo}, PC-stable \citep{PCstable}
and FCI \citep{spirtes2001causation} algorithms via implementations provided by the
\texttt{causal learn} \citep{causal_learn} python package.
These are among the best understood and most studied choices for constraint-based causal discovery.
Beyond demonstrating the capability of application to IID-data and flexibly exchanging underlying
CD-algorithms, there are at least three more questions that this section can answer:
\begin{itemize}
	\item
	In §\ref{sec:num_experiments}, we encountered an issue with union-graph discovery becoming
	worse given more data. We will find below that this problem does not occur for IID-data, so
	it is likely a time series effect.
	\item
	In §\ref{sec:num_experiments}, we noticed that of our two hyperparameter-sets, the one for larger
	regimes seems to almost always perform better, even for small regimes.
	For IID-data we will find that this effect becomes much weaker, but still seems to be present.
	We take this as an indication that both the way multiple independence-tests are combined in CD and
	time series effects should play a role for hyperparameter choices.
	\item
	For time series, different parameters like maximum time-lag and node-count are interdependent.
	Also, stability considerations make it difficult to systematically explore the behavior under
	different effect-to-noise ratios.
	So we study some additional parameters and the dependence on type and quantitative scale of noise-distributions
	on IID-algorithms instead.
\end{itemize}

Methods and evaluation metrics are the same as in the main text.

\subsection{Details on Time Series}

Results for time series are already discussed in the main text §\ref{sec:num_experiments}.
Here we provide some additional information about the impact of the choice of underlying
CD-method (PCMCI \citep{pcmci}, PCMCI+ \citep{pcmci_plus}, LPCMCI \citep{lpcmci};
using implementations of the \texttt{tigramite}%
\footnote{see \url{https://github.com/jakobrunge/tigramite}} python package)
and a more detailed discussion about the
evaluation of the behavior with changing numbers of ground-truth model-indicators.

\subsubsection{Comparison of Methods}\label{apdx:num_exp_compare}

\begin{figure}[ht]
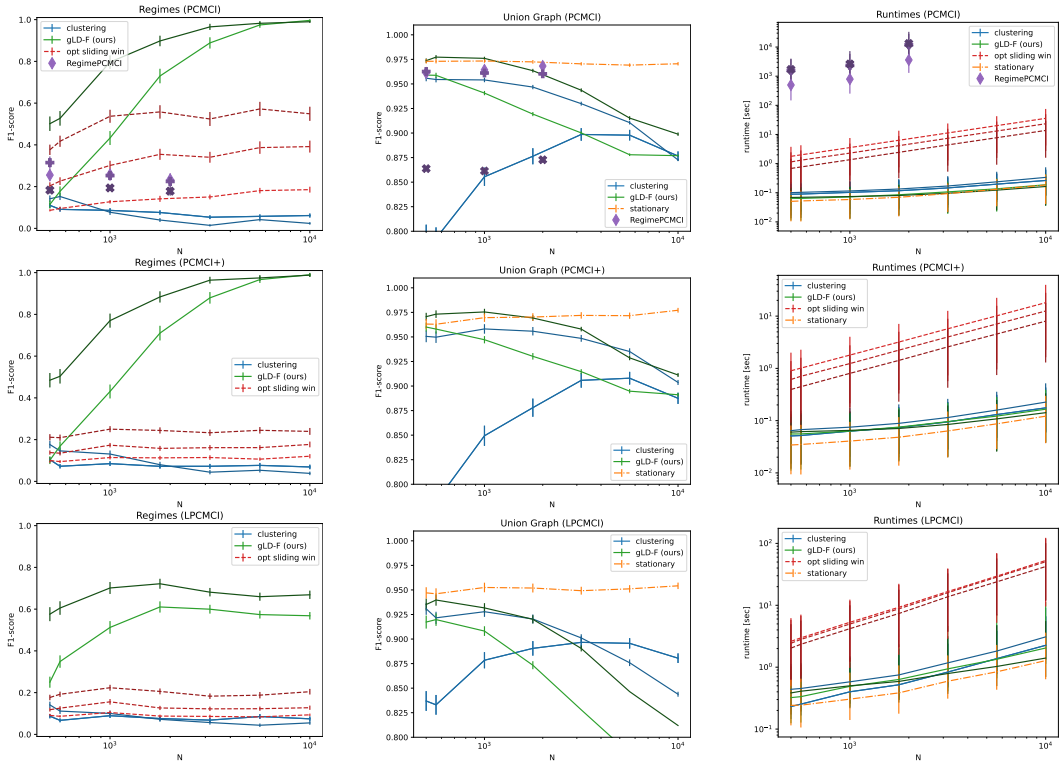

	\mayincludegraphics[width=0.3\textwidth]
	{cd_ts/sample_size/PCMCI_F1_regime}
	\hfill
	\mayincludegraphics[width=0.3\textwidth]
	{cd_ts/sample_size/PCMCI_F1_union_graph}
	\hfill
	\mayincludegraphics[width=0.3\textwidth]
	{cd_ts/sample_size/PCMCI_runtimes}
	\\
	\mayincludegraphics[width=0.3\textwidth]
	{cd_ts/sample_size/PCMCIplus_F1_regime}
	\hfill
	\mayincludegraphics[width=0.3\textwidth]
	{cd_ts/sample_size/PCMCIplus_F1_union_graph}
	\hfill
	\mayincludegraphics[width=0.3\textwidth]
	{cd_ts/sample_size/PCMCIplus_runtimes}
	\\
	\mayincludegraphics[width=0.3\textwidth]
	{cd_ts/sample_size/LPCMCI_F1_regime}
	\hfill
	\mayincludegraphics[width=0.3\textwidth]
	{cd_ts/sample_size/LPCMCI_F1_union_graph}
	\hfill
	\mayincludegraphics[width=0.3\textwidth]
	{cd_ts/sample_size/LPCMCI_runtimes}
	\\
	\caption{
		Comparison of results for PCMCI \citep{pcmci} (also shown as Fig.\ \ref{fig:dcd_ts_scaling}),
		PCMCI+ \citep{pcmci_plus} and LPCMCI \citep{lpcmci} (rows),
		for regimes, union-graph and runtime (columns)
		for different baseline-methods (colors) and different window-sizes\Slash{}hyperparameter sets
		(where applicable; shades, darker is larger).
	}
	\label{fig:cd_ts_methods}
\end{figure}
In figure \ref{fig:cd_ts_methods}, we compare results for different underlying (time series)
CD-algorithms.
The details of the MCI-implementation vary (see §\ref{apdx:MCI}).
For our method, PCMCI and PCMCI+ perform very similar, this behavior can vary for
other configurations of the PC$_1$-phase, see §\ref{apdx:MCI}.
For LPCMCI, the results for the union-graph seem to suffer stronger than for the other methods (see main text);
also the results for regime-detection seem to plateau, this is to some degree expected, as the current
MCI-logic (§\ref{apdx:MCI}) for LPCMCI may not execute enough tests as mCITs (consider too many tests \inquotes{PC$_1$-like})
to detect all model-indicators.

While the adaptation in the implementation of the MCI-philosophy for different time series methods
makes the application of our framework with different methods in this case more complicated than for the
IID-case, it is still possible with reasonable effort.
Our results show that for complex settings like for LPCMCI (time series with contemporaneous effects and
latent confounders) there is room for improvement by more specialized adaptation,
but nevertheless they also demonstrate the modularity and flexibility in choosing different underlying
CD-algorithms.

\subsubsection{Behavior with Indicator Count Revisited}\label{apdx:ts_indicator_count}

\begin{figure}[ht]
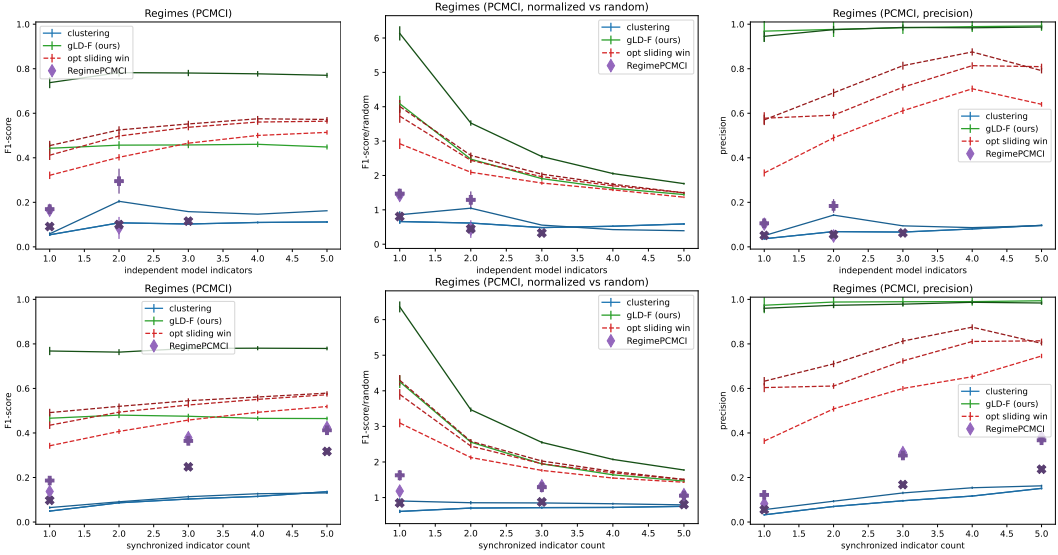

	\mayincludegraphics[width=0.32\textwidth]{cd_ts/indicator_count/PCMCI_F1_regime}
	\hfill
	\mayincludegraphics[width=0.32\textwidth]{cd_ts/indicator_count/PCMCI_F1_normalized_regime}
	\hfill
	\mayincludegraphics[width=0.32\textwidth]{cd_ts/indicator_count/PCMCI_precision_regime}
	\\
	\mayincludegraphics[width=0.32\textwidth]{cd_ts/indicator_count_sync/PCMCI_F1_regime}
	\hfill
	\mayincludegraphics[width=0.32\textwidth]{cd_ts/indicator_count_sync/PCMCI_F1_normalized_regime}
	\hfill
	\mayincludegraphics[width=0.32\textwidth]{cd_ts/indicator_count_sync/PCMCI_precision_regime}
	\\
	\caption{Results for the quality of regime-detection with increasing indicator count.
		Left-most is F$_1$-score, also shown in Fig.\ \ref{fig:dcd_ts_regime_complexity},
		center is F$_1$-score normalized by F$_1$-score of a random algorithm with the same recall,
		right-most is precision.
		In the upper row, indicators vary independently
		(see Def.\ \ref{def:causal_modularity_order}),
		in the bottom row, they are synchronized (change simultaneously).
	}
	\label{fig:cd_ts_indicator_count_apdx}
\end{figure}

In the main text (§\ref{sec:num_ex_scenarios}) we already indicated that F$_1$-score may not
be well-suited to discuss the results on behavior with different indicator counts.
To compare the effects of non-trivial indicator count in isolation, node-count, link-count and
sample-size are kept constant. But for a fixed link-count, increasing the number of non-trivial
model-indicators increases the actual (ground-truth) number of and \emph{rate of} positives.

A higher rate of actual positives means random guessing at fixed recall gets better.
So some insight can be gained, by normalizing the F$_1$-score by the F$_1$-score that
would be obtained by random guessing at the same recall. This is shown in the middle
panel of Fig.\ \ref{fig:cd_ts_indicator_count_apdx}. From this representation it 
becomes clear that compared to random performance, clustering and Regime-PCMCI do
\emph{not} get better with larger numbers of independent non-trivial indicators, as the
F$_1$-score might suggest. This score is not well-suited to compare different
methods (with different recall vs.\ precision), but (for low precision methods, see next paragraph)
it can likely capture the behavior with increasing indicator-count well.

For a guess to be accidentally correct -- and more likely so for higher rates
of actual positives -- one has to guess. As the last panel in 
Fig.\ \ref{fig:cd_ts_indicator_count_apdx} shows, our approach has very high
precision, thus it rarely \inquotes{guesses} positives, so that the effect discussed
above only has a small impact. Standard F$_1$-score may be more appropriate
in this case to study the behavior with larger non-trivial model-indicator counts.
The slight up-wards trend in the left panel of Fig.\ \ref{fig:cd_ts_indicator_count_apdx}
is probably due to the effects mentioned above, but the strong downward trend
in the middle-panel may not accurately represent the behavior.

\subsection{IID Methods}

\begin{figure}[ht]
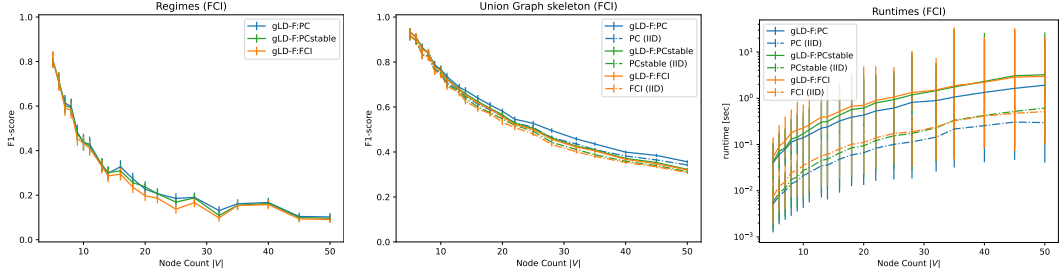

	\mayincludegraphics[width=0.32\textwidth]{cd_iid/methods/regimes}
	\hfill
	\mayincludegraphics[width=0.32\textwidth]{cd_iid/methods/union_graph}
	\hfill
	\mayincludegraphics[width=0.32\textwidth]{cd_iid/methods/runtimes}
	\\
	\caption{
		Comparison of the combination with different underlying CD-algorithms
		for the IID-case: PC \citep{PCalgo}, PC-stable \citep{PCstable}
		and FCI \citep{spirtes2001causation}.
	}\label{fig:cd_iid_methods}
\end{figure}

The difference between different underlying CD-methods (we used PC \citep{PCalgo},
PC-stable \citep{PCstable} and FCI \citep{spirtes2001causation};
using implementations in the \texttt{causal learn} \citep{causal_learn} python package)
is small, likely because for the sample-sizes where regime-detection works
adequately the union-graph discovery is mostly stable and these methods to a large
degree end up executing the same tests. We did not include systematic experiments
with latent confounders, so FCI tends to perform slightly worse than PC and PC-stable
(which leverage causal sufficiency).

The most pronounced difference seems to occur in the analysis of node-counts
(for details on this problem itself, see §\ref{apdx:cd_iid_node_count} below).
Results for different methods in this case are shown in Fig.\ \ref{fig:cd_iid_methods}.
This comparison may not be particularly interesting, but it demonstrates the
simple adaptation of our method to different underlying CD-algorithms.
The more relevant consequence of this adaptability is in the applicability
for example with latent confounding via FCI.

\subsection{Sample Size}\label{apdx:iid_sample_size}

\begin{figure}[ht]
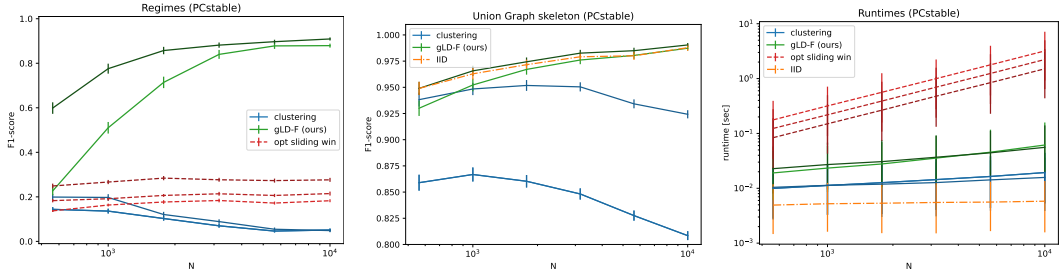

	\mayincludegraphics[width=0.32\textwidth]
	{cd_iid/sample_size/PCstable_F1_regime}				
	\hfill
	\mayincludegraphics[width=0.32\textwidth]
	{cd_iid/sample_size/PCstable_F1_union_graph}			
	\hfill
	\mayincludegraphics[width=0.32\textwidth]
	{cd_iid/sample_size/PCstable_runtimes}
	\caption{Behavior if IID-Methods with sample-size for the PC-stable algorithm.}		
	\label{fig:cd_iid_sample_count}		
\end{figure}

The behavior with increasing sample-size, see Fig.\ \ref{fig:cd_iid_sample_count}, with regard to the detection of regime-changes is similar to the time series case.
However, the union-graph-skeleton no longer shows a fall-off for larger data-sets.
This does not strictly imply that said falloff is a time series effect only -- it could still exist and negatively effect
convergence-speed, while not changing the sign -- but it is at least an indication that a time series effect
is a plausible explanation.
For our method, the next subsection (§\ref{apdx:cd_iid_node_count})
shows some additional $N$-dependence across different node-counts.

Here, and for other results in this section, our approach has higher runtime
(Fig.\ \ref{fig:cd_iid_sample_count}, right-most plot; shown are median and error-bars from
$10\%$--$90\%$ quantiles) than standard (union-graph only) causal discovery
or clustering plus CD (at least for the simple k-means clustering used here),
while it has lower runtime than sliding-window approaches (which likely suffer
from very slow iteration and additional overhead from multiple CD-executions in python).
The differences are much less pronounced than for example in comparison to more complex methods
like Regime-PCMCI (see §\ref{sec:num_experiments}).
For stability of runtimes, see also §\ref{apdx:cd_iid_node_count} below.

\subsection{Node Count / Model Complexity}\label{apdx:cd_iid_node_count}

\begin{figure}[ht]
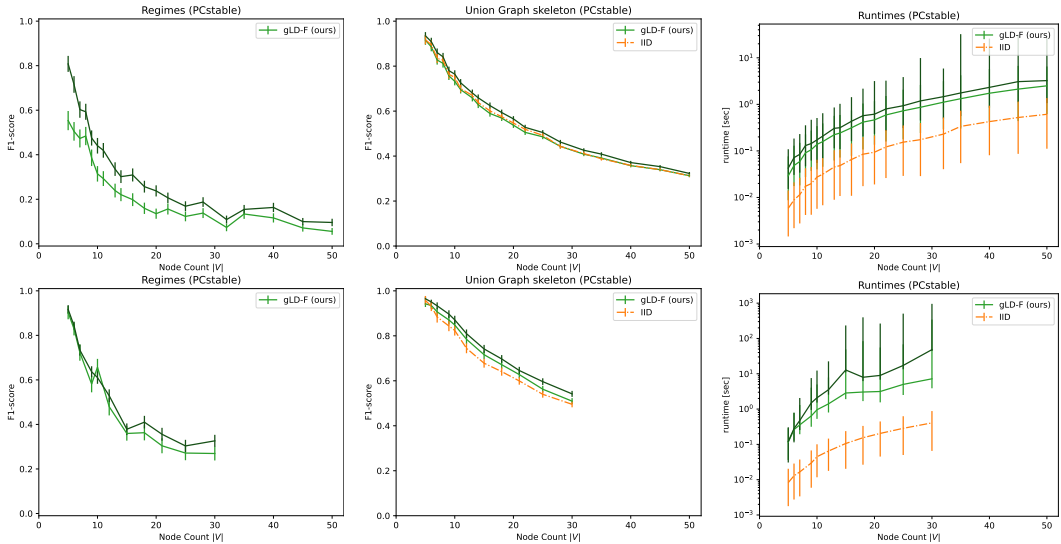

	\mayincludegraphics[width=0.32\textwidth]
	{cd_iid/node_count/PCstable_F1_regime}				
	\hfill
	\mayincludegraphics[width=0.32\textwidth]
	{cd_iid/node_count/PCstable_F1_union_graph}
	\hfill
	\mayincludegraphics[width=0.32\textwidth]
	{cd_iid/node_count/PCstable_runtimes}
	\\
	\mayincludegraphics[width=0.32\textwidth]
	{cd_iid/node_count/5k_PCstable_F1_regime}				
	\hfill
	\mayincludegraphics[width=0.32\textwidth]
	{cd_iid/node_count/5k_PCstable_F1_union_graph}
	\hfill
	\mayincludegraphics[width=0.32\textwidth]
	{cd_iid/node_count/5k_PCstable_runtimes}
	\caption{Behavior if IID-Methods with node-count for the PC-stable algorithm.
	First row is for $N=1000$, second row for $N=5000$.}			
	\label{fig:cd_iid_node_count}					
\end{figure}

Both the union-graph and especially the regime-recovery do suffer substantially from large node-counts,
see Fig.\ \ref{fig:cd_iid_node_count}.
It should be remarked that at up to 50 nodes and three independently changing indicators the setup is quite
challenging for recovery from only $N=1000$ data points.
Increasing to $N=5000$ seems to improve the situation somewhat, but leads to runtime-problems
(see below).
We target a link-density of $2$ parents per node. This density cannot be achieved for very small node-counts,
correspondingly the density ramps up to $2$ initially.
We believe that the strong falloff in the recovery of regime-changes is primarily driven by the difficulties
in detecting the union-graph correctly and the more frequent occurrence of larger conditioning sets.
With the sparser setup used for time series (Fig.\ \ref{fig:dcd_ts_complexity}) the recovery of regime-changes
could maintain a good quality better, which supports this interpretation. See also next subsection.

When increasing the sample size $N$, we find quite unstable runtimes. Note that,
in Fig.\ \ref{fig:cd_iid_node_count} (lower row, $N=5000$) the y-axis scale
of the runtime-plot (right hand side) is very different from the upper row ($N=1000$).
Further the error-bars around the median values ($10\%$--$90\%$ quantiles, see §\ref{apdx:data_gen_integrated})
become quite large; note, that the y-axis is log scale, so the $90\%$ quantile is almost
two orders of magnitude (a factor of $10^2$) above the median.
Both the large increase in median runtime and the encountered instabilities when
going to $N=5000$ may arise from lower false negative-rates:
Nodes that are separated by multiple \inquotes{hops} (subsequent edges) in the graph
are technically dependent. However, after how many hops this dependence is actually detected in 
the finite-sample case depends on $N$. False negatives (more likely for smaller $N$)
in this case may \emph{not} substantially worsen CD results (at least for the skeleton),
rather they make the graphs sparser more quickly, thus make CD much faster (and more stable).

Due to these stability issues, we included for $N=5000$ data points only up to 30 nodes.
For a comparison with other approaches (like sliding-window or clustering), see
the time series case §\ref{sec:num_experiments}.

\subsection{Link Count / Graph Density}\label{apdx:num_exp_link_density}

\begin{figure}[ht]
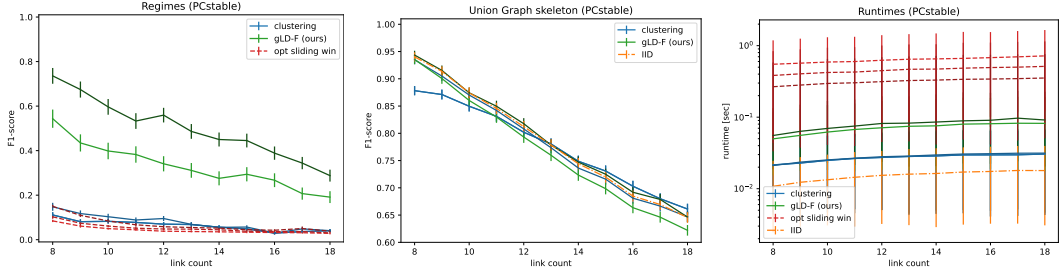

	\mayincludegraphics[width=0.32\textwidth]
	{cd_iid/link_count/PCstable_F1_regime}				
	\hfill
	\mayincludegraphics[width=0.32\textwidth]
	{cd_iid/link_count/PCstable_F1_union_graph}
	\hfill
	\mayincludegraphics[width=0.32\textwidth]
	{cd_iid/link_count/PCstable_runtimes}
	\caption{Behavior of IID-Methods with link-count for the PC-stable algorithm.}			
	\label{fig:cd_iid_link_count}
\end{figure}

The link count at a fixed node-count also provides a reasonable measure of \inquotes{model complexity}.
Results are shown in Fig.\ \ref{fig:cd_iid_link_count}.

\subsection{Regime Sizes}\label{apdx:num_exp_regime_size}

\begin{figure}[ht]
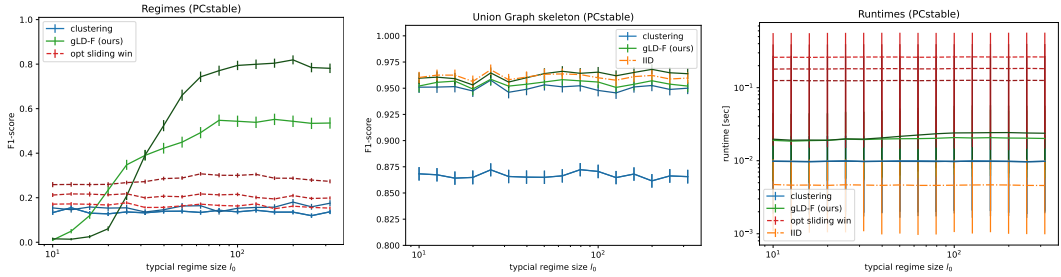

	\mayincludegraphics[width=0.32\textwidth]
	{cd_iid/regime_size/PCstable_F1_regime}				
	\hfill
	\mayincludegraphics[width=0.32\textwidth]
	{cd_iid/regime_size/PCstable_F1_union_graph}
	\hfill
	\mayincludegraphics[width=0.32\textwidth]
	{cd_iid/regime_size/PCstable_runtimes}
	\caption{Behavior if IID-Methods with ground-truth typical regime-size for the PC-stable algorithm.}			
	\label{fig:cd_iid_regime_size}
\end{figure}

Generally we expect our regime-detection to be sensitive down to a certain ground-truth typical regime-size.
A inductive bias towards larger regimes should cap out at a higher score eventually, but at the cost of falling off earlier
when going to smaller regime-sizes.
While this effect is not very strongly visible in Fig.\ \ref{fig:cd_iid_regime_size}
(compare dark green and light green lines),
it is visible much more clearly than for the time series results (Fig.\ \ref{fig:dcd_ts_regime_sizes}).
Nevertheless the hyperparameter set for larger regimes seems to work surprisingly well again. This is likely
due to the way multiple tests are combined in CD, where false negatives and false positives effectively have
unequal severity for overall score.

\subsection{Noise Types}\label{apdx:noise_types}

We use tests based on z-scores. In general z-scores approach a normal distribution rather quickly, however,
we are computing them on potentially small (a few tens of data points) blocks.
The resulting distribution will thus depend at least weakly on the true noise-distributions.
This effect seems to be typically rather small, but for more heavy-tailed distribution seem to negatively affect
the overall quality of regime detection. Among the examples analyzed below, this is particularly
visible for Cauchy-distributions.

We analyzed different noise-distributions
across different parameters shared across all nodes,
and also included a mixed setup, where each node is assigned one of the above distributions
at random (with equal probability).
Our analysis is structured to study more specifically:
\begin{itemize}
	\item Distributions that are relevant for modeling errors under different \emph{practical considerations}.
		Following the logic of \citep[4.4 (p.\,87--91)]{VapnikEstimation}
		normal, Laplace and uniform distributions provide reasonable corner-cases to study.
	\item The effect of heavy-tails and more generally challenging moment-(convergence-)properties. We include
		a Cauchy-distribution as practically relevant example.
	\item The effect of noise-scale, which can be parameterized well for the examples given above.
	\item The effect of asymmetry. As a concrete example we use $\beta$-distributions with
		suitable parameters to study this question.
	\item The effect of multi-modality, for simplicity studied on a multi-modal normal, with different
		distances between modes.
\end{itemize}

\subsubsection{Normal Distributions}\label{sec:noise_normal}

\begin{figure}[ht]
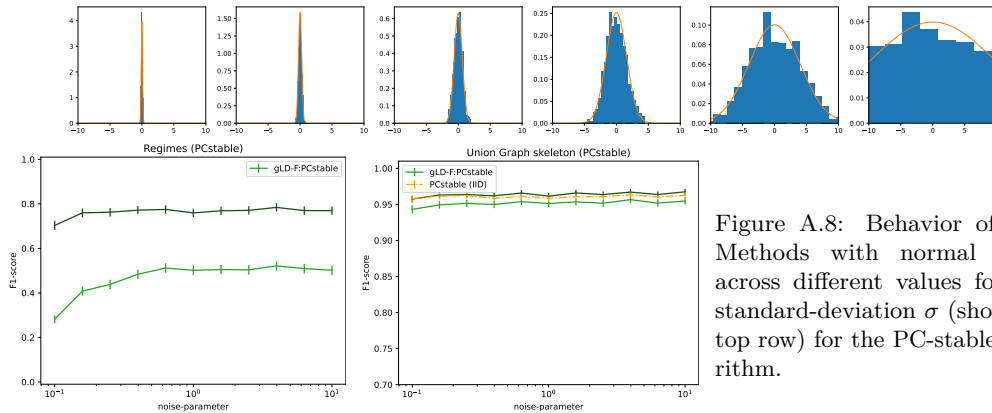
	
	\begin{minipage}{\textwidth}
		\centering
		\mayincludegraphics[width=0.9\textwidth]{figures/noise_types/normal}
	\end{minipage}\\
	\hfill
	\mayincludegraphics[width=0.32\textwidth]
	{cd_iid/noise_types/normal_regimes}				
	\hfill
	\mayincludegraphics[width=0.32\textwidth]
	{cd_iid/noise_types/normal_union_graph}					
	\hfill	
	\begin{minipage}[b]{0.32\textwidth}
		\caption{Behavior of IID-Methods with normal noise across different values for the standard-deviation $\sigma$
				(shown in top row) for the PC-stable algorithm.}
		\label{fig:cd_iid_noise_normal}
	\end{minipage}
	\hfill		
\end{figure}

Normal distributions are often employed as a plausible model for noise terms. There are many arguments justifying
this choice. For example, as \citet[4.4 (p.\,87--91)]{VapnikEstimation} points out, the normal distribution, having
maximal entropy given a fixed variance, models a minimum knowledge (pure noise) distribution under fixed experimental
conditions ('fixed' here means the variance itself remains fixed across repetitions of a measurement, see also the
next point §\ref{sec:noise_laplace} about Laplace distributions).

The width of the noise is picked uniformly for all nodes in the graph. The resulting signal
seen by the independence-test depends on the variation of the source (which is composed of the source's noise and
its ancestors' noises) and the (model-)effect-strength (the multiplier).
Thus signal-to-noise should not be affected by the (shared among nodes) width of the noise-distributions,
as this width contributes equally to signal and noise (within the limits of numerical floating-point precision,
but we should be far from these limits for the experiments presented here).
This does not make for terribly interesting plots, but the primary purpose of this subsection on
the behavior with varying noise-types is the \emph{validation} of approaches and implementations;
for this purpose the plots presented to provide useful information.

For normal distributions, results in Fig.\ \ref{fig:cd_iid_noise_normal}, this behavior is mostly generated as expected.
The only non-trivial feature is a slight fall-off of the lower-regime-length hyperparameter set results for regime
detection.

\subsubsection{Laplace}\label{sec:noise_laplace}

\begin{figure}[ht]
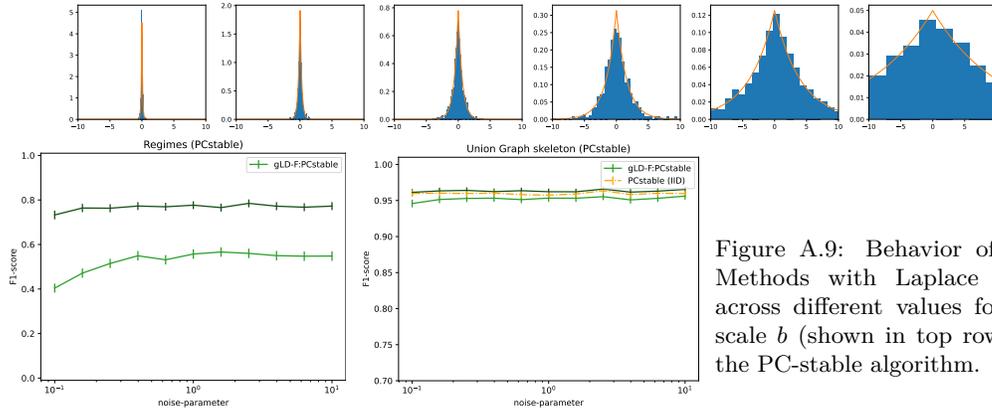
	
	\begin{minipage}{\textwidth}
		\centering
		\mayincludegraphics[width=0.9\textwidth]{figures/noise_types/laplace}
	\end{minipage}\\
	\hfill
	\mayincludegraphics[width=0.32\textwidth]
	{cd_iid/noise_types/laplace_regimes}				
	\hfill
	\mayincludegraphics[width=0.32\textwidth]
	{cd_iid/noise_types/laplace_union_graph}					
	\hfill	
	\begin{minipage}[b]{0.32\textwidth}
		\caption{Behavior of IID-Methods with Laplace noise across different values for the scale $b$
			(shown in top row) for the PC-stable algorithm.}
		\label{fig:cd_iid_noise_laplace}
	\end{minipage}
	\hfill		
\end{figure}

Laplace distributions are often employed as a plausible model for noise terms if experimental conditions are
believed to change between repeated measurements \citep[4.4 (p.\,87--91)]{VapnikEstimation}.
This is in contrast to normal distributions being well-suited to model noise under fixed conditions;
considering that we are interested in non-stationary\Slash{}non-IID data, this seems like a relevant
distribution to study. The above notion is made more precise for example by \citet[4.4 (p.\,87--91)]{VapnikEstimation},
based on an argument of maximizing the entropy of the distribution of $\sigma$, and then further
of the mixture with varying (according to this distribution of $\sigma$) variance.

Both our expectations and our observations do not differ substantially from the case of normal noise §\ref{sec:noise_normal},
see Fig.\ \ref{fig:cd_iid_noise_laplace}.
An interesting feature of Laplace-distributions is the \inquotes{singularity} (discontinuity of its derivative) of the density
at zero (or rather, at its location-parameter, here $0$).
However, we do not use density-estimation or smoothing-kernels, so this does not seem to considerably affect our method.

\subsubsection{Uniform}

\begin{figure}[ht]
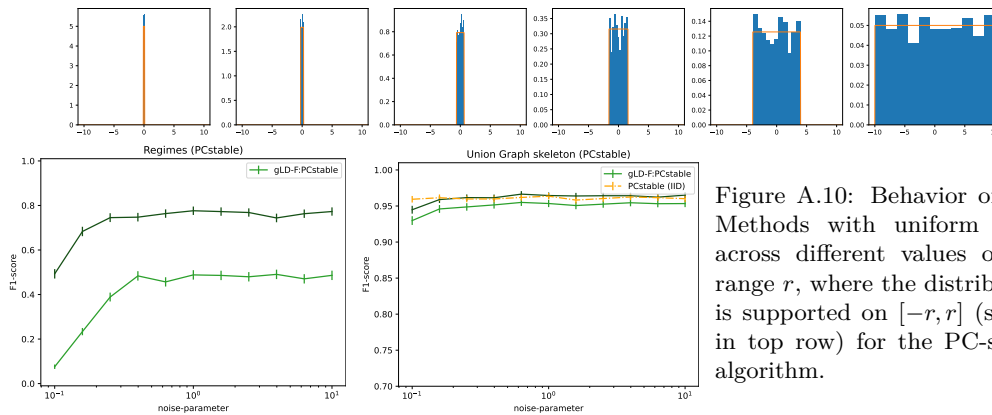
	
	\begin{minipage}{\textwidth}
		\centering
		\mayincludegraphics[width=0.9\textwidth]{figures/noise_types/uniform}
	\end{minipage}\\
	\hfill
	\mayincludegraphics[width=0.32\textwidth]
	{cd_iid/noise_types/uniform_regimes}				
	\hfill
	\mayincludegraphics[width=0.32\textwidth]
	{cd_iid/noise_types/uniform_union_graph}					
	\hfill	
	\begin{minipage}[b]{0.32\textwidth}
		\caption{Behavior of IID-Methods with uniform noise across different values of the range $r$,
			where the distribution is supported on $[-r,r]$
			(shown in top row) for the PC-stable algorithm.}
		\label{fig:cd_iid_noise_uniform}
	\end{minipage}
	\hfill		
\end{figure}

\citet[4.4 (p.\,87--91)]{VapnikEstimation} gives a simple and convincing argument for the relevance of uniform distributions:
With limited resolution or encoding precision, real-world measurements are often quantized
(in the literal sense, not necessarily in the sense of quantum theory). To make them easier to process,
it is often reasonable to de-quantize them. Given the absence of any knowledge on the distribution within such quanta,
the logical choice for de-quantization is a uniform distribution.

Generally, the observations we make are similar to those in the normal §\ref{sec:noise_normal} and
Laplace §\ref{sec:noise_laplace} case.
The fall-off of the regime-detection quality for small noise-values is stronger however,
see discussion \ref{sec:noise_discussion}.

\subsubsection{Cauchy}\label{sec:noise_cauchy}

\begin{figure}[ht]
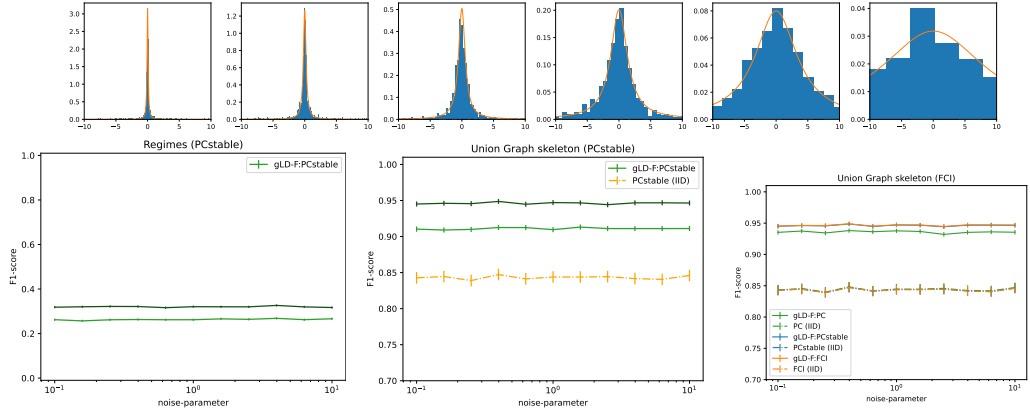
	
	\begin{minipage}{\textwidth}
		\centering
		\mayincludegraphics[width=0.9\textwidth]{figures/noise_types/Cauchy}
	\end{minipage}\\
	\hfill
	\mayincludegraphics[width=0.32\textwidth]
	{cd_iid/noise_types/cauchy_regimes}				
	\hfill
	\mayincludegraphics[width=0.32\textwidth]
	{cd_iid/noise_types/cauchy_union_graph}					
	\hfill	
	\mayincludegraphics[width=0.32\textwidth]
	{cd_iid/noise_types/cauchy_union_graph_compare_methods}	
	\\
	\caption{Behavior of IID-Methods with Cauchy noise across different values for the scale $\gamma$
		(shown in top row) for the PC-stable algorithm. Third plot shows difference between CD-methods.}
	\label{fig:cd_iid_noise_cauchy}
\end{figure}

Cauchy-distributions are an often employed example of distributions with challenging formal properties,
like the non-existence of (any) moments. This makes it \emph{interesting} to study,
but it is also \emph{relevant} as it naturally appears in statistical physics,
for example in (Doppler-free) spectroscopy applications.

While innocent-looking at first Fig.\ \ref{fig:cd_iid_noise_cauchy}, this density seems to have
the by far most dramatic impact on our method, and also on union-graph discovery.
We believe that this is due to outliers from the tails shaping the derived statistics in
a way that prevents fast convergence to normality, and thereby jeopardizes our null-distribution estimates.
We do not fully understand why our method seems to outperform the standard IID-method on union-graphs,
but the observed signal is rather strong.
Even more confusing, while for the other noise-experiments PC, PC-stable and FCI perform extremely similar
(on skeletons) for the case of Cauchy-distributions, the vanilla PC algorithm seems to fall off
compared to PC-stable and FCI in terms of recovery of the union-graph through our method.

A possible explanation for the difference on the union-graphs might be that our method, being
built to be (at least to some degree) robust against non-stationarities (see §\ref{apdx:conditional_tests}),
may also be able to compensate for rare (thus local in \inquotes{time}, confined to a single block) outliers
better, or to find dependence observed on rare outliers indirectly through the homogeneity-test.

There has been a substantial recent interest in causal discovery on heavy-tailed data
\citep{gnecco2020causaldiscoveryheavytailedmodels},
and our framework is in principle flexible enough to supplement independence-tests with,
for example,
tail-coefficients, so this might be a topic of interest for future work.

\subsubsection{Beta}

\begin{figure}[ht]
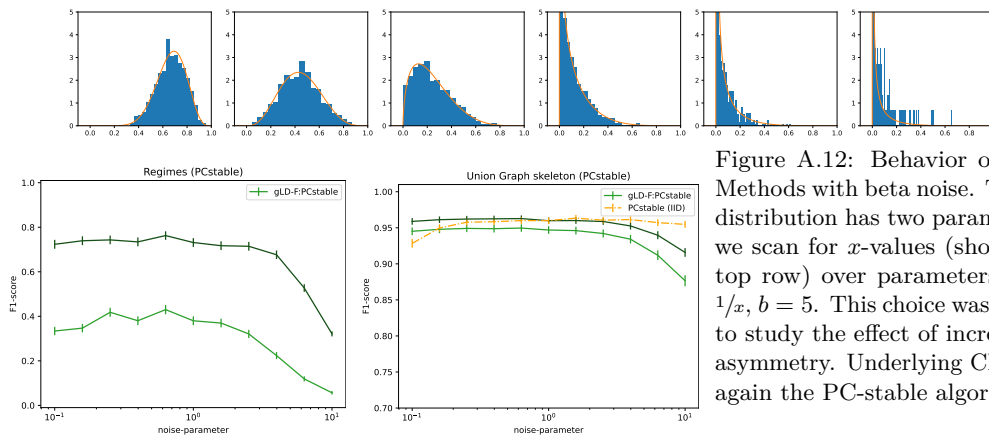
	
	\begin{minipage}{\textwidth}
		\centering
		\mayincludegraphics[width=0.9\textwidth]{figures/noise_types/beta}
	\end{minipage}\\
	\hfill
	\mayincludegraphics[width=0.32\textwidth]
	{cd_iid/noise_types/beta_regimes}				
	\hfill
	\mayincludegraphics[width=0.32\textwidth]
	{cd_iid/noise_types/beta_union_graph}					
	\hfill	
	\begin{minipage}[b]{0.32\textwidth}
		\caption{Behavior of IID-Methods with beta noise. The $\beta$-distribution has
			two parameters, we scan for $x$-values (shown in top row) 
			over parameters $a=\sfrac{1}{x}$, $b=5$. This choice was made to study the effect of
			increasing asymmetry. Underlying CD was again the PC-stable algorithm.}
		\label{fig:cd_iid_noise_beta}
	\end{minipage}
	\hfill		
\end{figure}

Beta-distributions are of great relevance for the modeling of random variables of compact support.
A systematic account for the behavior under arbitrary configurations of the (two parameter) family
would be out of scope for this appendix, however, we can employ a specific choice of
parameters to study the effect of asymmetry in the noise-density, see Fig.\ \ref{fig:cd_iid_noise_beta}.

We observe some falloff with high asymmetries, but in this case the distribution also becomes very narrow,
so that this observation may be due to the same effects observed for example for normal distributions §\ref{sec:noise_normal}.
In the parameter-regime in the middle of the plot, where the distribution is not yet narrow, but already
very asymmetric, there is nothing new observed, so that asymmetry in itself is likely not a substantial problem
to our method.

\subsubsection{Multi-Modal Normal}

\begin{figure}[ht]
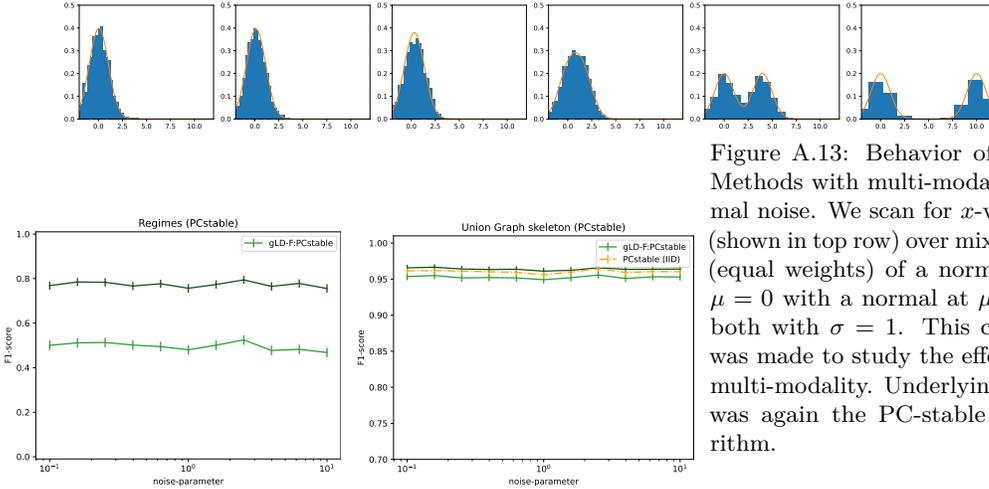
	
	\begin{minipage}{\textwidth}
		\centering
		\mayincludegraphics[width=0.9\textwidth]{figures/noise_types/multimodal_normal}
	\end{minipage}\\
	\hfill
	\mayincludegraphics[width=0.32\textwidth]
	{cd_iid/noise_types/multimodal_normal_regimes}				
	\hfill
	\mayincludegraphics[width=0.32\textwidth]
	{cd_iid/noise_types/multimodal_normal_union_graph}					
	\hfill	
	\begin{minipage}[b]{0.32\textwidth}
		\caption{Behavior of IID-Methods with multi-modal normal noise.
			We scan for $x$-values (shown in top row) 
			over mixtures (equal weights) of a normal at $\mu=0$
			with a normal at $\mu=x$, both with $\sigma = 1$.
			This choice was made to study the effect of
			multi-modality. Underlying CD was again the PC-stable algorithm.}
		\label{fig:cd_iid_noise_multimodal_normal}
	\end{minipage}
	\hfill		
\end{figure}

Much in our method revolves around the question of the presence of multiple peaks in
dependence-distributions (see Fig.\ \ref{fig:univariate_problem}), so one should naturally
be somewhat worried about the effect of multi-modal noise-distributions.

In practice, we do not observe any problems from the presence of multi-modal noises.

\subsubsection{Combinations}

\begin{figure}[ht]
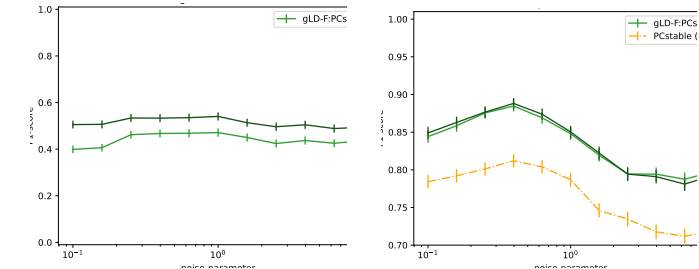
	
	\hfill
	\mayincludegraphics[width=0.3\textwidth,trim={0.2cm 0.2cm 1.25cm, 0.4cm},clip]
	{cd_iid/noise_types/mixed_regimes}				
	\hfill
	\mayincludegraphics[width=0.3\textwidth,trim={0.2cm 0.2cm 1.25cm, 0.4cm},clip]
	{cd_iid/noise_types/mixed_union_graph}			
	\hfill	
	\begin{minipage}[b]{0.3\textwidth}
		\caption{Behavior of IID-Methods with combinations of different noises.
			Here each node is randomly assigned one of the noise-types discussed above.
			This choice was made to study the effect of
			combining different noise-types within the same model. Underlying CD was again the PC-stable algorithm.}
		\label{fig:cd_iid_noise_mixed}
	\end{minipage}
	\hfill		
\end{figure}

Finally, one should wonder what happens, if different nodes follow different distributions.
As the results in Fig.\ \ref{fig:cd_iid_noise_multimodal_normal} show, the
detection-quality, in particular for the hyperparameter set for large regimes does suffer,
as does the union-graph.
Note, that while we keep the noise-parameter the same for all nodes in the system,
this parameter does not have a meaningful interpretation as a scale for all of the noise-types,
so in practice this setup is also expected to be more challenging as it combines different
\emph{noise-scales}, not just noise-types.

The union-graph results for the two hyperparameter sets for our method appear curiously close together.
While we of course cannot exclude a bug, since all noise-types (including the combinations one)
where generated, run and evaluated with precisely the same code, this seems, however, unlikely.
Again, our approach seems to provide better union-graphs, than the vanilla algorithm;
the only plausible explanation we have at the moment is as for the Cauchy case (see §\ref{sec:noise_cauchy}):
robustness against non-stationarity might help to cope with outliers.

\subsubsection{Discussion}\label{sec:noise_discussion}

We find two main challenges in our quick study of different noise-types.
The first arises for Cauchy distributions, where the quality of regime-detection
decreases substantially compared to for example normal-distributions.
Cauchy distributions are formally challenging. We do not currently know what precisely
causes the difficulties encountered, but we believe it could plausibly be an issue with
outliers. These outliers will slow down the approach to normality of z-scores,
especially on small datasets (or small blocks). However, the hyperparameter set for
larger regimes (with larger blocks) seems to suffer more than the smaller regimes
one, which seems to contradict this interpretation. Further also the standard
(on all data\Slash{}without blocks) union-graph discovery seems to suffer particularly strongly,
which poses a similar difficulty to this explanation.
Generally correlation-based test-approaches provide little control over the weighting of
outliers compared to other samples. It might be interesting to study the effect of up- or
down weighing outliers. Indeed causal discovery in the presence of heavy-tailed distributions
is an active area of research, and
also combinations of our framework with for example
tail-coefficient techniques \citep{gnecco2020causaldiscoveryheavytailedmodels}
may be interesting.

The second challenge arises for very small value-scales,
for example very small $\sigma$-values for
normal distributions. The effect seems to be present for all distributions
where we scanned across a parameter (barring Cauchy-distributions, where other effects are more
formative, see above) which can be interpreted as a proxy for scale: normal, Laplace
and uniform distributions.
We do not believe this is a problem of floating-point precision, but we cannot exclude
this at the moment either. 
There may in practice be some virtue in normalizing data.

\subsection{Preliminary Experiments}\label{apdx:prelim_num_results}

Before studying the hyperparameter dependence of the individual sub-modules
of the mCIT implementation discussed in §\ref{apdx:data_processing_layer},
we ran some initial benchmarks to assess the viability of our approach
with preliminary hyperparameters (educated guess based on expected typical regime-lengths
in the benchmark data-set).
We include these results in tables \ref{tab:prelim_small} and \ref{tab:prelim_large},
for completeness and transparency, but also because of an interesting observation:
Especially when used with PCMCI+, even with extremely crude hyperparameters
(fixed across all $Z$, and $N$ [thus both experiments], picked without any kind of
sophisticated procedure), our approach seems to be able to maintain good
precision, even though recall does suffer, especially for larger\Slash{}more complex models.

With our method being competitive in run-time to union-graph (standard) CD,
and the main effect of miss-specification of hyperparameters apparently being found
in low recall, our approach can likely be useful already as a \inquotes{low noise} (rare
false positives) extension
beyond union-graphs.

The runtimes listed in the tables are not immediately comparable -- Regime-PCMCI instances were run
on a cluster on 16 CPU-cores per instance (summed into the CPU-seconds given, \ie CPUs are total core-seconds),
while the preliminary runs with our method were done on a laptop
(with substantially different hardware).
Results in the main text show more comparable times (all on the same cluster)
and show a very similar picture.
For some hyperparameter choices for Regime-PCMCI, the rate of convergence (where
at least one of the re-runs produced a result) is low, it is unclear if chance of convergence
and error-probability are independent, thus the performance on converged runs only may
not accurately represent the true error-rates in these cases (assuming that convergence
is more likely on \inquotes{easy} examples, the shown numbers may be favoring configurations with
\emph{low} convergence-rate, which might explain the lower F$_1$ score of the hyperparameters
recommended by \citep{Saggioro2020} for complex systems on the complex system).

\begin{table}[ht]
	\small
	\begin{tabular}{c|l||c|c|c||c|c}
		Method & hyperparams & $F_1$-score & precision & recall
		& convergence & av.\ runtime\\
		\hline
		\multicolumn{7}{c}{Regime Dependent Edges (trivial-true $F_1$ would be $\sim 0.10$):}\\
		\hline
		Regime-PCMCI & tigramite &
		$0.48$ & $0.35$ & $0.77$ & $57.1\%$ & $569.2$\,CPUs\\
		Regime-PCMCI & \citep{Saggioro2020}, simple &
		$0.58$ & $0.46$ & $0.78$ & $80.7\%$ & $1682.7$\,CPUs\\
		Regime-PCMCI & \citep{Saggioro2020}, complex &
		$0.33$ & $0.22$ & $0.72$ & $94.0\%$ & $1877.2$\,CPUs\\
		\hline
		gLD-F PCMCI & preliminary &
		$0.78$ & $0.83$ & $0.73$ & 100\% & $0.044$\,CPUs\\
		gLD-F PCMCI+ & preliminary &
		$0.80$ & $0.92$ & $0.71$ & 100\% & $0.040$\,CPUs\\
		\hline
	\end{tabular}
	\caption{Preliminary results on a very simple time series setup:
		$5$ nodes, $5$ links plus
		$50\%$ of nodes with auto-correlation,
		acyclic summary-graph,
		max (and thus only) lag $1$,
		one changing link in ground-truth.
		(233 runs)}
		\label{tab:prelim_small}
\end{table}

\begin{table}[ht]
	\small
	\begin{tabular}{c|l||c|c|c||c|c}
		Method & hyperparams & $F_1$-score & precision & recall
		& convergence & av.\ runtime\\
		\hline
		\multicolumn{7}{c}{Regime Dependent Edges
			(trivial-true $F_1$ would be $\sim 0.07$):}\\
		\hline
		Regime-PCMCI & tigramite &
		$0.14$ & $0.08$ & $0.98$ & $21.6\%$ & $2\,485$\,CPUs\\
		Regime-PCMCI & \citep{Saggioro2020}, simple &
		$0.14$ & $0.07$ & $0.94$ & $27.6\%$ & $6\,111$\,CPUs\\
		Regime-PCMCI & \citep{Saggioro2020}, complex &
		$0.05$ & $0.03$ & $0.80$ & $99.5\%$ & $11\,267$\,CPUs\\
		\hline
		gLD-F PCMCI & preliminary &
		$0.52$ & $0.54$ & $0.51$ & 100\% & $0.979$\,CPUs\\
		gLD-F PCMCI+ & preliminary &
		$0.64$ & $0.90$ & $0.50$ & 100\% & $0.686$\,CPUs\\
		\hline
	\end{tabular}
	\caption{Preliminary results on a more challenging time series setup:
		$10$ nodes, $9$ links plus
		$50\%$ of nodes with auto-correlation,
		acyclic summary-graph,
		max lag $4$,
		two changing links in ground-truth.
		(199 runs)}
		\label{tab:prelim_large}
\end{table}

\subsection{Hyperparameter Policy}\label{apdx:hyper_param_policy}

After initial validation (see §\ref{apdx:prelim_num_results}) with hand-picked
fixed hyperparameters, we studied the different stages of the mCIT
on IID data (see §\ref{apdx:homogeneity_hyper_heuristic}, §\ref{apdx:weak_hyper_choices}),
with the goal of picking hyperparameters in a robust way:
An evaluation of numerical results on a mostly uninformative prior
is combined with theoretical considerations to pick hyperparameters in 
a \inquotes{minimax-like} spirit to provide good performance across
a large range of plausibly encountered model-priors.
These choices are made for IID-data, but conceptually they should
be useful for time series data as we adopt the MCI \citep{pcmci} idea
(see §\ref{apdx:MCI}); they may, however, not apply for the PC$_1$-phase
(also discussed in §\ref{apdx:MCI}), see also the discussion
of the quality of union-graph discovery in §\ref{sec:num_experiments}.

These hyperparameter sets were than frozen for all numerical experiments
on causal discovery in §\ref{sec:num_experiments} and §\ref{apdx:num_experiments}.
The only exception to this rule is that we fixed a bug that prevented
data-blocks from being centered correctly for computations of
correlation-scores and re-ran all CD-experiments with the fixed CIT-implementation.
This code-fix did not lead to perceptible changes in any of the results outside
of the study of different noise-types (§\ref{apdx:noise_types}).

\section{Further Topics}

There are a few closely related topics, whose inclusion can shed light on certain aspects of
our framework. We include a few short discussions on such topics here.

\subsection{Examples for Patterns}\label{apdx:patterns}

Our definition of patterns (Def.\ \ref{def:blocks}) is quite general:
It only requires prior knowledge of \inquotes{neighborhoods} (blocks) of a given size of points,
which define a tiling of the index set $I$ (assuming data is given in
the form $(X_i, Y_i, Z_i, \ldots)_{i\in I}$).
Indeed this tiling need not even be without gaps, as long as enough blocks are available
for a statistical analysis.

Patterns of this form include temporal regimes, where blocks are intervals
(example \ref{example:persistent_in_time_pattern}), spatially persistent regimes,
where blocks are for example rectangles or hexagons tiling the plane.
Generally any exogenous quantity (for the endogenous case, see §\ref{apdx:endogenous}),
like height above sea-surface for spatial locations, or terrain slope for river-catchments
can for example be rank-transformed and then employ example \ref{example:persistent_in_time_pattern}
(in such cases -- other than for example for temporal regimes -- assuming constant
information in the pattern often \emph{does} make sense, conceptually this makes
mCIT testing easier).
But also sufficiently regular networks (undirected graphs) -- for example computer-networks, electricity-grids
etc.\ --
where sufficiently regular may mean for example \inquotes{similar number of neighbors for all nodes},
can plausibly be approximately tiled.

These patterns can also be combined, in at least two meaningful ways:
\begin{enumerate}[label=(\roman*)]
	\item
		Assume each indicator follows at least one of the two patterns,
		for example for all times the same spatial pattern-realization, or everywhere (in space)
		the same temporal pattern-realization.
		In this case, homogeneity and weak-regime tests can be considered independently for both patterns.
		State-space construction §\ref{sec:indicator_translation}
		could then additionally exploit this product-structure on the state-space, with assignment
		to factors in this case being implicit in mCIT results.
	\item 
		Assume each indicator follows both patterns locally, for example
		varies slowly in space for all times and varies slowly in time everywhere in space.
		In this case blocks can be formed as suitable products (with potentially different weights
		for different patterns).
		The two-dimensional spatial patterns described above are a product of
		one-dimensional patterns as in example \ref{example:persistent_in_time_pattern}
		in this sense.
\end{enumerate}
Further, patterns can be combined with exact knowledge, \ie observed contexts $C$:
For example, for climate dynamics, a difference between sea-surface vs.\ land-surface may
be plausible. In such cases, it is in principle possible to find an (approximate) tiling
of either context (sea and land) individually to assess (marked) conditional independencies of the form
$X \independent Y | C$ better.
As another such example, assume we are given time series data for multiple (but not many)
spatial locations, which are believed to feature similar temporal regime structure-realizations
\citep{JPCMCI}.
We then may use intervals (as in example \ref{example:persistent_in_time_pattern}), but make them
shorter (improving our resolution) by considering all spatial locations for each interval
as part of the block.

\subsection{Endogenous Indicators}\label{apdx:endogenous}

So far our indicators are functions of sample-indices, which in practice
will typically mean they are functions of time or space or similar extraneous constructs.
This makes them exogenous to the system-variables.
Many real-world problems are changing in response to endogenous -- system-driven -- variables,
this includes both natural and technical (state-machine like) systems. Such
endogenous-indicator systems have been studied for causal discovery
\citep{EndoMethod,EndoTheory}, but only for the case where indicators are
known or well-approximated by known deterministic functions of observed variables.
The handling of selection-bias (from masking to a context)
proposed by \citep{EndoMethod} is implicitly achieved by similar logic in our framework.

So, how can our framework be extended to endogenous indicators?
The first clue is that a \inquotes{pattern} as we detect it on space or time etc.,
can in principle just as well be detected \emph{in the values of a system-variable}.
\begin{example}[Endogenous Pattern]
	Consider the particular simple case of a system with only three variables,
	$X$, $Y$, $C$, and causal graph $X \rightarrow Y \leftarrow C$ and
	mechanism at $Y$ of the form $Y= \alpha\mathbbm{1}(C \geq t) X + \eta_Y$,
	\ie $Y$ depends on $X$, but only if $C \geq t$ for an unknown threshold $t$.
	Since $C$ is not a descendant of either $X$ or $Y$, it behaves essentially as
	if it where exogenous to the problem, therefore we could try to find a pattern
	in $C$, by building blocks by similarity of values in $C$.
	Finding $t$ is a particularly simple (much simpler than the rather general case
	discussed in §\ref{sec:dyn_indep_tests}) case of pattern-detection.
	The pattern in this case is known a priori to be described by a single parameter.
	Our framework could exploit this knowledge through a suitable mCIT implementation.
	Even with only the knowledge that the system is likely to behave
	similar for similar values of $C$, the material in §\ref{sec:dyn_indep_tests} --
	and therefore our framework -- is applicable.
\end{example}
\begin{rmk}
	In practice, such effects often are additionally persistent in time or space.
	If the endogenous variable $C$ in the example above
	is auto-correlated, it may cross the threshold only rarely, effectively
	leading to persistent regimes.
	Our current implementation \emph{can} thus indirectly find these
	endogenous effects, but it cannot yet attribute them to $C$ (even though
	resolved detected indicators could of course be compared to $C$, see
	§\ref{apdx:JCI}) and further including prior knowledge about the dependence on $C$
	could substantially improve finite-sample properties.
\end{rmk}

Generally, there are some challenges involved in using endogenous variables to
find patterns.
For example selection bias could create the illusion of regime-specific links
between ancestors of $C$, if $C$ is the descendant of a collider on a path
that is open\Slash{}closed depending on the state of other model-indicators.
More surprisingly selection-bias, through shifts in observational support,
can effectively even make links between ancestors vanish \citep{EndoMethod},
see also §\ref{apdx:observational_support}.

A potential remedy could be the inspection of the union-graph and subsequent
restriction to testing between definitive non-ancestors of $C$ or an
adaptation of the testing-scheme of \citep{EndoMethod}.

Further, an endogenous indicator can, at least in principle, act as a mediator.
In particular, for the example above, also the link $C \rightarrow Y$ may
appear to be changing, but the interpretation is different: Here a simplified
latent (the binary indicator) suffices to explain the relation-ship (the dependence
factors through $c \mapsto (c\geq t)$). This is similar to the problem studied in
\citep{JPCMCI}, but not a regime-dependence in the sense one might expect.
In practice, adherence to a pattern (with large blocks) together with the binary value-space
of indicators, means the information-flow through them is very limited. So
in the finite sample case, it may not be possible to effectively detect these effects anyway;
even though the interpretation of change-points in indicators as events and application
of tests designed for point-processes might allow to sufficiently boost power.

\subsection{Interpretation under Limited Observational Support}
\label{apdx:observational_support}

As analyzed by \citep{EndoMethod}, different contexts can render certain
causal mechanisms (thus their corresponding links in the graph) effectively invisible
by concentrating the observational support in regions of low or vanishing
derivative of the mechanism, rather than by physically deactivating them.
These changes are dubbed \emph{descriptive} by \citep{EndoMethod} as opposed
to \emph{physical} changes for which multivariate mechanisms are constant
in one (or more) parents for specific values of a context-parent.

To tackle the problem in practice, \citep{EndoMethod} introduces assumptions
that ensure contexts are chosen sufficiently expressive (dubbed context-sufficiency,
in analogy to causal sufficiency), so that all changes
are physical when expressed relative to this context.
Such \inquotes{sufficient} contexts always exist:
If a mechanism is constant in one parent (say $X$) for certain values of $X$,
then formally one can always define a binary context variable that takes on one value if
$X$ is in the constant region and the other value for other values of $X$.
This is essentially the scheme proposed in §\ref{apdx:endogenous} for the
treatment of endogenous variables.

Indeed for endogenous variables our approach can thus be interpreted as finding
this (deterministic) endogenous indicator that suffices to describe the change
as physical. 
While in the endogenous case the change of context is effectively a detection
of structure in the functional form of a mechanism, our method for exogenous
patterns (\eg in time) as described in the main text of this paper finds
non-IID (or non-stationary) changes in mechanisms. Here the multivariate mechanism
whose form varies qualitatively is of the form $f_Y(X, t)$ with dependence on $t$
being subject to persistence (or any other pattern) and leading to qualitative
changes in the dependence on $X$ in a \inquotes{physical} sense.

\subsection{Relation to Clustering and CPD}\label{apdx:global_cpd_clustering}

The problem of HCCD studied in this paper is
closely related with a clustering problem.
In principle one could solve the multivariate clustering-problem
to find all ($\approx 2^k$) states of the system, then perform causal discovery
on each cluster individually.
Our approach can be understood as an automated
translation of this multivariate clustering problem
with a large number of clusters into a sequence of
simple univariate clustering tasks, by exploiting non-trivial structure of the problem:
\begin{enumerate}[label=(\arabic*)]
	\item \emph{Exploit causal structure} through a CD algorithm and §\ref{sec:indicator_translation} to obtain a problem-statement in
	terms of marked independencies (and implication tests).
	\item \emph{Exploit persistence} by applying a dependence score to blocks of data,
	to obtain a well-behaved univariate data-set.
	\item \emph{Exploit the binary nature} of the result of a hypothesis test
	(independence) and
	\item \emph{exploit the known value} $d=0$ in the independent regime
	to simplify a problem with many clusters
	into: Is more than one cluster, and is
	one cluster centered at zero?
\end{enumerate}
Leveraging change-point-detection (CPD) or clustering methodology
while also retaining the ability to exploit these special properties can be difficult in practice.
Our framework enables this combination:
A plausible approach
would be  and an application of the dependence-score
on CPD-assigned segments rather than blocks.
See §\ref{apdx:leveraging_cpd_clustering} for more information about how to
leverage CPD and clustering \emph{within} our framework
and §\ref{apdx:persistence_by_blocks} which summarizes why we currently use the simple
block-based approach.
In our framework CPD -- when used as a drop-in for the block-based approach --
can be applied locally
for example to gradients of scores or changes in the target (often, albeit not always,
a changing link will change the variance and entropy of its target)
or on dependence-scores on (smaller) blocks of data, avoiding the much more difficult
global\Slash{}multivariate problem.

The assumption of causal modularity (Def.\ \ref{def:causal_modularity_order})
we make, also leads to a signal-to-noise ratio problem for global algorithms.
One may intuitively think that the removal of a link leaves traces in the
values of many variables (which it does), but these changes\footnote{Assuming there
	is no substantial measurement noise.} are strongly correlated and including more
variables does not actually improve signal. It does of course increase noise however.
So, given causal modularity (Def.\ \ref{def:causal_modularity_order}),
local approaches should typically have a substantially better signal to noise
than global ones, especially on data-sets with many variables.

The problem of having a potentially very large number of clusters
(see Fig.\ \ref{fig:illustrate_global_vs_local_structure}) on complex datasets
leads to yet another problem for global approaches: Performing causal discovery per
cluster dramatically reduces the amount of available data per cluster.
A local method can always use all data for each local decision,
and only for weak-testing (§\ref{sec:core_algo}
has to split the data-set;
edge-orientations in the present baseline implementation of our framework
rely on the currently only partially local §\ref{sec:indicator_translation},
see §\ref{sec:ind_locality} however).
With clustering, error-propagation through the CD algorithm happens per cluster
(per global regime, cf.\ Fig.\ \ref{fig:illustrate_global_vs_local_structure}),
leading to potentially much less stable results.

See also §\ref{apdx:indicator_resolution}, where we return to the discussion of
CPD as a post-processing step \emph{after} graph-discovery.

\subsection{Relation to Sliding Windows}\label{apdx:sliding_window}

Sliding window approaches in principle apply quite generally:
As long as the used CD-algorithm applies per window (or the number of
windows with actual change points in them is at least low for practical purposes),
there are little restrictions. Time series, IID-data, latent variables etc.\ can
be handled if the used CD-algorithms is suitable for this kind of data.
Other than our present version of §\ref{sec:indicator_translation},
sliding window approaches can also handle cyclic union-graphs.

In many way, the details of our mCIT §\ref{apdx:data_processing_layer} implementation
are reminiscent of sliding window approaches.
There are, however, important differences because \dots
\begin{enumerate}[label=(\alph*)]
	\item \dots our approach is local:
	\begin{itemize}
		\item
			The \inquotes{validity} of windows (absence of change-points)
			relies only on \emph{local} change-points, avoiding exponential
			decay in available valid lengths with increasing indicator-count.
		\item 
			The actual causal discovery-logic (the underlying CD-algorithm),
			is executed on mCIT results using \emph{all} data, while
			for standard sliding windows CD first propagates finite-sample
			errors, then composes the obtained graphs (including propagated
			errors).
	\end{itemize}
	\item \dots our approach is direct:
	\begin{itemize}
		\item
			The choice of hyperparameters can be guided by trade-offs
			from the statistical analysis of the mixture-dataset
			(see Fig.\ \ref{fig:tradeoffs}), while for standard sliding-window
			approaches the practical choice (especially of the
			cutoffs $a_\pm$, see §\ref{sec:num_ex_scenarios}) is extremely difficult.
	\end{itemize}
\end{enumerate}

In §\ref{sec:num_ex_scenarios}
we compare to a posteriori optimal hyperparameters for the sliding-windows.
This removes the effect of (b) from the comparison.
The convergence-speed for increasing $N$ even so is lower than for our approach,
demonstrating the relevance of (a).
For any practical application, point (b) is of course similarly of
great relevance, as the a posteriori choice of hyperparameters is not
possible in practice (it requires perfect knowledge of the ground truth).

\subsection{Time Series}\label{apdx:MCI}

For the treatment of time series, we adopt the \inquotes{MCI-idea} \citep{pcmci}:
The PCMCI-family algorithms first systematically search for super-sets of time-lagged
parents (the \inquotes{PC$_1$-phase}), then test \emph{momentary} independence.
Momentary here means that $X \independent Y$ is always tested conditional
at least on all lagged parents; in the case of a regression based CIT like partial
correlation, this means lagged parents are regressed out.
The result is a data-set of approximately IID residuals; the application of IID tests
(and confidence-intervals) in the momentary sense has been systematically studied
and proven a very useful practical approach \citep{pcmci}.
In the same spirit we integrate PCMCI by running regime-detection
(marked independence tests) only (or only to the full extend, see below) in the MCI phase.

For other algorithms like PCMCIplus \citep{pcmci_plus} and LPCMCI \citep{lpcmci}
we use a similar idea, but with slightly different implementation.
For PCMCIplus, a ad hoc reinterpretation of different parts of the algorithm
as (non-)MCI works surprisingly well (indeed it outperforms PCMCI in
our numerical experiments, even though PCMCI's stronger assumptions \emph{are}
satisfied).
More generally, the first iteration of Algo.\ \ref{algo:core}
in the acyclic case produces a union-graph (its acyclification \citep{BongersCyclic}
in the cyclic case), which can also be used to provide supersets of (lagged) parents
for subsequent iterations.
Indeed this is how we integrate regime detection into LPCMCI:
After the first iteration, we consider a test MCI if its conditioning set
contains the lagged parents form the graph found in the initial iteration,
and apply full regime-detection only in these cases.
We do not actively enlarge conditioning-sets (to ensure the occurrence of sufficiently many MCI-like
tests), and we do not ensure that lagged effects through contemporaneous parents are blocked.
So this is a heuristic approach with potentially low expected (and numerically observed) recall
on regimes even asymptotically. Considering the complexity of the task
(time series, with contemporaneous links and latent confounding \emph{plus regime detection}),
that this integration with LPCMCI aims to solve, it performs rather well,
especially considering the extreme simplicity of the current integration.
Nevertheless, there is certainly room for future improvement on the combination
of our framework with LPCMCI.

\begin{figure}[ht]
	\begin{minipage}{0.6\textwidth}
		\mayincludegraphics[width=0.45\textwidth]{cd_ts/pc1_phase/regimes_pcmci_f1}
		\hfill
		\mayincludegraphics[width=0.45\textwidth]{cd_ts/pc1_phase/regimes_pcmci_plus_f1}
		\\
		\mayincludegraphics[width=0.45\textwidth]{cd_ts/pc1_phase/regimes_pcmci_precision}
		\hfill
		\mayincludegraphics[width=0.45\textwidth]{cd_ts/pc1_phase/regimes_pcmci_recall}
	\end{minipage}
	\hfill
	\begin{minipage}{0.35\textwidth}
		\caption{Results for regime-detection quality (see §\ref{sec:num_exp_eval_metrics}).
		Shown are F$_1$-score for PCMCI and PCMCI+ (plots in top row), as well as
		precision and recall for PCMCI (bottom row). Line-colors correspond to different $\alpha$-values, for both colors respectively light is the hyperparameter-set
		for generic\Slash{}dark for large regimes and
		line-styles corresponds to different configurations of the PC$_1$-phase, see main text.}
		\label{fig:pc1_phase_regimes}
	\end{minipage}
\end{figure}

\begin{figure}[ht]
	\begin{minipage}{0.6\textwidth}
		\mayincludegraphics[width=0.45\textwidth]{cd_ts/pc1_phase/union_graph_pcmci_f1}
		\hfill
		\mayincludegraphics[width=0.45\textwidth]{cd_ts/pc1_phase/union_graph_pcmci_plus_f1}
		\\
		\mayincludegraphics[width=0.45\textwidth]{cd_ts/pc1_phase/union_graph_pcmci_precision}
		\hfill
		\mayincludegraphics[width=0.45\textwidth]{cd_ts/pc1_phase/union_graph_pcmci_recall}
	\end{minipage}
	\hfill
	\begin{minipage}{0.35\textwidth}
		\caption{Results for union-graph-recovery (see §\ref{sec:num_exp_eval_metrics}).
			Shown are F$_1$-score for PCMCI and PCMCIplus (plots in top row), as well as
			precision and recall for PCMCI (bottom row). Line-colors correspond to different $\alpha$-values, for both colors respectively light is the hyperparameter-set for generic\Slash{}dark for large regimes and
			line-styles corresponds to different configurations of the PC$_1$-phase, see main text.
			Error-bar based on counting-statistics only are negligibly small.}
			\label{fig:pc1_phase_union_graph}
	\end{minipage}
\end{figure}

In figures \ref{fig:pc1_phase_regimes} and \ref{fig:pc1_phase_union_graph},
we show some numerical results for regime-detection and union-graph recovering
comparing different configurations for the PC$_1$-phase:
Here the \inquotes{skip} configuration (solid line) skips the homogeneity test
in the PC$_1$-phase. This means instead of the homogeneity-first approach §\ref{sec:testing_order}
the independence-test (with conservative $\alpha$) is considered more robust to auto-lags.
In the \inquotes{keep} configuration (dashed line), both PC$_1$- and MCI-phase use the
same $\alpha$-value for homogeneity-testing, while in the \inquotes{match} configuration
the ratio of $\alpha$-values in the PC$_1$- vs.\ in the MCI-phase is the same as for independence-testing
(here $5\%$ and $20\%$, blue line; $1\%$ and $10\%$, green line).
For the other numerical experiments (in particular those in §\ref{sec:num_experiments}),
the \inquotes{keep} configuration was used;
for PCMCI this seems to be a reasonable choice (there is little difference, thus little reason
to deviate from the logic of §\ref{sec:testing_order}).
For PCMCI+, skipping the homogeneity-test in the PC$_1$-phase seems to offset (some of) the
convergence problems for the union-graph.
The convergence problem further seems to originate from a loss of precision
(see lower row of plots in §\ref{fig:pc1_phase_union_graph}) on edge detection.

Currently in cases where the homogeneity test is used in the PC$_1$-phase,
it is used as is (potentially with a modified $\alpha$-value).
A more sophisticated approach using a homogeneity-test that is
properly adapted to the time series case for the PC$_1$-phase
could probably considerably improve finite-sample performance.
We leave the detailed study of such modifications to future work.

\subsection{JCI Ideas}\label{apdx:JCI}

Studying regimes with different graphs is very closely related to
the study of multi-context data as for example in \citep{SelectionVars,JCI,CD-NOD,EndoMethod}.
Usually these treatments assume prior knowledge of the assignments to
data-sets\Slash{}contexts (for a detailed discussion of the relation
between CD-NOD and our approach, see the related literature section, §\ref{sec:related_lit}).

There are some such joint causal inference (JCI; \citep{JCI})
ideas that apply almost immediately in the present context.
\begin{example}\label{example:JCIlike}
	Let $X \rightarrow Y$ be a changing link, if we call the (hidden)
	indicator $C$ and include it into the graph, we obtain $X \rightarrow Y \leftarrow C$,
	which has a (detectable given knowledge of $C$) unshielded collider at $Y$
	(this is the viewpoint taken by \citep{JCI}).
	Even if $C$ is initially hidden, its sufficient (potentially implicit)
	recovery from data can be used to orient the link $X\rightarrow Y$ \citep{CD-NOD}.
	In our case, this requires a stochastic time-resolution of the indicator
	(stochastic in the sense that other than for independence-testing, there
	is no deeper problem from the assignment being only approximate),
	which can be easily obtained for example by thresh-holding.
	Then the target of the changing link should depend on this approximate indicator,
	while the source should not (again in precisely the same spirit as
	in the JCI and CD-NOD cases).
\end{example}
There are, however, some difficulties. The JCI idea assumes that
conditional independencies involving both arbitrary context- and system-variables
can be tested. For context-variables that are not \inquotes{sufficiently independent}
(sufficiently independent in this sense are for example
different spatial directions or space and time directions), this is not
typically true (not even in principle). Even if context-variables are exogenous,
tests may not satisfy this assumption, \ie they may \emph{not} be testable, even in principle.
\begin{example}[Limits of JCI like Ideas]
	\hspace*{1em}\\[-1em]
	\begin{enumerate}[label=(\alph*)]
		\item 
		Consider two exogenous indicators $C$ and $C'$, such that
		both switch their value only once, \ie there is
		a unique $t$ with $C_{t+1} \neq C_t$ and a unique $t'$
		with $C'_{t+1} \neq C'_t$.
		How could we test $C \independent C'$?
		We could compare the distance $|t-t'|$ to the total
		number $N$ of data points. If switches of $C$ and $C'$ are independent
		and uniform ($t\sim\UniformDist(\{0, \ldots, N-1\})$) and always occur precisely once,
		then we can control errors at $\alpha$ if we reject independence whenever
		$|t-t'| < \alpha N$. However, this requires a lot of assumptions, and depending on
		the alternative it may gain power slowly or never with increasing $N$.
		\item 
		If we assume (as for JCI) that contexts are exogenous (thus pairwise independent),
		we never need to test $C \independent C'|Z$ for any $Z$ for causal discovery.
		Consider the following simple model:
		$X \rightarrow Y \rightarrow Z$ with both $\Rmodel_{XY}$ and $\Rmodel_{YZ}$
		non-trivial (corresponding to $C_{XY}$ and $C_{YZ}$).
		The argument of the previous example \ref{example:JCIlike}
		works for the link $X \rightarrow Y$, because both indicators are independent of $X$.
		But to orient the link $Y \rightarrow Z$, we have to figure out
		that $C_{YZ} \independent Y$. However, this pairwise test may be difficult:
		We know that $C_{XY} \notindependent Y$,
		if dependence were transitive, then the combination of both tests would
		show $C_{YZ} \independent C_{XY}$ (which is not testable by (a)).
		Now, dependence is \emph{not} transitive, but in our case indicators
		are binary and contain very little information, so for most models in practice
		(like $Y = X \times C_{XY} + \eta_Y$ and $Z = Y \times C_{YZ} + \eta_Z$)
		this conclusion would hold.
		\item 
		The problem from (b) actually \emph{can} be reasonably resolved:
		While $C_{YZ} \notindependent Y$ may not be testable,
		$C_{YZ} \notindependent Y|C_{XY}$ (which would also be ok for CD),
		seems to be testable:
		For example assume (for simplicity) that there are few invalid blocks $\chi \approx 0$,
		then on blocks that are \emph{above} a certain cutoff for the independence
		$X\independent Y$ (associated to $C_{XY}$) and
		below a p-value of say $\alpha$ we will have about a fraction $\alpha$
		of impurities (from $X\independent Y$).
		Next, test the dependence between $Y$ and $C_{YZ}$,
		discard the $\alpha$ highest values, and check if on the remaining blocks independence can still be rejected.
		Discarding $\alpha + \chi_0$ values could even maintain error-control in a conventional
		sense up to $\chi \approx \chi_0$.
		This may not be a sample-efficient procedure, but it demonstrates that
		the problem is surprisingly accessible.
	\end{enumerate}
\end{example}
While part (c) indicates that some JCI-like conclusions can still be drawn,
it may not always be obvious which tests have to be done.
The problem in the example arises, because true independence with one indicator
may be erroneously rejected due to the dependence on another indicator.
This leads to a false positive, which in turn affects our interpretation
of links out of contexts (which seems fine) but also the ability to orient links,
because the collider at $Z$ no longer is unshielded.
For the simple example above, the next iteration(s) of most CD-algorithms
would find the valid adjustment set $C_{XY}$ (or any parents of $Y$ blocking
the influence of other contexts) to delete the link $C_{YZ} \rightarrow Y$
which then allows for subsequent orientation of $Y \rightarrow Z$.
It is, however, not entirely clear if a suitable adjustment set always exists
and is always found by a classical CD-algorithm.
In any case, a suitable test (comparable to part (c) of the previous example)
is required. Hence the details on a systematic study of JCI-like techniques for orienting links
based on detected indicators is left to future work.
This is also related to §\ref{sec:ind_locality} and §\ref{apdx:non_modular_changes}.

\subsection{Lookup Regions}\label{apdx:cd_lookup_regions}

Sometimes, we need to be able to abstractly reason about the
\inquotes{sparsity of lookups} involved in the efficient algorithmic computation
of a CD mapping from independence-structures to graphs (see §\ref{sec:stationary_cd}).
The following simple formalized version will often suffice for this purpose;
note that we will never need to actually compute $\Lookup_{\CDAlg}$.

\begin{definition}[Lookup Region]\label{def:lookup_region}
	The lookup region $\Lookup_{\CDAlg}$ of an (abstract) constraint-based
	causal discovery algorithm $\CDAlg$
	is an algorithm-specific mapping of independence-structures $\IStruct$
	to sets of multi-indices
	$I_{\IStruct} = \{\vec{i}_0, \ldots, \vec{i}_M \} \subseteq \MultiindexSet$
	\begin{equation*}
		\Lookup_{\CDAlg}:
		\IStructs \rightarrow \subsets(\MultiindexSet),
		\IStruct \mapsto I_{\IStruct}\txt,
	\end{equation*}
	with the property
	\begin{equation*}
		\IStruct|_{\Lookup_{\CDAlg}(\IStruct)}
		= \IStruct'|_{\Lookup_{\CDAlg}(\IStruct)}
		\quad\Rightarrow\quad
		\CDAlg(\IStruct) = \CDAlg(\IStruct')
		\text.
	\end{equation*}
\end{definition}
\begin{rmk}[Lookup-Region Always Exists for Deterministic Algorithms]
	Think for example of PC which
	will execute tests $\vec{i}_0$, then $\vec{i}_1$ etc.\ 
	with the choice of later tests to execute depending on the results
	obtained on previous ones. These results on previous ones are formally
	encoded\Slash{}supplied by a fixed independence-structure $\IStruct$.
	After a finite number $M+1$ of tests the algorithm terminates.
	It will \inquotes{look up} precisely the tests
	$I_{\IStruct}=\{\vec{i}_0, \ldots, \vec{i}_M\}$ from $\IStruct$.
	Since subsequent tests are chosen depending on results of previous tests,
	$I_{\IStruct}$ depends on $\IStruct$.
	Finally, the algorithm of course decides its output based only on the tests it
	actually inspects, \ie based on the restricted mapping
	$\IStruct|_{I_{\IStruct}}$.
	
	The property $\IStruct|_{\Lookup_{\CDAlg}(\IStruct)}
	= \IStruct'|_{\Lookup_{\CDAlg}(\IStruct)}
	\quad\Rightarrow\quad
	\CDAlg(\IStruct) = \CDAlg(\IStruct')$
	is an assumption on the formal mapping, but automatically satisfied
	if there is a deterministic algorithm implementing the mapping:
	This follows by induction. We assume the algorithm executes identically
	for $\IStruct$ and $\IStruct'$ with $\IStruct|_{\Lookup_{\CDAlg}(\IStruct)}
	= \IStruct'|_{\Lookup_{\CDAlg}(\IStruct)}$
	up to step $k$ (with $k=-1$ for the start of induction).
	The (deterministic) algorithm chooses a test $\vec{i}_{k+1}$ based on
	results obtained for tests $I_{\IStruct}^k=\{\vec{i}_0, \ldots, \vec{i}_k\}$.
	For $k=-1$ (inductive start) $I_{\IStruct}^{-1} = \emptyset$ independently of
	$\IStruct$, otherwise $I_{\IStruct}^k=I_{\IStruct'}^k$ by inductive hypothesis.
	Thus the (deterministic) choice $\vec{i}_{k+1}$ agrees for $\IStruct$ and $\IStruct'$.
	In particular $\vec{i}_{k+1}\in I_{\IStruct}$, thus
	(by hypothesis) $\IStruct(\vec{i}_{k+1})=\IStruct'(\vec{i}_{k+1})$ and the algorithm
	executes identically in step $k+1$.
	
	So any deterministic algorithm which terminates after a finite number of steps 
	(this includes to our knowledge all constraint-based causal-discovery algorithms used in practice) always has a well-defined lookup-region in this sense.
\end{rmk}
\begin{example}
	In the skeleton-phase, the PC-algorithm checks first all pairwise
	independencies. Then (in the \inquotes{stable} variant) it checks
	$X \independent Y | Z$ if (and only if) $Y$ is pairwise dependent to both $X$ and $Z$.
	And so forth for increasing conditioning sets.
	In the oracle-case, given $G$, pairwise dependence occurs if and only if there is no
	d-separation given the empty-set. So for example for the graph
	$X \rightarrow Z \leftarrow Y$, in the oracle $X \independent Y$, so that
	the only triples tested are $X \dependent Z | Y$ and $Y \dependent Z | X$.
	Thus $I_{\IStructOracle} = \{(X,Y), (X,Z), (Y,Z), (X,Z,(Y)), (Y,Z,(X))\}$
	is the lookup-region for $\CDAlgPC_\text{skeleton}$ and
	in particular $(X,Y,(Z)) \notin I_{\IStructOracle}$.
\end{example}

\subsection{Indicator Resolution}\label{apdx:indicator_resolution}

One of the important novel conceptual ideas behind our approach was, to switch the order
of (time-)resolution of indicators and graph-discovery.
Albeit initially counter-intuitive, finding the causal graphs based on direct testing
turns out to be both in theory and in numerical experiments better, than to first resolve
indicators, then learn graphs per regime.
This does of course not mean that we consider the resolution of indicators as functions of
time (or any other parameter) not important, it simply becomes the second, rather than the first step.
Indeed, also indicator resolution can benefit from this reordering:
Having already discovered the causal graphs in the first step, we can now use this
knowledge to guide our indicator-resolution.

Resolving (discrete in the global case, but not necessarily binary) indicators in a global approach
has to solve a global (multi-variate) change-point detection (CPD) problem, then assign segments to regimes.
Multi-variate change-point detection on finite data is a challenging task, mostly for the usual problems
incurred by high-dimensional data. In our framework, we already know where to look for changes:
In the general case, given a set of graphs, we need to find for example causal effects that
are particularly well-suited to find the differences between any pair of graphs, then
run CPD on these univariate problems. In the modular case (cf.\ §\ref{sec:indicator_translation}),
the results for different model-indicators are even mostly independent of each other (so it is
enough to pick causal effects distinguishing the presence or absence of a model-indicator,
for example the representor of §\ref{sec:indicator_translation}).

While this means we have to solve multiple univariate CPD problems instead of a single
multivariate one, the total number of change-points remains the same
(see Fig.\ \ref{fig:illustrate_global_vs_local_structure}).
If we imagine the concatenation of the time series for all univariate problems, then this
is essentially a super-sampled (and thus easier) version of a single univariate problem
with the original sample-count $N$ and the same total number of change-points.
So what we got in the end is arguably even (considerably, if there are many model-indicators)
simpler than a single univariate CPD problem replacing the multivariate one!

Further, especially in the modular case, we also do not need a separate assignment of
segments to regimes. Indeed the modular case can here also be efficiently represented,
as there is simply one resolution per indicator (see Fig.\ \ref{fig:intro_toy_model}).
A simple representation (both formally and graphically) makes the result of our framework
particularly easy to interpret. The tendency to produce high precision (and loosing primarily
in recall for too small data-sets) on regime-detection (§\ref{sec:num_experiments}) further contributes
to a comparatively easy-to-read output.

A very important, yet so far understudied, challenge to the resolution of regimes
is the quantification of errors. The systematic simplification to binary indicators
and univariate problems may make this problem more accessible for future work.
Note, that in our case, there may be multiple, possibly to a sufficient degree
independent, causal effects (or dependence-scores) associated to the same
model-indicator. This could also provide information for consistency checks and
error-estimation.

So while we do not, in this paper, study the problem of indicator resolution, our framework
does in principle simplify this problem considerably.

\subsection{Causal Models and Predictive Quality}

Score-based methods like \citep{BalsellsRodas2023, mameche2025spacetime}
use a score (based on predictive quality of a model)
directly while also constraint-based approaches like \citep{Saggioro2020}
usually revert to such a score for assigning data points to regimes.
This requires fitting a model to data, which in the case of contemporaneous effects or IID-models,
and especially in the presence of latent confounders can limit applicability in practice.

Model-fitting, but also our mCIT-testing approach are considerably more difficult to
realize for non-linear models.
There is an additional problem one has to consider:
When is the causal model the best predictive (or best scoring) model?
This is usually referred to as consistency of the score \citep{Chickering2002OptimalSI},
and difficult to ensure beyond linear models with reasonable (\eg curved exponential) noises
\citep{Haughton1988OnTC}.
Considering this, the question arises, to what degree regime-assignment by score is viable
for non-linear models.
A conclusive answer does not seem to be known, but a potential difficulty is illustrated by the following example:
\begin{example}
	We consider three nodes $X\rightarrow Y \rightarrow Z$,
	in regime (A), in regime (B) we add a link $X\rightarrow Z$.
	Let the dependence of $Y$ on $X$ be by some (potentially complicated)
	function $f$, and of $Z$ on $Y$ by some (potentially complicated)
	$g \approx \alpha f^{-1}$, such that $g\circ f(x) \approx \alpha x$.
	Further, for the direct dependence in regime (B) of $Z$ on $X$
	also choose a simple (for example linear) function $h(x) = \beta x$.
	If we fit models to both regimes respectively, we have to be aware
	that non-parametric regressors will not converge uniformly on the function-space,
	thus the approximation of $f$ and $g$ may be much worse than the approximation
	of $h$.
	If $\alpha \approx \beta$ and both are considerably smaller than the (large) variance
	of $\hat{f}$ and $\hat{g}$, then there seems to be a rather high risk that
	the model from regime (B) has better predictive quality on
	segments or neighborhoods of change points of regime (A) than the correct
	model, which could lead to considerable instability, for example in
	sliding change-points during an optimization-step.
\end{example}

\subsection{Score-Based Methods}\label{apdx:score_based}

Our framework is formulated specifically for constraint-based methods and via a modification of
the independence structure. In practice, most score-based causal discovery approaches use a local
notion of consistency for (typically also local) scores. For example in the very influential
\citep[Def.\ 6]{Chickering2002OptimalSI} defines local consistency in a way that asymptotically
is equivalent to conditional independence-testing. Indeed the notion of correctness
in our Thm.\ \ref{thm:mCIT} should be suitable as a \inquotes{marked} local consistency statement.
Thus it should be possible to employ score-based technology with our framework as well.
We have not studied this possibility in detail.

\section{Details on Models and Data-Generation}\label{apdx:model_details}

Statistical modeling and data-generation for numerical experiments usually consists
of two stages:
A prior distribution describes the probability for the occurrence of a particular
set of true parameters (for example linear coefficients) defining a stochastic model.
This stochastic model then describes the probability of obtaining a specific data-set.
Data-generation for numerical experiments or benchmarking follows the same pattern;
for this reason we incorporate many details of our data-generation into
this section as examples for the more abstract probability theoretical and
statistical ideas they seek to emulate.

The second step, the generation of a data-set from a model, typically has to employ
assumptions about the reuse of model aspects to ultimately allow a meaningful statistical
analysis. For example, it is often assumed that the data-generation is IID,
in which case the distribution of a size $N$ data-set is a $N$-fold product,
where the model describes a factor and is reused for all samples.
Similarly in the stationary case (for Markov-sequences),
one usually assumes the existence of a model of the next time-step relative
to few previous steps in the past, which is reused to successively generate the time series.

In the non-IID or non-stationary case, this symmetry is explicitly broken.
Formally, this was approached in §\ref{sec:models} as follows:
All symmetry-broken information is by construction contained into
(binary\footnote{The non-additive
case Def.\ \ref{def:model_nonstationary_scm_nonadd} can be treated more generally when
including non-binary indicators, we focus on the binary case, however, for this treatment.})
indicators, the remainder of the model can be reused across data points;
this is also the reason why the parameters of this shared model
can be subject to statistical treatment in a standard way.

For real world phenomena we seek to model, $R$ itself is also random;
at least it \emph{should be modeled as random}, because it is not known a priori.
Fixing the mapping $R(t)$ as part of the model in the second -- data-generating --
stage, puts the full burden of describing this \emph{randomness} as part of the
prior distribution over models.
Conceptually this is fine, but we need a more fine-grained description
of what this means for several reasons:
\begin{enumerate}[label=(\arabic*)]
	\item
		The point-wise recovery of $R(t)$ is typically not possible
		from a single data-set, as this part of the model is not shared by
		multiple data points (see also §\ref{sec:persistence}).
		Instead we plan to recover in §\ref{sec:dyn_indep_tests}
		information of the form: Is $R(t)\equiv \const$?
		In practice, we observe a finite number $N$ data points,
		starting at some \emph{arbitrary} point $t_0$ in time.
		It is possible that this might select a range of time-points
		where the restriction $R|_{[t_0, t_0+N)}$ is constant,
		while $R$ is not (see reachable and reached states §\ref{sec:identifiable_states}).
	\item 
		Real-life phenomena often happen at some typical time-scale
		depending on the phenomenon. Indeed persistent non-stationarity is
		inherently a multi-scale problem.
		This should be modeled by a suitable model-prior.
		It is, however, not typically meaningful to try and reason on $N$ (consecutive)
		discrete time-steps separated by $\Delta t$,
		about phenomena happening at time-scales $\tau$ with
		$\tau \ll \Delta t$ or $\tau \gg N \Delta t$,
		a restriction that is formally captured by the
		Nyquist--Shannon sampling theorem.
\end{enumerate}
We find the following additional structure helpful to study these problems:
We assume the prior distribution describes, \emph{for each fixed $N$},
an \inquotes{indicator meta-model} $\mathcal{R}_N$. This means, other than
for example in the prior over linear coefficients, we explicitly allow for $N$-dependence
of some aspects of the indicators; this seems to be necessary to accurately
capture point (2), and provides interesting insights into convergence behavior.
These indicator meta-models are formally binary random sequences, \ie
random-elements taking values in $\Map(\mathbb{Z}, \{0,1\})$.
We will call realizations $R\in\Map(\mathbb{Z}, \{0,1\})$ unrestricted indicators.
In the next step, given a random\footnote{The choice of prior over $t_0$ will not matter much,
as we focus on stationary sequences $\mathcal{R}$, the core idea is that in principle
multiple starting-points could be picked for the same unrestricted indicator $R$.}
starting time $t_0$, the restricted indicator $R|_{[t_0, t_0+N)}$
is simply a restriction of the mapping $R$.
This restricted indicator can be used as indicator in a
non-stationary SCM as described in §\ref{sec:models}.

Formally, this only factorizes what we previously called the first, or model-prior,
stage: For a fixed $N$,
instead of taking a single model from the prior, to than sample a single data-set of
size $N$,
the prior now describes probabilities of $N$-independent model information
like linear coefficients
and $N$-specific indicator meta-models $\mathcal{R}_N^{ij}$ for each pair of variables $i$, $j$
and for this $N$. From each indicator meta-model we sample a single unrestricted indicator
$R_{ij}$ (for each $i$, $j$). Then we pick a single random $t_0$ and obtain a single
restricted indicator $R|_{[t_0, t_0+N)}$. Finally we combine the
$N$-independent model information like linear coefficients with the
restricted indicators into a non-stationary SCM as described in §\ref{sec:models}.
That is, the task and end-product remain the same, but we have subdivided
the process in intermediate steps that (formally) record additional information,
like the unrestricted indicators $R_{ij}$.

At this point, our description is already detailed enough that reachable and
reached states are a simple and well-defined concept:
\begin{definition}[Formal Reachability]\label{def:reachable_and_reached}
	For a model obtained by the procedure described above,
	we call the image $\im(R_{ij}) \subseteq \{0,1\}$ reachable
	model-indicator values and the image
	$\im(R_{ij}|_{[t_0, t_0+N)}) \subseteq \{0,1\}$ reached
	model-indicator values.
	We call a state $s\in S$ (Def.\ \ref{def:states})
	reachable, if its model-indicator values are reachable,
	and reached, if its model-indicator values are reached.
\end{definition}
Further, the explicit $N$-dependence of $\mathcal{R}_N^{ij}$ leaves us
with plenty of freedom to describe meaningful models and study
interesting asymptotic behaviors.

The formulation above can easily be generalized from contiguous
(uninterrupted) time series to more general data-sets,
by replacing $t\in \mathbb{Z}$ by suitable indices
$i\in I$ and restricting indicators to suitable random subsets
of fixed size $N$ rather than to intervals of size $N$.
Nevertheless many of the relevant aspects can be most easily illustrated
on the time series case with persistent over time regimes.
For this reason, we will also discuss this special case in more
detail, to than return to the general case with the time series results
as examples to illustrate the more abstract description.

\begin{notation}
	In this section we will usually loosely state constructions of the form \inquotes{sample XYZ from $P_{\txt{XYZ}}$, then \ldots},
	where XYZ is some random element. This can of course be written down formally as plugging an element
	$\omega \in \Omega$ of the sample-space into the mapping XYZ to obtain an instance, then
	do subsequent constructions as compositions of mappings to obtain a suitable random-element in the end.
	This should not lead to confusion, but dramatically simplifies notation and makes it much easier to
	clearly lay out the underlying ideas.
\end{notation}

\subsection{Properties of Priors and Meta-Models}

In principle, a model could have only few reachable states with potentially
complicated transition patterns. Our priority at the moment is not the study of
complicated state-machines but of the reconstruction of changing aspects of
causal models.
For causal models, it is believed that mechanisms typically change independently (cf.\ \eg \citep{Elements,CD-NOD}), a principle known as causal modularity.
As laid out in the introduction §\ref{sec:intro},
we want to prioritize such locally changing models.
Formally this idea can be captured as follows.

\begin{definition}[Causal Modularity of Changes]\label{def:causal_modularity_order}
	We say a prior describes models changing in a modular way
	(changing of order $\leq 1$, see below),
	if the indicator meta-models $\mathcal{R}_{ij}$ are independent\footnote{While
	such independence is difficult to test (§\ref{apdx:JCI}), there is no
	problem in formally assuming it.}.
	If the non-trivial (see Rmk.\ \ref{notation:non_trivial_indicator_meta}) indicator meta-models
	can be assigned into sets with at most $k$ elements each,
	such that any two such sets are
	(jointly) independent of each other, we say the prior describes models
	changing of order $\leq k$;
	of order $k$ if $k$ is minimal with this property.
	In particular a prior of stationary models (all indicators are trivial)
	is changing of order $0$;
	a prior describes models that are changing, but in a modular way (see above)
	is changing of order $1$.
\end{definition}
\begin{lemma}[Modular Reachability]\label{lemma:model_modular_reachable}
	If a prior describes models changing in a modular way,
	then the states defined by Def.\ \ref{def:states}
	are precisely the reachable (but not necessarily reached) states.
\end{lemma}
\begin{proof}
	This follows directly from Def.\ \ref{def:causal_modularity_order},
	Def.\ \ref{def:reachable_and_reached} and Def.\ \ref{def:states}.
\end{proof}
\begin{example}[Data-Generation]\label{datagen:modularity_order}
	Most of our numerical experiments focus on the modular case.
	However, we do include some experiments specifically to study the effect of
	higher order changes (see Fig.\ \ref{fig:dcd_ts_regime_complexity}).
	In these experiments all indicators are \inquotes{synchronized}
	(\ie the order is the total number of non-trivial indicators;
	indicators are not just dependent but deterministically related\Slash{}agree up to \inquotes{sign}).
	Together with experiments on the effect of the total number of 
	(independent) indicators, this provides good insight into the
	effect of the order of changes.
\end{example}

For a fixed $N$, drawing a single data-set produces a single indicator (per link),
so properties like the ratio $\chi$ of valid blocks for a given pattern (Def.\ \ref{def:blocks}),
are well-defined \emph{numbers}. But when we study the statistical power of our
tests later on, we will be interested in the behavior under limits, thus for different $N$
entailing different $\mathcal{R}_N$.
\begin{rmk}\label{rmk:trivial_scaling_issues}
	If the indicator meta-models are of the form
	$\mathcal{R}_N(t) = \bar{\mathcal{R}}(t / \Lambda(N))$, \ie if the behavior for large
	$N$ changes only by modifying the time-scale by a factor $\Lambda(N)$,
	this description can be retained, by using the same $\bar{\mathcal{R}}$ for all $N$.
	However, such a scaling behavior does not seem to be sufficiently expressive to
	capture realistic or interesting examples; it also comes with an entire host of new
	problems: For example $\bar{\mathcal{R}}$ should now be defined on the real numbers
	$\mathbb{R}$ for this notation to make sense. But going to large $\Lambda$ effectively
	zooms into this the real axis, in particular it could reveal completely
	new behavior of the mapping $\bar{\mathcal{R}}$ on small scales.
	An exception is the case where $\Lambda(N)=\const$ (see \inquotes{infinite
	duration limit}, Def.\ \ref{def:limits_temporal}), but this limit
	turns out to be too challenging: By Prop.\ \ref{prop:imposs_B},
	it is typically not possible to achieve such convergence for weak-regime
	(and thus mCIT) tests.
	On the other hand, one can introduce a lower bound $\Delta t_{\text{min}}$
	on smallest structures in $\bar{\mathcal{R}}$,
	but this does not seem to lead to interesting insights
	(see Rmk.\ \ref{rmk:inifite_sampling_uninformative}).
\end{rmk}

Instead it seems to be easier to treat properties like $\chi_N$ as
sequences of random variables.

\begin{definition}[Indicator Properties]\label{def:indicator_properties}
	Recall the definition of a pattern in the sense of Def.\ \ref{def:blocks}:
	A pattern  $\Pattern$ is a mapping $\Pattern : \mathbb{N} \rightarrow \mathcal{B}_I$ to collections $\Pattern(b)$ of disjoint subsets $T_i \in \Pattern(b)$ of $T$ with individual size $|T_i|=b$.
	
	Given an indicator $R: T \rightarrow \{0,1\}$,
	a pattern $\Pattern$ and a block-size $b$, there is a well-defined
	invalid-block-ratio $\chi$ (with $R$-validity of a block $T_i$ defined in Def.\ \ref{def:blocks} as $R|_{T_i}\equiv\const$)
	\begin{equation*}
		\chi(R,\Pattern,b) := \frac{\setelemcount{\{ T_i \in \Pattern(b) | T_i \txt{ is not $R$-valid } \}}}{\setelemcount{\Pattern(b)}}
		\txt.
	\end{equation*}
	This property of individual indicators induces a random variable
	on indicator meta-models
	(technically we assume the following mapping is measurable, similar below)
	\begin{equation*}
		\chi_N(\mathcal{R}, \Pattern, b) : \Omega \rightarrow [0,1], \omega \mapsto \chi(\mathcal{R}_N(\omega), \Pattern, b)
		\txt.
	\end{equation*}
	Analogously we can define a fraction $A$ of valid independent blocks
	(with corresponding induced random variables)
	\begin{align*}
		A(R,\Pattern,b) &:= \frac{\setelemcount{\{ T_i \in \Pattern(b) | R|_{T_i} \equiv 0 \}}}{\setelemcount{\Pattern(b)}}
		\txt,\\
		A_N(\mathcal{R},\Pattern,b) &: \Omega \rightarrow [0,1], \omega \mapsto A(\mathcal{R}_N(\omega), \Pattern, b)
		\txt.
	\end{align*}
	We will sometimes allow $b=b(N)$ to depend on $N$, which similarly induces
	random variables $\chi_N$ and $A_N$.
\end{definition}
We will later often consider limits where $\chi\rightarrow 0$,
so the following property is useful:
\begin{lemma}[Relating $A$ to $\chi$]\label{lemma:relate_A_to_chi}
	Given the (true) regime-fraction $a$,
	\begin{equation*}
		a - \chi \leq A \leq a
	\end{equation*}
\end{lemma}
\begin{proof}
	$A\leq a$ is clear from its definition.
	There are at most $\chi \setelemcount{\Pattern(b)}$ invalid blocks,
	thus at most $\chi \setelemcount{\Pattern(b)} b = \chi N$ data points
	contained within any invalid block.
	There are $a N$ data points in the independent regime in total,
	thus at least $(a-\chi)N$ data points in valid blocks.
	These are $(a-\chi)\setelemcount{\Pattern(b)}$ many, such that indeed $A \geq a-\chi$.
\end{proof}
We will treat the regime-fraction $a$ as a regular model-parameter (\ie the
same way we treat for example linear coefficients). It would of course
be possible to treat $a$ (or any other parameters) as $N$-dependent as well,
however, in practice this does not seem to lead to substantial insights.
Hence we restrict the $N$-dependent treatment to the aspects (in particular $\chi$)
of the indicator-modeling that seem to be required (see introduction to this
§\ref{apdx:model_details}) to keep things as simple as reasonable.

\begin{notation}\label{notation:non_trivial_indicator_meta}
	We will also consider the triviality-property (constness) of model-indicators as logically
	$N$-independent. This is formally not a problem: Even if $\mathcal{R}_N$ may produce
	a non-trivial unrestricted or restricted indicator $R$ only with a certain probability
	$p_{ij} \neq 0,1$, we can always rewrite it as a \inquotes{structural} prior
	(describing the causal graph and non-triviality of model-indicators) with
	a an additional parameter $p_{ij}$ for each link and a logically decoupled description
	of a non-trivial indicator meta-model (for example the original meta-model conditioned
	on this property being true) only to be invoked in case a \emph{non-trivial} indicator
	would be encountered based on the roll on $p_{ij}$.
	This is the same as factoring
	\begin{equation*}
		P(\mathcal{R})=P(\mathcal{R}|\text{$R$ non-trivial}) P(\text{$R$ trivial} =\boolFalse)
		+ P(\mathcal{R}|\text{$R$ trivial})P(\text{$R$ trivial}=\boolTrue)
		\text,
	\end{equation*}
	where $P(\mathcal{R}|\text{$R$ trivial})$ is binary (it only decides if $R$ is trivially
	dependent or independent) and $P(\text{$R$ trivial})$ is binary, so all the interesting
	structure is in $P(\mathcal{R}|\text{$R$ non-trivial})$.
	Triviality can refer to both the restricted or unrestricted indicator,
	we will for simplicity focus on the non-triviality of the restricted indicator.
	For the remainder of this section we consider only non-trivial indicator meta-models
	in isolation.
\end{notation}
\begin{example}[Data-Generation]\label{datagen:reachedness}
	Our data-generation includes a sanity-check for non-triviality of generated indicators.
	This does not guarantee reachedness of all \emph{states}, but together with reasonable
	maximal regime-lengths and modular changes,
	it provides sufficiently reached states for statistical analysis of our (and other) method(s).
\end{example}

\subsection{Statistical Testing}\label{apdx:models_and_testing}

Conventionally, statistical tests are considered consistent if they control
false-positive rate and gain power asymptotically point-wise on the support of the prior.
Here, point-wise, means for example when testing independence,
for any model satisfying the null hypothesis the error-rate is controlled
and that for every $\epsilon > 0$
and for every individual model satisfying the alternative hypothesis
(with true dependence value $d_1\neq 0$;
a point in the space of possible alternative hypotheses),
there is a $N_0$ such that $\forall N\geq N_0$ the test will have
power at least $1-\epsilon$ on samples of size $N$ on this alternative.

As already noted before (see Rmk.\ \ref{rmk:trivial_scaling_issues}),
a definition of point-wise that is both meaningful and achievable
in terms of controlling false-positives point-wise is difficult to achieve
(for details see §\ref{apdx:data_processing_layer}).
At the same time, note that gaining power point-wise can be too weak
for our use-case for the following subtle reason:
From a practical viewpoint, we want our test to gain power in a
Bayesian sense on an unknown (thus optimally on any \inquotes{reasonable}) prior,
that is drawing an alternative at random from this prior and using the
point-wise power as likelihood of a correct detection, we want the posterior
probability of a correct detection to approach $1$. If the model-prior
is independent of $N$, than this follows from the point-wise property
for example under reasonable continuity assumptions.
However, generally, the pointwise existence of $N(\epsilon)$ may not be sufficiently uniform, the Bayesian criterion would require that the prior-weighted integral of pointwise values of $N(\epsilon)$ (\ie the expectation value $E[N(\epsilon)]$) does not diverge, which may not
be the case.

In our case, the prior changes with $N$. Indeed, it changes
\emph{critically} with $N$: By Prop.\ \ref{prop:imposs_B},
it is not possible to achieve convergence in an infinite duration limit
(for the weak-regime test), yet we will show that under
plausible assumptions on the prior, convergence in a Bayesian sense
\emph{is possible}.
Why this is the case, will become much clearer on the
example of persistent-in-time regimes §\ref{apdx:temporal_regimes} below
(see Rmk.\ \ref{rmk:error_type_split}).
In the end, what we show is that our mCIT implementation
is asymptotically correct in the following \inquotes{weak\footnote{We
call this notion of correctness \inquotes{weak}, because it does not
provide finite-sample error-guarantees, which would require a point-wise
guarantee.} Bayesian} sense:

\begin{definition}[Asymptotic Correctness]\label{def:error_control_power_meta}
	Let $\{\hat{T}_N\}_{N\in\mathbb{N}}$ be a family statistical tests of hypothesis $T$.
	We formally allow $T$ to take a finite number (possibly $\geq 2$, \eg the mCIT
	has $T\in \{0,1,\mCIToutR\}$) of values.
	We say $\hat{T}_N$
	is weakly asymptotically correct under the indicator meta-model $\mathcal{R}_N$,
	if for any (fixed) parameters $\theta=(d_0, d_1, a, \dots)$
	\begin{equation*}
		\lim_{N\rightarrow\infty}P_\theta\big(\hat{T}_N(D_N) = T(R|_N, \theta)\big) = 1
		\txt,
	\end{equation*}
	where $R|_N$ is the restricted indicator of length $N$
	and $D_N$ the corresponding data-set of size $N$.
	Thus $N$-independent parameters $\theta$ are fixed (point-wise),
	but convergence over the $N$-dependent indicator meta-model $\mathcal{R}_N$
	is in a Bayesian sense. The hypothesis $T$ depends
	on the non-stationary SCM
	(Def.\ \ref{def:model_nonstationary_scm}\Slash{}\ref{def:model_nonstationary_scm_nonadd}),
	thus on $R|_N$ and $\theta$; for example it can be a marked independence
	(Def.\ \ref{def:marked_independence_rel_d}; Is $R|_N$ trivial, and is $d_0=0$?).
\end{definition}
\begin{rmk}\label{rmk:inifite_sampling_uninformative}
	One of the reasons, why we do not restrict ourselves to much simpler meta-models
	of the form $\mathcal{R}_N(t) = \bar{\mathcal{R}}(t/\Lambda(N))$
	(see Rmk.\ \ref{rmk:trivial_scaling_issues}) is that it does not
	seem to be informative here.
	For example in the CPD literature, assuming an infinite sampling limit
	(define $R$ on an interval, for example $[0,1]$ with some smallest structure
	for example $\Delta t_{\text{min}}$, then subdivide into $N$ equal steps)
	is quite common. This is often reasonable to achieve a goal like the correct
	discovery of \emph{all} change-points, where in practice the number
	of change-points cannot arbitrarily increase with $N$ in order to
	be achievable. Our goal (\inquotes{Is the indicator constant?})
	also gets more difficult to achieve if the number of change-points increases,
	but is much more tolerant to errors on individuals ones.
	In practice this means that an infinite-sampling limit is not informative:
	Writing down any plausible method, we find convergence in this limit for
	arbitrary plausible hyperparameters. Thus this notion of convergence
	does not inform us about which method or hyperparameters are a \emph{good}
	choice.
	Note, that our results \emph{do} also hold in this simpler case,
	cf.\ Example. \ref{example:lambda_scaling_from_trivial_scaling} and Cor.\ \ref{cor:mCIT_on_trivial_scaling}.
\end{rmk}

\subsection{Persistent-In-Time Regimes}\label{apdx:temporal_regimes}

We start with an exemplary treatment of persistent-in-time regimes.
This will shed light both on the potential form of indicator meta-models,
and on the intuition behind what makes the test-specification
Def.\ \ref{def:error_control_power_meta} accessible.
For our present purposes, we take persistent-in-time regimes to mean the applicability of the
time-aligned pattern of the main text, recall:

\theoremstyle{definition}
\newtheorem*{example_time_aligned_blocks}
{Example \ref{example:persistent_in_time_pattern}}
\begin{example_time_aligned_blocks}[Time-Aligned Blocks]
	The time-aligned pattern is the subdivision of $I=T=[0,N)\cap\mathbb{Z}$
	into size $B$ blocks of the form $b_\tau = [\tau B,(\tau+1)B) \subseteq T$ with
	$\tau \in [0,\Theta)\cap\mathbb{Z}$.
\end{example_time_aligned_blocks}

Persistent non-stationarity as discussed here is inherently a multi-scale problem.
There is some process driving regimes that happens at larger time-scales than the
primarily modeled problem.
For most realistic scenarios such a regime-driving process will itself have its typical time-scale.
For example daily climate data may feature regimes driven by weather phenomena happening at a weekly scale,
by El Niño like effects happening over months or years, by climate change happening
over decades etc. For each phenomenon individually the
assumption of IID segment-lengths may be plausible and useful.
But with typical segment-lengths substantially different for different regime-driving phenomena,
we get a more concrete version of the point (2) at the beginning of §\ref{apdx:model_details}:
We have to distinguish between a \inquotes{random} indicator \emph{before} fixing
a phenomenon (thus length-scale), in our formalism $\mathcal{R}_N$,
compared to a \inquotes{random} indicator \emph{after} fixing
a phenomenon (thus length-scale), in our formalism the choice of $t_0$
and were to restrict $R$.

\begin{rmk}[Exchangeability]\label{rmk:indicators_exchangability}
	Formally this means, we replace the basic idea of
	IID segment-lengths by exchangeable segment lengths.
	If \inquotes{choosing a typical length-scale} is modeled in a sufficiently expressive way
	(as a random measure describing the length distribution for this phenomenon),
	then randomizing first length-scale, then IID segment lengths per scale is by de Finetti's
	theorem precisely the same as allowing exchangeable segment lengths in the
	indicator meta-model.
\end{rmk}
\begin{notation}
	In the exchangeable case (Rmk.\ \ref{rmk:indicators_exchangability}),
	the distribution over the lengths of segments (intervals with $R$ taking the same value,
	and maximal with this property) do not depend on the absolute value of $t$,
	so there are random variables $L_\pm$ with these laws (dependent '$+$' for dependent,
	'$-$' for independent).
	We will for the persistent-in-time case focus on such exchangeable models and
	switch between the encoding as $\mathcal{R}$ and $\mathcal{L}$.
	We denote the exchangeable meta-model
	as $\mathcal{L}$ and the IID (see Rmk.\ \ref{rmk:indicators_exchangability}) \inquotes{instance model}
	(after drawing a length-scale) as $L$.
\end{notation}

\subsubsection{Indicator Meta-Model for Temporal Regimes}

There are many plausible models for the indicator meta-model.
One typically attempts to rely on few simple
properties of the prior (here: the indicator meta-model),
such that conclusions are likely
to apply to actual (unknown) priors.
Further we want to understand how to formulate an \inquotes{uninformative} prior,
a model, for the case where we do not know anything special about the regime-length
distributions, \ie a prior to express our lack of knowledge; this is
relevant especially for numerical experiments.
We focus on structures similar to the following one:

\begin{example}[Data-Generation]\label{datagen:indicator_metamodel}
	Let the scale $S \sim \NormalDist(\log(\ell), \gamma)$, then we call
	\begin{equation*}
		L = \exp(S)
	\end{equation*}
	the scale-normal length-distribution at $\ell$ of width $\gamma$.
	Note, that $S$ has median $\log(\ell)$ and $\exp$ is a strictly monotonic function,
	thus $L$ has median $\ell$.
	
	We define a meta-model -- a model providing a model for length-distributions in form
	of a random variable $L$ -- with parameters $A$, $B$ and $\gamma$ as follows:
	Let $\log(\ell) \sim \UniformDist(\log(A), \log(B))$ and
	the resulting model (random length-scale $L$)
	is the scale-normal length-distribution at $\ell$ of width $\gamma$.
	In principle, it may make sense to have separate length-scales $\ell_+$ and $\ell_-$ for
	segments in the dependent and independent regime respectively.
	In practice, a large discrepancy, \eg $\ell_+ \gg \ell_-$, leads to very imbalanced data-sets
	($a\approx 0$ or $a\approx 1$),
	a challenge which we leave to future work.
	
	In our numerical experiments we typically use this meta-model together with $A=5$, $B=\sfrac{N}{2}$
	and a fixed value $\gamma \approx 0.3$.
	Instead of separate scales $\ell_+$ and $\ell_-$, we use a single scale and a uniform at random imbalance
	(approximately, $a \sim \UniformDist(0.15,0.85)$)
	to slightly modify $\ell_+$ vs.\ $\ell_-$.

	For the numerical experiments of the full (causal discovery) method (§\ref{sec:num_experiments}),
	the lower bound was increased to $A=200$ to make the baselines used for comparison
	more applicable. Our method seems to be comparatively stable also for lower
	regime-sizes, see Fig.\ \ref{fig:dcd_ts_regime_sizes} and Fig.\ \ref{fig:cd_iid_regime_size}.
\end{example}

The most important aspect here is that the meta-model first fixes a typical time-scale $\ell$
for each indicator.
Each indicator individually will then contain segments of similar scale.
We will consider an indicator meta-model \inquotes{uninformative} if $\ell$ is distributed uniformly across scales.
For true multi-scale problems this seems to be a plausible choice.
We first need a formal convention for what \inquotes{typical scale} will mean on arbitrary priors:

\begin{definition}[Typical Regime-Length]
	For a random (segment-length) variable $L$ we define the typical regime-length $\ell$
	as the median $\ell = \median(L)$.
\end{definition}

Then we can already formalize what we mean by an uninformative prior:

\begin{definition}[Uninformative Prior]\label{def:uninformative_prior_temporal}
	We call a meta-model for $L$ (inducing $\mathcal{R}$) an uninformative prior if
	the typical segment-size $\ell_N:=\median(L_N)$ is distributed uniformly across \emph{scales}
	between a single sample and the sample size $N$, \ie if
	$\log(\ell_N) \sim \UniformDist(0, \log(N))$.
\end{definition}
\begin{example}[Data-Generation]\label{datagen:uninf_approx}
	The data-generation as described above (example \ref{datagen:indicator_metamodel}) is
	not in a strict sense uninformative in this sense (because it disregards length-scales that are too small or too large
	for stable numerical experiments), but especially for large $N$ it is a very good approximation
	of an uninformative prior in this sense.
\end{example}

This does not imply any specific model for $L$, and indeed no such specific model will be needed
for theoretical results;
this is due to the median entailing a lower bound on valid blocks with increasing certainty
for increasing block-size, see also (proof of) Lemma \ref{lemma:time_regimes_limit_chi}.
\begin{lemma}[Median-Limited Switch-Frequency]\label{lemma:regime_switches_bounded_via_median}
	Given $N=b \Theta$ data points, and $\mathcal{R}$ with regime
	segment lengths $L$ and median $\ell=\median(L)$,
	such that the law of $L$ is non-singular at $\ell$.
	For $\epsilon >0$, there is $\Theta_0$ such that $\forall \Theta\geq\Theta_0$,
	the regime changes at most
	$\frac{2N}{\ell-\epsilon}$ times (\ie there are at most $\frac{2N}{\ell-\epsilon}$ values
	$t\in T$ with $R(t+1) \neq R(t)$)
	with probability at least $1-\epsilon$.
\end{lemma}
\begin{proof}
	This is a simple application of Markov's inequality
	which states that for any non-negative random variable $X$ and
	any $t>0$, the inequality $P(X\geq t)\leq\frac{E[X]}{t}$ is satisfied.
	$X=L$ is clearly non-negative, thus for $t=\ell$ the (true) median, the left-hand-side
	$P(L\geq \ell) = \frac{1}{2}$ (by non-singularity at this point).
	The right-hand-side is $\frac{E[L]}{\ell}$ and thus $E[L]$ is at least
	\begin{equation*}
		E[L] \geq \frac{\ell}{2}
		\txt.
	\end{equation*}
	In particular the mean $\bar{L}$ eventually (beyond some $\Theta_0$)
	is larger $\frac{\ell-\epsilon}{2}$
	with probability at least $1-\epsilon$ by CLT.
	Note that the number of switches $s$ (the number of values
	$t\in T$ with $R(t+1) \neq R(t)$)
	is
	\begin{equation*}
		s=\frac{N}{\bar{L}}-1
		\leq
		\frac{2N}{\ell-\epsilon}
		\txt{ with probability at least $1-\epsilon$.}
	\end{equation*}
\end{proof}

\subsubsection{Asymptotic Limits}\label{apdx:limits}

As the total sample-size $N$ grows, asymptotic properties strongly depend
on the behavior of \inquotes{typical} segment-sizes $\ell$ relative to $N$.
We will classify limits by their relative growth rates.
We fix an indicator meta-model $\mathcal{R}$.
First recall the following standard definition:
\begin{definition}[Stochastic Ordering]\label{def:stochastic_ordering}
	Given two real-valued random variables $A$ and $B$, then $A$ is stochastically less than $B$, denoted
	$A \preceq B$, if the cdf is pointwise less
	\begin{equation*}
		A \preceq B \quad:\Leftrightarrow\quad \forall x\in\mathbb{R}: \Pr(A>x) \leq \Pr(B>x)
		\txt,
	\end{equation*}
	and similar for strict versions (replace $\preceq$ and $\leq$ by $\prec$ and $<$).
\end{definition}
\begin{example}\label{example:stoch_ordering_same_var}
	If $X \sim \NormalDist(\mu_X, \sigma^2)$ and $Y \sim \NormalDist(\mu_Y,\sigma^2)$ with
	the same variance parameter $\sigma^2$, then $X \preceq Y \Leftrightarrow \mu_X \leq \mu_Y$.
\end{example}
For persistent-in-time regimes specifically we use the following preliminary
(for the general case, see §\ref{apdx:models_general_patterns})
notion
\begin{definition}[Asymptotic Limits, Temporal Case]\label{def:limits_temporal}
	We say regimes grow at least $\Lambda$-uniformly for a length-scale
	$\Lambda : \mathbb{N} \rightarrow \mathbb{R}$ if
	$\ell$ is bounded and $\Lambda$-decreasing in the sense that
	there are $N_0 \in \mathbb{N}$, $C$ with
	$\Pr(\frac{\ell_\pm(N_0)}{\Lambda(N_0)} < C) > 1-\epsilon$ and
	\begin{equation*}
		N < N'
		\quad\Rightarrow\quad
		\frac{\ell_\pm(N)}{\Lambda(N)}
		\preceq
		\frac{\ell_\pm(N')}{\Lambda(N')}
		\txt.
	\end{equation*}
	We will additionally require that any length-scale $\Lambda$ is monotonically growing and such that
	$\Lambda(N) \leq N$.
	We say regimes grow at least $\Lambda$-uniformly eventually, if there is
	$N_0'$ such that the implication above holds for $N,N' \geq N_0'$.
\end{definition}
\begin{example}[Trivial Scaling]\label{example:lambda_scaling_from_trivial_scaling}
	If $\mathcal{R}_N(t)=\bar{\mathcal{R}}(t/\Lambda(N))$ for
	$\bar{\mathcal{R}}:\mathbb{R} \rightarrow \{0,1\}$ with
	lower bound $\Delta t_{\text{min}}$	on smallest structures,
	then regimes grow at least $\Lambda$-uniformly,
	because $\frac{L_\pm(N)}{\Lambda(N)}=\frac{L_\pm(N')}{\Lambda(N')}$,
	thus $\frac{\ell_\pm(N)}{\Lambda(N)}=\frac{\ell_\pm(N')}{\Lambda(N')}$.
\end{example}
\begin{example}\label{example:limits}
	Some intuitive examples are:
	\begin{enumerate}[label=(\alph*)]
		\item $\Lambda_\gamma(N) = N^\gamma$ is a simple but often useful choice.
		\item $\Lambda(N) = N$ (or $\gamma = 1$): This is an \inquotes{infinite sampling
		limit}, where the regime-structure is (essentially) fixed on a finite interval (say $[0,1]$), and
		for increasing $N$ the time-steps are chosen as $\Delta t = \sfrac{1}{N-1}$.
		This limit is often considered for CPD consistency.
		\item $\Lambda(N) = \const$ (or $\gamma = 0$):
		This is an \inquotes{infinite duration limit}, where the same system for increasing $N$ runs for a time proportional
		$N$ effectively appending new samples with the same regime-structure.
		This limit captures most finite sample properties.
	\end{enumerate}
\end{example}
This scaling information can be employed as sufficient condition to ensure validity of most blocks,
in the sense of $\chi \rightarrow 0$:
\begin{lemma}[Relative Asymptotic Block-Size]
	\label{lemma:time_regimes_limit_chi_helper_lemma}
	If regimes grow at least $\Lambda$-uniformly eventually for a length-scale
	$\Lambda : \mathbb{N} \rightarrow \mathbb{R}$ and
	$\frac{B(N)}{\Lambda(N)}\rightarrow 0$
	then, in probability
	\begin{equation*}
		\frac{B(N)}{\ell(N)} \xrightarrow{\txt{P}} 0
		\txt.
	\end{equation*}
\end{lemma}
\begin{proof}
	Let $\epsilon > 0$ be arbitrary.	
	By definition (Def.\ \ref{def:limits_temporal}), there are $N_0 \in \mathbb{N}$, $C$ with 
	$\Pr(\frac{\ell_\pm(N_0)}{\Lambda(N_0)} < C) > 1-\epsilon$.
	With $N < N' \Rightarrow \frac{\ell_\pm(N)}{\Lambda(N)}	\preceq	\frac{\ell_\pm(N')}{\Lambda(N')}$
	(again by Def.\ \ref{def:limits_temporal}),
	this remains true for all $N\geq N_0$.
	
	By $\frac{B(N)}{\Lambda(N)}\rightarrow 0$, there is $N_1$ such that $\forall N\geq N_1$,
	$\frac{B(N)}{\Lambda(N)} < \frac{\epsilon}{C}$,
	thus for $N\geq\max(N_0,N_1)$, with probability at least $1-\epsilon$,
	\begin{equation*}
		\frac{B(N)}{\ell(N)}
		=\frac{B(N)}{\Lambda(N)}\times\frac{\Lambda(N)}{\ell(N)}
		<\frac{\epsilon}{C} \times C = \epsilon
		\txt.
	\end{equation*}
\end{proof}
\begin{lemma}[Sufficient Growth Implies Asymptotic Validity]\label{lemma:time_regimes_limit_chi}
	If regimes grow at least $\Lambda$-uniformly for a length-scale
	$\Lambda : \mathbb{N} \rightarrow \mathbb{R}$ and
	$\frac{B(N)}{\Lambda(N)}\rightarrow 0$
	then
	\begin{equation*}
		\chi(B(N)) \xrightarrow{\txt{P}} 0
		\txt.
	\end{equation*}
\end{lemma}
\begin{proof}
	Let $\epsilon > 0$ be arbitrary.
	There are at most as many invalid blocks as regime-switches $s$,
	that is $\chi(N) \leq \frac{s(N) \times B(N)}{N}$.
	Thus it is enough to show $s \frac{B(N)}{N} \xrightarrow{\txt{P}}0$.
	With $\frac{B(N)}{\Lambda(N)}\rightarrow 0$ by hypothesis and $\Lambda(N) \leq N$ by definition
	(Def.\ \ref{def:limits_temporal}),
	the relative size of blocks $\frac{B(N)}{N}\rightarrow 0$ and
	the number of blocks $\Theta(N) = \lfloor\frac{N}{B(N)}\rfloor \rightarrow \infty$ approaches infinity.
	With $\Theta(N) \rightarrow \infty$, we eventually have
	$s \leq \frac{2 N}{\ell} + \epsilon N$
	with probability at least $1-\onehalf\epsilon$ by Lemma
	\ref{lemma:regime_switches_bounded_via_median}.
	We get $\chi(B(N)) \leq (\frac{2N}{\ell} + \epsilon N) \frac{B(N)}{N}$ with probability
	at least $1-\frac{\epsilon}{2}$.
	Applying Lemma \ref{lemma:time_regimes_limit_chi_helper_lemma},
	we get $\frac{B(N)}{\ell} < \frac{\epsilon}{4}$ with probability
	at least $1-\frac{\epsilon}{4}$ eventually.
	Thus eventually with probability greater $1-\epsilon$
	\begin{equation*}
		\chi(B(N)) \leq \frac{\epsilon}{2} + \epsilon \frac{B(N)}{N}\txt,
	\end{equation*}
	and by $\frac{B(N)}{N}\rightarrow 0$ (see above) eventually
	$\chi(B(N)) < \epsilon$ with probability greater $1-\epsilon$.
\end{proof}
\begin{rmk}\label{rmk:error_type_split}
	Essentially these arguments work, because error-rates $\epsilon$ can be split such that
	with probability $1-\frac{\epsilon}{2}$ the setup is \inquotes{benign} in the sense
	that regimes are at least of typical size $\propto\Lambda$, while
	giving up on the remaining setups (occurring with prior
	probability $\leq \frac{\epsilon}{2}$).
	Then on these benign setups error-rates are uniformly reduced below $\frac{\epsilon}{2}$.
	Combining both error-types leads to an error-rate of less than $\epsilon$
	in a Bayesian sense, but without uniform pointwise guarantees.
\end{rmk}
Uninformative priors in the previous sense have scales growing logarithmically.
\begin{lemma}[Logarithmic Growth of Uninformative Temporal Regimes]
	\label{lemma:uninformative_prior_scaling_temporal}
	Let $\mathcal{R}_N$ be an uninformative prior
	(Def.\ \ref{def:uninformative_prior_temporal}),
	and $\Lambda(N)$ such that $\frac{\Lambda(N)}{\log(N)} \rightarrow 0$,
	then segments of $\mathcal{R}_N$ grow at least $\Lambda$-uniformly eventually.
\end{lemma}
\begin{proof}
	By Def.\ \ref{def:uninformative_prior_temporal},
	$\log(\ell_N) \sim \UniformDist(0, \log(N))$,
	in particular Def.\ \ref{def:limits_temporal} is satisfied
	for $\Lambda$ satisfying the hypothesis.
\end{proof}

\subsection{General Pattern Applicability}\label{apdx:models_general_patterns}

For more general patterns, like multi-dimensional spatio-temporal ones, manifolds beyond the $\mathbb{R}^n$,
periodicity or neighborhoods in graphs\Slash{}networks the \inquotes{applicability} of a pattern is
hard to capture geometrically (cf.\ typical lengths $\ell$ above).
However, the previous subsection showed, how temporal persistence as a simple type of pattern is compatible
with certain limits and uninformative model-priors in the sense of ensuring useful behaviors of
invalid block rates $\chi$ in the asymptotic limit (and similarly for valid independent
blocks $A\rightarrow a$ by Lemma \ref{lemma:relate_A_to_chi}).
This information will turn out to be sufficient to assess the applicability of mCIT tests
(§\ref{apdx:data_processing_layer}).
It therefore makes sense to elevate these properties to the status of a definition for the general case.

\begin{definition}[Pattern Compatibility]\label{def:pattern_compatible}
	Let $\Lambda : \mathbb{N} \rightarrow \mathbb{R}$ be a scale (see Def. \ref{def:limits_temporal}).
	A pattern $\Pattern$ is asymptotically $\Lambda$ compatible
	with an indicator meta-model $\mathcal{R}$ if
	for all $B(N)$ with $\frac{B(N)}{\Lambda(N)} \rightarrow 0$, for $N \rightarrow \infty$
	\begin{equation*}
		\chi_{\mathcal{R}}(\Pattern, B(N)) \xrightarrow{\txt{P}} 0 
		\txt.
	\end{equation*}
\end{definition}
\begin{example}\label{example:persistent_pattern_is_compatible}
	A time-aligned pattern (example \ref{example:persistent_in_time_pattern}) is asymptotically $\Lambda$ compatible with
	any meta-model whose regimes grow at least $\Lambda$-uniformly in the sense of Def.\ \ref{def:limits_temporal}.
	This result was formally shown as Lemma \ref{lemma:time_regimes_limit_chi}.
	
	A persistent-in-time pattern is asymptotically $\Lambda$ compatible with
	the uninformative prior of Def.\ \ref{def:uninformative_prior_temporal},
	if $\frac{\Lambda(N)}{\log(N)} \rightarrow 0$.
	This follows from Lemma \ref{lemma:uninformative_prior_scaling_temporal}
	via Lemma \ref{lemma:time_regimes_limit_chi}.
\end{example}

\subsection{Data Generation for Numerical Experiments}\label{apdx:model_joint}

Here we summarize how our data-generation integrates the indicator meta-models
discussed above with random graphs and mechanism\Slash{}noises into
non-stationary SCMs in the sense of Def.\ \ref{def:model_nonstationary_scm}.

\subsubsection{Causal Graph}\label{apdx:data_gen_graph}

In our numerical experiments, we usually first generate a DAG.
In the IID case this DAG will become the union-graph.
In the time series case it will become the the union-summary-graph (see §\ref{apdx:data_gen_ts}).

We generate the DAG with causal-order equal to index-order (\ie ancestors are always lower index,
descendants always higher index), a random causal order is generated in the very end
to permute the resulting data. This is equivalent to generating a random causal order
in the beginning but, easier to implement and thus less error-prone.

For each node, up to \texttt{max\Underbar{}parents\Underbar{}per\Underbar{}node} parents are drawn 
(uniformly, see below) from the possible ancestors (nodes of lower index) for each node.
Depending on the requirements of the setup, for example the density of links
\texttt{target\Underbar{}density}
(number of actual links relative to number of possible links in a DAG of
equal node-count, disregarding \texttt{max\Underbar{}parents\Underbar{}per\Underbar{}node})
or an average number of links per node
\texttt{average\Underbar{}links\Underbar{}per\Underbar{}node}
is fixed in advance; this fixes a number
of requested links.
We then analyze how many links \emph{could} at most be included
while still satisfying the requirement of at most \texttt{max\Underbar{}parents\Underbar{}per\Underbar{}node}
parents per node (nodes of low index have less than \texttt{max\Underbar{}parents\Underbar{}per\Underbar{}node} possible ancestors).
From the set of these \inquotes{legal slots} for parents we then draw a random
subset of size equal to the number of requested links (see above).
Finally we fill these selected \inquotes{legal slots} with
respectively valid parents (valid here means \inquotes{of lower index}, ensuring a consistent
causal total ordering, thus a DAG).

This allows to draw DAGs with available slots for parents populated uniformly
(in a reasonable sense), while also providing direct control over
high-level parameters like \texttt{max\Underbar{}parents\Underbar{}per\Underbar{}node}
and \texttt{target\Underbar{}density} or \texttt{average\Underbar{}links\Underbar{}per\Underbar{}node}.

\subsubsection{Changing Links}\label{apdx:data_gen_changing_links}

All our experiments use a fixed in advance number of changing links.
This number is known to the clustering and Regime-PCMCI algorithms (which need
this information), but hidden from our method (which does not need this information,
and in its current form also cannot exploit this information if known).
In the sliding window case, the number of changing links is hidden from the
method at runtime, but with hyperparameters chosen a posteriori (see §\ref{sec:num_exp_methods}),
it is effectively known (at least in part).

This fixed number of links is picked uniform at random among all (non-autolag in the
time series case §\ref{apdx:data_gen_ts}) edges, see also §\ref{apdx:data_gen_graph}.
To these randomly chosen links, a non-trivial indicator is assigned.
Non-trivial indicators are generated as described in detail in example \ref{datagen:indicator_metamodel}:
For each changing link individually, a typical length-scale $\ell$ between
the fixed parameters $\ell_{\text{min}}$ and $\ell_{\text{max}}$ is sampled log-uniformly.
Then the random-segment-lengths for this indicator are drawn at random but on this scale $\ell$
(see example \ref{datagen:indicator_metamodel}). The initial (starting at $t=0$)
segment-length is multiplied by a uniform at random number $\sim \UniformDist(0,1)$
(as this segment may have started in the past).
Indicators on different links are drawn independently.
We employ a sanity-check that tests if the drawn indicator happens to be constant by coincidence.
If this (or any other sanity-check) fails, data-generation is restarted (this happens extremely
rarely; so while this happens more often for very large $\ell$,
it does not measurably affect the distribution of $\ell$).
We modify resulting segment-lengths according to a regime-fraction target $a$,
by multiplying lengths of independent segments by $2 a$, dependent segments by $2 (1-a)$
(for example if $a=\frac{1}{2}$, all segment lengths remain unchanged).
The true regime-fraction $A$ may not be the same as the target $a$ by randomness of segment-lengths.
$a$ itself is randomized per indicator, typically (approximately) $a\sim \UniformDist(0.2, 0.8)$
for the sub-system experiments in §\ref{apdx:data_processing_layer}
and $a\sim \UniformDist(0.3, 0.7)$ for integrated experiments in §\ref{sec:num_experiments} and
§\ref{apdx:num_experiments}.

Changing links are taken into account by the actual data-generating SCM
as specified by Def.\ \ref{def:model_nonstationary_scm}.
Currently our data-generation starts all regimes in the dependent states,
but other conventions do not seem to change the observed behavior.

To ensure stable applicability of all methods, the minimum typical
\textbf{regime-length is picked large} $\ell_{\text{min}}=200$ for time series experiments.
The exception to this rule is Fig.\ \ref{fig:dcd_ts_regime_sizes} where on the x-axis different
typical regime-sizes are shown.
To make results immediately comparable, this convention also applies to the IID-case, again
a detailed analysis of smaller typical regime-sizes is given in §\ref{apdx:num_exp_regime_size}.

\subsubsection{Time Series}\label{apdx:data_gen_ts}

We use a DAG as union-summary-graph.
Generally the summary graph would not have to be acyclic. Even with the
current implementation of the state-space construction §\ref{sec:indicator_translation},
which requires an acyclic union-graph, only the \emph{window}-union-graph
as to be acyclic.
We opted to run time series experiments with DAGs for summary-graphs,
because otherwise (in the presence of time-resolved feedback-loops) the
models tend to become less stable. This is in principle an interesting problem
in itself, but not the primary subject of study here. Using union-summary-DAGs
avoids the mingling of regime-dependent behavior (the primary interest of this paper)
and such stability problems.

The union-summary-graph is supplemented by additional auto-lag edges:
In the shown experiments, each variable has a (lag $1$) auto-dependence with
probability $0.5$.

Our time series experiments use $\tau_{\text{min}}=\tau_{\text{max}}=1$.
The value $\tau_{\text{min}}\geq 1$ is required for comparison with Regime-PCMCI
\citep{Saggioro2020}. The value of $\tau_{\text{max}}$ could be varied, the resulting
effect is very similar to increasing the number of nodes (in steps of the number of variables)
while simultaneously changing link-densities and (effective) noise distributions.
We instead included a detailed study of the dependence on
node-counts (§\ref{apdx:cd_iid_node_count}, in steps of $1$),
noises (§\ref{apdx:noise_types}) and link-densities(§\ref{apdx:num_exp_link_density})
individually, which allows us to analyze the individual behaviors in a decoupled and transparent way.

Time series data is generated with a burn-in period of $150$ time-steps with the model
residing (for these $150$ time-steps) in the initial state (where data-generation proper will begin).
The used models do not feature feed-back loops (see above) or very large auto-lags,
thus they approach the stationary distributions (of the current regime) comparatively quickly.
After $150$ time-steps the initial time-step is effective sampled from the 
stationary distribution of the initial regime.
All typical time-scales $\ell$ (see §\ref{apdx:data_gen_changing_links})
should similarly be large compared to the typical approach to stationarity within each regime,
so the choice of picking the initial data point (via burn-in) from a stationary distribution is
reasonable.

\subsubsection{Noises}

Most of our experiments use standard-normal noises for simplicity.
The effect of different noises is studied in detail in §\ref{apdx:noise_types}.

\subsubsection{Mechanisms}\label{apdx:datagen_mechanisms}

We focus on linear models. Coefficients are drawn log-uniform,
typically between $0.2$ and $5.0$ for the sub-system experiments in §\ref{apdx:data_processing_layer},
between $0.5$ and $1.0$  for integrated experiments in §\ref{sec:num_experiments} on 
time series (to avoid instabilities; auto-lags are drawn as all other effects) and for
§\ref{apdx:num_experiments} (even though in §\ref{apdx:noise_types} the scale of noises
is varied, which for IID-models effectively also explores different effect-scales).
The sign of each effect is drawn by a fair coin-flip.

For evaluating mCITs on weak regimes in §\ref{apdx:data_processing_layer},
the independent regime is replaced by a weak regime with
effect-strength $\sim\UniformDist(0.1,0.9)$ times the effect strength
in the dependent regime, with independently randomized sign.
Conditional test setups in this case
inject a random number of mediators and confounders summing up to a fixed
total number $\setelemcount{Z}$. Noise on the confounders and mediators
is normal with $\sigma$ log-uniform between $0.1$ and $5.0$.
Effects involved with these mediators and confounders are drawn as described above
with one exception: For mediators, the correlation with the source $X$ is bounded
to $0.75$ (if it is higher than this value, the variance of the noise
on the mediator is increased to reduce the correlation to $0.75$).
The reason for doing so is that if the mediator $M$ is too good a
proxy for the value of $X$ the independence-testing becomes extremely noisy
(this is not just a problem of our mCIT, it happens for CITs in general).
This happens reasonably rarely, and our setup still seems to be
very challenging and informative.

\subsubsection{Integrated Scenarios}\label{apdx:data_gen_integrated}

\NewDocumentCommand{\colheader}{O{60} O{2em} m}{\makebox[#2][l]{\rotatebox{#1}{#3}}}

\begin{table}[ht]
	\small
	\begin{tabular}{l|ccccccc}
		& \colheader{Discussion}
		& \colheader{Sample-Size $N$}
		& \colheader{Node Count}
		& \colheader{Indicator Count}
		& \colheader{Link-Density}
		& \colheader{Run Count}
		& \colheader{Runs (R-PCMCI)}
		\\\toprule
		PC$_1$
		& §\ref{apdx:MCI}
		& \dots
		& 5
		& 1
		& 0.5
		& 2066
		& --
		\\\midrule
		Sample Size
		& Fig.\ \ref{fig:dcd_ts_scaling}
		& \dots
		& 5
		& 1
		& 0.5
		& 543
		& 233, 294, 122
		\\\midrule
		Node Count
		& Fig.\ \ref{fig:dcd_ts_complexity}
		& 1000
		& \dots
		& 1
		& 0.3
		& 1059
		& 240
		\\\midrule
		Indicator Count
		& Fig.\ \ref{fig:dcd_ts_regime_complexity}
		& 1000
		& 8
		& \dots
		& 0.4
		& 976\Slash{}1687
		& 198\Slash{}237
		\\\midrule
		Regime Sizes
		& Fig.\ \ref{fig:dcd_ts_regime_sizes}
		& 1000
		& 8
		& 1
		& 0.4
		& 1118
		& 159
		\\\midrule\midrule
		Sample Size (IID)
		& §\ref{apdx:iid_sample_size}
		& \dots
		& 5
		& 1
		& 0.5
		& 958
		& --
		\\\midrule
		Node Count (IID)
		& §\ref{apdx:cd_iid_node_count}
		& \makecell{1000\\5000}
		& \dots
		& 3
		& \makecell{average:\\$|\Pa|=2$}
		& \makecell{120\\60}
		& --
		\\\midrule
		Link Count (IID)
		& §\ref{apdx:num_exp_link_density}
		& 1000
		& 8
		& 1
		& \dots
		& 394
		& --
		\\\midrule
		Regime Sizes (IID)
		& §\ref{apdx:num_exp_regime_size}
		& 1000
		& 5
		& 1
		& 0.5
		& 1019
		& --
		\\\midrule
		Noises (IID)
		& §\ref{apdx:noise_types}
		& 1000
		& 5
		& 1
		& 0.5
		& 1424
		& --
		\\\bottomrule
	\end{tabular}
	\caption{Individual parameters for integrated CD-experiments.
		Shared parameters are listed in the text proper; sub-system level experiments use only shared parameters (with the 	exception of implication
		test demonstrators, which use larger typical regime-sizes, cf.\ 
		Fig.\ \ref{fig:implcation_testing_methods} and \ref{fig:implcation_testing_behavior}.).
		Parameters that are actively varied in each experiment (shown as x-axis in corresponding
		plots) are indicated by \inquotes{\dots} in the respective column.
		Run counts for Regime-PCMCI are total runs not successful (converged) runs,
		convergence rate ranges from few percent (multiple indicators) to
		80--90\% depending also on the hyperparameter sets, see error-bars.
		For sample-sizes, run-counts of Regime-PCMCI depend on $N$,
		for indicator-count, the run-counts are given as
		independent indicators\Slash{}synchronized indicators.
		}
	\label{tab:data_gen_params}
\end{table}

\begin{table}[ht]
	\small
	\begin{tabular}{l|cccccccccc}
		& \colheader{Test Order}
		& \colheader{Homogeneity Null}
		& \colheader{Homogeneity Hyper}
		& \colheader{Homogeneity Cond.}
		& \colheader{Homogeneity Vali.}
		& \colheader{Weak Hyper}
		& \colheader{Weak Conditional}
		& \colheader{Indicator Relat.}
		& \colheader{mCIT}
		& \colheader{mCIT Cond.}
		\\\toprule
		Discussion
		& F.\ \ref{fig:global_dependence}
		& F.\ \ref{fig:homogeneity_overview}
		& \ref{apdx:homogeneity_hyper_heuristic}
		& \ref{apdx:homogeneity_hyper_heuristic}
		& \ref{apdx:homogeneity_hyper_heuristic}
		& \ref{apdx:weak_hyper_choices}
		& \ref{apdx:weak_hyper_choices}
		& \ref{apdx:implication_testing}
		& \ref{apdx:full_mCIT}
		& \ref{apdx:full_mCIT}
		\\
		Run Count
		& 1405
		& 4801
		& 2876
		& 726
		& 3997
		& 2742
		& 778
		& 4470
		& 16496
		& 1233
	\end{tabular}
	\caption{
		Actual number of runs executed for the generation of individual benchmark
		databases, see text. These are runs per pixel and per scenario.
	}
	\label{tab:data_gen_run_counts_mCIT}
\end{table}

With the different numerical experiments presented in §\ref{sec:num_experiments} and
§\ref{apdx:num_experiments}, we seek to explore across various
data-generation parameters explicitly while keeping the other ones fixed.	
The used system configurations are summarized in Tab.\ \ref{tab:data_gen_params}.

We did not enforce a target run count for experiments, instead we allocated
a fixed total runtime for each experiment, potentially appending runtime
if error-bars or noise-levels were unsatisfactory.
This is easier to implement and much more effective to parallelize in practice.
It also avoids the suggestion of a controlled error level and
(potentially wrong) expectations that a fixed run count may induce.
All our results are presented either as one or multiple curves with error-bars or
as one or multiple 2D images. This allows for an assessment of the stability of results
by comparison to neighboring points and pixels. Such inspection of noisiness of
shown results seems a much more informative and reliable means of assessing results
than an overemphasis of run-counts; nevertheless run-counts provide additional
information and are summarized in Tab.\ \ref{tab:data_gen_params}, \ref{tab:data_gen_run_counts_mCIT}.

Error-bars were generated by starting from binomial counting-statistics
(based on observed rates), then using normal approximations and Gaussian
error-propagation. The approximations involved seem to be reasonably accurate
for these purposes.
For run-times, error-bars indicate $10\%$--$90\%$ quantiles.
Run-times (curves) are median values for CD, and mean-values for mCITs;
for mCITs, there is little difference between mean and median,
for CD -- see also §\ref{apdx:cd_iid_node_count} --
mean-values may not reasonably represent runtimes.

\subsection{Non-Additive Models}\label{apdx:models_non_additive}
This subsection generalizes the construction in §\ref{sec:model_defs} to non-additive models.
In the general (non-additive) case, constness can be modeled by replacing
a parent as input by an exogenous indicator as follows:
\begin{definition}[Non-Stationary General SCM]\label{def:model_nonstationary_scm_nonadd}
	Given variables $\{X_i\}_{i\in I}$ with $\setelemcount{I} = K$, (possibly trivial)
	non-additive mappings $\{r_{ij}\}_{i,j \in I}: T \rightarrow \mathbb{R} \sqcup \{*\}$
	(formally add a disjoint special point $*$),
	and time-independent
	structural mappings $f_{j} : \mathbb{R}^{K+1} \rightarrow \mathbb{R}$.
	Define the shorthand
	\begin{equation*}
		\tilde{X}^t_{ij}
		:=
		\begin{cases}
			X^t_i & \txt{ if } r_{ij}(t)=* \\
			r_{ij}(t) & \txt{ if } r_{ij}(t)\neq*
		\end{cases}
	\end{equation*}		
	The non-stationary SCM $M$ is given on
	$\{X_i\}_{i\in I}$ by
	\begin{equation*}
		X^t_j = f_{j}(\eta^t_j, \tilde{X}^t_{1j}, \ldots, \tilde{X}^t_{Kj})\txt.
	\end{equation*}
\end{definition}
\begin{rmk}
	The mappings $r_{ij}$ taking values in $\mathbb{R} \cup \{*\}$
	induce binary indicators $R_{ij}$ by collapsing the $\mathbb{R}$-component to a point:
	Set $R_{ij}(t) = 1$ if $r_{ij}(t)=*$
	and $R_{ij}(t) = 0$ otherwise.
	From these $R_{ij}$ parents and graphs may be defined as in the additive case.
\end{rmk}	
\begin{rmk}
	To restrict further for example to temporal regimes (discrete change-points with stationary
	behavior also in dependent regions),
	one may assume that the $r_{ij}$ are constant between $*$-values in the sense
	that $\forall t<t'$: $r_{ij}(t) \neq r_{ij}(t')$
	$\Rightarrow$ $\exists t_*\in [t,t']$ \st $r_{ij}(t_*) = *$.
	In this case the regime-structure assumption as described in Ass.\ \ref{ass:segment_structure}
	is restored.
	More generally, to this end the $r_{ij}$ have to be constant on non-trivial segments.
	We assume that for values (other than $*$) taken by the $r_ij$, the (joint)
	parents still satisfy a minimality condition like \citep[Def.\ 2.6]{BongersCyclic}.
\end{rmk}
\begin{example}
	Let $T=[0,100)\cap\mathbb{Z}$ and
	start from the stationary SCM with $Y = f(X, Z, \eta_Y)$ using
	\begin{equation*}
		r_{XY}(t) =
		\begin{cases}
			* & \text{ on } [20, 40) \cup [60,80)\text,\\
			a & \text{ on } [0,20)\text,\\
			b & \text{ on } [40,60)\text,\\
			c & \text{ on } [80,100)\text.
		\end{cases}
	\end{equation*}
	Then the $\tilde{X}^t$ defined above is (here $X$ is short for $X_i$, and $Y=X_j$,
	we write $\tilde{X}^t_Y$ for $\tilde{X}^t_{ij}$)
	\begin{equation*}
		\tilde{X}^t_Y
		:=
		\begin{cases}
			X^t & \text{ for } t\in [20, 40) \cup [60,80)\text,\\
			a & \text{ for } t\in [0,20)\text,\\
			b & \text{ for } t\in [40,60)\text,\\
			c & \text{ for } t\in[80,100)\text.
		\end{cases}
	\end{equation*}
	an thus we obtain $Y^t$ being defined as:
	\begin{equation*}
		Y^t =
		\begin{cases}
			f(X^t,Z, \eta_Y) & \text{ for } t\in [20, 40) \cup [60,80)\text,\\
			f(a,Z, \eta_Y) & \text{ for } t\in [0,20)\text,\\
			f(b,Z, \eta_Y) & \text{ for } t\in [40,60)\text,\\
			f(c,Z, \eta_Y) & \text{ for } t\in [80,100)\text.
		\end{cases}
	\end{equation*}
	Note, that in the non-additive case,
	$f(a,\cdot,\cdot) - f(b,\cdot,\cdot) \not\equiv \const$,
	\ie the functional behavior in different (independent) regimes
	can differ. However $X^t \independent Y^t \Leftrightarrow t\notin [20, 40) \cup [60,80)$
	only depends on the induced binary indicator $R(t) = 1 \Leftrightarrow t\in [20, 40) \cup [60,80)$.
	
	Think for example of a mechanical system, that during our observations gets stuck thrice, in positions
	$a$, $b$ and $c$, while under standard operating conditions, it follows the behavior of $X$
	transmitting some sort of effect onto $Y$.
	More concretely, say $X$ is the water-level of a reservoir,
	measured by a floating indicator $L$ which sometimes
	gets stuck, and a controller that based on measured water-level $L$ and temperature $Z=T$
	non-linearly regulates an irrigation system. Realistically for this simple scenario
	typically $L$ would be observed, rather than $X$, but we could for example
	observe water-levels independently of how they are perceived by the system.
\end{example}

\subsection{Proofs of Claims in the Main Text}\label{apdx:model_detail_proofs}

In §\ref{sec:model_states} we stated that the assignment of time-indices to graphs $G: T \rightarrow \mathcal{G}$
factors through the space of states $S$.
\begin{proof}[Proof of Lemma \ref{lemma:graphs_factor_through_states}]
	The graph $G(t)$ by definition depends only on the values of the indicators $R_{ij}(t)$ at $t$.
	By definition \ref{def:states} a state $s$ is a value taken by $s:=\sigma(t)=(R_1(t), \ldots, R_{\kappa}(t))$,
	where $R_1, \ldots, R_{\kappa}$ are precisely the (re-indexed) non-trivial indicators.
	Thus the value of $s$ fixes all non-trivial indicator-values. All trivial indicator-values
	are fixed in the (meta-) model, so given the model $M$ and $s=\sigma(t)$
	all $R_{ij}(t)$ and thus $G(t)$ are fixed.
\end{proof}

\FloatBarrier

\section{Details on the Data Processing Layer}\label{apdx:data_processing_layer}

In this section, we study in detail the
methods realizing the individual steps of mCIT-testing as outlined
in the main text §\ref{sec:dyn_indep_tests} up to and including
in their combination into an mCIT.
We put particular emphasis on the existence and understanding of
fundamental limitations to the realization of mCITs and potential
ideas and interpretations that allow to retain a meaningful method
of testing, as already in part laid out in §\ref{apdx:models_and_testing}.

An important role in maintaining a meaningful test-interpretation
is played by assumptions and the interplay of trade-offs
and hyperparameters. Since the choice of reasonable hyperparameters
is also very much of practical concern, we also use detailed numerical
experiments to validate our theoretical understanding of behaviors
under different hyperparameters and to inform semi-heuristic schemes
to choose them in robust ways.

We start by discussing persistence (or pattern-compatibility more generally)
and its relevance to the identifiability of the mCIT result.
Then we discuss some simplifying assumptions needed in particular for
keeping results on statistical power of our tests accessible.
Next, details on the homogeneity test and the weak-regime test are given,
and both are combined into a full mCIT.
Finally, we give an overview of some relevant directions for future work.

As explained in Rmk.\ \ref{rmk:relation_to_general_cits}, it makes
sense to study the mCIT problem as formulated by Def.\ \ref{def:marked_independence_rel_d}:
This means, we consider for example the interpretations of 
dependent but uncorrelated regimes for the use of correlation-based tests
to be an independence testing (not mCIT-specific) problem.
We focus on the mCIT-specific problems.

\subsection{Persistence}

Persistence, or more generally compatibility with patterns, is
of core importance to our approach.
We first give a brief illustration of the limits of identifiability that make additional
assumptions necessary. We also demonstrate how and why persistence can serve
as such an assumption. This further helps to shed light on the relative difficulty of individual
sub-tasks, in particular it aids in understanding why the weak-regime test becomes
the bottle-neck for finite-sample performance in numerical experiments -- this is not
(only) the result of a rather primitive method, but actually deeply rooted
in a second impossibility-result.
Finally, we briefly explain which considerations lead us to the use of blocks
of data, rather than a more sophisticated subdivision of data, for example via CPD.

\subsubsection{Formal Necessity}\label{apdx:necessity_persistence}

As we demonstrate in this subsection, there are
fundamental limitations to the identifiability of regime-structures;
there are cases, where gaining power (even asymptotically) is not possible.
This problem is closely related to limitations known on randomness\Slash{}IIDness tests
\citep{WaldWolfowitzRuntest, RandomnessTestsOnline}.
It necessitates the incorporation of additional assumptions
like persistence of regimes into the formal approach.

\begin{rmk}
	The necessity of an assumption does of course not imply that this assumption 
	has to be about a non-IID structure (like persistence).
	Indeed if one insists on remaining in the IID case, the existence of a mixture is identifiable if and only if
	a suitable linear independence of the involved distribution-families
	is satisfied \citep{yakowitz1968identifiability}.
	See also \citep{BalsellsRodas2023} of an application of such parametric assumptions to
	a HCCD-like problem.
\end{rmk}

To demonstrate the problem, we start from a (potential) mixture $P = \lambda P_0 + (1-\lambda) P_1$
of two distributions $P_0$ and $P_1$, where one (say $P_0$) is known.
We want to decide, if $\lambda = 0$; we show that this question is not identifiable in general.
For simplicity we assume that densities $p_0$ and $p$ (of the potential mixture) exist.
Then the following holds:
\begin{lemma}[Representation by Mixture]
	Define:
	\begin{equation*}
		\Xi_0(p,p_0) := \sup \{ \beta \in [0,1) | \forall \vec{x}:
		\beta p_0(\vec{x}) \leq p(\vec{x}) \}
	\end{equation*}
	Given $P$, $P_0$ with densities $p$, $p_0$, then
	there exists a density $p_1$ and a $\lambda \geq \Xi_0(p,p_0)$
	with $p = \lambda p_0 + (1-\lambda) p_1$.
\end{lemma}
\begin{proof}
	Since we do not know $p_1$, we can always write $p = \Xi p_0 + (p- \Xi p_0)$.
	From $\int p = 1$ and $\int \Xi p_0 = \Xi$ together with $p>\Xi p_0$
	we have $\mathcal{N} := \int (p- \Xi p_0) = 1 - \Xi$.
	So after normalizing $p_1' := \mathcal{N}^{-1} (p- \Xi p_0)$ is a probability density
	(again by $p- \Xi p_0 \geq 0$).
	We can, using $\lambda = \Xi$, write
	$p = \lambda p_0 + (1-\lambda) p_1'$ as a mixture
	of $p_0$ and $p_1'$ with $\lambda = \Xi$.
\end{proof}
This is bad: We had not yet assumed $p$ \emph{is} a mixture, rather this always works,
irrespective of whether or not $p$ was in any sense generated by a \inquotes{true mixture}.
For the case of relevance here, we do not know $P_0$, only its factorization-property
(as an independence null-distribution);
knowing less will of course not improve identifiability:
\begin{lemma}[Representation by Mixture with Independent Component]
	\label{lemma:apdx_persistence_general_density}
	Given a joint distribution $P(X,Y)$ with density $p$ and
	\begin{equation*}
		\Xi = \sup_{p_0(X,Y) = p_0^X(X) \times p_0^Y(Y)}
		\sup \{ \beta \in [0,1) | \forall x,y:
		\beta p_0(x,y) \leq p(x,y) \}
	\end{equation*}
	then
	$\forall \epsilon > 0$
	there exist densities $p_0^X(x), p_0^Y(y), p_1(x,y)$
	and a $\lambda \geq (1-\epsilon) \Xi$
	with $p(x,y) = \lambda p_0^X(x)\times p_0^Y(y) + (1-\lambda) p_1(x,y)$.
\end{lemma}
\begin{proof}
	Let $\epsilon > 0$ be arbitrary, \asswlog $\epsilon < 1$.
	There are (by definition the $\Xi$)
	$p^X_{0,\epsilon}(x), p^Y_{0,\epsilon}(y)$ with
	$(1-\epsilon) \Xi (p^X_{0,\epsilon}(x) \times p^Y_{0,\epsilon}(y)) \leq p(x,y)$.
	
	Thus the $\Xi_0$ of the previous lemma satisfies, using $p_{0,\epsilon}:=p_0^X(x)\times p_0^Y(y)$,
	that $\Xi_0(p,p_{0,\epsilon}) \geq (1-\epsilon) \Xi$.
	If we apply the previous lemma to $p$ and $p_{0,\epsilon}=p_0^X(x)\times p_0^Y(y)$,
	we therefore find a $p_{1,\epsilon}$ with
	$p = \lambda_\epsilon p_{0,\epsilon} + (1-\lambda_\epsilon) p_{1,\epsilon}$ for
	$\lambda \geq \Xi_0(p,p_{0,\epsilon}) \geq (1-\epsilon) \Xi$.
\end{proof}
The value of $\Xi$ can be large (close to $1$) even for seemingly harmless examples,
in particular counter-examples \emph{do exist}:
\begin{example}\label{example:bivar_normal_require_persistence}
	Let $P$ be a bivariate normal distribution with means $\mu_X$ and $\mu_Y$
	and variances $\sigma_X^2$, $\sigma_Y^2$
	and correlation coefficient $\rho$.
	Pick $\sigma_-^2$
	as the smaller eigenvalue of the covariance-matrix $\Sigma$
	(for example $1-|\rho|$ if $\sigma_X=\sigma_Y=1$).
	Then by construction
	$\vec{x}^T \Sigma^{-1} \vec{x} \leq \frac{\norm{\vec{x}}^2}{\sigma_-^2}$.
	Define $p_0(x,y)$ as the product of normal distributions with means $\mu_X$ and $\mu_Y$
	and standard-deviations $\sigma_X' = \sigma_Y' = \sigma_-$. Then this $p_0$ is such that
	the exponential term in the density contains $\frac{\norm{\vec{x}}^2}{\sigma_-^2}$
	where $p$ contains $\vec{x}^T \Sigma^{-1} \vec{x}$ so that it falls off faster when moving
	away from the mean values; thus 
	$\beta p_0 \leq p$ iff $\beta p_0(\mu_X,\mu_Y) \leq p(\mu_X,\mu_Y)$ and
	we find from the normalizations of these densities that this inequality is
	satisfied for $\beta = \frac{\sigma_-^2}{\sigma_X\sigma_Y\sqrt{1-\rho^2}}$.
	\Eg if $\sigma_X=\sigma_Y=1$ this is $\beta= \sqrt{\frac{1-|\rho|}{1+|\rho|}}$
	(by writing $1-\rho^2=(1-\rho)(1+\rho)$ in the denominator and case-distinction for $\rho\leq0$).
	Finally $\Xi \geq \beta$, because $\Xi$ was defined as the supremum of
	all values satisfying this inequality.
	For example, if $\sigma_X=\sigma_Y=1$, then $\Xi \geq \sqrt{\frac{1-|\rho|}{1+|\rho|}}$;
	it should be noted in particular that this approaches $1$ for $|\rho| \rightarrow 0$.
\end{example}
\theoremstyle{plain}
\newtheorem*{prop_imposs_a}
{Proposition \ref{prop:imposs_A}}
\begin{prop_imposs_a}[Impossibility Result A]
	Given data as a set (without non-IID structure on the index-set),
	then the marked independence statement is not in general identifiable.
	By Lemma \ref{lemma:necessity_of_dyn_indep_struct}
	(see also Rmk.\ \ref{lemma:necessity_of_dyn_indep_struct})
	also the HCCD-problem is not in general identifiable.
\end{prop_imposs_a}
\begin{proof}
	By Lemma \ref{lemma:apdx_persistence_general_density} the same (law of)
	the \emph{set} of data obtained as a random element can be obtained
	from a mixture and a non-mixture for $a < \Xi$.
	In particular the distinction is not identifiable from knowledge of the
	(law of) the random set alone in such cases.
	By example \ref{example:bivar_normal_require_persistence}, cases with $0<a < \Xi$ exist.
\end{proof}
\begin{cor}[Non-Testability A]
	Given a joint distribution $P(X,Y)$ with density $p$ and
	\begin{equation*}
		\Xi = \sup_{p_0(X,Y) = p_0^X(X) \times p_0^Y(Y)}
		\sup \{ \beta \in [0,1) | \forall x,y:
		\beta p_0(x,y) \leq p(x,y) \}
	\end{equation*}
	then there is no hypothesis-test that assumes the data to be IID
	(considers only \emph{sets} of data points without non-IID structure)
	distinguishing between true regimes and global dependence that, for
	regime-fractions $\lambda$ with $\lambda < \Xi$, controls either error-rate (FPR\Slash{}FNR) at
	$\leq\alpha$ and has asymptotically
	non-trivial power (\ie which would have TPR\Slash{}TNR $> \alpha$ for any sample-size $N$).
\end{cor}
\begin{proof}
	Let $\hat{T}$ be an arbitrary test that given $N$ data points $D=\big\{(x_1,y_1), \ldots, (x_N, y_N)\big\}$
	outputs a value $\hat{T}(D)\in\{0,1\}$ where \eg $0$ means true regime and $1$ means global dependence
	(the choice of $0$ vs.\ $1$ is arbitrary).
	Let $p(x,y)$ be given, and consider $N$ data points $(x_1,y_1), \ldots, (x_N, y_N)$ drawn IID from $p(x,y)$.
	This is a value of the random element $D_p^N$.
	Then $\hat{T}(D_p^N)$ is a random variable taking values in $\{0,1\}$, in particular
	(as any random binary variable) it is described by a single value $q_p^N\in[0,1]$,
	by $P(\hat{T}(D_p^N) = 1) = q_p^N$ and $P(\hat{T}(D_p^N) = 0) = 1 - q_p^N$.
	
	\emph{Controlling FPR:}
	Let $\lambda>0$ with $\lambda < \Xi$ be arbitrary.
	Set $\epsilon = \frac{\Xi - \lambda}{2} > 0$.
	By the previous lemma,
	we can write $p$ as $p = p_{\text{mix}} = \lambda p_{0,\epsilon} + (1-\lambda) p_{1,\epsilon}$.
	To control the rate of false positives, \ie outcomes where $\hat{T}(D_p^N) = 1$ (where $T$ claims there is no regime)
	on the ground-truth scenario described by the mixture $p_{\text{mix}}$,
	necessarily $P(\hat{T}(D_{p_{\text{mix}}}^N) = 1)\leq \alpha$.
	But without further knowledge about $\lambda$ (like a persistence assumption on $\lambda$)
	-- it is important here that by hypothesis $\hat{T}(D)$ depends \emph{only} on the set $D$ without
	employing any additional non-IID-structure -- the \emph{random set} $D$ only depends
	on the density, \ie $D_{p_{\text{mix}}}^N=D_p^N)$ and thus
	$q_p^N = P(\hat{T}(D_p^N) = 1) = P(\hat{T}(D_{p_{\text{mix}}}^N) = 1) \leq \alpha$.
	Similar remarks on the factorization through $D$ apply to all other paragraphs below.
	
	On the other hand consider the (non-mixture) $p_{\text{homog}} = 0 + 1 * p$ ground-truth scenario.
	Here the power to correctly reject the true regime is
	$P(\hat{T}(D_{p_{\text{homog}}}^N) = 1) = P(\hat{T}(D_p^N) = 1) = q_p^N \leq \alpha$.
	That is $\hat{T}$ does not have non-trivial power if controlling FPR on mixtures with  $\lambda < \Xi$
	for any $N$.
	
	\emph{Controlling FNR:}
	Let $\lambda>0$ with $\lambda < \Xi$ be arbitrary.
	Consider the (non-mixture) $p_{\text{homog}} = 0 + 1 * p$ ground-truth scenario.
	To control the rate of false negatives (where $T$ claims there is a regime)
	necessarily $P(\hat{T}(D_{p_{\text{homog}}}^N) = 0)=P(\hat{T}(D_p^N) = 0) = 1-q_p^N \leq \alpha$.
	
	Set $\epsilon = \frac{\Xi - \lambda}{2} > 0$.
	By the previous lemma,
	we can write $p$ as $p = p_{\text{mix}} = \lambda p_{0,\epsilon} + (1-\lambda) p_{1,\epsilon}$.
	The power to correctly classify $p_{\text{mix}}$, \ie outcomes where $\hat{T}(D_p^N) = 0$
	(where $T$ correctly claims there is a regime)
	on the ground-truth scenario described by the mixture $p_{\text{mix}}$,
	is given by $P(\hat{T}(D_{p_{\text{mix}}}^N) = 0)=P(\hat{T}(D_p^N) = 0) = 1-q_p^N \leq \alpha$.	
	That is $\hat{T}$ does not have non-trivial power if controlling FNR on mixtures with  $\lambda < \Xi$
	for any $N$.
\end{proof}

\subsubsection{How Persistence Helps}\label{apdx:persistence_helps_necessity}

To see the effect a persistence-assumption will have,
we again look at the initial Lemma \ref{lemma:apdx_persistence_general_density}, in particular
at the definition
\begin{equation*}
	\Xi_0(p,p_0) := \sup \{ \beta \in [0,1) | \forall \vec{x}:
	\beta p_0(\vec{x}) \leq p(\vec{x}) \}\txt.
\end{equation*}
Next we divide the data into blocks of $k$ data points each.
Those blocks are sampled according to $p^k$ if the data is IID.
Note that (by monotonicity of the root\Slash{}exponentiation by $\sfrac{1}{k}$)
\begin{align*}
	\Xi_0(p^k,p_0^k)
	:&= \sup \{ \beta' \in [0,1) | \forall \vec{x}:
	\beta' p^k_0(\vec{x}) \leq p^k(\vec{x}) \} \\
	&= \sup\{ \beta \in [0,1) | \forall \vec{x}:
	\beta^{\sfrac{1}{k}} p_0(\vec{x}) \leq p(\vec{x}) \}\txt.
\end{align*}
Thus $\Xi' = \Xi_0(p^k,p_0^k) = \Xi_0(p,p_0)^k$, and with $\Xi_0(p,p_0)<1$ further
$\Xi_0(p^k,p_0^k) \xrightarrow{k \rightarrow \infty} 0$.
So have we magically resolved the problem?
Of course the answer is no -- we are now comparing those blocks whose $k$ data points
\emph{all} came from $p_0$ to the rest of the data. For IID data, the fraction
of such blocks (by binomial statistics) is simply $\lambda' = \lambda^k$.
Our impossibility-result really is about the comparison $\lambda \geq \Xi$ which is
equivalent to $\lambda' \geq \Xi'$, so we have gained nothing, \emph{so far}.

The simplest type of persistence here would be to assume that we know that regime-switches occur
only every $k$ data points, so that blocks containing only data points from
$p_0$ occur as the total fraction of data points in $p_0$ which is $\lambda$
(as opposed to $\lambda^k$ for random\Slash{}IID blocks).
Thus in this case $\lambda' = \lambda$, so that with $\Xi' \rightarrow 0$ for $k\rightarrow \infty$
indeed $\lambda' \geq \Xi'$ is always satisfied for $k$ large enough.

In practice regime-switches occur not exclusively at multiples of $k$ like this.
However, if regime-switches occur rarely (compared to $k$), $\lambda \approx \lambda'$ is still a good approximation (see Lemma \ref{lemma:relate_A_to_chi}).
So a persistence-requirement should be understood in the following sense:
To resolve small (true) regime-fractions $\lambda$, we need a small enough $\Xi$ -- which depends also
on the \inquotes{distance} between the mixed distributions (cf.\ example \ref{example:bivar_normal_require_persistence},
where the correlation $\rho$ is a suitable distance-measure) -- and therefore a large enough
ground-truth persistence length $L \gg k_\Xi$ such that $\lambda \approx \lambda'$ still holds.

Finally, inspecting Fig.\ \ref{fig:persistence_illustrate} (middle panel in particular),
the use of blocks (for say $k=2$) can make precise the intuitive idea why this pair of examples
can be distinguished using time-resolution. For any block (of $k=2$ subsequent data points),
one can assess how often both data points are in the same regime, which given many
blocks provides useful information about which of the densities on the left-hand side were active.

The above argument only takes into account \inquotes{fully valid}
(all data points in the same context) blocks. However, the case where all blocks are valid
should be easiest, thus limitations to the identifiability in this case should also
apply to all other cases. More subtle, we only consider a single choice of $L_0$ at once,
but for the case where all blocks are valid the joint distribution over $L_0$ data points
contains all information about the joint over $L_0-1$ etc.\ data points by marginalization.
So up to a suitable persistent length $L_0$ (where still reasonably many blocks are valid)
the additional use of shorter (than $L_0$) blocks should not gain substantial information either
(\inquotes{substantial} in the sense of affecting the impossibility result or general trade-off;
finite sample performance may very plausibly improve by a suitable combination of multiple
block-sizes, see §\ref{apdx:combining_block_sizes}).

\subsubsection{Relative Difficulty of Subtasks}\label{apdx:why_weak_is_harder}

The arguments in §\ref{apdx:necessity_persistence} and §\ref{apdx:persistence_helps_necessity}
concern the \inquotes{full} problem of testing for the existence of true regimes.
Meanwhile our arguments below and in the main text further divide the problem into subtasks
(testing homogeneity and testing for weak regimes).
Thus it is clear that the difficulties do \emph{not} arise
from this reformulation into sub-problems, rather they exist independently of this approach.
However, it mandates the question,
where precisely the difficulties come from in terms of these subtasks.
The main insight discussed in this subsection is that testing homogeneity is
(relatively speaking) easy, while testing for weak regimes is particularly hard.

An easy way to understand why homogeneity\Slash{}IIDness-testing is comparatively easy
is that given \emph{any} (non-zero) persistence in the data,
that is if $\lambda_t \notindependent \lambda_{t+1}$, then subsequent data points
are not independent $(X_t, Y_t) \notindependent (X_{t+1}, Y_{t+1})$ either
(under mild assumptions). % cf. 21.07.
In principle any (non-parametric) independence-test can thus distinguish between homogeneous and
persistent-heterogeneous models.
\begin{rmk}
	Using an independence test directly may not be a good idea. Not only is the
	non-parametric testing-problem challenging on finite data (and in principle \citep{shah2020hardness}),
	but even more
	the \emph{type} of heterogeneity we are interested in for subsequent
	weak-regime\Slash{}true-regime testing is \emph{not} (auto-)lag behavior of the
	underlying model, but larger
	persistent data-regions (cf.\ §\ref{apdx:persistence_helps_necessity}).
	It thus seems to make more sense to allocate statistical power to finding
	such heterogeneity, hence our choice of inductive bias in §\ref{sec:testing_homogeneity}
	is tailored primarily to this alternative.
\end{rmk}

If the fundamental difficulty encountered in §\ref{apdx:necessity_persistence}
does not appear to its full extend in homogeneity-testing, then it must appear when
testing weak-regimes.
This problem becomes more apparent if one considers not just (infinite sample)
identifiability, but also finite-sample questions §\ref{apdx:imposs_finite_sample}.
The relative difficulty of the weak-regime test is heuristically
supported for example by the results presented in Fig.\ \ref{fig:mCIT_validation_pairwise}
and in theory by the statistical power-results in §\ref{apdx:homogeneity_scaling}
and §\ref{apdx:weak_power}.

\subsubsection{Conceptual Finite Sample Challenges}\label{apdx:imposs_finite_sample}

In hypothesis-testing, one is often interested in the ability to control finite-sample
errors in one direction at a finite (and known) rate. This has many advantages, in
particular for the interpretability of results.
Considering the true regime case, in the formulation via a homogeneity and a weak-regime
step, we note the following apparent problem:
Homogeneity can be accepted or rejected (with finite sample control on false positives),
and independence of the weak regime can be accepted or rejected
(with finite sample control on false positives), but the true regime case is the combination
\emph{reject} homogeneity, then \emph{accept} independence.
This combination does not naturally allow for the control of any error-rate in the
finite sample case.

Again, we want to see if this is a problem of our reformulation into sub-tasks
or a genuine issue.
The intuition behind our argument is that usually (for most model-classes)
there are meaningful model-properties $d_0$ measuring dependence
(in the weak or only regime) and $\Delta d$ measuring non-homogeneity.
With finite-sample error control being achieved through p-value bounds on the uncertainty
of some test-statistic. This test-statistic may not involve an intermediate estimate
of $d_0$ and $\Delta d$, but if these properties are well-defined and in a suitable sense \inquotes{sufficiently continuous},
then confidence regions will have non-trivial extension (for example containing the true value in its topological
interior) in the $d_0$--$\Delta d$-plane.
In this case error-control for standard hypotheses works, because the alternatives
(to homogeneity\Slash{}independence) are open (in the sense of topology),
or equivalently because the null hypotheses are closed.
In this intuition, the true regime case corresponds to $d_0 = 0$ and $\Delta d \neq 0$,
which is neither open nor closed, thus is neither suitable as null hypothesis nor as
alternative in the usual sense.

Capturing the above argument more formally requires two types of formal assumptions.
First, there is a statistical aspect related to non-triviality of the problem.
If by our prior knowledge, we already know that one or more of the hypotheses never occur,
\eg if $P(d_0=0,\Delta d=0) = 0$, then the testing problem will become solvable for trivial reasons.
This also means that we have to consider priors on $d_0$ and $\Delta d$ which are
\emph{not} continuous in the $d_0$--$\Delta d$ plane (the point $(0,0)$ corresponding
to global independence must have singular weight and also the axis $\Delta d=0$ of
global\Slash{}non-regime models must have non-zero weight despite being a Lebesgue zero-set).
We also have to account for a different problem with similar consequences which
arises if there are a priori known \inquotes{gaps} in parameters, for example if we
knew that $|d_0|\notin (0,d_{\txt{min}})$, \ie if there is a lower bound for non-zero
values taken by $d_0$, then it might suddenly be possible to control errors (for
independence-testing in this case) in both directions. We are interested in the
generic case and want to exclude this special case.

Second, there is a topological aspect, for example the measure of proximity
in the sense of finite-sample bounds on estimates $\hat{d_0}$ and $\hat{\Delta d}$
(possibly implicit in $\hat{T}$) must be
compatible with the topology of $\mathbb{R}_{\geq 0}^2$.
There are many ways to formalize this, we opt for requiring
$\Pr(\hat{T}=\mCIToutR|d_0,\Delta d)$
to be continuous. It is difficult to judge how plausible this assumption is.
While it may appear somewhat ad hoc at first, this is also \emph{because} it is chosen rather weakly:
For example, one might also ask for data-sets $\vec{X} = (\vec{x}_1, \dots, \vec{x}_N) \in \mathbb{R}^{2N}$
to arise from a continuous (in $\vec{X}$, $d_0$, $\Delta d$) conditional density $p(\vec{X}| d_0, \Delta d)$
and an mCIT $\hat{T}$ deciding based on a continuous (in $\vec{X}$) p-value estimate $p_R(\vec{X})$.
This implies (is stronger than) the above assumption about their \inquotes{composition}.
However, this continuous factorization might for example not apply to our binomial homogeneity test
(where p-values are based on discrete block-counts), while the composition (averaging over realizations)
should, for most constructions of this kind, smooth this problem away, so that the hypothesis
as given in the proposition should still apply.
The main point here is not maximal generality of the given formulation, but a formal illustration
of the problem.

\theoremstyle{plain}
\newtheorem*{prop_imposs_B}
{Proposition \ref{prop:imposs_B}}
\begin{prop_imposs_B}[Impossibility Result B]
	Assume the model-class considered has well-defined model-properties
	$d_0 \in\mathbb{R}_{\geq 0}$ and $\Delta d \in\mathbb{R}_{\geq 0}$
	such that a model $M$ has an independent regime iff $d_0(M) = 0$ and
	is homogeneous iff $\Delta d(M) = 0$.
	Further assume that we are given an mCIT $\hat{T} \in \{0,1,\mCIToutR\}$
	such that $\Pr(\hat{T}=\mCIToutR|d_0,\Delta d)$ is continuous in $d_0$ and $\Delta d$,
	and that there is no a priori known gap in realized parameters $d_0$ and $\Delta d$,
	that is $P(d_0,\Delta d) > 0$ near\footnote{\inquotes{Near \dots} in this case means on an open neighborhood of \dots, as this is a topological requirement.} $d_0 = 0$,
	and there is no a priori known gap in $\Delta d$ on weak regimes,
	that is $P(\Delta d|d_0=0) > 0$ near $\Delta d=0$.
	
	Then, given $\alpha > 0$, no such mCIT can for any finite sample-size $N$
	\begin{enumerate}[label=(\alph*)]
		\item on true regimes uniformly control
		FPR at $<\alpha$ and have non-trivial power $>\alpha$ on any global independence alternative or
		\item on global dependence or the union of global hypotheses control
		FPR at $<\alpha$ and have non-trivial power $>\alpha$ on any true regime alternative.
	\end{enumerate}
\end{prop_imposs_B}
\begin{proof}
	First note, that uniform error-control over models implies uniform error-control over well-defined
	model-properties (where supported, cf.\ gap-hypotheses).
	
	Part (a):
	Let $M$ be an arbitrary global independence alternative (to a true regime),
	that is $d_0(M)=0$ and $\Delta d(M)=0$ by hypothesis.
	By absence of a gap in $\Delta d$ on weak regimes $P(\Delta d|d_0=0) > 0$ near $\Delta d=0$,
	uniform (over $d_0, \Delta d$, see initial comment of this proof) $\alpha$-error control
	requires $\Pr(\hat{T}=\mCIToutR|d_0=0, \Delta d) \geq 1-\alpha$ for all $\Delta d > 0$ near $\Delta d=0$.
	By continuity of $\Pr(\hat{T}=\mCIToutR|d_0, \Delta d)$ (and for example the sequence criterion), thus
	$\Pr(\hat{T}=\mCIToutR|d_0=0, \Delta d=0) \geq 1-\alpha$.
	This implies $\Pr(\hat{T}=0|d_0=0, \Delta d=0) \leq \alpha$ and by $d_0(M)=0$ and $\Delta d(M)=0$
	and uniformity of error-control (over models)
	$\Pr(\hat{T}=0|M) \leq \alpha$. That is $\hat{T}$ does not have non-trivial power on $M$.
	Since the model $M$ was arbitrary (among global independence alternatives),
	this shows part (a).
	
	Part (b):
	Uniform error control over the union of global hypotheses in particular implies
	uniform error control over the (subset of) global dependence models, so it suffices to
	proof the claim for global dependence models.
	Let $M$ be an arbitrary true regime alternative (to a globally dependent scenario),
	that is $d_0(M)=0$ and $\Delta d(M)\neq0$ by hypothesis.
	By absence of a gap in parameters $P(\Delta d,d_0) > 0$ near $\Delta d=0$,
	uniform (over $d_0, \Delta d$, see initial comment of this proof) $\alpha$-error control
	requires $\Pr(\hat{T}=1|d_0, \Delta d) \geq 1-\alpha$ for all $\Delta d, d_0$ near $\Delta d=0$.
	This implies $\Pr(\hat{T}=\mCIToutR|d_0, \Delta d) \leq \alpha$ for all $\Delta d, d_0$ near $\Delta d=0$.
	By continuity of $\Pr(\hat{T}=\mCIToutR|d_0, \Delta d)$ (and for example the sequence criterion), thus
	$\Pr(\hat{T}=\mCIToutR|d_0=0, \Delta d=\Delta d(M)) \leq \alpha$.
	By $d_0(M)=0$
	and uniformity of error-control (over models)
	$\Pr(\hat{T}=\mCIToutR|M) \leq \alpha$. That is $\hat{T}$ does not have non-trivial power on $M$.
	Since the model $M$ was arbitrary (among true regime alternatives),
	this shows part (b).
\end{proof}

\subsubsection{Rationale Behind the Use of Blocks}\label{apdx:persistence_by_blocks}

There are multiple reasons why we subsequently specialize to the simple block-based approach:
It is very simple to analyze theoretically and very fast in practice.
However, it is also surprisingly effective for the type of applications we have in mind.
These applications are larger data-sets, with at least a few thousand data points,
as we believe that in practice regime-structure -- as any multi-scale structure --
typically is interesting in particular for larger datasets.

In this case, the effectiveness of the simple block-based approach can be understood as follows.
As illustrated in §\ref{sec:univar_problem},
the aggregation of data from the full index-set $I$ to the final result about
marked independence is itself performed in two steps: Use persistence to aggregate
similar data (\eg in blocks), then use the statistical distribution over blocks
to \inquotes{directly} (without explicit time-resolution of indicators) reason about
marked independence. Especially for larger datasets, it turns out that this second step is
decisive for finite sample performance; essentially increasing block-sizes leads
to diminishing returns, so increasing the number of blocks becomes more reasonable.
The simple subdivision into blocks provides blocks independently of the data (the
mapping on indices is fixed a priori), while \eg CPD-segments are inherently
dependent on the underlying data, with biases the aggregation-procedure itself.
Also blocks are of equal size and thus later on
produce the same distribution after applying a dependence-score $\hat{d}$, again
this is not the case for CPD-segments. These two properties make the statistical analysis
of distributions over blocks comparatively easy and effective.
One may thus think of the block-based approach compared to CPD
as trading block-size (thus variance of $\hat{d}$) at a given rate of valid blocks
for a more well-defined and accessible direct testing problem in the second step.
This tends to be a particularly favorable trade on \emph{large} datasets.
It was also relevant to this decision that we currently specialize to the case
of partial correlation, which initially converges much faster than
non-parametric dependence-scores, thus suffers less from small block-sizes.

Finally, our pattern approach via blocks works quite generally, also for
multi-dimensional (for example spatial) and even more exotic patterns
like neighborhoods in graph-networks or periodic structures.
While there is a vast array of one-dimensional CPD literature,
the availability of established methods for more complicated
setups is limited.

This does not mean that CPD could not be used at this point,
rather that it does not help to tackle the core problem.
For this reason we focus on the direct testing problem and
leave the additional possibility to leverage CPD to future work, see §\ref{apdx:leveraging_cpd_clustering}.

\subsection{Simplifying Assumptions}\label{apdx:simplify_d_for_power}

Here, we provide notation and assumptions for the subsequent theoretical study
of the statistical tests proposed for homogeneity and weak-regime testing.
Especially for the power results, some simplifying assumptions are made.

\subsubsection{Underlying Dependence-Score and Estimator}\label{apdx:underlying_d}

At the beginning of §\ref{sec:univar_problem}, we fixed a notion of an underlying
dependence-measure $d$ with an estimator $\hat{d}$ (Ass.\ \ref{ass:underlying_d}).
The choice of underlying dependence-measure $d$ (as opposed to its estimator)
can be important here, as our ground-truth scenarios are formulated relative
to $d$ (see Rmk.\ \ref{rmk:relation_to_general_cits} and Def.\ \ref{def:marked_independence_rel_d}).
For example for the case of correlations (or z-values)
this means, to maintain a standard interpretation in terms of dependence
(rather than correlation) statements, we implicitly assume that
models are such that both cases coincide.

\begin{assumption}[Underlying CIT Convergence]\label{ass:apdx_sqrt_consistent_d}
	When we say \inquotes{$\hat{d}$ consistently approximates $d$} in Ass.\ \ref{ass:underlying_d},
	more precisely, this will mean
	that when estimating $\hat{d}_N$ on $N$ data points and $\hat{d}_{N'}$
	for $N'>N$ data points from the same generating process (with true value $d$), then
	(using the notation $\preceq$ for stochastic ordering, cf.\ Def.\ \ref{def:stochastic_ordering})
	\begin{equation*}
		|d - \hat{d}_{N'}| \sqrt{N'} \preceq |d - \hat{d}_{N}| \sqrt{N}
	\end{equation*}
	Note that for parametric tests,
	this is primarily an implementation	assumption (Rmk.\ \ref{rmk:sqrt_consistence_of_d_est}).
\end{assumption}
\begin{rmk}\label{rmk:sqrt_consistence_of_d_est}
	As phrased, Ass.\ \ref{ass:apdx_sqrt_consistent_d}
	requires parametric convergence:
	Given any parametric estimator $\hat{d}$ and $N' = k\times N$
	(for simplicity assume $k\in\mathbb{N}$), dividing the set of $N'$ data points into
	$k$ disjoint subsets of $N$ data points each,
	computing $\hat{d}_{N}$ on these $k$ subsets and $\hat{d}_{N'}$ as the mean-value
	satisfies the above equation (by applying central limit theorem, then Chebyshev’s inequality; assuming second moments exist).
	This does not violate the Cramér–Rao bound, rather it states that for larger $N$ we get closer to the Cramér–Rao bound
	(or put differently, the smaller $N$, the harder it is to get close to this bound).
	
	However, the specific form of asymptotic behavior (\ie parametric convergence,
	$\sim N^{-\frac{1}{2}}$) is \emph{not} the relevant aspect of
	Ass.\ \ref{ass:apdx_sqrt_consistent_d}. Indeed subsequent formal results hold
	also for more general convergence-behavior, thus for
	non-parametric tests.
	Our (formal) hyperparameter choices for power-results
	(see also examples \ref{example:binom_test_assymptotic_hyper_exist} and
	\ref{example:choices_for_interval_power_scaling_abs_exist}) use the
	asymptotic $\sim N^{-\frac{1}{2}}$ scaling
	for concreteness, but the specific form of scaling is not critical.
	Only the (practical) hyper-parameter choices should (unsurprisingly)
	depend on the choice of the estimator $\hat{d}$.
	
	The \emph{relevant aspect} of Ass.\ \ref{ass:apdx_sqrt_consistent_d}
	is actually in the existence of an estimator $\hat{d}$,
	which converges (at any rate):
	As \citep{shah2020hardness} show, even convergence
	(at any rate) is in full generality not possible.
	This is a general (unrelated to non-IIDness) hardness of conditional
	independence testing.
	In this sense, the use of a convergence-assumption like Ass.\ \ref{ass:apdx_sqrt_consistent_d}
	(or possibly slower than square-root rates) for $\hat{d}$ separates
	these problems (the necessity of assumptions
	\citep{shah2020hardness}), which are typical for CITs in general (cf.\ Rmk.\ \ref{rmk:relation_to_general_cits}),
	from those specific to mCITs, which we are primarily interested in here.
\end{rmk}
\begin{example}
	For z-transformed partial correlation with conditioning-set dimension $\setelemcount{Z}$,
	in good approximation $\hat{d}_{N} \sim \NormalDist\big(d, \sigma_N^2=(N-3-\setelemcount{Z})^{-1}\big)$.
	In particular $\sqrt{N}\sigma_N > 1$ and monotonically decreasing with limit $1$.
	Thus Ass.\ \ref{ass:apdx_sqrt_consistent_d} is satisfied.
\end{example}
Further we will (primarily to make statements and proofs less intricate) assume that estimates
of $\hat{d}$ are stochastically monotonic.
\begin{assumption}[Stochastic Monotonicity of Dependence-Estimators]
	\label{ass:est_d_monotonicity}
	We assume that given $N$ data points from generating processes with true values $d_0 < d_1$,
	then the distributions of the estimator $\hat{d}$ applied to such datasets $\hat{d_0^N}$, $\hat{d_1^N}$
	are stochastically ordered:
	\begin{equation*}
		d_0 < d_1 \quad\Rightarrow\quad
		\hat{d_0^N}\preceq\hat{d_1^N}
	\end{equation*}
	We will assume this ordering to be strict on the support (\ie where cdfs are not $0$ or $1$).
\end{assumption}
\begin{example}
	Let $\hat{d}$ be an approximately normal, unbiased, variance-stabilized estimator,
	then (see example \ref{example:stoch_ordering_same_var})
	Ass.\ \ref{ass:est_d_monotonicity} is satisfied.
	In particular for z-transformed partial correlation
	Ass.\ \ref{ass:est_d_monotonicity} is satisfied.
\end{example}
Sometimes we need a slightly stronger variant:
\begin{assumption}[Dependence-Estimators on Mixed Data]\label{ass:est_d_mixing}
We assume that given $N=n_0 + n_1$
data points from generating processes with true values $d_0 < d_1$,
then the distributions of the estimator $\hat{d}_{\txt{mix}}$ applied to the mixed data-set
are stochastically ordered relative to application of $\hat{d}_0, \hat{d}_1$ to size-$N$ data-sets
drawn exclusively from either generating process:
\begin{equation*}
	d_0 < d_1 \txt{ and } n_0,n_1 \neq 0 \quad\Rightarrow\quad
	\hat{d_0}\preceq\hat{d}_{\txt{mix}}\preceq\hat{d_1}
\end{equation*}
We will assume this ordering to be strict on the support (\ie where cdfs are not $0$ or $1$).
Additionally we assume that for $N\rightarrow \infty$ this ordering also holds for
$\hat{d}_{\txt{mix}}$ evaluated on invalid blocks (we never assumed \inquotes{independence}%
\footnote{Blocks are chosen a priori, but switches could occur periodically,
	so block-boundaries and switches may appear dependent.}
of switches and block-boundaries,
in principle every switch could happen after the first data point of the respective invalid block, so that
for example $n_0=1$, $n_1 \rightarrow \infty$ as $N\rightarrow\infty$; we exclude such singular mixing behavior by assumption).
\end{assumption}

For simplicity, we in places (see Rmk.\ \ref{rmk:normality_of_d}) assume that the dependence measure $\hat{d}$ (on homogeneous datasets) is normally distributed, this is the case at least
approximately for most estimators given sufficiently large blocks:
\begin{assumption}[Dependence-Estimator Normality]\label{ass:normality_of_d}
	Given a (homogeneous) dataset with expectation-value $d = E[\hat{d}]$,
	the estimator $\hat{d}\sim\NormalDist(d, \sigma)$ is unbiased and normal distributed.
\end{assumption}
\begin{rmk}\label{rmk:normality_of_d}	
	Beyond Lemma \ref{lemma:mixed_sum_cutoff_general} this is primarily an
	assumption to simplify notation.
	In Lemma \ref{lemma:mixed_sum_cutoff_general}, it is only really required for the distribution
	around $d_0$, and Lemma \ref{lemma:mixed_sum_cutoff_general} in turn is used in
	Lemma \ref{lemma:interval_method_apdx} to ensure error-control. Thus it is used
	only under the null, which means for $d_0=0$.
	So this assumption is needed primarily for $d=0$.
	
	Normal estimators are compatible with stochastic monotonicity
	(Ass.\ \ref{ass:est_d_monotonicity}) only if the variances agree
	$\sigma_0=\sigma_1$.
	The assumption $\sigma_0=\sigma_1$ is reasonable for parametric models, where often
	variance-stabilized estimators (like z-transformed correlation) are used in practice.
	For non-parametric estimators this equality may not hold.
	In practice normality is here a good approximation for the \inquotes{bulk} of the
	data, but breaks down on tails; indeed stochastic monotonicity seems a plausible assumption
	in these cases nevertheless (it would have to be violated in the tails, so this
	does not contradict the claim about normal approximations before).
	At the same time, we use normality primarily for truncated-normal approximations
	(see proof of Lemma \ref{lemma:mixed_sum_cutoff_general}), so as long as the cutoffs
	are reasonably small (in terms of quantiles), this approximation still makes sense.
	A reasonable size of the quantiles is effectively ensured by the
	$n_c \geq n_0^{\txt{min}}$ in Lemma \ref{lemma:interval_method_apdx}.
\end{rmk}

\subsubsection{Regime-Structure}\label{apdx:simplify_regime_structure}

For the theoretical analysis, especially of statistical power, it is often
helpful to restrict possible models to the case of precisely two
different regimes with individually constant (ground-truth) dependence
values $d_0\neq d_1$.
For real-world data-sets, there may of course be both more than two
regimes, and regimes of varying (for example smoothly drifting)
dependence-values.
Since the independent regime satisfies $d_0=0$, these variations are
naturally limited to the dependent regime.
Most of our arguments do not inherently depend on $d_1$ being constant,
even though this can improve statistical power (compare also
§\ref{apdx:homogeneity_quantile_est} below), this assumption is
primarily made to simplify theoretical arguments.
One notable exception is that of multiple regimes with different
sign (or drifts reaching both signs), where a one-sided cut-off
may be inappropriate (see §\ref{apdx:weak_methods});
this of course only happens for \emph{signed} dependence-scores $\hat{d}$
like (signed) correlation.
It is also interesting to note that, for the application in causal
discovery, the regime-structure for valid adjustment sets
(those where the regime actually becomes conditionally independent)
may be systematically \emph{simpler} than for general (weak) regimes.
This is, because the conditional independence usually is achieved
by conditioning on (for example) a parent-set, which effectively
removes regime-structure other than from the particular link
under investigation from the problem.

\begin{assumption}[Exactly Two Regimes]\label{ass:exactly_two_regimes}
	We assume there are (locally, on the CIT)
	precisely two regimes, with $\hat{d}$ applied to valid blocks on either regime
	having expectation-values $d_0 \neq d_1$.
\end{assumption}
\begin{notation}[Standardize Signs and Order]\label{notation:standardize_signs}
	\Wlog $d_1 > 0$ and $|d_1| > |d_0|$ (otherwise flip signs or switch the indexing);
	$d_0$ may be negative.
\end{notation}

\subsection{Homogeneity}\label{apdx:testing_homogeneity}

We provide details on the homogeneity-test proposed in §\ref{sec:testing_homogeneity}.

\subsubsection{Quantile Estimation}\label{apdx:homogeneity_quantile_est}

Estimating lower bounds on quantiles for the $d_1$ regime is surprisingly easy,
but the formulation of requirements on \inquotes{consistent} estimators
is rather subtle.
This subtlety is related to the existence of results of the form
of Lemma \ref{lemma:homogeneity_power_inf_duration} (see discussion in §\ref{apdx:homogeneity_scaling}).

\begin{definition}[Quantile Estimator]\label{def:quantile_est} 
	We call $\hat{q}_\beta^\leq$ a sound estimator of a lower bound of the $\beta$-quantile of the distribution
	in the $d_1$-regime if on the null-distribution of homogeneity (only one regime)
	with dependence $d_1$ estimated on blocks by $\hat{d_1}$ it satisfies
	\begin{equation*}
		\Pr(\hat{d_1} \leq \hat{q}_\beta^\leq) \leq \beta
		\txt.
	\end{equation*}
	We further call $\hat{q}_\beta^\leq$ consistent if additionally there is a (fixed) mapping $\beta(B)$
	-- we will formally consider the pair $(\hat{q}_\beta^\leq, \beta(B))$ as defining information of
	the estimator -- such that
	on any alternative satisfying Ass.\ \ref{ass:exactly_two_regimes} (precisely two regimes with $d_0 < d_1$)
	with $d_0$ estimated on valid blocks of size $B$ of the independent regime by $\hat{d_0}$
	the following conditions are fulfilled:
	\begin{enumerate}[label=(\alph*)]
		\item
		$\beta(B) \rightarrow 0$ for $B\rightarrow\infty$.
		\item
		$\Pr(\hat{d_0} \leq \hat{q}_{\beta(B)}^\leq)\rightarrow 1$ for $B\rightarrow\infty$.
	\end{enumerate}
\end{definition}
\begin{example}[Analytic Quantile Estimation for Partial Correlation]\label{example:quantile_parcorr_analytic}
	Consider $\hat{d}$ given as the z-value of partial correlation.
	Define first (using $\Theta = \lfloor\frac{N}{B}\rfloor$ the number of blocks)
	\begin{equation*}
		\hat{d}_1^\leq := \Theta^{-1}\sum_b \hat{d}(b)
	\end{equation*}
	given by the mean-value over estimates on blocks.
	We start by inspecting its properties on the null-distribution (only one regime)
	of dependence $d_1$.
	If $B$ is the block-size and $d_z:=\setelemcount{Z}$ the size of the conditioning set,
	then $\hat{d}$ evaluated on blocks is (approximately, we assume $B$ is large enough)
	normal distributed with mean $d_1$ and variance $\sigma_B^2 = (B - 3 - d_z)^{-1}$.
	Thus $\hat{d}_1^\leq \sim \NormalDist(d_1, \Theta^{-1}\sigma_B^2)$.
	For simplicity assume $\Theta$ is large, such that $\Theta^{-1}\sigma_B \ll \sigma_B$
	and $d_1 \approx \hat{d}_1^\leq$
	(otherwise replace the p-value estimate at the end based on $\sigma_B$ by
	a combinations of estimates based on $\sigma_B$ and $\Theta^{-1}\sigma_B$ respectively).
	Then $\hat{d}$ evaluated on blocks is normal distributed with
	known mean $\hat{d}_1^\leq$ and variance $\sigma_B^2$.
	We define (with $\phi$ denoting the standard-normal cdf)
	\begin{equation*}
		\hat{q}_\beta^\leq
		:=
		\hat{d}_1^\leq + \phi^{-1} (\beta)\sigma_B
		\txt.
	\end{equation*}
	By construction, this estimate is sound on the null distribution.
	
	This estimator is also consistent in the sense of Def.\ \ref{def:quantile_est}
	under Ass.\ \ref{ass:est_d_mixing} (monotonicity of $\hat{d}$):
	Define $\beta(B) := 1 - \phi(B^{\frac{1}{4}})$.
	By $\phi(t) \rightarrow 1$ for $t\rightarrow \infty$
	and $B^{\frac{1}{4}}\rightarrow\infty$ for $B \rightarrow\infty$,
	also $\beta \rightarrow 0$ for $B\rightarrow \infty$
	showing property (a).
	Using $\sigma_B \sim B^{-\frac{1}{2}}$ in the defining equation of 
	$\hat{q}_\beta^\leq :=\hat{d}_1^\leq + \phi^{-1} (\beta)\sigma_B$ the second term
	$\phi^{-1} (\beta(B)) \sigma_B \sim -B^{\frac{1}{4}} B^{-\frac{1}{2}} = - B^{-\frac{1}{4}} \rightarrow 0$
	for $B \rightarrow \infty$, thus $\hat{q}_\beta^\leq \rightarrow \hat{d}_1^\leq$.
	By Ass.\ \ref{ass:est_d_mixing} (and $a\neq 0$ when in the alternative),
	the mean $\hat{d}_1^\leq \rightarrow d(a) > d_0$.
	Finally for $B\rightarrow\infty$ (and thus $\sigma_B \rightarrow 0$), we have $\hat{d}_0^B \xrightarrow{\txt{P}} d_0 < d(a)$,
	which also implies $P(\hat{d}_0^B < d(a)-\epsilon) \rightarrow 1$.
	By the above argument $\hat{q}_\beta^\leq \rightarrow d(a)$ eventually becomes larger $d(a)-\epsilon$,
	thus also $P(\hat{d}_0^B < \hat{q}_\beta^\leq) \rightarrow 1$ and property (b) is satisfied.
	
	Finally note, that the \inquotes{effective sample count} approximation of the distribution
	per block can be replaced by other analytical estimates, for example a large $N$ expansion.
\end{example}
\begin{rmk}[Meaning of hyperparameters]\label{rmk:apdx_homogeneity_hyper_tradeoff}
	The choice of $\beta$ used in the example above may not be practically useful,
	it illustrates an important principle however:
	Eventually, $\beta$ needs to become smaller than $a$ (see proof of Lemma \ref{lemma:power_of_homogeneity})
	to gain power on small (ground-truth) regime-fractions, but at the same time $\beta$ should be chosen
	large enough, such that $|q_\beta - d_1|$ is small compared to $\frac{\Delta d}{\sigma_B}$.
	Thus $\beta$ parameterizes the trade-off between gaining power faster on
	regimes with small fraction $a$ (prefers smaller $\beta$) and gaining power faster on
	regimes with low separation $\Delta d = |d_1-d_0|$ (prefers larger $\beta$).
	
	The second hyperparameter $B$ is easier to interpret: Larger $B$ improves effective
	separation $\frac{\Delta d}{\sigma_B}$ (because $\sigma_B$ becomes smaller),
	but it also increases $\chi$ (as each switch of regimes invalidates approximately one block,
	thus $B$ data points).
\end{rmk}
\begin{example}[Bootstrapped Quantile Estimation]\label{example:quantile_bootstrap}
	If the estimator $\hat{d}$ is more complex, a simple analytical estimate may not be possible.
	However, one seemingly (see below) can bootstrap quantiles simply as follows:
	Draw $B$ samples from all data at random, and compute $\hat{d}$ on these samples,
	repeat. Estimate $q(\beta)$ as the $\beta$-th quantile of the resulting distribution.
	While this may seem trivial, it is crucial to note that this leads
	to \inquotes{unaligned} (with the time-axis) blocks and typically produces the same
	distribution as evaluating $\hat{d}$ on time-aligned blocks if the data are IID
	(in particular this estimator is sound in terms of Def.\ \ref{def:quantile_est})
	but not if there are persistent regimes.
	By Ass.\ \ref{ass:est_d_mixing} (estimating $\hat{d}$ on invalid blocks is ordered relative to valid blocks),
	and Ass.\ \ref{ass:apdx_sqrt_consistent_d} (square-root consistent convergence of $\hat{d}$),
	the choice $\beta(B) := 1 - \phi(B^{\frac{1}{4}})$ again makes this a consistent estimator
	(in the sense of Def.\ \ref{def:quantile_est}).
	
	There is also another catch: The naive bootstrap idea above actually over-controls false-positives,
	thus looses some power. This can be understood as follows:

	Clearly the bootstrap has to break the time-alignment of the data to work.
	Interestingly this is also true under the null.
	Consider the following oversimplified example:
	If we bootstrapped random blocks \emph{from the time-aligned blocks} $b_\tau$,
	pick a quantile $q(\beta)$, and compare the values on the original $b_\tau$ to
	this $q$, we will find about $k\approx\beta \Theta$ blocks below $q$.
	However, while this agrees with the expectation of the binomial distribution
	$\BinomDist(n=\Theta, p=\beta)$, the result for $k$ is \emph{not} binomially
	distributed! Actually for $\Theta\rightarrow\infty$ we have $q \rightarrow q_{\txt{sample}}$,
	with $q_{\txt{sample}}$ the quantile of the original sample.
	The bootstrap introduces some randomness around this mean, but by the
	law of large numbers, it approaches the mean for large bootstrap-count,
	in particular $\Var(k)\rightarrow 0$. Meanwhile for the binomial
	$\Var(k)=\Theta \beta (1-\beta) \rightarrow \infty$.
	For large bootstrap-count the approximation obtained for $k$ in lemma \ref{lemma:binomial_test}
	via this \inquotes{aligned} bootstrap is thus extremely close to $k \approx b\tau$
	and very rarely (with probability $\ll \alpha$)
	exceeds the upper $\alpha$ percentile of the much wider binomial.
	
	We do not bootstrap from aligned blocks, but there is of course some non-trivial
	overlap and thus $q \notindependent q_{\txt{sample}}$.
	Indeed this effect is visible in benchmarking (see for example Fig.\ \ref{fig:homogeneity_overview}).
	It can be avoided by cross-fitting (bootstrapping on two halves of the sample for
	the respective other half), up to the analytically tractable problem
	from rounding to integer cutoffs for $k$ (see Rmk.\ \ref{rmk:binomial_rounding_control}).
	While the discrepancy in statistical power is not high, so that in most cases
	the simple\Slash{}fast bootstrap is sufficient for practical purposes,
	it seems important to understand \emph{why} it overcontrols false positives.
	Also the above argument shows that the bootstrap-count should, somewhat counter-intuitively
	\emph{not be chosen too large}, to avoid loss of power. Since for the unaligned
	bootstrap, the random alignment is small, it makes still sense to have a large bootstrap-count in practice.
\end{example}
\begin{rmk}\label{rmk:better_quantile_est}
	As will be seen in the results concerning statistical power (§\ref{apdx:homogeneity_scaling})
	below, the difference between estimating $q_\beta$ in an unbiased way and
	the consistency-requirement given by Def.\ \ref{def:quantile_est}
	does substantially affect scaling behavior of the test.
	Since heuristically it seems to still converge much faster
	than the weak regime test, this does likely not substantially affect
	our ability to implement an mCIT.
	Nevertheless, it should be noted
	that there are a few rather simple steps that may be
	taken to improve the estimation of $q_\beta$.
	\begin{enumerate}[label=(\roman*)]
		\item Instead of through a mean, $d_1$ in the partial correlation approach
			(example \ref{example:quantile_parcorr_analytic}) could be estimated
			via a quantile $\beta'$. If $\beta' > a$, then this provides for $B\rightarrow\infty$
			an unbiased estimate of $d_1$, and thus of $q_\beta$.
		\item For homogeneity testing, other than for weak-regime testing, the value $d=0$ is not
			special and the problem is essentially symmetric under exchange of $d_0$ and $d_1$.
			Instead of estimating a lower quantile of $d_1$ and comparing to $\hat{d}_0$,
			we could as well estimate an upper quantile of $d_0$ and compare to $\hat{d}_1$.
			In particular we could test in both directions, individually -- and thus after correcting for multiple testing combined -- control FPR and reject if either one has power.
		\item In the light of (ii), we can assume \asswlog $a \leq \onehalf$.
			Then the requirement $\beta' > a$ in (i) is effectively always satisfied (for at least
			one of the two tests) if $\beta' > \onehalf$. Thus for partial correlation and
			for example picking $\beta' = \sfrac{3}{4}$
			at least one of the two tests ends up with an unbiased estimate of $q_\beta$
			as $B \rightarrow \infty$.
	\end{enumerate}
	Note, however, that while this may help keeping bias controlled better for finite $B$,
	it still becomes unbiased only for $B\rightarrow \infty$ (not for $N\rightarrow \infty$).
	Indeed by the impossibility of general IIDness\Slash{}randomness tests (and Lemma \ref{lemma:homogeneity_power_inf_duration}),
	clearly an unbiased estimate of $q_\beta$ requires at least $B\geq 2$.
\end{rmk}

\begin{figure}[ht]
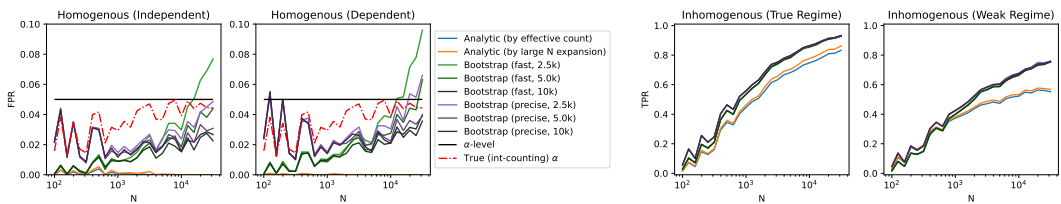
		
	\mayincludegraphics[width=\textwidth]{homogeneity/basic_all}
	\caption{Overview of results for homogeneity tests at $\alpha=5\%$.
		Shown are FPR (type I errors) and TPR (power)
		for different estimates of the quantile $q$.
		Used $\alpha=0.05$, $\beta=0.1$ and block size $B=15$, see §\ref{apdx:homogeneity_hyper_heuristic}.
		The red dash--dot line (lhs plot of FPRs) is slightly below $\alpha$
		from rounding to an integer cutoff, cf.\ Rmk.\ \ref{rmk:binomial_rounding_control}.
	}
	\label{fig:homogeneity_overview}
\end{figure}

Figure \ref{fig:homogeneity_overview} shows results (FPR and TPR) of numerical experiments
for the homogeneity test combined with different quantile estimators.
These numerical results support the functioning of this test at large,
with the restriction that bootstrap approaches in particular seem to struggle with
error-control on very large data-sets and require a sufficiently high number of
bootstrap-samples to control errors on larger data-sets. Slightly more surprising,
the analytical estimates are able to almost match bootstrap power despite over-controlling FPR
on true-regime ground-truth scenarios. The main difference materializes for weak-regime cases.
However, errors confusing homogeneous dependent and weak regime do not affect the
marked independence tests immediately (they lead to the same output).
This could affect the 
power of the homogeneity-first approach, discussed in §\ref{sec:testing_order}.
Intriguingly it does not: The fall-off in Fig.\ \ref{fig:homogeneity_overview} on weak regimes
is primarily due to regimes with the same sign (of correlation)
which reduces the separation, the need for homogeneity-first
arises if the dependence-test does not perform well due to opposite sign
(of correlation) in the different regimes
(indeed for Fig.\ \ref{fig:global_dependence} an analytic estimate was used).
This is also discussed in the paragraph \inquotes{Validation and Weak Regimes}
of §\ref{apdx:homogeneity_hyper_heuristic}.

Based on these considerations, we henceforth use the analytical estimates in our
numerical experiments. For our general analysis, the potential performance gain
by using a bootstrapped estimate does not seem to justify the runtime-cost.
In applications and for extensions to time series (beyond MCI, cf.\ §\ref{apdx:MCI}),
the use of (suitable) bootstrapping strategies should be given further consideration.

\subsubsection{Error Control}\label{apdx:homogeneity_err_control}

Using the quantile-estimates as described in the previous section,
error-control is rather quick to show.

\theoremstyle{plain}
\newtheorem*{lemma_binomial_test_error_control}
{Lemma \ref{lemma:binomial_test}}
\begin{lemma_binomial_test_error_control}
	Fix a hyperparameter $\beta \in (0,1)$. Given $\alpha \in (0,1)$ and
	a sound estimator $\hat{q}_\beta^\leq$ of a lower bound of the $\beta$-quantile of the distribution
	in the $d_1$-regime (Def.\ \ref{def:quantile_est}), then we can reject data-homogeneity
	with type I errors controlled at $\alpha$ as follows:
	
	Let $k := \setelemcount{\{\tau | d_\tau < q_\beta \}}$ be the number of blocks
	with $d_\tau$ less than $q_\beta$.
	Denote by $\phi^{\txt{binom}}_{\Theta, \beta}$ the cumulative distribution-function
	of the binomial distribution $\BinomDist(n=\Theta, p=\beta)$.
	Then $p_0 := 1-\phi^{\txt{binom}}_{\Theta, \beta}(k)$ is a
	valid p-value under the null:
	\begin{equation*}
		P_{\txt{homogenous}}( p_0 < \alpha ) \leq \alpha\txt.
	\end{equation*}
	In particular rejecting homogeneity iff $p_0<\alpha$ leads to a valid test
	(Def.\ \ref{def:error_control_power_meta}).
\end{lemma_binomial_test_error_control}

\begin{proof}
	By soundness of the quantile-estimator $p=\Pr(\hat{d} < \hat{q}_\beta^\leq) \leq \beta$.
	The total number $k$ of blocks with $\hat{d} < q$ under the null
	is binomially distributed as $\BinomDist(n=\Theta, p)$. With $p\leq \beta$,
	we have $\phi^{\txt{binom}}_{\Theta, p} \leq \phi^{\txt{binom}}_{\Theta, \beta}$.
	Validity of $p_0$ then follows directly from the definition
	of the cumulative distribution function.
\end{proof}

\subsubsection{Power and Scaling Behavior}\label{apdx:homogeneity_scaling}

Error-control (Lemma \ref{lemma:binomial_test}) alone does not make for a meaningful statistical test.
Even more so, since testing IIDness in general is not possible (cf.\ Prop.\ \ref{prop:imposs_A}),
the question arises, under which circumstances, our test will gain power
asymptotically, and how the choice of hyperparameters will affect
its convergence-speed.
We discuss the asymptotic behavior theoretically and hyperparameter-choices semi-heuristically
with the help of numerical experiments in §\ref{apdx:homogeneity_hyper_heuristic}.

\begin{lemma}[Binomial Test, Power]\label{lemma:power_of_homogeneity}
	Let $\Lambda : \mathbb{N} \rightarrow \mathbb{R}$ be a scale (see Def. \ref{def:limits_temporal}),
	and assume we are given a pattern $\Pattern$ (Def.\ \ref{def:blocks}) which is asymptotically
	$\Lambda$-compatible (Def.\ \ref{def:pattern_compatible})
	with the indicator meta-model $\mathcal{R}$
	(see intro to §\ref{apdx:model_details}) generating our data.
	Further we restrict to alternatives
	with exactly two regimes (Ass.\ \ref{ass:exactly_two_regimes}).
	
	Let $B(N)\rightarrow \infty$, such that $\frac{B(N)}{\Lambda(N)}\rightarrow 0$ and
	w.\,l.\,o.\,g.\footnote{As long as $B(N)\rightarrow \infty$ is maintained,
		we may choose a smaller $B(N)$. In particular, it is always possible to ensure $\Theta(N) \rightarrow \infty$.}
	$\Theta(N) = \lfloor\frac{N}{B(N)}\rfloor \rightarrow \infty$.
	We are given a consistent estimator $(\hat{q}_\beta^\leq, \beta(B))$ of a lower bound of the $\beta$-quantile of the distribution
	in the $d_1$-regime (Def.\ \ref{def:quantile_est}).
	Fix $\alpha(N) \rightarrow 0$ such that $\frac{\phi^{-1}(1-\alpha(N))}{\sqrt{\Theta(N)}} \rightarrow 0$,
	where $\phi$ is the standard-normal cdf.
	Such $\alpha(N)$ exist (cf.\ example \ref{example:binom_test_assymptotic_hyper_exist}).
	
	Then the test described in Lemma \ref{lemma:binomial_test} with
	hyperparameters $B(N)$ and $\beta(B(N))$
	is weakly asymptotically correct (Def.\ \ref{def:error_control_power_meta}).
\end{lemma}
\begin{proof}
	In this case (homogeneity as opposed to weak-regime, see §\ref{apdx:testing_weak_regime}),
	there is uniform finite-sample error-control at $\alpha(N)$
	by Lemma \ref{lemma:binomial_test},
	thus by $\alpha(N) \rightarrow 0$ the probability of a correct results on
	models satisfying the null hypothesis (homogeneity) approaches $1$,
	and we can focus on establishing the power-result, which by
	Ass.\ \ref{ass:exactly_two_regimes} has exactly two regimes with $d_0 \neq d_1$
	(see hypothesis).
	
	Define
	\begin{equation*}
		\beta_0(N) := P\Big(\hat{d}^0_{B(N)} < \hat{q}^\leq_{\beta(B(N))}\Big)\txt,
	\end{equation*}
	where $\hat{d}^0_B$ has the distribution of $\hat{d}$ applied to a valid block of size $B$ in regime $0$ with
	true value $d_0$.
	By definition \ref{def:quantile_est} of consistency of $\hat{q}_\beta^\leq$,
	$\beta_0 \rightarrow 1$.
	The count $k$ in Lemma \ref{lemma:binomial_test} can be split as follows,
	where $\tau\in V$ are the valid blocks from the $d_0$-regime
	\begin{equation*}
		k := \setelemcount{\big\{\tau \big| d_\tau < q_\beta \big\}}
		= \setelemcount{\big\{\tau \big| d_\tau < q_\beta, \tau \in V \big\}}
	+ \setelemcount{\big\{\tau \big| d_\tau < q_\beta, \tau \notin V \big\}}
	=: k_0 + k_1
	\txt.
	\end{equation*}
	In particular $k\geq k_0$
	and $\frac{k_0}{\Theta} \rightarrow \beta_0 A \rightarrow A\rightarrow a > 0$
	by asymptotic $\Lambda$-consistency and Lemma \ref{lemma:relate_A_to_chi}.
	Let $\epsilon > 0$ be arbitrary.
	By $\frac{k_0}{\Theta} \rightarrow a > 0$, eventually 
	$k_0 > \frac{\delta}{2} \Theta(N)$ with probability greater $1-\epsilon$.
	
	The null used by Lemma \ref{lemma:binomial_test} is $k_{\txt{null}} \sim \BinomDist(\Theta, \beta(N))$,
	with $\beta \rightarrow 0$ thus eventually $\beta < \frac{\delta}{4}$.
	In particular, eventually $k_{\txt{null}} \preceq \BinomDist(\Theta, \frac{\delta}{4})$,
	and the cdf satisfies
	$\phi^{\txt{binom}}_{\Theta,\beta(N)}(k) \geq \phi^{\txt{binom}}_{\Theta,\frac{\delta}{4}}(k)$.
	Further, $\BinomDist(\Theta, \frac{\delta}{4})$ approaches (for $\Theta \rightarrow \infty$),
	the normal distribution $\NormalDist(\mu=\Theta\frac{\delta}{4},
	\sigma^2=\frac{\delta}{4}(1-\frac{\delta}{4})\Theta)$.
	Thus eventually
	\begin{equation*}
		\phi^{\txt{binom}}_{\Theta,\beta(N)}(k)
		\geq
		\phi^{\txt{binom}}_{\Theta,\frac{\delta}{4}}(k)
		\approx
		\phi^{\txt{normal}}_{\mu=\Theta\frac{\delta}{4},
			\sigma^2=\frac{\delta}{4}(1-\frac{\delta}{4})\Theta}(t)
			\txt.
	\end{equation*}
	We reject for $k$ above
	\begin{align*}
		k_{\txt{null}}^{\alpha(N)}
		= \big(\phi^{\txt{binom}}_{\Theta,\beta(N)}\big)^{-1}(1-\alpha)
		&\leq
		\big( \phi^{\txt{normal}}_{\mu=\Theta\frac{\delta}{4},
			\sigma^2=\frac{\delta}{4}(1-\frac{\delta}{4})\Theta} \big)^{-1}(1-\alpha)\\
		&= \Theta\frac{\delta}{4} + \phi^{-1}(1-\alpha)\sqrt{\frac{\delta}{4}(1-\frac{\delta}{4})\Theta}\\
		&= \Theta\Big(
				\frac{\delta}{4}
				+
				\sqrt{\frac{\delta}{4}(1-\frac{\delta}{4})} \times \frac{\phi^{-1}(1-\alpha)}{\sqrt{\Theta}}
			\Big)
			\txt,
	\end{align*}
	where $\phi$ is the standard normal cdf.
	By hypothesis $\frac{\phi^{-1}(1-\alpha)}{\sqrt{\Theta}} \rightarrow 0$,
	thus eventually $k_{\txt{null}}^{\alpha(N)} < \frac{\delta}{2} \Theta(N)$.
	Since $k \geq k_0 > \frac{\delta}{2} \Theta(N)$ with probability greater $1-\epsilon$,
	we eventually reject the null with probability greater $1-\epsilon$.
	The claimed power-result
	follows because $\epsilon>0$ was arbitrary.
\end{proof}
\begin{example}[Existence of hyperparameters]\label{example:binom_test_assymptotic_hyper_exist}
	We show that choices for $\alpha$ with $\alpha(N) \rightarrow 0$ such that
	$\frac{\phi(1-\alpha(N))}{\sqrt{\Theta(N)}} \rightarrow 0$ exist.
	An example is $\alpha(N) = 1 - \phi(\Theta(N)^{\frac{1}{4}})$:
	By $\Theta\rightarrow\infty$ we get $\phi(\Theta(N)^{\frac{1}{4}})\rightarrow 1$
	and thus $\alpha \rightarrow 0$.
	Further $\frac{\phi^{-1}(1-\alpha)}{\sqrt{\Theta}} =
	\frac{\Theta(N)^{\frac{1}{4}}}{\Theta(N)^{\frac{1}{2}}} = \Theta(N)^{-\frac{1}{4}} \rightarrow 0$.
\end{example}

In Rmk.\ \ref{rmk:better_quantile_est}, we have seen that obtaining an \emph{unbiased} estimator
of the $\beta$-quantile $\hat{q}_\beta$ is difficult.
Indeed this subtlety is not by accident. Rather \emph{knowing an unbiased estimator} would be
extremely powerful, as the following \emph{infinite duration limit} (see example
\ref{example:limits}) result shows:

\begin{lemma}[Binomial Test, Power with Unbiased Quantile]
	\label{lemma:homogeneity_power_inf_duration}
	Modify the hypothesis of the previous Lemma \ref{lemma:power_of_homogeneity},
	by replacing $B \rightarrow \infty$ by $B= \const$, and the pattern-applicability
	assumption by knowledge of an unbiased estimator $\hat{q}_\beta$ of the $\beta$-quantile of the distribution
	in the $d_1$-regime (Def.\ \ref{def:quantile_est}).
	Finally additionally assume Ass.\ \ref{ass:est_d_mixing} (invalid blocks are stochastically ordered relative to
	valid blocks).
	Then the conclusion that the test gains power asymptotically still holds.
\end{lemma}
\begin{proof}
	Define, as before,
	\begin{equation*}
		\beta_0 := P\Big(\hat{d}^0_B < \hat{q}_\beta\Big)\txt,
	\end{equation*}
	where $\hat{d}^0_B$ has the distribution of $\hat{d}$ applied to a valid block of size $B$ in regime $0$ with
	true value $d_0$. By Ass. \ref{ass:est_d_monotonicity} (implied by the stronger
	Ass.\ \ref{ass:est_d_mixing} given by hypothesis) $\beta_0 > \beta$.
	Still as before, the count $k$ in Lemma \ref{lemma:binomial_test} can be split,
	where $\tau\in V$ are the valid blocks from the $d_0$-regime
	\begin{equation*}
		k := \setelemcount{\big\{\tau \big| d_\tau < q_\beta \big\}}
		= \setelemcount{\big\{\tau \big| d_\tau < q_\beta, \tau \in V \big\}}
		+ \setelemcount{\big\{\tau \big| d_\tau < q_\beta, \tau \notin V \big\}}
		=: k_0 + k_1
		\txt.
	\end{equation*}
	By Ass.\ \ref{ass:est_d_mixing} (invalid blocks are stochastically ordered relative to
	valid blocks)
	$\hat{d}$ applied to an invalid block has $P(\hat{d} < \hat{q}_\beta)\geq \beta$,
	thus $k_1 \succeq \BinomDist((1-A)\Theta, \beta)$ (where $\Theta = \lfloor \frac{N}{B} \rfloor$).
	Further $k_0 \sim \BinomDist(A\Theta, \beta_0)$.
	For $\Theta\rightarrow \infty$ thus
	$\frac{k_0}{\Theta} \xrightarrow{\txt{P}} A\beta_0$
	and with probability approaching unity $\frac{k_1}{\Theta} \geq (1-A)\beta-\epsilon$,
	for any $\epsilon>0$ in particular for $\epsilon = \frac{A}{2}(\beta_0-\beta)>0$.
	Therefore
	\begin{align*}
		\frac{k}{\Theta} = \frac{k_0+k_1}{\Theta} \geq A \beta_0 + (1-A)\beta - \epsilon
		&= \beta (A + (1-A)) + A(\beta_0 - \beta) - \epsilon \\
		&= \beta +  \frac{A}{2}(\beta_0-\beta) \halfquad>\halfquad \beta\txt.
	\end{align*}
	with probability approaching unity.
	The null used by Lemma \ref{lemma:binomial_test} is $k_{\txt{null}} \sim \BinomDist(\Theta, \beta)$,
	thus $\frac{k_{\txt{null}}}{\Theta} \xrightarrow{\txt{P}} \beta$
	and is rejected with probability approaching one (for details
	on the values of $k$, see the proof of
	Lemma \ref{lemma:power_of_homogeneity}).
\end{proof}

The reason why this argument could not be made in Lemma \ref{lemma:power_of_homogeneity} is that for
a lower bound $q_\beta^\leq$, the rate at which valid blocks in the $d_1$-regime have dependence
below $q_\beta^\leq$ does not approach $\beta$ (as for an unbiased estimator). Instead, if $q_\beta^\leq < d_1$
is strictly less, it will approach $0$ (because $\hat{d}_1\xrightarrow{\txt{P}}d_1$ in Lemma \ref{lemma:power_of_homogeneity}).
Thus $\frac{k_1}{\Theta}$ may get arbitrarily close to zero, and all contributions to $k$ have to come from $k_0$ alone.
Evidently, this is only possible if $a > \beta$, thus a consistent choice with $\beta\rightarrow 0$ is needed.
This also further clarifies the choice of hyperparameters (see Rmk.\ \ref{rmk:apdx_homogeneity_hyper_tradeoff}).
These problems necessarily require that $B\rightarrow \infty$ if no unbiased estimator is available.

\subsubsection{Semi-Heuristic Hyperparameter Choices}\label{apdx:homogeneity_hyper_heuristic}

While the previous subsections provide a reasonable theoretical understanding
of the meaning of the hyperparameters $\beta$ (quantile)
and $B$ (block-size), the question for their practical numerical choice
so far was not addressed.
We do so by numerical experiments and a semi-heuristic interpretation.
Given the systematic reduction of the HCCD problem to simple stages
within a local mCIT (see §\ref{sec:dyn_indep_tests}),
typical runtimes of modular components (here: the homogeneity-test)
are usually less than a millisecond per execution.
Together with a comparatively simple specification of the individual
subproblems solved by each stage, this allows for a rather comprehensive
study of the behavior with different hyperparameters, ground-truth scenarios
and data-set sizes.
The present analysis also serves as a demonstrator for this
advanced introspective capability.

Generally, the more prior information a user can supply, the better a choice 
of hyperparameters can be made.
However, few applicants have infinite amounts of time to allocate to this task and
to grind out every last bit of optimization here
will often not be the best distribution of effort.
To address real-world applications, we provide semi-heuristic
hyperparameter-sets, where the user only has to specify one simple binary
prior information: Are segments expected to be very large (hundreds of data points)?
In this case, statistical power can be allocated to finding such very large
structures substantially more efficiently than if one wants to search
for smaller structures as well.

\begin{convention}\label{convention:hyperparam_sets}
	We provide two hyperparameter sets: one for generic regime-sizes,
	one for large regime-sizes.
	Variation with sample-size $N$ and conditioning-set size $|Z|$
	will be handled semi-heuristically as explained below.
\end{convention}

For this reason, most subsequent figures come in pairs (for example
Fig.\ \ref{fig:homogeneity_hyper_generic_regimes}
and Fig.\ \ref{fig:homogeneity_hyper_large_regimes}),
once for data-generation with generic regimes (§\ref{apdx:temporal_regimes})
and once with this same data-set filtered to include only typical lengths in
the range $\ell=100\ldots 1000$ to better understand behavior of different
hyperparameters on large regimes (see also Fig.\ \ref{fig:homogeneity_power}a in the main text).

\begin{figure}[ht]
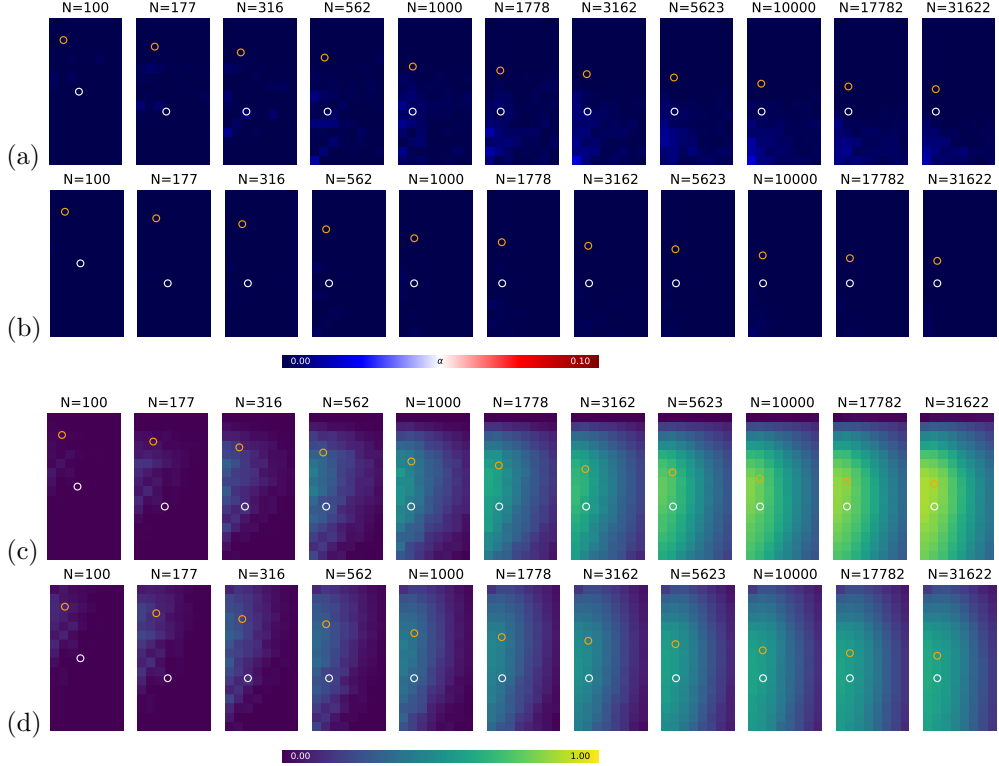

	(a)
	\mayincludegraphics[width=0.9\textwidth]
	{homogeneity/generic_regimes/nulldist_independent}\\
	(b)
	\mayincludegraphics[width=0.9\textwidth]
	{homogeneity/generic_regimes/nulldist_dependent}\\
	\hspace*{10em}
	\mayincludegraphics[width=0.3\textwidth]
	{color_scheme_fpr_05}\\[1em]
	(c)
	\mayincludegraphics[width=0.9\textwidth]
	{homogeneity/generic_regimes/power_true_regime}\\
	(d)
	\mayincludegraphics[width=0.9\textwidth]
	{homogeneity/generic_regimes/power_weak_regime}\\
	\hspace*{10em}
	\mayincludegraphics[width=0.3\textwidth]
	{color_scheme_tpr}\\
	\caption[Homogeneity (hyperparams): Pairwise Tests (generic)]
	{\emph{Pairwise Tests (generic $L$):} Behavior of false positives ((a) globally independent, (b) globally dependent) and true positives ((c) true regime, (d) weak regime) for generic
	prior of regime-lengths across hyperparameters (per image: x-direction is quantile $\beta=0.05 \ldots 0.4$ linear, left-to-right, and y-direction is block-size
	$B=5\ldots 80$ logarithmic, top-to-bottom).
	Images within each row are for different sample-sizes
	$N=10^2\ldots10^{4.5}$, logarithmic left-to-right.
	Indicated by circles are heuristic hyperparameter choices
	(orange: generic, white: optimized for large regimes).}
	\label{fig:homogeneity_hyper_generic_regimes}
\end{figure}
\begin{figure}[ht]
	(a)
	\mayincludegraphics[width=0.9\textwidth]
	{homogeneity/large_regimes/nulldist_independent}\\
	(b)
	\mayincludegraphics[width=0.9\textwidth]
	{homogeneity/large_regimes/nulldist_dependent}\\
	\hspace*{10em}
	\mayincludegraphics[width=0.3\textwidth]
	{color_scheme_fpr_05}\\[1em]
	(c)
	\mayincludegraphics[width=0.9\textwidth]
	{homogeneity/large_regimes/power_true_regime}\\
	(d)
	\mayincludegraphics[width=0.9\textwidth]
	{homogeneity/large_regimes/power_weak_regime}\\
	\hspace*{10em}
	\mayincludegraphics[width=0.3\textwidth]
	{color_scheme_tpr}\\
	\caption[Homogeneity (hyperparams): Pairwise Tests (large)]{\emph{Pairwise Tests (large $L$):}
	Same as Fig.\ \ref{fig:homogeneity_hyper_generic_regimes}, but
	for ground truth regime-length in the range $\ell=100\ldots1000$,
	\ie comparatively large.}
	\label{fig:homogeneity_hyper_large_regimes}
\end{figure}

\paragraph{Pairwise Tests:}
We start by investigating the behavior of different hyperparameters
for pairwise (unconditional, $Z=\emptyset$) tests under the four
ground-truth scenarios of the univariate problem, see Fig.\ \ref{fig:univariate_problem}
in the main text; that is for global independence, global dependence, true regimes and
weak regimes.
Both global hypotheses are homogeneous (ground-truth negatives, null),
while the true and weak regime
cases are heterogeneous (ground-truth positives, alternatives)
for the homogeneity test.
In Fig.\ \ref{fig:homogeneity_hyper_generic_regimes}, we thus study false positive rates (FPR)
for the global hypotheses (on a color-scale that makes it easy to compare to $\alpha=0.05$),
and true positive rates  (TPR) for the regime hypotheses.
The underlying data-generation is for generic regimes as described in
§\ref{apdx:temporal_regimes}, Fig.\ \ref{fig:homogeneity_hyper_large_regimes}
shows the same information, but for larger ground-truth regimes.

We note first that FPR-control seems to work, but errors are over-controlled. This is expected
(we use the analytic estimate, which really only bounds the error); for the shown
results, $\alpha = 5\%$.
Concerning statistical power (TPR), the most evident observation is that larger $N$ improves
power; this is a reasonable consistency-check, but of course expected
-- even though the scaling here is by an uninformative prior
(see §\ref{apdx:model_details}, in particular example \ref{datagen:indicator_metamodel}),
that is typical segment-sizes only
scale like $\log(N)$, which is a comparatively challenging benchmark.

Next, the left hand side region of the plots (small $\beta$)
typically seem to perform favorable, even though Fig.\ \ref{fig:homogeneity_hyper_large_regimes}
suggests that at least for large ground-truth regimes, at smaller sample-sizes $N$
a slightly higher choice of $\beta$ can be favorable.
Further a choice of $\beta$ that is too small (especially for small $N$) seems to
negatively affect stability and FPR-control;
the panels (a) and (b) in the figures show only
the average FPR (per pixel = hyperparameter choice), while FPR-control should work uniformly
(not just across the the two global hypotheses (a) and (b), but also across regime-sizes,
regime-fractions etc.). This is not really surprising, as small $\beta$, for small $N$
(thus small number of blocks $\Theta$) means the counting statistic
$\BinomDist(A\Theta,\beta)$ has to decide on the basis of typically very small counts.
The choice of a decreasing $\beta$ is also consistent with our theoretical considerations,
even though the motivation there was in part different
(see Rmk.\ \ref{rmk:apdx_homogeneity_hyper_tradeoff}).

At least for the generic regime-lengths Fig.\ \ref{fig:homogeneity_hyper_generic_regimes},
there is a trend discernible that for larger $N$ larger $B$ are more favorable.
This does make sense, as for our data-set, the relative frequency of smaller
regimes \emph{does} at some point decay (see §\ref{apdx:limits}), even though only
logarithmically fast (Lemma \ref{lemma:uninformative_prior_scaling_temporal}).
For already initially large regime-lengths this effect is much weaker, this could
be simply because the logarithmic increase no longer matters at that point (relative to other
factors). It could also be special to the (partial-)correlation setting, where
dependence-estimates converge fast and increasing block-size has
quickly diminishing return
(increasing block-size decreases block-count, one would expect to operate on
an effective sample-count $N\frac{B-3}{B}$ for correlation, which for large
$B$ no longer substantially depends on $B$).

We make our hyperparameter choices -- indicated in
Fig.\ \ref{fig:homogeneity_hyper_generic_regimes} and
Fig.\ \ref{fig:homogeneity_hyper_large_regimes} by orange (generic lengths)
and white (large lengths) circles -- for pairwise tests based on these observations:
For large regimes, we pick a constant $B=30$ (but with a minimum block-count of $5$ for stability,
which is why in the first column $N=100$ and thus $B=\lfloor\frac{100}{5}\rfloor=20<30$,
moving the white circle up a notch in the first image).
For the smaller regime-length case, we use the idea outlined in the explanation above,
of logarithmically increasing typical lengths, to motivate a form $B=a \log_{10}(N) + b$,
where a heuristic guess based on the numerical experiments for $a$ and $b$ is made
as $a=5$ and $b=-3$.

For $\beta$, we first have to discuss another issue:
As already noted in Rmk.\ \ref{rmk:binomial_rounding_control}, the count $k$
of below-$q_\beta$ blocks used in the binomial test is discrete, so error-control
will \emph{not} happen at precisely $\alpha$ if we simply fix a value for $\beta$,
then reject if $k\geq k_\beta$. Rather this controls somewhere below $\alpha$ (which is a valid
test, but certainly not an optimal hyperparameter; as noted before, with the current
parcorr-specific analytic estimate for $q_\beta$, we still overcontrol errors.
Nevertheless we do not want to over-control more than we have to.
Further, for other choices of quantile-estimators, see §\ref{apdx:homogeneity_quantile_est},
this may become more relevant).
The approach we employ instead is to start from an initial $\beta_0$,
compute for the present values of $N$ and $B$ (thus of $\Theta$),
the value $k^0_\beta$ at which we could first reject with FPR less than $\alpha$.
Then we pick $\beta$ as the largest value, for which rejecting above $k^0_\beta$
still controls FPR at $\alpha$.
Since this does not change the acceptance-threshold on $k$, only the estimate $q_\beta$,
this can in practice only increase power.
This procedure tends to produce $\beta \approx \beta_0$ for large $\Theta$,
while typically $\beta > \beta_0$ for small $\Theta$ (\ie in particular for the combination
of small $N$ and large $B$, thus mostly for the large-regime hyperparameters and for small
data-sets). Indeed starting from a reasonable initial value $\beta_0=0.1$
typically seems to lead to acceptable results (the circles\Slash{}parameters shown in the figures);
thus there does not seem to be much reason for a more complicated choice of $\beta_0$
(given prior knowledge, or particular interest in small regime-fractions $a$ may, however,
motivate a deviation from this argument).
Finally, this choice of $\beta$ does not appear special in the figures. This is,
because of the figures' finite resolution, providing pixels on nearby $\beta$ only,
which does not seem to resolve this rounding-optimization (on $k_\beta$).

\begin{figure}[ht]
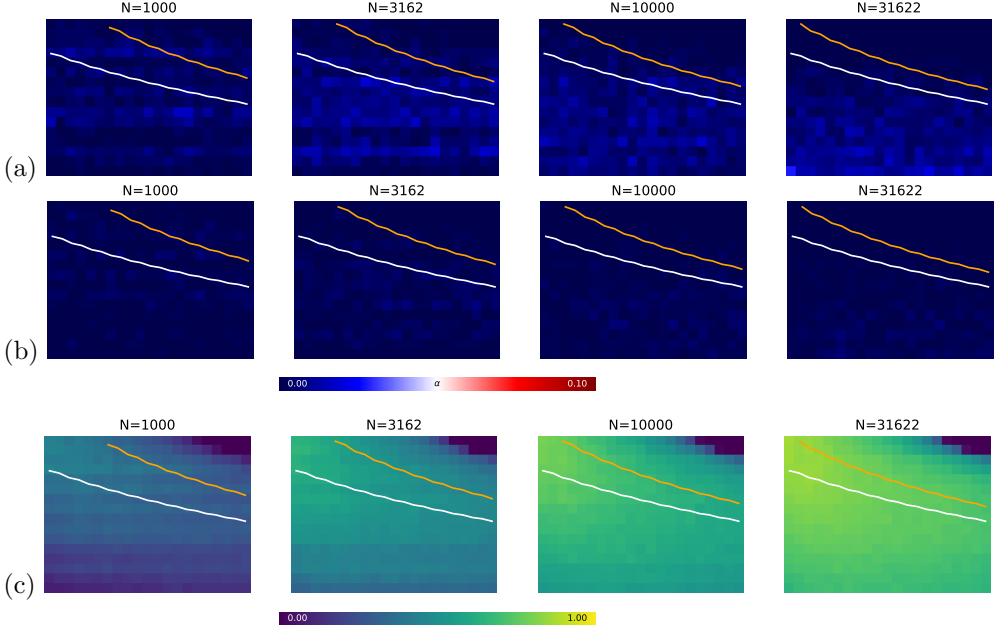

	(a)
	\mayincludegraphics[width=0.9\textwidth]
	{homogeneity/conditional/local_05_fpr_independent}\\
	(b)
	\mayincludegraphics[width=0.9\textwidth]
	{homogeneity/conditional/local_05_fpr_dependent}\\
	\hspace*{10em}
	\mayincludegraphics[width=0.3\textwidth]
	{color_scheme_fpr_05}\\[1em]
	(c)
	\mayincludegraphics[width=0.9\textwidth]
	{homogeneity/conditional/local_05_power_true_regime}\\
	\hspace*{10em}
	\mayincludegraphics[width=0.3\textwidth]
	{color_scheme_tpr}\\
	\caption[Homogeneity (hyperparams): Conditional Tests (generic)]{
		\emph{Conditional Tests (generic $L$):}
		Behavior of false positives
		((a) globally independent, (b) globally dependent)
		and true positives ((c) true regime) for generic
		prior of regime-lengths
		for different conditioning-set sizes $|Z|=0\ldots 20$ (per image: x-direction;
		linear, left-to-right)
		across different block-sizes $B=20\ldots 150$
		(larger than before; logarithmic, top-to-bottom; per image: y-direction)
		and for different sample-sizes.
		Indicated by plotted lines are heuristic hyperparameter choices
		(orange: generic, white: optimized for large regimes).}
	\label{fig:homogeneity_hyper_generic_regimes_conditional}
\end{figure}
\begin{figure}[ht]
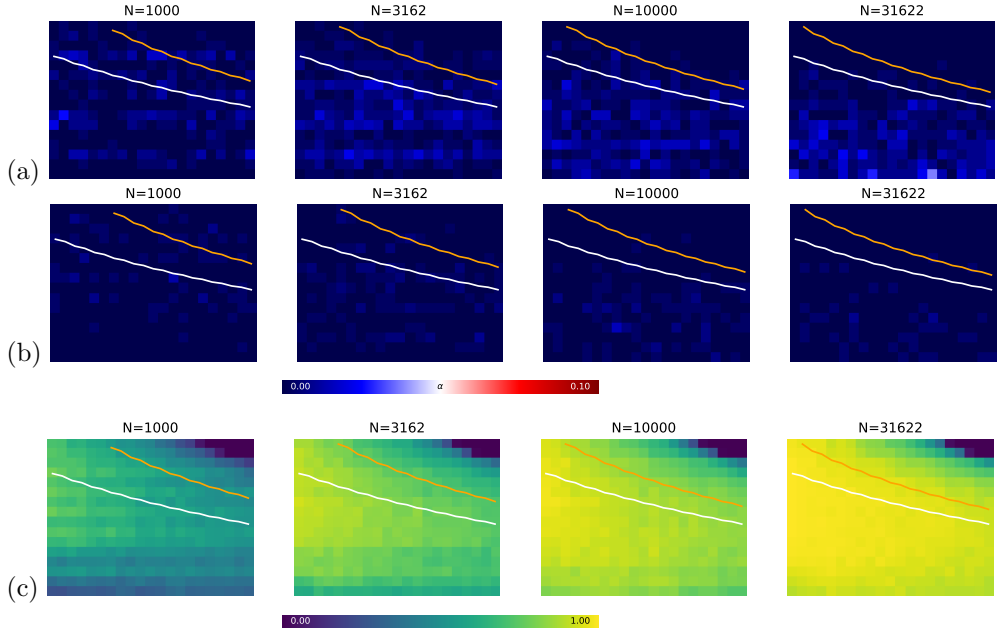

	(a)
	\mayincludegraphics[width=0.9\textwidth]
	{homogeneity/conditional_large/local_05_fpr_independent}\\
	(b)
	\mayincludegraphics[width=0.9\textwidth]
	{homogeneity/conditional_large/local_05_fpr_dependent}\\
	\hspace*{10em}
	\mayincludegraphics[width=0.3\textwidth]
	{color_scheme_fpr_05}\\[1em]
	(c)
	\mayincludegraphics[width=0.9\textwidth]
	{homogeneity/conditional_large/local_05_power_true_regime}\\
	\hspace*{10em}
	\mayincludegraphics[width=0.3\textwidth]
	{color_scheme_tpr}\\
	\caption[Homogeneity (hyperparams): Conditional Tests (large)]{
		\emph{Conditional Tests (large $L$):}
		Same as Fig.\ \ref{fig:homogeneity_hyper_generic_regimes_conditional}, but for ground truth regime-length in the range $\ell=100\ldots1000$,
		\ie comparatively large.}
	\label{fig:homogeneity_hyper_generic_regimes_conditional_large}
\end{figure}

\paragraph{Conditional Tests:}
So far, we have only considered pairwise tests. We now move on to conditional tests
(see §\ref{apdx:model_joint} for details about the underlying model\Slash{}data-generation).
As discussed above, the choice of $\beta$ (or rather $\beta_0=0.1$) can be kept simple.
To avoid too high-dimensional data-sets, we thus only analyze different choices of $B$
for different conditioning-set sizes $|Z|$.
Figures \ref{fig:homogeneity_hyper_generic_regimes_conditional} and
\ref{fig:homogeneity_hyper_generic_regimes_conditional_large} again show results for
generic and filtered to large regimes respectively. The x-axes of the plots now show
conditioning-set sizes from $0$ to $20$.

First, we note that again FPR-control seems to work.
For statistical power, we initially focus on the true-regime case in these figures
(see below for results on weak regimes).

We inspect panels (c) in both figures more closely.
Focusing on a single row of pixels for any of the images
shows a decrease of power with increasing $|Z|$. This does of course
make sense, as larger conditioning-sets correspond to a more complex
problem to solve.
There is further an evident \inquotes{dark corner} in the top right, corresponding
to large $|Z|$ and small $B$. Note that for partial correlation, the effective
sample count per block is $B-3-|Z|$ and if this number becomes non-positive,
no rejection of the homogeneity-hypothesis is possible.
For example the top-most row corresponds to $B=20$, so $B-3-|Z| \leq 0$ for
all $|Z| \geq 17$, \ie for the top row, the last four pixels to the right
will never have non-trivial power.

A closer look at individual columns per image to check for optimal (per $|Z|$)
values of $B$ suggests the use of larger $B$ for larger $|Z|$,
which makes sense, as we trade effective (per block) sample-size
$B - 3 - |Z|$ vs.\ rate of invalid blocks $\chi$; a trade-off which is affected by $|Z|$.
This effective-sample count also suggests a linear behavior with $|Z|$,
\ie $B = B_0 + c |Z|$. Heuristically, we find $c=\frac{3}{2}$ as a reasonable choice
(the y-axis in the images displayed in the figure is logarithmic).

\begin{figure}[ht]
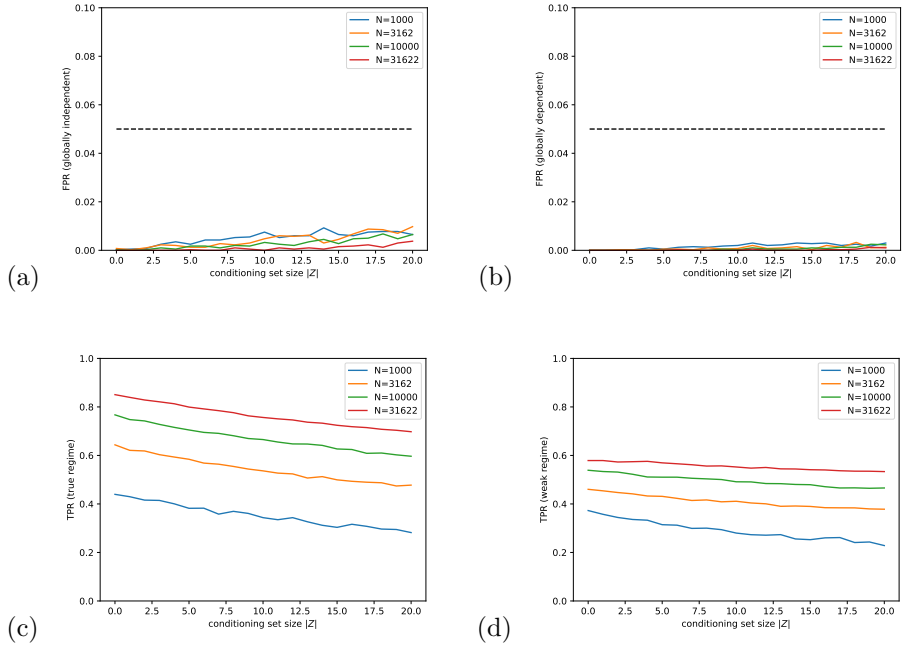

	(a)
	\mayincludegraphics[width=0.4\textwidth]
	{homogeneity/validate/fpr_independent}
	(b)
	\mayincludegraphics[width=0.4\textwidth]
	{homogeneity/validate/fpr_dependent}\\[1em]
	(c)
	\mayincludegraphics[width=0.4\textwidth]
	{homogeneity/validate/tpr_true_regime}
	(d)
	\mayincludegraphics[width=0.4\textwidth]
	{homogeneity/validate/tpr_weak_regime}\\
	\caption{Behavior of false positives ((a) globally independent, (b) globally dependent) and true positives ((c) true regime, (d) weak regime) for generic
		prior of regime-lengths validating heuristic hyperparameter sets
		for different sample-sizes.}
	\label{fig:homogeneity_hyper_generic_regimes_validate}
\end{figure}
\begin{figure}[ht]
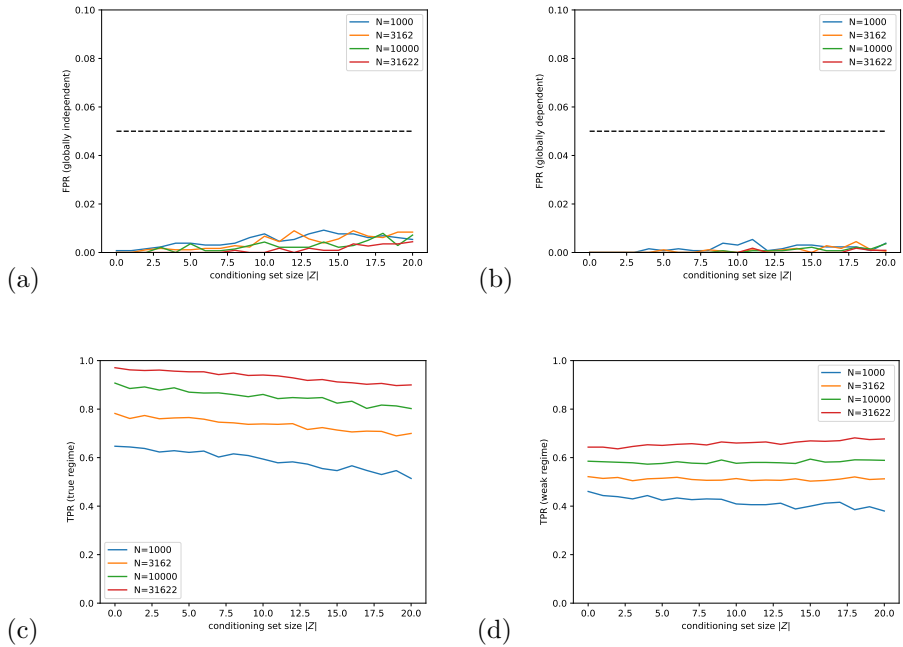

	(a)
	\mayincludegraphics[width=0.4\textwidth]
	{homogeneity/validate/fpr_independent_large}
	(b)
	\mayincludegraphics[width=0.4\textwidth]
	{homogeneity/validate/fpr_dependent_large}\\[1em]
	(c)
	\mayincludegraphics[width=0.4\textwidth]
	{homogeneity/validate/tpr_true_regime_large}
	(d)
	\mayincludegraphics[width=0.4\textwidth]
	{homogeneity/validate/tpr_weak_regime_large}\\
	\caption{Same as Fig.\ \ref{fig:homogeneity_hyper_generic_regimes_validate},
	but for ground truth regime-length in the range $\ell=100\ldots1000$,
	\ie comparatively large.}
	\label{fig:homogeneity_hyper_generic_regimes_validate_large}
\end{figure}

\paragraph{Validation and Weak Regimes}

Finally, we validate our semi-heuristic hyperparameter choices for different $N$ and $|Z|$
in figures \ref{fig:homogeneity_hyper_generic_regimes_validate} and
\ref{fig:homogeneity_hyper_generic_regimes_validate_large} (for generic and large regimes respectively).
Of particular note is the behavior on ground-truth weak-regimes (panel (d)).

The power on weak regimes is less relevant for
the full mCIT. The mCIT only requires the homogeneity test \emph{or} the global dependence
test to have power on weak regimes. Global dependence primarily suffers a loss of power from
cancellations between regimes. These occur for regimes with different signs (of correlation).
Different signs tend to lead to large separation $\Delta d = |d_1 - d_0|$.
Homogeneity tends to suffer from \emph{low} separation $\Delta d$. Thus
both tests are complementary in a very fortunate way.
The homogeneity test, for example in Fig.\ \ref{fig:homogeneity_hyper_generic_regimes_validate}d,
seems to approach TPR of around $0.5$ rather quickly (different lines\Slash{}colors),
while it afterwards improves only much slower. To our understanding,
this reflects the coin-flip in data-generation on relative sign of $d_0$ and $d_1$.
Since (see above) the global dependence test works well for same-sign scenarios,
it is in practice (for the mCIT) sufficient that the homogeneity-test
gains power fast on the different-sign scenarios. This is reflected by the
fast approach to around $0.5$ in the figure.

The choice of $\alpha$ (not shown in included figures), does not seem to substantially
affect the choice of other hyperparameters and produces very similar output.

\FloatBarrier

\subsection{Weak Regimes}\label{apdx:testing_weak_regime}

We analyze the weak regime problem (there are two regimes with
dependencies $d_0$ and $d_1$, but both are non-zero) and the proposed
acceptance-interval method (§\ref{sec:weak_regimes}) in more detail.
To do so, we will have to introduce some more notation and a few basic results concerning
truncated normal distributions and mixtures.
This makes a subsequent discussion of FPR-control, the interpretation of the
$\chi \approx 0$ assumption of the main text (cf.\ Ass.\ \ref{ass:interval})
and finally of statistical power analysis more concise.
We again conclude with semi-heuristic choices of hyperparameters.

\subsubsection{Truncated Normal Distributions}

We need to fix notation and collect some standard results on truncated normal distributions.
This is a rather standard topic, for a textbook treatment see for example
\citep[§13]{ContinuousUnivariateDistr1}. Most computations can be
reduced to statements about mills-ratios for which a vast array of
approximations are known, see for example \citep{gasull2014approximating}.
\begin{definition}[Truncated Normal Distribution]
	Let $X \sim \NormalDist(\mu, \sigma^2)$ and $a < b$, then
	$P(X|a \leq X \leq b)$ is the truncated at $a$ and $b$ normal distribution.
	If $a = -\infty$, $b=c$, we call $P_{\leq c}:=P(X|X\leq c)$ truncated above the cutoff $c$.
\end{definition}
\begin{notation}\label{notation:weak_regime}
	We use the following notations (usually these will be recalled or be clear from context where used):
	\begin{align*}
		\beta &= \frac{c-\mu}{\sigma} \txt{ (for a truncated above $c$ normal
			with parameters $\mu$, $\sigma$)} \\
		\varphi(t) &\txt{ is the standard-normal pdf at $t$}\\
		\phi(t) &\txt{ is the standard-normal cdf at $t$}\\
		\phi^{-1}(p) &\txt{ is the standard-normal ppf (inverse cdf) at $p$}\\
		\mu_{m,s}^c &= E_{\mu=m, \sigma=s}[X|X\leq c]\\
		m_{\mu, \sigma}^{c}(n) &= n^{-1}\sum_{i=1}^n X \txt{ for } X\sim \NormalDist(\mu, \sigma)_{\leq c}
		\txt{ is the random variable describing}\\&\quad\txt{the sample-mean of a truncated normal distribution.} \\
		\sigma_\alpha &= \sigma \phi^{-1}\left(1-\frac{\alpha}{2}\right)
			\txt{ is the two-sided $\alpha$ confidence}
		\txt.
	\end{align*}
	Further, given $n$ IID data points $D = \{x_1, \ldots, x_n\}$,
	then:
	\begin{align*}
		D_c &= \{ x_i \in D | x_i \leq c \} \txt{ is the cut-off data-set with}\\
		n_c &= \setelemcount{D_c} \txt{ data points.}
	\end{align*}
\end{notation}

\begin{lemma}[Properties of Truncated Normal Distributions]\label{lemma:truncated_normal_expectation}
	For a normal distributed
	$X \sim \NormalDist(\mu,\sigma)$ the truncation below
	$c$ has the following properties:
	\begin{align*}
		E[X|X\leq c] &= \mu - \sigma \frac{\varphi(\beta)}{\phi(\beta)}\\
		\Var[X | X\leq c] &= \sigma^2 \left(1- \beta \frac{\varphi(\beta)}{\phi(\beta)} -
		\Big(\frac{\varphi(\beta)}{\phi(\beta)}\Big)^2\right) < \sigma^2
		\txt.
	\end{align*}
\end{lemma}% cf. 22.07., 23.09.
\begin{proof}
	These are standard results, see for example \citep[§13, Eq.\ 13.134, Eq.\ 13.135]{ContinuousUnivariateDistr1}.
	A simple computation can already give good insight however:
	If $X \sim \NormalDist(0,1)$ is standard normal, then the density
	of $P(X|\alpha\leq X\leq \beta)$ is
	$p(x) = \frac{\varphi(x)}{\phi(\beta)-\phi(\alpha)}$, thus
	\begin{equation*}
		 E[X|\alpha\leq X\leq \beta]
		 = \frac{1}{\phi(\beta)-\phi(\alpha)} \int_\alpha^\beta x \varphi(x) \dx
		 \txt.
	\end{equation*}
	A change of variables to $y:= \frac{x^2}{2}$ (and splitting the integral at $0$, if
	$\alpha$, $\beta$ have different signs) plugged into the explicit form of $\varphi(x)$,
	leads to an integral of the form $\int e^{-y}$ which can of course be solved.
	By reinstating $x$ the result of the integration is again of the form
	$\varphi(\beta)-\varphi(\alpha)$. With the normalization of $p(x)$ and setting
	$\alpha = -\infty$ we get the result for $\mu=0$ and $\sigma=1$.
	The general case follows by suitable change of variables.
	
	The variance result can be obtained for example by using $s:= \sigma^{-2}$
	and $\frac{d}{ds} \varphi\big(\frac{x-\mu}{\sigma}\big) =
	-\frac{1}{2}(x-\mu)^2\varphi\big(\frac{x-\mu}{\sigma}\big)$
	to solve the quadratic integral (by exchanging order of integration with differentiation).
	
	For the inequality for the variance,
	we have to show $\beta + \frac{\varphi(\beta)}{\phi(\beta)}\geq 0$.
	If $\beta \geq 0$, this is evident. If $\beta <0$, use $\tilde{\beta} = -\beta > 0$,
	then the above inequality becomes
	$-\tilde{\beta} + \frac{\varphi(\tilde{\beta})}{1-\phi(\tilde{\beta})}\geq 0$.
	Since $\tilde{\beta} \geq 0$, the mills-ratio
	$m(\tilde{\beta})=\frac{1-\phi(\tilde{\beta})}{\varphi(\tilde{\beta})}$
	is known to satisfy $m(\tilde{\beta}) <  \tilde{\beta}^{-1}$
	(see \eg \citep[Eq.\ 32, p.\ 1844]{gasull2014approximating}, which gives a slightly stronger statement).
	Thus $\frac{\varphi(\tilde{\beta})}{1-\phi(\tilde{\beta})} > \tilde{\beta}$
	and therefore indeed	
	$-\tilde{\beta} + \frac{\varphi(\tilde{\beta})}{1-\phi(\tilde{\beta})}\geq 0$.
\end{proof}

\subsubsection{Mixtures}

The direct testing problem naturally is formulated on mixtures (see also the illustration in
Fig.\ \ref{fig:univariate_problem}).
Formally we define mixtures follows:

\begin{definition}[Mixture]\label{def:mixture}
	Given two random variables $X$ and $Y$, then for $\lambda \in [0,1]$,
	the $\lambda$-mixture $(Z,C)$ consists of an observed random variable
	$Z = C Y + (1-C) X$ and an unobserved binary random variable $C \in \{0,1\}$
	with $P(C=0) = \lambda$.
	
	Given two random variables $X$ and $Y$,
	then a ($n^X$, $n^Y$)-mixed data-set of size $n=n^X + n^Y$ is
	is a set $D=\{z_1, \ldots, z_n\}$ drawn from (an $n$-fold product of) a $\lambda$-mixture (see above)
	conditioned on $\sum_i c_i = n^Y$; the choice of $\lambda$ is arbitrary (by the conditioning),
	as long as $\lambda \in (0,1)$ so that the conditioned distribution is defined.
\end{definition}

\begin{lemma}[Means of Cutoff Mixtures]\label{lemma:mixed_sum_cutoff_general}
	Let $X \sim \NormalDist(\mu, \sigma)$, and $Y \sim Q$ independent stochastically ordered $X \preceq Y$
	(Def.\ \ref{def:stochastic_ordering}) random variables.
	Given is a $(n^X, n^Y)$ mixed data-set $z_1, \ldots, z_n$ of size $n$ (cf.\ Def.\ \ref{def:mixture}),
	set $\lambda := \sfrac{n^X}{n}$.
	Further define $q_\lambda$ as the $\lambda$-quantile of this mixture.
	Then $\forall c\leq q_\lambda$, on the cut-off dataset $D_c=\{z_i| z_i < c\}$
	with $n_c$ elements and mean
	\begin{equation*}
		m:=n_c^{-1}\sum_{i|z_i<c} z_i
	\end{equation*}
	the following inequalities are satisfied for large $n_c$:
	\begin{align*}
		\Pr\Big(m > \mu + \frac{\sigma_\alpha}{\sqrt{n_c}}\Big) &\leq \frac{\alpha}{2}\\		
		\Pr\Big(m < \mu^c_{\mu,\sigma} - \frac{\sigma_\alpha}{\sqrt{n_c}}\Big) &\leq \frac{\alpha}{2}
	\end{align*}
\end{lemma}
\begin{proof}
	From our definition of mixtures (Def.\ \ref{def:mixture}),
	it is clear that we can formally (using the unobserved but well-defined values $C_i$)
	count data points in the cut-off mixture by their origin ($X$ or $Y$)
	by defining $n_c^X := \sum_{i|z_i<c} c_i$ and $n_c^Y := n_c - n_c^X$
	(in analogy to the non-cutoff $n^X$ and $n^Y$).
	We introduce the notation (cf.\ Fig.\ \ref{fig:tradeoffs})
	$n_{\txt{lost}} := n^X - n_c^X$ the number of data points \inquotes{lost}
	by cutting of at $c$ and call data points $n_c^Y$ from $Y$ included into the cut-off
	data-set impurities.
	
	From $c\leq q_\lambda$ and $\lambda = \sfrac{n^X}{n}$ it is immediately evident
	that $n_c \leq n^X$, which implies a bound on impurities $n_c^Y$ by lost data points $n_{\txt{lost}}$
	\begin{equation*}
		n_c^Y = n_c - n_c^X \leq n^X - n_c^X = n_{\txt{lost}}
		\txt.
	\end{equation*}
	Using this we readily obtain
	\begin{align*}
		m = n_c^{-1}\sum_{i|z_i<c} z_i &=
		n_c^{-1}\Big( \sum_{i|z_i<c,c_i=0} z_i + \sum_{i|z_i<c,c_i=1} z_i \Big) \\
		&< n_c^{-1} \Big( \sum_{i|z_i<c,c_i=0} z_i + n_c^Y \times c \Big) \\
		&\leq n_c^{-1} \Big( \sum_{i|z_i<c,c_i=0} z_i + \sum_{i|c\leq z_i\leq q^X,c_i=0} z_i \Big) =: m'
		\txt,
	\end{align*}
	where $q^X$ in the last term is the $\sfrac{n_c}{n^X}\leq 1$ quantile such that the two sums in the last
	line contain precisely $n_c$ terms (the first one contains $n_c^X$ terms,
	thus the second contains $n_c^Y$ terms each of magnitude $\geq c$, so the last inequality holds),
	actually the $n_c \leq n^X$ (see above) lowest values
	drawn from $X$.
	The right-hand side is a mean $m^{q^X}_{\mu, \sigma}(n_c)$
	over $n_c$ values drawn from a cut-off at $q^X$ normal-distribution,
	thus by
	\begin{align*}		
		\Pr\Big(m > \mu + \frac{\sigma_\alpha}{\sqrt{n_c}}\Big) 
		&\overset{m \leq m'}{\leq}
		\Pr\Big(m'=m^{q^X}_{\mu, \sigma}(n_c) > \mu + \frac{\sigma_\alpha}{\sqrt{n_c}}\Big)\\
		&\leq
		\Pr\Big(m^\infty_{\mu, \sigma}(n_c) > \mu + \frac{\sigma_\alpha}{\sqrt{n_c}}\Big)
		= \frac{\alpha}{2}
		\txt.
	\end{align*}
	This is the first inequality we have to show.
	
	For the second inequality we first use the stochastic ordering hypothesis
	to employ the cut-off at $c$ distribution of $X$ as a (stochastic)
	lower bound for the mixture (by $X \preceq Y$, any $\lambda$-mixture $Z$
	of $X$ and $Y$ is such that $X \preceq Z$).
	so we have to show that taking a mean over $n_c$ data points drawn from a cut-off
	above $c$ normal $\NormalDist(\mu, \sigma)$ satisfies this inequality.
	Note that for $n_c$ large enough, by CLT the distribution of this mean-value approaches a normal distribution
	with mean $\mu_{\mu,\sigma}^c$ and variance $n_c^{-1} \sigma_c^2$.
	By Lemma \ref{lemma:truncated_normal_expectation}
	$\sigma_c^2 \leq \sigma^2$, and thus the second inequality holds for large $n_c$.
\end{proof}

\subsubsection{The Approximation of 'Few' Invalid Blocks}

The main text, in Ass.\ \ref{ass:interval}, assumes $\chi \approx 0$,
\ie there are \inquotes{few} invalid blocks. In the finite-sample case (and in practice) this can
be understood as an approximation: As long as $\chi$ is small, the formal results provided
yield a reasonable estimate.
In the case of an asymptotic limit, a little bit more care should be taken
in defining what \inquotes{$\chi \approx 0$} is supposed to mean formally, and in validating
that this approximation is meaningful (can be satisfied, hopefully plausibly so)
even in a limit.

To this end, we first proof a slight modification of 
the soundness of the interval test (Lemma \ref{lemma:interval_method}),
the version in the main text follows as described in Rmk.\ \ref{rmk:interval_method_main_text_version} below:
\begin{lemma}[Acceptance Interval]\label{lemma:interval_method_apdx}
	Choose a cutoff $c>0$
	and a minimal valid size $n_0^{\txt{min}}$, compute the mean $\hat{d}_{\txt{weak}}$
	of blocks of dependence $\hat{d}$
	with $\hat{d} < c$
	where the number of such blocks is $n_c$.
	\begin{equation*}
		\hat{d}_{\txt{weak}} = n_c^{-1} \sum_{\tau|\hat{d}_\tau\leq c} \hat{d}_\tau
	\end{equation*}
	We assume Ass.\ \ref{ass:underlying_d}\Slash{}\ref{ass:apdx_sqrt_consistent_d}
	(there is an underlying $\hat{d}$ with known or estimated null-distribution,
	for its variance on block-size $B$ write $\sigma_B^2$)
	additionally satisfying Ass.\ \ref{ass:est_d_monotonicity} and
	Ass.\ \ref{ass:est_d_mixing} (estimates by $\hat{d}$ are stochastically
	ordered, and on invalid blocks are stochastically ordered relative to valid blocks),
	Ass.\ \ref{ass:normality_of_d} ($\hat{d}$ is approximately normal; for valid
	blocks with $d_0=0$) and Ass.\ \ref{ass:exactly_two_regimes} (there are exactly two
	regimes).
	
	Given an error-control target $\alpha > 0$,
	and a hyperparameter $\sigma_\chi(\Theta, A_c)\geq 0$ (for $A_c$ see below),
	let $\sigma_{n_c}^\alpha = n_c^{-\frac{1}{2}}\sigma_B\phi(1-\frac{\alpha}{2})$
	be the (two-sided) $\alpha$-interval around $0$.
	Define an acceptance interval $I := [\mu_{0,\sigma_B}^c-\sigma_{n_c}^\alpha, \sigma_{n_c}^\alpha+\sigma_{n_c}^\chi]$.
	Accept the null hypothesis (independence of the weak regime)
	if $n_c \geq n_0^{\txt{min}}$ and $\hat{d}_{\txt{weak}} \in I$,
	otherwise reject it.
	
	For any $N$,
	defining $A_c := \frac{n_c}{\Theta}$
	(recall that $A$ was the fraction of data points in valid $d_0$-blocks,
	and that $A \in [a-\chi, a]$, Lemma \ref{lemma:relate_A_to_chi}),
	this method controls false positives if
	\begin{enumerate}[label=(\roman*)]
		\item
			$\big( 1 - \frac{A}{A_c} \big) c \leq \sigma_\chi(\Theta, A_c)$ and
		\item
			$A \geq a^{\txt{min}}(n_0^{\txt{min}})$,
			see Rmk.\ \ref{rmk:choice_of_n0_min} below for $a^{\txt{min}}(n_0^{\txt{min}})$.
	\end{enumerate}
\end{lemma}
\begin{proof}
	By construction of $a^{\txt{min}}(n_0^{\txt{min}})$ (see Rmk.\ \ref{rmk:choice_of_n0_min}),
	if (ii) is satisfied $n_c \geq n_0^{\txt{min}}$
	(this may be true only up to $\tilde\alpha<\alpha$, in which case
	$\alpha$ below is to be replaced by $\alpha - \tilde\alpha$ to account for multiple
	testing) and we only have to check
	the rejection-criterion $\hat{d}_{\txt{weak}} \in I$.
	
	We want to apply Lemma \ref{lemma:mixed_sum_cutoff_general} with
	$X=\hat{d}_0$ on valid blocks from regime $0$
	and $Y=\hat{d}_1$ a mixture of valid blocks from regime $1$ and invalid blocks.
	By the mixing assumption (Ass.\ \ref{ass:est_d_mixing}; applying $\hat{d}$
	to invalid blocks is stochastically ordered relative to valid blocks)
	these satisfy $X \preceq Y$.
	By definition of $A$ there are $n^X = A \Theta$ valid blocks in regime $0$
	(using the naming conventions of Lemma \ref{lemma:mixed_sum_cutoff_general})
	and $n=n^X + n^Y=\Theta$. To apply Lemma \ref{lemma:mixed_sum_cutoff_general},
	we need that for $\lambda = \frac{n^X}{n}=A$, the $\lambda$-quantile $q_\lambda$
	(on all considered blocks) satisfies $c \leq q_\lambda=q_A$.
	
	If $c \leq q_A$, we can thus apply Lemma \ref{lemma:mixed_sum_cutoff_general} as is.
	If $c > q_A$, we proceed as follows: Define $\tilde{c} := q_A<c$,
	then apply Lemma \ref{lemma:mixed_sum_cutoff_general} with $\tilde{c}$,
	yielding (under the null, \ie for $d_0=0$) a mean-value
	$\tilde{m} \in [\mu_{0,\sigma_B}^c - \frac{\sigma_\alpha}{\sqrt{\tilde{n}_c}},
	\frac{\sigma_\alpha}{\sqrt{\tilde{n}_c}}]$ with probability $>1-\alpha$,
	where $\tilde{n}_c = A \Theta$
	(by choice of $\tilde{c}$).
	Note that (using $\tilde{c}<c$, see above)
	\begin{align*}
		\hat{d}_{\txt{weak}}
		&= n_c^{-1} \sum_{\tau|\hat{d}_\tau\leq c} \hat{d}_\tau \\
		&= n_c^{-1} \Big( \sum_{\tau|\hat{d}_\tau\leq \tilde{c}} \hat{d}_\tau
			+ \sum_{\tau|\tilde{c}<\hat{d}_\tau\leq c} \hat{d}_\tau
		\Big) \\
		&\leq \Big(
			\frac{\tilde{n}_c}{n_c} \tilde{n_c}^{-1} \sum_{\tau|\hat{d}_\tau\leq \tilde{c}} \hat{d}_\tau
			\Big)
			+
			\Big(
			n_c^{-1} (n_c - \tilde{n}_c) c
			\Big)\\
		&= \frac{\tilde{n}_c}{n_c} \tilde{m}
			+ \Big(1-\frac{A}{A_c}\Big) c
		\txt,
	\end{align*}
	where the inequality holds because the second sum has $n_c-\tilde{n}_c$ terms
	in the range $\tilde{c}$ to $c$.
	This provides an upper bound.
	By $\tilde{c}> 0$ (already the valid blocks from $d_0$ are distributed around zero, thus
	$q_A$ should always be positive, at least if $n_c \gg 1$, which we may assume by suitable choice
	$n_0^{\txt{min}}$)
	the second sum can be approximated \emph{below} by zero to yield
	$\hat{d}_{\txt{weak}} \geq \frac{\tilde{n}_c}{n_c} \tilde{m}$.
	From the $\alpha$-bounds on $\tilde{m}$ and $\tilde{n}_c<n_c$,
	thus $\sqrt{\frac{\tilde{n}_c}{n_c}}<1$, we have
	$m\in [\mu_{0,\sigma_B}^c - \frac{\sigma_\alpha}{\sqrt{n_c}},
	\frac{\sigma_\alpha}{\sqrt{n_c}} + \big(1-\frac{A}{A_c}\big) c]$
	up to probability $\alpha$.
	In particular, if $\big( 1 - \frac{A}{A_c} \big) c \leq \sigma_\chi(\Theta, A_c)$
	the acceptance-interval $I$ contains an $\alpha$-confidence region.
\end{proof}
\begin{rmk}[Choice of $n_0^{\txt{min}}$]\label{rmk:choice_of_n0_min}
	Instead of choosing $n_0^{\txt{min}}$, it is also possible to choose a minimum regime-size $a^{\txt{min}}$,
	or to express a given $n_0^{\txt{min}}$ as $a^{\txt{min}}(n_0^{\txt{min}}, \chi)$ as follows.
	We are given precisely $n^0 = A \Theta$ valid blocks from the $d_0$-regime,
	each of these has a $\hat{d}$ score below the cutoff $c$ with probability $\Pr(\hat{d}_0^B \leq c)$.
	Let $n_c^0$ be the number of valid $d_0$ blocks with $\hat{d}\leq c$.
	Clearly $n_c \geq n_c^0$
	and $n_c^0$ is distributed as $\BinomDist(n^0, \Pr(\hat{d}_0^B \leq c))$.
	Let $\phi_{n_c}$ be the cdf of this distribution, and define $n_0^{\txt{min}}:=\phi_{n_c}^{-1}(\alpha)$.
	If $A \geq a^{\txt{min}}$ (as assumed by the lemma) then
	$\Pr(n_c < n_0^{\txt{min}}) < \alpha$ by construction.
	Choosing appropriate $\alpha$ values here and for $\sigma_\alpha$ to account for multiple testing
	provides a valid test.
	
	In practice, neither $a$ nor $\chi$ are known, and a heuristic choice of $n_0^{\txt{min}}$ is employed
	with the understanding that $\alpha$ error control accounts for uncertainty in the dependence-estimate
	only, not for the uncertainty in the relation of $n_0^{\txt{min}}$ to $a^{\txt{min}}$.
\end{rmk}
Here $\sigma_{n_0}^\chi$ is an additional hyperparameter. This hyperparameter in many ways exists merely in a formal sense:
It is not independent of $\alpha$ and can formally also be absorbed (in a suitable sense, see below) into $c$
as long as $c < q_a$ is strictly smaller $q_a$. This is also the reason, why Lemma \ref{lemma:interval_method}
as stated in the main text is correct with the interpretation of $\chi\approx 0$ as meaning
$\chi \leq a - A_c$:
\begin{rmk}\label{rmk:interval_method_main_text_version}
	For the formulation in the main text (Lemma \ref{lemma:interval_method}), where $c < q_a$
	(thus $A_c < a$),
	we can then interpret $\chi \approx 0$ formally as $\chi \leq a - A_c$, because then
	$A\geq a-\chi \geq A_c$ and
	Lemma \ref{lemma:interval_method_apdx} applies with $\sigma_\chi=0$.
	
	Similarly, 
	for $c \leq q_A$ (that is replacing the independent regime-fraction $a$ by
	the fraction of valid independent blocks $A$ on the quantile: $q_a \mapsto q_A$),
	we have $A_c\leq A$, thus $\big(1-\frac{A}{A_c}\big) c \leq 0$ and Lemma \ref{lemma:interval_method_apdx} again applies with $\sigma_\chi = 0$.
\end{rmk}

\subsubsection{Power and Scaling Behavior}\label{apdx:weak_power}

We start by a result for a (fixed) choice $c(N)$ in Lemma \ref{lemma:interval_power_scaling_abs}.
It appears to be the case that also the choice of a quantile%
\footnote{The quantile $q_{\bar{c}}$ on the full dataset is of course known from data,
given $\bar{c}$.}
$\bar{c}$ entailing $c = \max(0, q_{\bar{c}})$ leads to a similar result.
We focus on the fixed $c(N)$ case, at it seemed to produce
more stable results in practice.
The asymptotic behavior in that case is insightful only in part
as the reliance in particular of step 2b on rejection by $n_0^{\txt{min}}$
seems to confound the finite-sample behavior rather than explain it.

Indeed a more insightful behavior is encountered
when studying infinite-duration limits,
see Lemma \ref{lemma:interval_power_finite}. In this case, $I$ will
(and can, Prop.\ \ref{prop:imposs_B}) no longer 
collapse to a single point. However, it becomes clearer which phenomena
on finite sample sizes are actually relevant for gaining power fast.
This second result as stated is not used elsewhere in this paper,
but its proof provides substantial insight into the problem.

\begin{lemma}[Interval-Test, Power for fixed $c(N)$]\label{lemma:interval_power_scaling_abs}
	Let $\Lambda : \mathbb{N} \rightarrow \mathbb{R}$ be a scale (see Def. \ref{def:limits_temporal}),
	and assume we are given a pattern $\Pattern$ (Def.\ \ref{def:blocks})
	which is asymptotically $\Lambda$-compatible (Def.\ \ref{def:pattern_compatible})
	with the indicator meta-model $\mathcal{R}$ (see intro to §\ref{apdx:model_details})
	generating our data.
	Further
	we assume Ass.\ \ref{ass:underlying_d}\Slash{}\ref{ass:apdx_sqrt_consistent_d}
	(there is an underlying $\hat{d}$ with known or estimated null-distribution,
	for its variance on block-size $B$ write $\sigma_B^2$)
	additionally satisfying Ass.\ \ref{ass:est_d_monotonicity} and
	Ass.\ \ref{ass:est_d_mixing} (estimates by $\hat{d}$ are stochastically
	ordered, and on invalid blocks are stochastically ordered relative to valid blocks),
	Ass.\ \ref{ass:normality_of_d} ($\hat{d}$ is approximately normal) and Ass.\ \ref{ass:exactly_two_regimes} (there are exactly two
	regimes).
	
	Fix the hyperparameters according to the following
	(satisfiable, see example \ref{example:choices_for_interval_power_scaling_abs_exist})
	requirements:
	\begin{enumerate}[label=(\roman*)]
		\item
			$B(N)\rightarrow \infty$, while
			 $\frac{B(N)}{\Lambda(N)}\rightarrow 0$
			 and $\Theta(N) = \lfloor\frac{N}{B(N)}\rfloor \rightarrow \infty$
		\item
			$c(N) \rightarrow 0$
			while $c(N) \sqrt{B(N)}\rightarrow\infty$,
		\item
			$\frac{n_0^{\txt{min}}(N)}{\Theta(N)}\rightarrow 0$,
			while
			$\frac{n_0^{\txt{min}}}{\Theta(N)}\exp(B(N))\rightarrow \infty$
			and $n_0^{\txt{min}}\rightarrow \infty$,
		\item
			$\alpha(N) \rightarrow 0$
			while $B(N)^{-\frac{1}{2}} n_0^{\txt{min}}(N)^{-\frac{1}{2}} \phi^{-1}(1-\frac{\alpha}{2}) \rightarrow 0$,
		\item
			$\sigma_\chi(N) \rightarrow 0$
			while $\sigma_\chi(N) \geq \varepsilon c(N)$ for some $\varepsilon >0$.
	\end{enumerate}
	The test described in Lemma \ref{lemma:interval_method_apdx} with
	hyperparameters satisfying the above requirements
	is weakly asymptotically correct (Def. \ref{def:error_control_power_meta}).
\end{lemma}
\begin{proof}
	We show first that the test gains power (steps 1 and 2),
	then that it controls false positives in a suitable sense (steps 3 and 4).	
	
	\textbf{The test gains power:}
	We show that we always reject either for the $n_c < n_0^{\txt{min}}$ criterion,
	or because $I\rightarrow\{0\}$ (step 1)
	and eventually the estimate $\hat{d}_0 \neq 0$ (step 2).
	
	\emph{Step 1 (size of acceptance region approaches $0$):}
	We first show that $I \rightarrow \{0\}$. If $n_c < n_0^{\txt{min}}$,
	we reject for the $n_c < n_0^{\txt{min}}$ criterion.
	Otherwise $n_c \geq n_0^{\txt{min}}$ and thus
	$\sigma_{n_c}^\alpha \leq \sigma_{B(N)} n_0^{\txt{min}}(N)^{-\frac{1}{2}} \phi^{-1}(1-\frac{\alpha}{2})
	\sim B(N)^{-\frac{1}{2}} n_0^{\txt{min}}(N)^{-\frac{1}{2}} \phi^{-1}(1-\frac{\alpha}{2}) \rightarrow 0$
	(by Ass.\ \ref{ass:apdx_sqrt_consistent_d} and hypothesis iv).
	The acceptance interval $I$ uses $\mu_{0,\sigma_B}^c$, \ie it
	truncates a normal with true parameter $\mu=0$, calling
	the corresponding $\beta$ (cf.\ notations \ref{notation:weak_regime})
	$\beta_0 := \frac{c-\mu}{\sigma_B}$, we thus have
	$\beta_0(N) = \frac{c(N)}{\sigma_{B(N)}} \rightarrow \infty$
	by $c(N) \sqrt{B(N)}\rightarrow\infty$ (hypothesis ii) and $\sigma_B^{-2} \sim B$ by Ass.\ \ref{ass:underlying_d},
	thus $\varphi(\beta_0(N)) \rightarrow 0$ and $\phi(\beta_0(N)) \rightarrow 1$ such that
	by Lemma \ref{lemma:truncated_normal_expectation}
	\begin{equation*}
		\mu_{0,\sigma_B}^c = -\sigma_B \frac{\varphi(\beta_0(N))}{\phi(\beta_0(N)))} \rightarrow 0
		\txt.
	\end{equation*}
	Finally $\sigma_\chi(N) \rightarrow 0$ (by hypothesis v).
	
	\emph{Step 2 (rejection outside acceptance region):}
	Since we are showing that the weak-regime test gains power,
	we study a model satisfying the alternative hypothesis, thus $d_0\neq 0$.
	We distinguish sub-cases by the sign of $d_0$ (relative to $d_1$,
	notation \ref{notation:standardize_signs}).
	
	\emph{Case a} ($d_0 < 0$):
	By $d_1 > 0$ and $\sigma_B \rightarrow 0$
	(by hypothesis i and Ass.\ \ref{ass:apdx_sqrt_consistent_d})
	we have $P(\hat{d}_1^B \leq c) \rightarrow 0$
	via $c\rightarrow 0$ (hypothesis ii) and $d_1 > 0$
	(notation \ref{notation:standardize_signs}).
	With $a > 0$ thus $m \rightarrow \mu_{d_0,\sigma_B}^c < d_0 < 0$.
	Thus eventually $m \notin I$ by step 1.
	
	\emph{Case b} ($d_0 > 0$):
	By $c \rightarrow 0$ (hypothesis ii) and $\sigma_B \rightarrow 0$
	(by hypothesis i and Ass.\ \ref{ass:apdx_sqrt_consistent_d})
	the value $d_0>0$ eventually becomes large compared to $c$ relative to the
	scale set by $\sigma_B$; with $d_1 > d_0$ this is also true for $d_1$,
	so eventually $n_c$ becomes small.
	We have to show that $n_c \rightarrow 0$ faster than $n_0^{\txt{min}} \rightarrow 0$.
	Formally $P(\hat{d}_0^B \leq c) \rightarrow \phi(\sfrac{d_0}{\sigma_B})< \const(d_0) \times \exp(-\sigma_B^{-2})$.
	And by $d_1 > d_0$ (by Ass.\ \ref{ass:est_d_mixing})
	also on invalid or dependent blocks
	$P(\hat{d}_1^B \leq c) \leq P(\hat{d}_0^B \leq c)
	<\const(d_0) \times \exp(-\sigma_B^{-2})$ is bounded by this term.
	This thus covers \emph{all} blocks in the data-set
	such that $n_c = P(\hat{d}^B\leq c) \times\Theta(N)
	< \const(d_0) \times \exp(-\sigma_B^{-2}) \times \Theta(N)$.
	Thus $\frac{n_c}{\Theta}\exp(B(N)) \rightarrow \const$
	(where potentially $\const =0$, the previous term was an upper \emph{bound};
	using again $\sigma_B^{-2}\sim B$).
	Meanwhile (by hypothesis iii) for $n_0^{\txt{min}}$ we have
	$\frac{n_0^{\txt{min}}}{\Theta(N)} \exp(B(N)) \rightarrow \infty$.
	Therefore, eventually $n_c < n_0^{\txt{min}}$, which leads to
	rejection of the null hypothesis.
	
	\textbf{The test controls false positives:}
	Whenever Lemma \ref{lemma:interval_method_apdx}
	applies, \ie if $\big( 1 - \frac{A}{A_c} \big) c \leq \sigma_\chi(\Theta, A_c)$ and $A \geq a^{\txt{min}}(n_0^{\txt{min}}, \chi)$,
	false positives are controlled up to $\alpha(N) \rightarrow 0$ and the claim will follow.
	Thus we show that the asymptotic behaviors of
	$\big( 1 - \frac{A}{A_c} \big) c \leq \sigma_\chi(\Theta, A_c)$ (step 3)
	and $A \geq a^{\txt{min}}(n_0^{\txt{min}}, \chi)$ (step 4) are such that
	the Lemma applies with probability approaching $1$.
	Further, for false-positive control, we are considering the null,
	that is $d_0=0$.

	\emph{Step 3:}
	Again by $c(N) \sqrt{B(N)}\rightarrow\infty$ (hypothesis ii) and $\sigma_B^{-2} \sim B$ (by Ass.\ \ref{ass:underlying_d}) we find $\frac{c(N)}{\sigma_B(N)} \rightarrow \infty$.
	This mean, eventually (up to any $\epsilon$) all independent valid
	blocks are in the cut-off data-set thus $n_c \geq A \Theta$.
	With $\sigma_B \rightarrow 0$ and $d_1>0$ eventually valid dependent
	blocks are not in the cut-off data-set thus $n_c \leq (A +\chi) \Theta$.
	By $\chi \rightarrow 0$ (by hypothesis on pattern compatibility),
	thus $n_c \rightarrow A \Theta$ and therefore
	$A_c \rightarrow A$. So for any $\varepsilon >0$ (in particular for the one
	in hypothesis v) eventually	$1-\frac{A}{A_c} < \varepsilon$.
	Therefore eventually $\big( 1 - \frac{A}{A_c} \big) c \leq \varepsilon c \leq \sigma_\chi$
	(the last inequality holds by hypothesis v).

	\emph{Step 4:}
	For $A \geq a^{\txt{min}}(n_0^{\txt{min}}, \chi)$, note that by asymptotic $\Lambda$-compatibility
	of the pattern (Def.\ \ref{def:pattern_compatible}) $A \rightarrow a >0$
	(using Lemma \ref{lemma:relate_A_to_chi}).
	Because	$\frac{n_0^{\txt{min}}(N)}{\Theta(N)}\rightarrow 0$ (by hypothesis iii),
	$a^{\txt{min}}(n_0^{\txt{min}}, \chi) \rightarrow 0$ (see Rmk.\ \ref{rmk:choice_of_n0_min}).
	In particular eventually $A \geq a^{\txt{min}}(n_0^{\txt{min}}, \chi)$ is satisfied
	with probability approaching one.
\end{proof}
\begin{example}[Existence of hyperparameters]
	\label{example:choices_for_interval_power_scaling_abs_exist}
	First fix a $B(N)$ with property (i), which clearly exists if $\Lambda \rightarrow \infty$.
	Choices for the other hyperparameters
	with the required scaling-properties exist
	for any such $B(N)$.
	For example:
	\begin{align*}
		c(N) &= B(N)^{-\frac{1}{4}} \\
		n_0^{\txt{min}}(N) &=
			\Theta(N) \times \max\big(\Theta(N)^{-\frac{1}{2}}, B(N)^{-1}\big)\\
		\sigma_\chi(N) &= \varepsilon c(N)\\
		\alpha(N) &= 2 - 2\phi\big(B(N)^{\frac{1}{2}}\big)
	\end{align*}
\end{example}

The increasing scale $\Lambda$ is required to ensure that the acceptance-region
asymptotically does exclude any alternative and
achieves false-positive control by ensuring applicability of
Lemma \ref{lemma:interval_method_apdx}.
The study of power under a finite acceptance region (conceding
the possibility of gaining power on some alternatives)
unveils interesting formal properties of the problem (cf.\ proof, case 2).

\begin{lemma}[Interval-Test, Power for constant $B$ and $c$]\label{lemma:interval_power_finite}
	We assume Ass.\ \ref{ass:underlying_d}\Slash{}\ref{ass:apdx_sqrt_consistent_d}
	(there is an underlying $\hat{d}$ with known or estimated null-distribution,
	for its variance on block-size $B$ write $\sigma_B^2$)
	additionally satisfying Ass.\ \ref{ass:est_d_monotonicity} and
	Ass.\ \ref{ass:est_d_mixing} (estimates by $\hat{d}$ are stochastically
	ordered, and on invalid blocks are stochastically ordered relative to valid blocks),
	Ass.\ \ref{ass:normality_of_d} ($\hat{d}$ is approximately normal; for valid
	blocks with $d_0=0$) and Ass.\ \ref{ass:exactly_two_regimes} (there are exactly two
	regimes).
	
	Fix $\alpha>0$, $B\in \mathbb{N}$ and $c > \sigma_B$, assuming Ass.\ \ref{ass:interval}
	(acceptance interval applicability).
	If $d_0 \notin [\mu_{0,\sigma_B}^c, \sigma_B]$, then the interval method asymptotically approaches unity power,
	\ie for ground-truth weak regimes with $d_0\notin [\mu_{0,\sigma_B}^c, \sigma_B]$
	the probability $p_{\txt{test}}$ of correctly classifying this scenario as weak regime
	satisfies $p_{\txt{test}} \rightarrow 1$ for $N\rightarrow\infty$.
\end{lemma}
\begin{rmk}\label{rmk:monotonicity_truncations_for_interval_power}
	We will assume for the proof that
	$\frac{\varphi(\beta)}{\phi(\beta)}$ is monotonically decreasing,
	and that $\frac{\varphi(\beta)}{\phi(\beta)}+\beta$ is monotonically increasing.
	Numerical approximation very strongly suggests that this is true.
	It should follow from the results of \citep{gasull2014approximating},
	but unfortunately the authors were unable to find an elegant argument proofing this claim.
	We do not use this lemma elsewhere, it is only presented here to elucidate the
	behavior of the test further.
\end{rmk}
\begin{proof}
Let $\epsilon > 0$ be arbitrary.
Since $d_0 \notin [\mu_{0,\sigma_B}^c, \sigma_B]$ we are in one of two cases:

\emph{Case 1} ($d_0 < \mu_{0,\sigma_B}^c$):
Let $\delta := \frac{\mu_{0,\sigma_B}^c - d_0}{2} > 0$.
For $N \geq N_0$,
the mean $m$ estimated on blocks of $\hat{d}<c$ satisfies
$m \leq d_0 + \delta < \mu_{0,\sigma_B}^c$ with probability $P > 1-\epsilon$.

\emph{Case 2} ($d_0 > \sigma_B$):
We will use two further sub-cases approximated around two limits:
The first limit has $c\rightarrow\infty$,
where the cutoff becomes irrelevant
and $\mu_{0,\sigma_B}^c \approx d_0$. The other limit has $\mu\rightarrow\infty$, where
the mean within the tail $\mu_{0,\sigma_B}^c \approx c$ approaches the cutoff%
\footnote{This is a well-known principle from extreme-value theory,
	but for our purposes a very simple approximation will suffice.}.
Most of the actual work to be done is to show that always at least one of these approximations 
is precise up to at most $\sigma_B$ such that together with $c > \sigma_B$ (by hypothesis of the Lemma) and $d_0 > \sigma_B$ (by hypothesis of the case)
the result $\mu_{0,\sigma_B}^c > 0$ can be established.
We make heavy use of the Notations \ref{notation:weak_regime}.

\emph{Case 2a} ($d_0 > \sigma_B$ and $c \geq d_0$):
Let $\delta := \frac{d_0-\sigma_B}{2} > 0$.	
For $N \geq N_0$, the mean $m$ estimated on blocks of $\hat{d}<c$ satisfies
$m \geq \mu_{d_0,\sigma_B}^c - \delta$ with probability $P > 1-\epsilon$.
In this case $\beta \geq 0$,
using that $\frac{\varphi(\beta)}{\phi(\beta)}$ is monotonically decreasing in $\beta$
(Rmk.\ \ref{rmk:monotonicity_truncations_for_interval_power}) and
$\frac{\varphi(0)}{\phi(0)}=\frac{2}{\sqrt{2\pi}}<1$ thus $\frac{\varphi(\beta)}{\phi(\beta)}<1$.
Therefore $\mu_{d_0,\sigma_B}^c = d_0 - \sigma_B \frac{\varphi(\beta)}{\phi(\beta)} > d_0 - \sigma_B > \delta$.
Thus $m \geq \mu_{d_0,\sigma_B}^c - \delta > 0$ with probability $P > 1-\epsilon$.

\emph{Case 2b} ($d_0 > \sigma_B$ and $c < d_0$):
Let $\delta := \frac{c-\sigma_B}{2} > 0$.	
For $N \geq N_0$, the mean $m$ estimated on blocks of $\hat{d}<c$ satisfies
$m \geq \mu_{d_0,\sigma_B}^c - \delta$ with probability $P > 1-\epsilon$.
Define a correction term $\gamma(\beta) := \frac{\varphi(\beta)}{\phi(\beta)} + \beta$ measuring
the error incurred by the approximation $\frac{\varphi(\beta)}{\phi(\beta)} \approx -\beta$.
Note that $\gamma(\beta)$ is monotonically increasing in $\beta$
(Rmk.\ \ref{rmk:monotonicity_truncations_for_interval_power}) and
$\gamma(0) = \frac{\varphi(\beta)}{\phi(\beta)}<1$ (see above) and $\beta \leq 0$ in the present
case, so $\gamma(\beta) < 1$.
Therefore
$\mu_{d_0,\sigma_B}^c = d_0 - \sigma_B \frac{\varphi(\beta)}{\phi(\beta)} = c - \sigma_B \gamma(\beta) > c - \sigma_B > \delta$.
Thus $m \geq \mu_{d_0,\sigma_B}^c - \delta > 0$ with probability $P > 1-\epsilon$.

In any case, we found $N_0$ such that for $N \geq N_0$,
$p_1 > 1-\epsilon$. The claim follows since $\epsilon>0$ was arbitrary.	

\end{proof}

\subsubsection{Beyond Acceptance Intervals}\label{apdx:weak_methods}

The main point of a direct test for weak-regime dependence
is the generation of a suitable data-set to access $d_0$
without the necessity to identify individual blocks as
belonging to any regime.
There are of course many ways to generate such a data-set,
such that false-positive control becomes valid under suitable
assumptions. By Prop.\ \ref{prop:imposs_B}, no general
error-control (without assumptions), is possible however.
The design of a test should thus include careful consideration
of its associated assumptions, as they contribute to the interpretation
of test-results.

\NewDocumentCommand{\drawnext}{mm}{
	\begin{tikzpicture}
		\draw (2, 0.75) node [anchor=west] {\textbf{#2}};
		
		\draw (0.3,0) node [rotate=90]{FPR};
		\draw (0.4,0) node [anchor=west]
		{\mayincludegraphics[scale=0.44]{weak/#1_FPR}};
		
		\draw (0,-1) node [rotate=90]{$F_1$};
		\draw (0.4,-1) node [anchor=west]
		{\mayincludegraphics[scale=0.44]{weak/#1_F1}};
		
		\draw (0.3,-2) node [rotate=90]{precision};
		\draw (0.4,-2) node [anchor=west]
		{\mayincludegraphics[scale=0.44]{weak/#1_precision}};
		
		\draw (0,-3) node [rotate=90]{recall};
		\draw (0.4,-3) node [anchor=west]
		{\mayincludegraphics[scale=0.44]{weak/#1_recall}};
	\end{tikzpicture}
}
\begin{figure}[h!t]
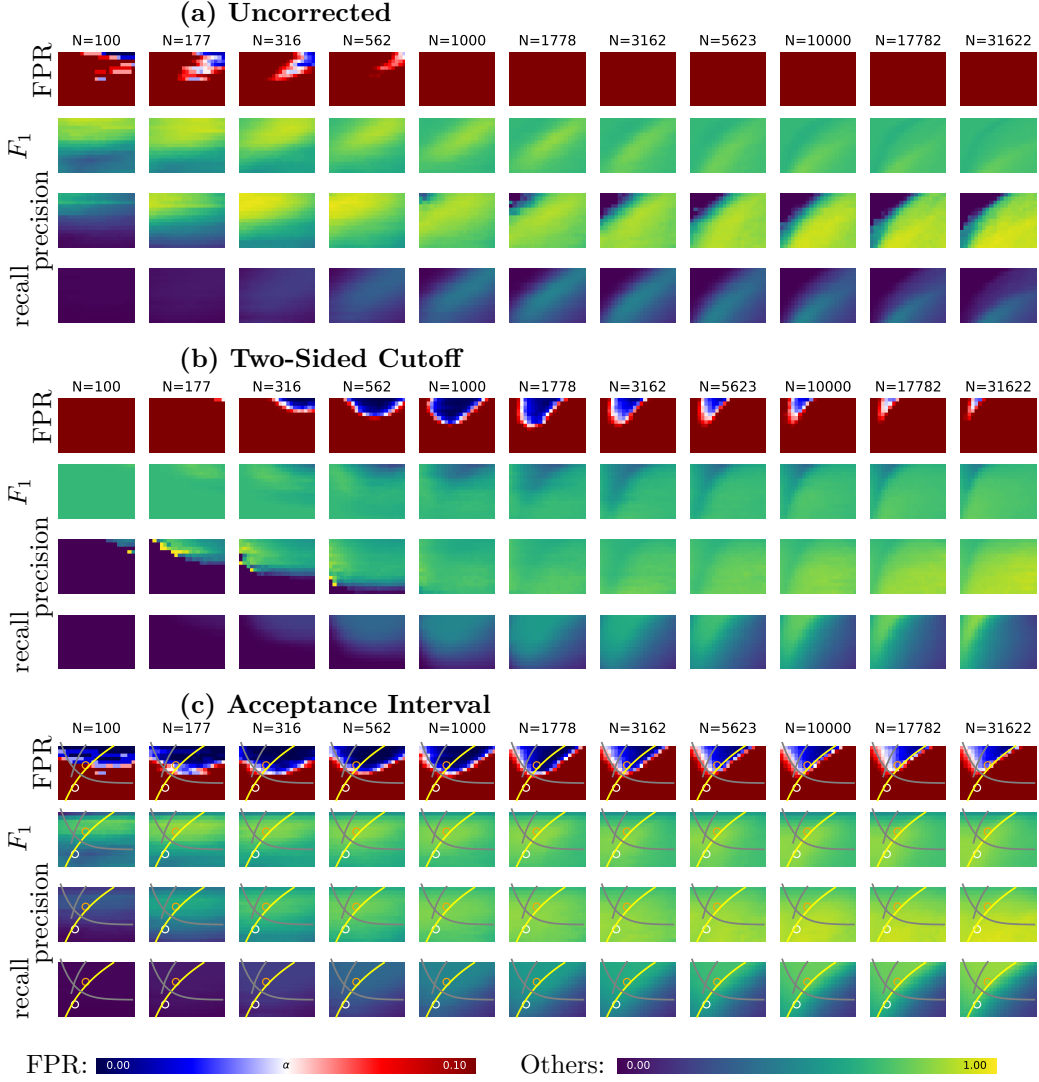

	\drawnext{uncorrected}{(a) Uncorrected}
	\drawnext{baseline}{(b) Two-Sided Cutoff}
	\drawnext{hyper_interval}{(c) Acceptance Interval}\\[0.5em]
	\hspace*{1em}FPR: \mayincludegraphics[scale=0.4]{color_scheme_fpr_05}
	\hspace*{1em}
	Others: \mayincludegraphics[scale=0.4]{color_scheme_tpr}\\
	\caption{
		\emph{Method Comparison for Weak-Regime:}
		Different metrics (see left-hand side labels)
		for the method variations described in
		§\ref{apdx:weak_methods}.
		Image coordinates correspond to hyperparameters,
		x-axis: cutoff $c=0.1\ldots 0.6$ (linear, left-to-right)
		and y-axis: block-size $B=5\ldots 50$ (logarithmic, top-to-bottom).
		The different images per row are for different sample-sizes $N=10^2\ldots10^{4.5}$
		(logarithmic, see labels).
		In panel (c), hyperparameter choices are
		indicated orange (generic regimes) and white (large regimes).
		Additional markers and curves shown in (c) are discussed in
		§\ref{apdx:weak_hyper_choices}.
	}
	\label{fig:weak_hyper_pairwise}
\end{figure}

In this section we detail two basic considerations behind our approach.
First, we discuss the relation to other, potentially simpler, approaches.
This also elucidates the applicability assumption \ref{ass:interval}.
Then, we discuss the plausibility, relevance and extension beyond
cases with exactly two regimes (Ass.\ \ref{ass:exactly_two_regimes}).

\paragraph{Simpler Methods and Applicability:}

We start by giving the limits in which an \inquotes{uncorrected} testing approach
on a cut-off data-set should work, then explore a simple but insightful baseline method.	
By an uncorrected test, we mean:
\begin{definition}[Uncorrected Test]
	Choose a cutoff $c$, independence-test $\hat{T}_{\txt{indep}}$ on samples from blocks of dependence $d$
	less than $d < c$ (we assumed signs are flipped, such that $d_1 > 0$, see above)
	or compute the mean of such blocks.
\end{definition}
Recomputing $\hat{T}_{\txt{indep}}$ by merging all relevant blocks into a single data-set
can have statistical advantages if
regimes are sufficiently homogeneous internally;
this choice does not change the overall understanding.
The uncorrected test is applicable if none of the problems (i-iii)
listed in §\ref{sec:weak_regimes} are relevant:
\begin{assumption}[Uncorrected Applicability]
	Require a large cutoff $c \gg \sigma_B$, few impurities $d_1 \gg c$
	and few invalid blocks $\chi \approx 0$ (or more intuitive: $L \gg B$).
	Since $\sigma_B \approx \sfrac{1}{\sqrt{B}}$ these can be summarized as
	(with $L$ and $d_1$ \inquotes{fixed by nature}):
	\begin{equation*}
		L \gg B \gg \sfrac{1}{\sqrt{c}} \gg \sfrac{1}{\sqrt{d_1}}
	\end{equation*}
\end{assumption}
We first inspect the condition $c \gg \sigma_B$.
Assuming $\hat{T}$ to be normal, one could compute the mean of a
truncated at $c$ normal distribution, which is known, this
is what eventually leads to the acceptance-interval approach
described above.
There is, however, a very simple procedure that may be more robust
against violations of the assumption of exactly two regimes
and illustrates some interesting
aspects of our setup:
\begin{definition}[Two-Sided Cutoff]
	Choose a cutoff $c$, test $\hat{T}$ on samples from blocks of dependence $d$
	with $d\in [-c, c]$ or compute the mean of such blocks.
\end{definition}
One could of course have chosen the absolute value of $d$ as score in the first place,
but there is an interesting feature to the distinction of the two formulations as given:
The first one (\inquotes{uncorrected}) splits the data into two regimes, then tests on one of these.
The second one (\inquotes{baseline}) splits the data into three parts, and even though for blocks
with $d<-c$ we should be rather convinced that they are likely in the low-dependence regime
(they are far from $d_1 > 0$),
we may still \inquotes{sacrifice} quality of the regime-estimates
(and data points) if that serves
our understanding of the final result which we are actually interested in.
This freedom was gained by asking \emph{directly} (in the sense of §\ref{sec:intro})
for the final result, rather than insisting on a time-resolved estimate of indicators first.	\begin{assumption}[Two-Sided Cutoff Applicability]
	Require few impurities $\sfrac{d_1}{c}\gg \sigma_B$ and few invalid blocks $\chi \approx 0$.
	\begin{equation*}
		L \gg B \quad\text{and}\quad c \sqrt{B} \gg d_1
	\end{equation*}
\end{assumption}
Note, that this cost us some data points however, and while we do not rely on
$c \gg \sigma_B$ anymore, larger $c$ will still lead to larger sample counts.

In figure \ref{fig:weak_hyper_pairwise}, some numerical results on these
different approaches (for pairwise tests, $Z=\emptyset$) are shown.
Due to the prior distribution of ground-truth parameters in the data-set,
the validity of error-control (generally depending on ground-truth and hyperparameters
point-wise)
translates at least in part into valid regions in hyperparameter space
for the \inquotes{corrected} approaches.
Generally, the acceptance-interval method is from a practical perspective
easier to apply, because the precision is rather homogeneous, making hyperparameter choices
easier. At least within the range of our numerical experiments, also the overlap
of hyperparameters with good precision \emph{and} good recall is substantially larger
than for the simpler approaches.

Finally, in terms of interpretability, the acceptance-interval method having
the simplest applicability assumption is easiest to interpret, providing yet another
advantage over simpler approaches.

\paragraph{Unified Description:}
These three approaches are all special cases of
using two (possibly different) cutoffs at the left and right hand side.
Indeed, the theoretical analysis of the acceptance-interval approach
given above also applies to this more general description (the modification
of the interval by a bias $\mu_{m,\sigma}^c$ on one side and $\sigma_\chi$ on
the other side then turns into different biases $\mu_{m,\sigma}^{c_\pm}$ 
on both sides plus possibly different choices of $\sigma_\chi$ for both sides).

The uncorrected approach then uses $c_\pm =\infty$, the two-sided cutoff
$c_+=c_-$ and the acceptance-interval $c_-=\infty$.

\paragraph{Beyond Exactly Two Regimes:}
The relevance and plausibility of the assumption of exactly two regimes
was already discussed in §\ref{apdx:simplify_regime_structure},
where one remaining and challenging case was that of multiple
dependent regimes with different signs of their dependencies.
In this case, separation $\Delta d$ needs to be sufficiently
large to the lowest absolute value of $d_1$ on either side, and
a cutoff $c$ needs to be chosen accordingly.
For smoothly varying dependence, a suitable interpretation
for example of time spent near zero needs to be adapted.
A method in the sense of the \inquotes{unified} description
(see previous paragraph) can in principle improve robustness against
these problems.
We leave a more detailed study of more sophisticated
regime-structure and a suitable interpretation of results in such cases
to future work.

\subsubsection{Semi-Heuristic Hyperparameter Choices}
\label{apdx:weak_hyper_choices}

As before (see §\ref{apdx:homogeneity_hyper_heuristic} for general principles),
we provide semi-heuristic choices for the simplified
prior knowledge of generic vs.\ large regimes (cf.\ Convention \ref{convention:hyperparam_sets}).

\paragraph{Pairwise tests:}

\begin{figure}[h!t]
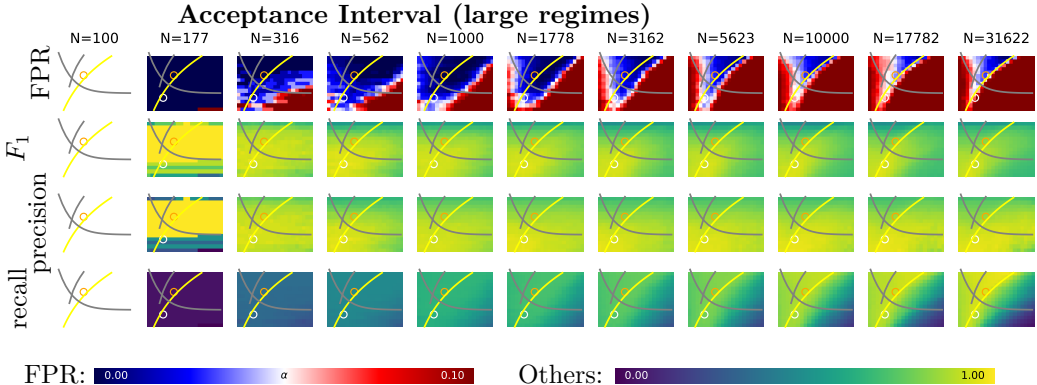

	\drawnext{large_interval}{Acceptance Interval (large regimes)}\\[0.5em]
	\hspace*{1em}FPR: \mayincludegraphics[scale=0.4]{color_scheme_fpr_05}
	\hspace*{1em}
	Others: \mayincludegraphics[scale=0.4]{color_scheme_tpr}\\
	\caption{
		\emph{Pairwise Tests of Weak-Regime (large regimes):}
		Same as Fig.\ \ref{fig:weak_hyper_pairwise}c,
		but for large ground-truth regimes.
		The first column has 0 data points, the second only $\approx 20$,
		per pixel, which is why they look unsmooth (or blank) and should not be
		over-interpreted.
	}
	\label{fig:weak_hyper_pairwise_large}
\end{figure}

In Fig.\ \ref{fig:weak_hyper_pairwise}c (generic regimes) and \ref{fig:weak_hyper_pairwise_large}
(large regimes), we show results for pairwise tests using the acceptance interval method.
The larger grid (containing smaller images) is organized by sample-count (left to right)
and different scores (rows).
Within each smaller image, the axes correspond to hyperparameters:
The x-axis corresponds to different values of the cutoff
$c=0.1\ldots 0.6$ (linear)
and the y-axis to block-size $B=5\ldots 50$ (logarithmic, from top to bottom).

We first focus on the first row of images, showing FPR.
Especially for larger $N\geq 1000$, there is surprisingly little dependence on $N$,
primarily the image gets sharpened. This is initially surprising.
We believe it is explained best, by this primarily visible structure being related
to applicability of Lemma \ref{lemma:interval_method_apdx} (given the model-prior\Slash{}benchmark
dataset) and thus FPR-control rather than finite-sample behavior in applicable cases.
Indeed the structure does seem to be interpretable in this sense.
To this end, we focus on \ref{fig:weak_hyper_pairwise_large} (large regimes) and high $N$,
where the structure visible best.

The first evident feature is the break-down of FPR-control near the left edge of
the images.
The reason
for this is in the convention\footnote{Changing this convention in the test, changes this feature in results, not shown.} of accepting
only for $n_c \geq n_0^{\txt{min}}$. This makes sense from the perspective of
detecting regimes conservatively (in doubt, return a global hypothesis and avoid clutter;
see also Rmk.\ \ref{rmk:low_sample_count_as_validity} and the discussion there, relating
this choice to \inquotes{validity assessment}).
\begin{rmk}
	It is of course possible to output a special token for rejection based on
	$n_c < n_0^{\txt{min}}$ for downstream processing to express
	potential invalidity of the test.
\end{rmk}
A more detailed analysis turns out to be more difficult.
We make the following ansatz:
\begin{equation*}
	\tag{$*$}
	n_c = a \left( \phi(\frac{c}{\sigma_B}) - \frac{B}{N} s \right)
	\txt,
\end{equation*}
where $s$ is the total number of regime-switches; further
we approximate $\sigma_B \approx B^{-\frac{1}{2}}$. This equation is motivated
by $a$ being the true regime-fraction, $\phi(\frac{c}{\sigma_B})$
being the probability of a valid $d_0$ block having $\hat{d}\leq c$
(thus contributing to the count $n_c$) and
$\frac{B}{N} s \approx \chi$ (if there is at most one switch per block,
then there are $\frac{B}{N} s$ data points in invalid blocks).
Now, $a$ is a random variable on the data-set, so really there is some
value of $\phi(\frac{c}{\sigma_B}) - \frac{B}{N} s = \nu$ where
we pass a level-line (\eg 5\% error rate) and
we make a (fixed) guess for $\xi = \frac{s}{N} \approx \frac{1}{\ell}$, the typical
ratio of switches to total data points (for the large-regime case, we know
$\xi < 0.005$).
Shown in the images as gray curves are (numerical) solutions to this equation
for $c$ and $B$ for values $\xi=0.02$, $\nu = 0.5$ (a guess for $\nu$ obtained by
based on typical values of $a$ and refinement after fixing $\xi$
by comparison to the numerical results),
corresponding to the gray downward curve (henceforth referred to as generic regime guess),
and for $\xi=0.004$, $\nu = 0.7$ corresponding to the gray upward curve 
(henceforth referred to as large regime guess).
It is not trivial to see why, but solutions to the above equation at some point transition
from a downward to an upward slope (compare the qualitatively different behavior
of the two gray lines shown in the figures). Intriguingly a similar behavior
is observed for the boundary of the valid region of the acceptance interval
test for generic and large regimes respectively.
Indeed the generic regime guess fits the left-hand side boundary 
in Fig.\ \ref{fig:weak_hyper_pairwise}c surprisingly well,
while the large regime guess does not seem to make any sense here.
In Fig.\ \ref{fig:weak_hyper_pairwise_large} on the other hand,
the large regime guess is an ok fit to the data (for large $N$), while 
the generic regime guess seems completely off.

In summary, we find both in numerical solution of an analytic approximation
and in numerical experiments a kind of phase-transition in the form of
the left boundary of valid (on average on the uninformative prior)
for FPR-control hyperparameters.
Thus in this case, the choice of hyperparameters for large as opposed
to generic regimes should reflect this qualitative change.
As a consequence, the large-regime hyperparameter set may have low regime recall
(from high weak-regime-independence FPR) on small regimes,
but perform substantially better on large regimes.

On the right hand side of the images, there is (again especially clearly visible
in Fig.\ \ref{fig:weak_hyper_pairwise_large} for large $N$), an apparent  concave valid region
at the top.
When assessing the different validity requirements Ass.\ \ref{ass:interval}
(see also statement of Lemma \ref{lemma:interval_method_apdx}), one routinely obtains
terms of the form $\frac{c}{\sigma_B} \approx c \sqrt{B}$.
This happens for example also for the equation $(*)$ above, in the limit where
$\chi \approx 0$ (which is different from the challenge studied above)
for $n_c$, and then for $\sigma_\chi$ (by plugging into $A_c = \frac{n_c}{\Theta}$ and
then into $(1-\frac{A}{A_c}) c \leq \sigma_\chi$ in Lemma \ref{lemma:interval_method_apdx}).
Since we use $\sigma_\chi = 0 = \const$ in our numerical experiments
(see Rmk.\ \ref{rmk:interval_method_main_text_version}),
manifolds of fraction of constant rate of invalid models thus correspond
to lines of $c \sqrt{B}=\const$.
In the images of Fig.\ \ref{fig:weak_hyper_pairwise}c and
Fig.\ \ref{fig:weak_hyper_pairwise_large} 
one such level line (for $c \sqrt{B}=1$, which to our best understanding
only co-incidentally matches the data, there is no obvious reason for this
constant to be equal to $1$) has been indicated in yellow.
It matches the right-hand side border of the valid region astonishingly 
well for large $N$.

We can thus explain the observed structure qualitatively reasonably
well from (semi-)analytic ideas that depend only on rather generic
validity ideas without detailed prior knowledge about regime-structure.
So our semi-heuristic hyperparameter choices should not deviate
substantially from this FPR-validity region.
However, in practice, we are also interested
in precision and recall of the implied regime-detection.
The dependence on $N$ does not seem strong or clear enough,
to mandate $N$-dependent hyperparameters.
We pick one $(c,B)$ pair for generic regimes (orange circle;
$c=0.275$, $B=11$, these are pixel coordinates hence the odd numbers)
and one pair for large regimes (white circle; $c=0.2$, $B=31$)
based on the idea of respecting FPR-validity,
maintaining precision and then assessing for good recall
on the uninformative prior benchmark.

\paragraph{Conditional Tests:}

\begin{figure}[h!t]
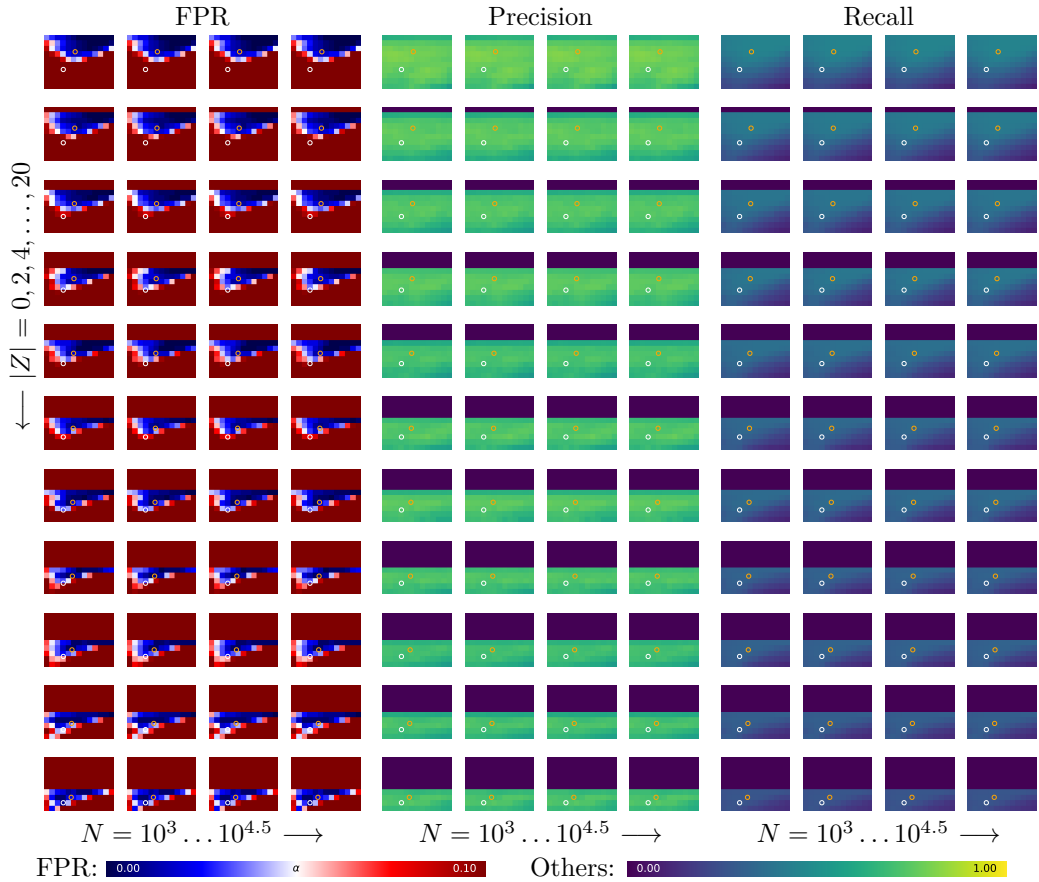

	\rotatebox{90}{$\longleftarrow$ $|Z|=0,2,4, \ldots, 20$}
	\begin{minipage}{0.3\textwidth}
		\centering
		FPR\\[0.3em]
		\mayincludegraphics[width=\textwidth]{weak/cond/interval_fpr}\\
		$N=10^3\ldots 10^{4.5}$ $\longrightarrow$
	\end{minipage}
	\hspace*{0.2em}
	\begin{minipage}{0.3\textwidth}
		\centering
		Precision\\[0.3em]
		\mayincludegraphics[width=\textwidth]{weak/cond/interval_precision}\\
		$N=10^3\ldots 10^{4.5}$ $\longrightarrow$
	\end{minipage}
	\hspace*{0.2em}
	\begin{minipage}{0.3\textwidth}
		\centering
		Recall\\[0.3em]
		\mayincludegraphics[width=\textwidth]{weak/cond/interval_recall}\\
		$N=10^3\ldots 10^{4.5}$ $\longrightarrow$
	\end{minipage}\\[0.5em]
	\hspace*{1em}FPR: \mayincludegraphics[scale=0.4]{color_scheme_fpr_05}
	\hspace*{1em}
	Others: \mayincludegraphics[scale=0.4]{color_scheme_tpr}\\[-0.5em]
	\caption{
		\emph{Conditional Weak-Regime Test:}
		FPR, precision and recall for the acceptance-interval test.
		Shown are for four different $N=10^3\ldots 10^{4.5}$
		sample-sizes (columns) and
		conditioning-set sizes $|Z|=0,2,4, \ldots, 20$
		(rows) images resolving hyperparameters
		$c=0.1\ldots 0.5$ in x-direction (linear, left-to-right)
		and $B=5\ldots 80$ in y-direction (logarithmic, top-to-bottom).
		Circles are again proposed hyperparameter choices.
	}\label{fig:weak_hyper_cond}
\end{figure}
\begin{figure}[h!t]
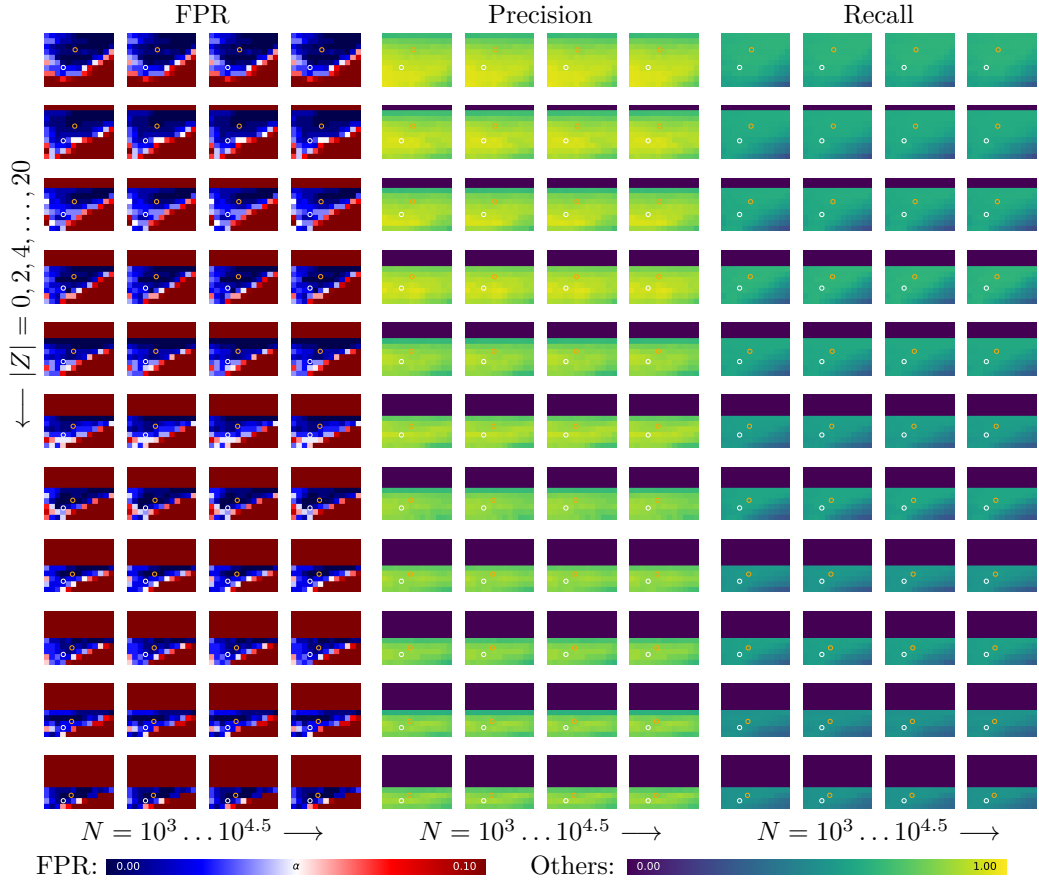

	\rotatebox{90}{$\longleftarrow$ $|Z|=0,2,4, \ldots, 20$}
	\begin{minipage}{0.3\textwidth}
		\centering
		FPR\\[0.3em]
		\mayincludegraphics[width=\textwidth]{weak/cond/large_interval_fpr}\\
		$N=10^3\ldots 10^{4.5}$ $\longrightarrow$
	\end{minipage}
	\hspace*{0.2em}
	\begin{minipage}{0.3\textwidth}
		\centering
		Precision\\[0.3em]
		\mayincludegraphics[width=\textwidth]{weak/cond/large_interval_precision}\\
		$N=10^3\ldots 10^{4.5}$ $\longrightarrow$
	\end{minipage}
	\hspace*{0.2em}
	\begin{minipage}{0.3\textwidth}
		\centering
		Recall\\[0.3em]
		\mayincludegraphics[width=\textwidth]{weak/cond/large_interval_recall}\\
		$N=10^3\ldots 10^{4.5}$ $\longrightarrow$
	\end{minipage}\\[0.5em]
	\hspace*{1em}FPR: \mayincludegraphics[scale=0.4]{color_scheme_fpr_05}
	\hspace*{1em}
	Others: \mayincludegraphics[scale=0.4]{color_scheme_tpr}\\[-0.5em]
	\caption{
		\emph{Conditional Weak-Regime Test (large):}
		Same as Fig.\ \ref{fig:weak_hyper_cond},
		but for large true regimes.
	}\label{fig:weak_hyper_cond_large}
\end{figure}

Figures \ref{fig:weak_hyper_cond} and \ref{fig:weak_hyper_cond_large}
show (for generic and for large regimes respectively),
in three large columns: FPR, precision and recall; within each large column
a grid of images, where within this grid, the x-direction is sample-size $N$,
and the y-direction is conditioning-set size $|Z|$. Each individual image
is a hyperparameter scan (x-direction: $c$, y-direction: $B$, increasing top to bottom).
The most immediate feature is a region of trivial (constant) results in the top of the images
that grows with increasing $|Z|$ (moving down in the grid).
This region of trivial decisions occurs, because there are no longer enough data points
within a single block of size $B$ (small $B$ are at the top of individual images).
Partial correlation has an effective number of degrees of freedom (per block)
of $B - 3 - |Z|$. So for the lowest row of images ($|Z|=20$), $B$ needs to be at least $24$
for the test to apply. With $B$ being shown on a logarithmic scale the middle-point
of the y-axis for $B=5\ldots 80$ is at $B=20$ (because $\frac{20}{5} = 4 = \frac{80}{20}$).
Thus this trivial region extends beyond the middle of the images in the lower part of the grid.

Besides this apparent, but easily explicable feature, the behavior seems to be similar to the
unconditional case. The trivial region at the top does not simply overlay the previous picture,
but rather \inquotes{compresses} it downward. This makes sense, as the trade-offs
of the unconditional case ($B$ vs.\ $\chi$) really are trade-offs $\sigma_B$ vs.\ $\chi$
of effective degrees of freedom, so this behavior is expected.
Again (as in the unconditional case) there does not seem to be a strong $N$-dependence,
and we choose hyperparameters independent of $N$.

In terms of actual hyperparameter choices, we start with $c$.
For the \emph{large} regime case Fig.\ \ref{fig:weak_hyper_cond_large}, there is not much
visible change, so we keep the
previous value of $c$ unchanged (and independent of $\setelemcount{Z}$).
For the generic regime case, Fig.\ \ref{fig:weak_hyper_cond_large}, best visible
for the FPR-plots, the best values of $c$ start at larger values
than for the large-regime case, but than seem to become progressively smaller.
This could be, because for larger $|Z|$ the required large (effective) size of
$B$ moves the trade-off with $\chi$ to regions of the behavior of $\chi$ more
similar to the large regime case.
The trend is rather small, and a simple linear (in $|Z|$) choice of $c$ seems to
account reasonably well for it.
For a concrete choice, see below.

The choice of $B$ mostly has to adapt for the downward-shift due to the trivial region.
this shift is linear in $|Z|$, so again a linear ansatz seems reasonable.

For concrete numerical choices (both for $c$ and for $B$),
we pick an \inquotes{optimal} pixel in the image in the top row and in the bottom row of the grid
and interpolate linearly. This leads to somewhat odd numbers, but should otherwise be acceptable.
This leads to, for large regimes: $c=0.2$ (as before) and
$B=31 + |Z| \frac{59-31}{20} \approx 31 + 1.4 |Z|$
and for generic regimes:
$c=0.275 + |Z| * \frac{0.25-0.275}{20} \approx 0.275 - 0.00125 |Z|$
and $B=11 + |Z| * \frac{43-11}{20} \approx 1.6 |Z|$.
The most relevant part is scaling of $B$ with $|Z|$ using a coefficient greater $1$
to stay outside the trivial part of the hyperparameter range.

\FloatBarrier

\subsubsection{Indicator Relation Tests}\label{apdx:implication_testing}

For the material in §\ref{sec:indicator_translation} will need a test
for indicator implications in the sense of Def.\ \ref{def:indicator_implciation}.
This problem can essentially be reduced to the question of controlling weak regimes
discussed in section §\ref{apdx:testing_weak_regime} up to this point as follows:
We first focus on the simplest case $K=1$, \ie given
the marked independencies $X_1 \independent_{R_1} Y_1 | Z_1$
with associated indicator $R_1$ (Def.\ \ref{def:detected_indicator}),
$X_2 \independent_{R_2} Y_2 | Z_2$ with associated indicator $R_2$
we want to know if
\begin{equation*}
	R_1(t) = 0
	\quad\Rightarrow\quad
	R_{2}(t) = 0
	\txt.
\end{equation*}
Previously, in this section §\ref{apdx:testing_weak_regime}, the task of weak-regime
testing was approached by assigning blocks
to a low-dependence data-set (by a cutoff),
and then test on this selected data for independence in an consistent and interpretable way
respecting the selection-bias (mean-value shift) in the truncated low-dependence results
and impurities included from the other regime.
Here, we assign blocks to a proxy data-set based on low dependence
in the first independence $X_1 \independent Y_1 | Z_1$ (see \inquotes{data-set selection} below),
and then test on this selected data for the second independence
$X_2 \independent_{R_2} Y_2 | Z_2$.

The previous case (weak-regime testing), then corresponds to the case where
$X_1 = X_2$, $Y_1 = Y_2$, and $Z_1 = Z_2$, \ie when both tests (and block-scores $d_1$, $d_2$)
are deterministically\Slash{}maximally dependent.
Now, depending on the graphical structure, both tests may be anywhere
between independent and strongly dependent.
Since the dependent case works as before, we first focus on what changes in the
independent case (or if tests\Slash{}scores become more independent).

The independent case is in part simpler,
for example, there would be no selection bias on the independent regime (mean-value shift by truncation).
However, this also means that \inquotes{lost} data points no longer cancel out
\inquotes{impurities} (for small enough cutoffs) in the sense of Lemma \ref{lemma:mixed_sum_cutoff_general}.
We account for this by a simple approximation of: The true regime-fraction
as $a^{\text{est}}$, the number of blocks below the cutoff relative to the total number of blocks
(under Ass.\ \ref{ass:interval} this is an upper bound), with the constraint that
$a^{\text{est}}\in[a_{\txt{min}},1-a_{\txt{min}}]$; The dependence of the dependent regime by the heuristic
$d_1^{\text{est}}\approx \frac{d^{\text{est}} - a^{\text{est}} d_0^{\text{est}}}{1-a^{\text{est}}}$,
where $d^{\text{est}}$ is the mean
over $\hat{d}$ of all blocks and $d_0^{\text{est}}$ the mean over blocks below the cutoff.
Then we shift the upper bound of the acceptance interval by
$d_1^{\text{est}} (1-\phi(\frac{c}{\sigma_B}))$, where the survival-function
$(1-\phi(\frac{c}{\sigma_B}))$ (for $\sigma_B$ the standard-deviation of $\hat{d}$ on blocks
of size $B$ and $c$ the cutoff) estimates the number of \inquotes{lost} data points,
thus $d_1^{\text{est}} (1-\phi(\frac{c}{\sigma_B}))$ mimics the term that in
Lemma \ref{lemma:mixed_sum_cutoff_general} cancels the contribution from impurities.

In practice, these tests\Slash{}scores are not independent:
These tests will typically
be used to resolve induced indicator problems as in example \ref{example:induced_indicators}.
So in practice the problem can be very similar to what we had before (especially given
the considerations in §\ref{sec:ind_locality}).
Indeed using the acceptance interval method proposed above we find
in numerical experiments similar behavior
as for the weak-regime case. Results tend to be worse on finite-sample experiments,
which is not surprising, considering the higher complexity of this problem. 

If the left-hand-side of the implication contains more than one indicator
$K>1$ (see Def.\ \ref{def:indicator_implciation}),
we assign the blocks to the candidate data-set if they are in the low-dependence
regime for all of the indicators on the left-hand-side.
This can lead to low sample-counts and may make it hard to reject an implication
with many terms on the left; in §\ref{sec:indicator_translation}, we show that
at least in the acyclic case, only few implication-tests and
mostly a small number of entries on the left-hand-side are required.

\begin{figure}[ht]
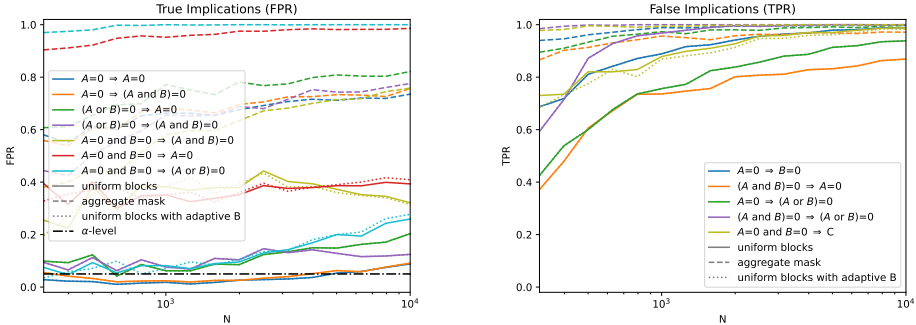

	\mayincludegraphics[width=0.45\textwidth]{implication_testing/methods_FPR}
	\hspace*{0.2em}
	\mayincludegraphics[width=0.45\textwidth]{implication_testing/methods_TPR}\\
	\caption{
		\emph{Implication Testing (methods):}
		Here segment-lengths of underlying indicators are chosen comparatively large
		$\geq 100$, as boolean combinations may produce substantially smaller
		effective segments. The present version of this test is
		preliminary, see text.
	}\label{fig:implcation_testing_methods}
\end{figure}

\paragraph{Data-Set Selection:}
As indicated above, we want to select a low-dependence candidate data-set for the right-hand side
test from the left-hand side tests.
If hyperparameters, in particular $B$, are selected test-specific (for example dynamically from data,
or because the conditioning-sets are of different size), we might have $B_1 \neq B_2$.
The simplest solution is to use the same, for example the largest of the $B_i$, value
for the block-size $B$ for all computations of $\hat{d}$ (referred to as
\inquotes{uniform blocks} below). 
Another possible approach is, to aggregate assignments of blocks on the left hand side
(potentially at different $B_i$) into a mask \emph{per-sample}.
Then compute the score for the right-hand-side test on all masked data (without blocks).
We refer to this approach as \inquotes{aggregate mask}. For partial correlation
the change in effective sample count from computing $\hat{d}$ as a mean over blocks
to computing it on all data can be used to heuristically relate the result to previous
considerations, nevertheless getting good results on numerical experiments is difficult.
A comparison of these data-set selection methods is shown in
Fig.\ \ref{fig:implcation_testing_methods}.

\begin{figure}[ht]
	\mayincludegraphics[width=0.45\textwidth]{implication_testing/large_vs_small_FPR}
	\hspace*{0.2em}
	\mayincludegraphics[width=0.45\textwidth]{implication_testing/large_vs_small_TPR}\\
	\small (a) Comparison of hyperparameter sets.\\[0.5em]
	\mayincludegraphics[width=0.45\textwidth]{implication_testing/alphas_FPR}
	\hspace*{0.2em}
	\mayincludegraphics[width=0.45\textwidth]{implication_testing/alphas_TPR}\\
	\small (b) Comparison of $\alpha$-values used in the implication test.\\[0.5em]
	\caption{
		\emph{Implication Testing (behavior):}
		Here segment-lengths of underlying indicators are chosen comparatively large
		$\geq 100$, as boolean combinations may produce substantially smaller
		effective segments. The present version of this test is
		preliminary, see text.
	}\label{fig:implcation_testing_behavior}
\end{figure}

\paragraph{Empirical Behavior}

In Fig.\ \ref{fig:implcation_testing_behavior}, we show results for different
test-scenarios and different hyperparameter choices.
In terms of FPR-control, the tests need further work, but for the present
application, these tests are much less relevant than the mCITs discussed at length before.
The primary goal of the implementation here is to demonstrate feasibility
on a conceptual level,
and given that our basic implementation already performs substantially better than
random shows that this kind of indicator--indicator comparison is a solvable problem.
Nevertheless, it is also a fundamentally hard problem, thus as already
pointed out in §\ref{sec:indicator_translation}, avoiding this type of test
where possible is likely a more important next step.

\FloatBarrier

\subsection{Marked Independence}\label{apdx:full_mCIT}

In view of the impossibility results Prop.\ \ref{prop:imposs_A} and Prop.\ \ref{prop:imposs_B},
\inquotes{strong} correctness in the na\"{i}ve sense is not possible for an mCIT.
However, there are still at least two types of meaningful results:
A \inquotes{weak Bayesian convergence} (see §\ref{apdx:models_and_testing})
shown below,
or assessment of validity, by assumptions as Ass.\ \ref{ass:interval} or
by inspecting the data see future work §\ref{apdx:assess_validity}.

By combining independence-testing, homogeneity-testing and weak-regime testing,
we get the following convergence result for the induced mCIT as specified by
Def.\ \ref{def:marked_independence_rel_d}
(\ie relative to $d=0$,
disregarding general CI-testing problems \citep{shah2020hardness},
cf.\ Rmk.\ \ref{rmk:relation_to_general_cits}):

\theoremstyle{plain}
\newtheorem*{thm_mCIT}
{Theorem \ref{thm:mCIT}}
\begin{thm_mCIT}
	We assume Ass.\ \ref{ass:underlying_d}\Slash{}\ref{ass:apdx_sqrt_consistent_d}
	(there is an underlying $\hat{d}$ with known or estimated null-distribution,
	for its variance on block-size $B$ write $\sigma_B^2$)
	additionally satisfying Ass.\ \ref{ass:est_d_monotonicity} and
	Ass.\ \ref{ass:est_d_mixing} (estimates by $\hat{d}$ are stochastically
	ordered, and on invalid blocks are stochastically ordered relative to valid blocks),
	Ass.\ \ref{ass:normality_of_d} ($\hat{d}$ is approximately normal, see Rmk.\ \ref{rmk:normality_of_d}) and Ass.\ \ref{ass:exactly_two_regimes}
	(there are exactly one or exactly two regimes).

	Let $\Lambda : \mathbb{N} \rightarrow \mathbb{R}$ be a scale (see Def. \ref{def:limits_temporal}) with $\Lambda \rightarrow \infty$,
	and assume we are given a pattern $\Pattern$ (Def.\ \ref{def:blocks})
	which is asymptotically-$\Lambda$ compatible (Def.\ \ref{def:pattern_compatible})
	with the indicator meta-model $\mathcal{R}$ (see intro to §\ref{apdx:model_details})
	generating our data.
	
	Construct an mCIT (see Algo.\ \ref{algo:mCIT}) by
	first testing homogeneity, if the result is
	\begin{enumerate}[label=(\alph*)]
		\item
		homogeneity is accepted, then test dependence (on basis of
		Ass.\ \ref{ass:underlying_d})
		\begin{enumerate}[label=(\roman*)]
			\item
			if independence is accepted, the result is global independence,
			\item
			if independence is rejected, the result is global dependence. 			
		\end{enumerate}
		\item
		homogeneity is rejected, then test for a weak regime,
		\begin{enumerate}[label=(\roman*)]
			\item
			if $d_0=0$ is accepted, the result is a true regime,
			\item
			if $d_0=0$ is rejected, the result is global dependence. 			
		\end{enumerate}
	\end{enumerate}
	Then there exist hyperparameters (cf.\ statements of Lemma \ref{lemma:power_of_homogeneity}
	and Lemma \ref{lemma:interval_power_scaling_abs}) for the binomial homogeneity-test
	and the acceptance-interval weak-regime test, such that the
	result is weakly asymptotically correct (Def.\ \ref{def:error_control_power_meta})
	as specified by Def.\ \ref{def:marked_independence_rel_d}
	(cf.\ Rmk.\ \ref{rmk:relation_to_general_cits}).
\end{thm_mCIT}
\begin{proof}
	This follows from the weak asymptotic correctness of the constituting tests, shown in
	Lemma \ref{lemma:power_of_homogeneity} (quantile estimates can be
	bootstrapped, example \ref{example:quantile_bootstrap}, or CIT-specific
	analytical estimates, see example \ref{example:quantile_parcorr_analytic}
	for partial correlation) and Lemma \ref{lemma:interval_power_scaling_abs}.
	The existence of hyperparameters is shown in examples \ref{example:binom_test_assymptotic_hyper_exist},
	\ref{example:choices_for_interval_power_scaling_abs_exist}.
	The subcases of (a), global dependence and independence, are formally trivial by our
	specification of results Def.\ \ref{def:marked_independence_rel_d} and
	Ass.\ \ref{ass:apdx_sqrt_consistent_d} (consistency of $\hat{d}$);
	in practice this means, they are distinguished correctly if the underlying
	CIT based on $\hat{d}$ is applicable (see Rmk.\ \ref{rmk:relation_to_general_cits}).
\end{proof}
\begin{rmk}
	It is important to note that hyperparameters can be chosen uniformly (without inspecting
	the model or data). For adaptive (to data) hyperparameters, see §\ref{apdx:future_adaptive_hyper}.
\end{rmk}

This includes the \inquotes{uninformative} temporal cases:

\begin{cor}\label{cor:mCIT_on_uninformative}
	Given an uninformative temporal prior (Def.\ \ref{def:uninformative_prior_temporal}),
	using a persistent-in-time pattern (Def.\ \ref{def:blocks})
	then given $B(N)$ with $B(N)\rightarrow\infty$ and $\frac{B(N)}{\log(N)}\rightarrow 0$,
	there are hyperparameters such that the mCIT described
	in Thm.\ \ref{thm:mCIT} is weakly asymptotically correct.
\end{cor}
\begin{proof}
	See example \ref{example:persistent_pattern_is_compatible}:
	A persistent-in-time pattern is asymptotically weakly $\Lambda$-compatible with
	the uninformative prior of Def.\ \ref{def:uninformative_prior_temporal},
	if $\frac{\Lambda(N)}{\log(N)} \rightarrow 0$.
	This follows from Lemma \ref{lemma:uninformative_prior_scaling_temporal}
	via Lemma \ref{lemma:time_regimes_limit_chi}.
	
	The claim then follows via the Theorem \ref{thm:mCIT}, and the fact that
	suitable hyperparameter choices exist after fixing
	$B$ first (see examples \ref{example:binom_test_assymptotic_hyper_exist},
	\ref{example:choices_for_interval_power_scaling_abs_exist}).
\end{proof}

It also includes all \inquotes{trivial} scaling cases except for
the infinite duration ($\Lambda=\const$) case, which is not possible
(Prop.\ \ref{prop:imposs_B}):

\begin{cor}\label{cor:mCIT_on_trivial_scaling}
	Let $\mathcal{R}_N(t) := \bar{\mathcal{R}}(t/\Lambda(N))$ for
	$\bar{\mathcal{R}}:\mathbb{R}\rightarrow \{0,1\}$ with lower bound $\Delta t_{\text{min}}$
	on smallest structures (Rmk.\ \ref{rmk:trivial_scaling_issues}).
	If $\Lambda(N) \rightarrow \infty$, given
	$B(N)$ with $B(N)\rightarrow\infty$ and $\frac{B(N)}{\Lambda(N)}\rightarrow 0$,
	then
	using a persistent-in-time pattern (Def.\ \ref{def:blocks})
	there are hyperparameters such that the mCIT described
	in Thm.\ \ref{thm:mCIT} is weakly asymptotically correct.
\end{cor}
\begin{proof}
	By example \ref{example:lambda_scaling_from_trivial_scaling},
	the regimes grow at least $\Lambda$-uniformly, thus by
	Lemma \ref{lemma:time_regimes_limit_chi}, $\mathcal{R}_N$ is
	$\Lambda$-compatible with the persistent-in-time pattern,
	so the claim follows via Theorem \ref{thm:mCIT}, and the fact that
	suitable hyperparameter choices exist after fixing
	$B$ first (see examples \ref{example:binom_test_assymptotic_hyper_exist},
	\ref{example:choices_for_interval_power_scaling_abs_exist}).
\end{proof}

\subsubsection{Numerical Experiments and Validation}

\begin{figure}[h!t]
	\begin{tikzpicture}
		\draw (0,0) node (base) {\mayincludegraphics[scale=0.4]{mCIT/generic_5_5}};
		\draw[->] (base.east) -- +(1.5,0)
			node[pos=0.5, above]{\small larger}
			node[pos=0.5, below]{\small regimes}
			node [anchor=west]
			{\mayincludegraphics[scale=0.35]{mCIT/large_5_5}};
		\draw[->] (base.south) -- +(0,-1)
			node[pos=0.5, left]{\small $\alpha: 5\% \searrow 1\%$}
			node[pos=0.5, right, align=left]{\small (dependence and\\\small homogeneity)}
			node [anchor=north]
			{\mayincludegraphics[scale=0.35]{mCIT/generic_1_5}};
		\draw[->] (base.south east) -- +(1.5,-1)
			node[pos=0.5, right, align=left]{\small \hspace*{1em}$\alpha_w: 5\% \nearrow 20\%$ (weak test)}
			node [anchor=north west]
			{\mayincludegraphics[scale=0.35]{mCIT/generic_5_20}};
	\end{tikzpicture}
	\caption{
		\emph{Validation of pairwise mCIT:}
		Shown is the correct rate of classification of different scenarios by the mCIT.
		The outputs of \inquotes{weak regime} will be merged into \inquotes{dependent}.
	}
	\label{fig:mCIT_validation_pairwise}
\end{figure}
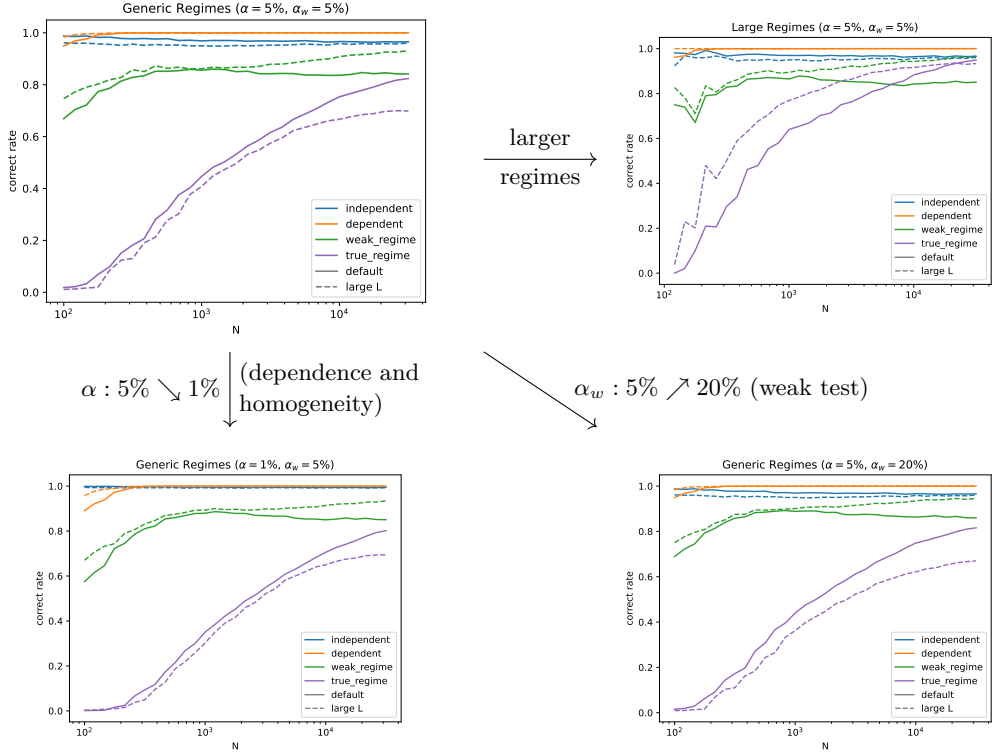

For the sake of completeness, we also included some results from numerical
experiments for the assembled full mCIT.
Figure \ref{fig:mCIT_validation_pairwise} shows the rates of correct assignment
of different ground-truths (distinguishing between week regime and globally dependent
scenarios, which is not required for the finial test-output, but informative).
As is expected from the previous discussions, there is a clear
an systematic hierarchy visible: On the standard global hypotheses
the test converges
very fast and reaches $1$ or $1-\alpha$ already for small $N$.
Next the weak regime-case increases also rather quickly, but then
(especially for the generic regimes hyperparameters) plateaus at a finite value,
likely this is related to test-validity (Ass.\ \ref{ass:interval}) being
violated at a finite rate. Finally the ratio of correct results for
the true-regime case grows slowest, which is also consistent with our previous
findings of having difficulty in achieving high recall on true regimes.

\paragraph{Conditional Tests:}
\begin{figure}[h!t]
	\mayincludegraphics[width=0.45\textwidth]{mCIT/cond/true_regime_gt_generic_hyper_generic}
	\hspace*{2em}
	\mayincludegraphics[width=0.45\textwidth]{mCIT/cond/weak_regime_gt_generic_hyper_generic}\\
	\small (a) hyperparameter set for generic regimes,
		true regime (lhs) and weak regime (rhs).	\\[0.5em]
	\mayincludegraphics[width=0.45\textwidth]{mCIT/cond/true_regime_gt_generic_hyper_large}
	\hspace*{2em}
	\mayincludegraphics[width=0.45\textwidth]{mCIT/cond/weak_regime_gt_generic_hyper_large}\\
	\small (b) hyperparameter set for large regimes,
	true regime (lhs) and weak regime (rhs).	\\[-0.5em]
	\caption{
		\emph{Validation of conditional mCIT (generic regimes):}
		Shown is the correct rate of classification of different scenarios
		(true regime and weak regime) by the mCIT.
		Image coordinates are conditioning-set size $|Z|=1\ldots 20$ on
		the (linear) x-axis and sample-size $N=10^2\ldots 10^{4.5}$
		on the (logarithmic) y-axis.
	}
	\label{fig:mCIT_validate_conditional}
\end{figure}
\begin{figure}[h!t]
	\mayincludegraphics[width=0.45\textwidth]{mCIT/cond/true_regime_gt_large_hyper_generic}
	\hspace*{2em}
	\mayincludegraphics[width=0.45\textwidth]{mCIT/cond/weak_regime_gt_large_hyper_generic}\\
	\small (a) hyperparameter set for generic regimes,
	true regime (lhs) and weak regime (rhs).	\\[0.5em]
	\mayincludegraphics[width=0.45\textwidth]{mCIT/cond/true_regime_gt_large_hyper_large}
	\hspace*{2em}
	\mayincludegraphics[width=0.45\textwidth]{mCIT/cond/weak_regime_gt_large_hyper_large}\\
	\small (b) hyperparameter set for large regimes,
	true regime (lhs) and weak regime (rhs).	\\[-0.5em]
	\caption{
		\emph{Validation of conditional mCIT (large regimes):}
		Same as Fig.\ \ref{fig:mCIT_validate_conditional}, but for ground
		truth regimes of larger length.
	}
	\label{fig:mCIT_validate_conditional_large}
\end{figure}

Figures \ref{fig:mCIT_validate_conditional} and \ref{fig:mCIT_validate_conditional_large}
show results for conditional mCITs. Shown is again the rate of
correct results, this time only for the more interesting cases of 
weak and true regimes (we do have plots for the global hypotheses,
they do not show anything new).
The main observation is that especially for the generic regime hyperparameter
set, large $N$ are required to achieve reasonable recall on true regimes.

\FloatBarrier

\subsection{Conditional Tests}\label{apdx:conditional_tests}

\begin{figure}[ht]
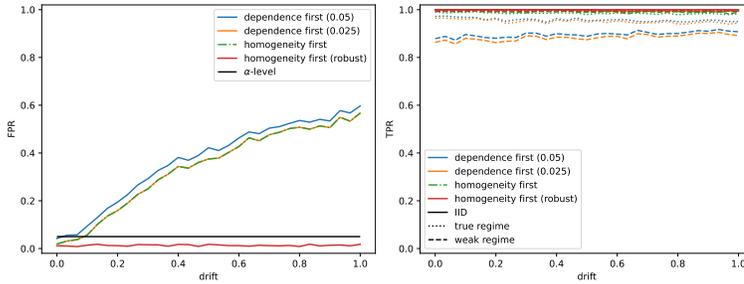

	\mayincludegraphics[width=0.35\textwidth]{conditional/drift_fpr}
	\hfill
	\mayincludegraphics[width=0.35\textwidth]{conditional/drift_power}
	\hfill
	\begin{minipage}[b]{0.28\textwidth}
		\caption{Conditional testing with and without robust configuration, that is,
			using block-wise (local in the pattern) regressing out.
			Sample-size is $N=1000$,
			block size is $B=25$.}
		\label{fig:conditional_robust}
	\end{minipage}
\end{figure}

Conditional tests have been already discussed with hyperparameter choices
(§\ref{apdx:homogeneity_hyper_heuristic}, §\ref{apdx:weak_hyper_choices}),
and explicitly in the main text at §\ref{sec:testing_conditional}.
An interesting observation that we have not yet discussed is that
a block-wise (or at least partially local in the pattern) approach
to regressing out for CIT-testing (in this case via partial correlation,
thus linear regressors), technically uses samples less efficiently in the
IID case, but provides additional robustness against non-IID (or non-stationary)
behavior in the mechanisms underlying the effect of the conditioning set $Z$
on the pair $X$, $Y$.

For example consider a simple confounder-model $X \leftarrow Z \rightarrow Y$
with linear effects $X = \alpha Z + \eta_X$, $Y= \beta Z + \eta_Y$.
In the IID case the independence $X \independent Y | Z$ is easily
found. But what if the coefficients $\alpha$ and $\beta$ become even slightly
time-dependent?
Many coefficients in real world models \emph{are} really time dependent,
think for example of electrical resistance as coefficient between voltage and
current, which is really temperature dependent, thus often de facto also time-dependent,
albeit slowly varying.
In practice we then effectively test if $\alpha(t)$ and $\beta(t)$ are sufficiently
independent -- however, even if they are, given that they vary on a much larger time-scale,
they have a much lower effective number of degrees of freedom, so analytical
estimates of the null-distribution are bound to fail if not carefully adapted;
estimates for example
by bootstrap\Slash{}shuffle approaches will likely not correctly respect this structure either.

Fig.\ \ref{fig:conditional_robust} shows the behavior of simple drifts:
$\alpha$ and $\beta$ start\Slash{}end at $1+\gamma$ times the value they
end\Slash{}start with, drifting linearly (in time) in between.
The x-axis in Fig.\ \ref{fig:conditional_robust} encodes this $\gamma$ factor
describing drift-strength ($\gamma=0$ means no drift, thus IID data).
Already at low drifts $\gamma = 0.1\ldots 0.2$ FPR-control starts to fail
visibly, a trend that becomes even much worse for stronger drifts.

We expect, already from potential non-trivial model-indicators, the presence of
some non-stationarities in mechanisms, thus we use the more robust
block-wise regressing out in the numerical results of this paper.
The reasons for this choice are similar to those given at the end of §\ref{sec:testing_order}.

\FloatBarrier

\subsection{Future Work}\label{apdx:mCIT_future_work}

We briefly summarize some relevant topics that could be subject to
future work.

\subsubsection{Adaptive hyperparameters}\label{apdx:future_adaptive_hyper}

In principle, hyperparameters for the mCIT can be chosen locally,
\ie for each test $X \independent Y | Z$ individually.
Doing so by hand is typically not a plausible approach,
however, it means that our approach could benefit particularly strongly
from automated hyperparameter choices.

To some small degree, the dependence of our simplified hyperparameter sets
on the size $\setelemcount{Z}$ of the conditioning set already chooses
parameters per test.
More general a procedure for inspecting data and adapting hyperparameters
correspondingly seems plausible.
One observation that could be leveraged to this end is
the trade-off shown in Fig.\ \ref{fig:tradeoffs}:

On the collection of values obtained by applying $\hat{d}$ to blocks,
larger blocks lead to better separated peaks at $d_0$ and $d_1$,
but they also eventually lead to more invalid blocks,
thus to worse separation of peaks (see black dash--dot line in Fig.\ \ref{fig:tradeoffs}
illustrating the combined\Slash{}full data-set distribution).
In particular by trying different block-sizes $B$, there should be
an \inquotes{optimal separation} for some $B_0$.
This $B_0$ might not yet be the optimal choice for $B$ itself,
but it can probably be related to a good choice for $B$ or leveraged
to estimate the true typical regime-length $\ell$;
numerical experiments with kernel-smoothing density estimation support this idea.

Similarly, if finite-data results are good enough to obtain a
two-peak structure, the relative size of peaks and\Slash{}or location
of minimum in between could be used to assess $a$ or directly $q_a$
to inform the choice of $c$ in order to maintain validity of the weak-regime test
(Ass.\ \ref{ass:interval}).

\subsubsection{Assessing Validity}\label{apdx:assess_validity}

An assessment of data as in §\ref{apdx:future_adaptive_hyper} could
further be employed to validate applicability-assumptions
of the weak-regime test (Ass.\ \ref{ass:interval}).
For example, given estimates (and approximations of their uncertainty)
for the terms appearing in Ass.\ \ref{ass:interval} could in principle
allow for $\frac{\alpha}{2}$ confidence applicable hyperparameters
plus $\frac{\alpha}{2}$ error-control under applicability leading
to an overall error-control at $\alpha$.
As a concrete example, consider that for partial correlation the value $\sigma_B$
and thus the width of peaks induced by valid blocks can be analytically estimated,
hence $\chi\approx 0$ can in principle be falsified by finding strong deviations from this
expected peak-width.

It will in practice (except for huge data-sets) likely be difficult to
always find $\frac{\alpha}{2}$ confidence applicable hyperparameters,
that also provide reasonable power;
indeed by Prop.\ \ref{prop:imposs_B} this may not even be possible.
A potentially practicable
compromise could be to only assess the plausibility of
used hyperparameters being valid, to inform the user (or the further
processing pipeline) about the potential issue.

\subsubsection{Leverage CPD and Clustering Techniques}
\label{apdx:leveraging_cpd_clustering}

Global CPD and clustering techniques as alternatives to our framework are
discussed in §\ref{apdx:global_cpd_clustering}.
Here we focus instead on how to leverage CPD and clustering within
our framework as approaches to exploit persistence and to solve
the univariate problem (Fig.\ \ref{fig:univariate_problem}).
In §\ref{apdx:persistence_by_blocks} we already summarized, why we currently
do \emph{not} use CPD, but instead aggregate data into a priori fixed blocks.

In principle the persistence assumption can be employed in ways different from
the approach via blocks of data used in this paper. A standard approach
would be change-point-detection (CPD) and an application of the dependence-score
on segments rather than blocks.
In our framework CPD -- when used as a drop-in for the block-based approach --
can be applied locally
for example to gradients of scores or changes in the target (often, albeit not always,
a changing link will change the variance and entropy of its target)
or on dependence-scores on (smaller) blocks or sliding windows of data,
avoiding the much more difficult
global\Slash{}multivariate problem.
Indeed for comparatively small data-sets and \emph{non-parametric independence-tests}
this may be a very practical and useful approach.

Further, the univariate problem §\ref{sec:univar_problem},
in particular the weak-regime test §\ref{sec:weak_regimes}, might substantially benefit
from the use of clustering approaches applied to (for example per block)
dependence-scores.
Indeed, as Fig.\ \ref{fig:univariate_problem} illustrates,
the univariate problem really \emph{is} a clustering problem with some
interesting constraints.
However, it seems in practice to be difficult to obtain reliable confidence intervals
for cluster-locations -- this is the reason why we formulate the weak-regime test
as a modified independence-test rather than a modified clustering method.

\subsubsection{Combining Block-Sizes}\label{apdx:combining_block_sizes}

We always only use a single block-size to assess each test.
In practice it is often better to combine multiple sizes
(see for example \citep{politis1994stationary}, using random
block-sizes to solve a different but in technical aspects related problem).
The use of different block-sizes could also be employed to 
gain insights about the reliability of a result.

Block-sizes (and even the whole aggregation approach)
need not necessarily be the same for
all variables involved in a test. \Eg for conditional tests, regressing out of conditions (or conditional shuffle-testing) could
reasonably be done on a larger time-scale than a dependence-scoring on the residuals,
see §\ref{sec:testing_conditional}.

\section{Details on the Core Algorithm}\label{apdx:core_algorithm}

Here, a formal analysis of the core algorithm is given. The main result
is a proof of its oracle-consistency (Thm.\ \ref{thm:core_algo}).

\subsection{Finding States On-The-Fly}\label{apdx:core_algo_on_the_fly}

During each iteration in the core algorithm (Algo.\ \ref{algo:core}),
the set  $J$ of marked (regime-dependent) independencies contains only
those marked independencies already found in previous iterations.
Even in the oracle case, the typical underlying CD-algorithm will
\emph{not} check all possible independence-statements, only those in
its lookup-region (see §\ref{apdx:cd_lookup_regions}; depending on which results
the checked independencies provide). This means the set $J$ found in the core-algorithm
is incomplete: generally $J \varsubsetneq \IStructOracleM^{-1}(\{\mCIToutR\})$.
We have to verify that the set $J$ (at termination of the algorithm) is yet always
large enough to find all identifiable states, and we need our state-space construction
algorithm to consistently (in a suitable sense to be specified) operate on incomplete $J$
in initial rounds.
In particular, we have to specify \inquotes{$J$-specific} state-spaces that \emph{should}
be recovered.

This subsection provides suitable $J$-specific state-space definitions from the
model-per\-spec\-ti\-ve and relates them to representations via mappings $J\rightarrow \{0,1\}$.
Then it is shown that $J$ can be discovered on-the-fly in the sense that
there is always at least one marked independence result of type \inquotes{true regime} if
$J$ is not large enough for $J$-states to capture all $\CDAlg$-identifiable states.
The core algorithm \ref{algo:core} only terminates once no such new marked results
are found (once $J'=\emptyset$), thus termination of the algorithm implies
that $J$ \emph{was} indeed large enough for $J$-states to capture all $\CDAlg$-identifiable states.
Finally, we fix a formal specification that \algConstructStateSpace{} has to satisfy
to be applicable with our proofs. This specification is satisfied by the
results of §\ref{sec:indicator_translation} (see Thm.\ \ref{thm:state_space_construction}).

\begin{notation}\label{notation:core_algo_states}
	In Def.\ \ref{def:states}, we defined states $s\in S$ which have an associated oracle independence-structure $\IStructOracle(M_s)$
	by example \ref{example:oracle}. This is itself a mapping $\MultiindexSet\rightarrow \{0,1\}$,
	but it can also be understood as a mapping in $s$.
	We denote these as follows:
	\begin{align*}
		\IStructOracle(M_*) &: S_T \rightarrow \IStructs,
			s\mapsto\IStructOracle(M_s) \\
		s &: \MultiindexSet \rightarrow \{0,1\},
			\vec{i} \mapsto \IStructOracle(M_s)(\vec{i})
	\end{align*}
	We will keep the notation of replacing a open (input)
	variable by a $*$ also in other places throughout this section.
\end{notation}

Given only partial knowledge (from the previous iteration) of marked independencies $J\subseteq \IStructs$,
we have to
\begin{enumerate}[label=(\alph*)]
	\item
	represent reached states in a practically accessible way,
	\item
	give a criterion when $J$ is large enough to find all identifiable states and
	\item 
	ensure this criterion is satisfied before the algorithm terminates.
\end{enumerate}
We start by working towards a representation.

\begin{definition}[Model $J$-States]\label{def:Jstates}
	Let $J \subseteq \MultiindexSet$ be a fixed set independence-test indices.
	We will call $J$ the marked independencies due to its usage in the core algorithm.
	On reached states $S_T$ (Def.\ \ref{def:reached_states}), define an equivalence-relation $\sim_J$ as follows
	(using notation \ref{notation:core_algo_states}):
	\begin{equation*}
		s \sim_J s' :\Leftrightarrow \forall \vec{j}\in J: s(\vec{j}) = s'(\vec{j})\txt.
	\end{equation*}
	We then define $S_T^J$ as the set of equivalence-classes and call it the set of $J$-states.	
\end{definition}
\begin{example}
	If $J = \emptyset$, then $S_T^J = \{s_{\emptyset}\}$ contains a single element, as any two states agree on $J=\emptyset$
	(and $S_T \neq \emptyset$).
	This also shows that evidently two states of the same equivalence class $s \sim_J s'$
	can have different (ground-truth) graphs $g_s \neq g_{s'}$
	if $J$ is chosen too small.
\end{example}
As the example shows, $J$ can be too small to capture all state-dependence, thus
motivating the following definition.
\begin{definition}[Marked Completeness]
	We call a set $J\subseteq \IStructs$ of marked independencies ($\CDAlg$-)complete,
	iff $\forall s,s' \in S_T$: $s\sim_J s' \Rightarrow s\sim_{\CDAlg} s'$.
\end{definition}
Completeness of $J$ ensures that all state-descriptions fit together nicely.
\begin{lemma}[Complete $J$-States Factor Identifiable States]\label{lemma:J_complete_factor_S}
	If $J$ is complete, then the quotient map $S_T \twoheadrightarrow S_T^{\CDAlg}$
	factors as $S_T \twoheadrightarrow S_T^J \twoheadrightarrow S_T^{\CDAlg}$.
\end{lemma}
\begin{proof}
	Immediate from the definition of completeness of $J$.
\end{proof}

$J$-states further have a simple representation in terms of mappings:
\begin{lemmaNotation}[Map-Represented $J$-States]\label{lemma:JstateDescr}
	There is an injective mapping (again using notation \ref{notation:core_algo_states}),
	\begin{equation*}
		S_T^J \hookrightarrow \Map(J, \{0,1\})\halfquad,\quad
		s_J \mapsto s_J|_J := \big( \vec{j} \mapsto \tilde{s}(\vec{j}) \big)\txt,
	\end{equation*}
	where $\tilde{s}$ is any representative
	of $s_J$ (for well-definedness, see proof).
	We denote the image of $S_T^J$ in $\Map(J, \{0,1\})$
	also as $S_T^J$.
	That is, we regard $J$-states $S_T^J$ as a subset of $\Map(J, \{0,1\})$.
\end{lemmaNotation}
\begin{proof}
	The mapping is well-defined:	
	Let $s_J\in S_T^J$ and representatives $\tilde{s},\tilde{s}' \in s_J$ be arbitrary.
	By construction of $s \sim_J s'$, thus for $\vec{j}\in J$: $s(\vec{j}) = s'(\vec{j})$.
	
	Injectivity:
	Let $s_J, s'_J$ with $s_J|_J \equiv s'_J|_J$ be arbitrary.
	Equivalence of mappings on $J$ precisely means $\forall \vec{j}\in J: s_J(\vec{j}) = s_J(\vec{j})$,
	which by definition \ref{def:Jstates} is equivalent to $s_J \sim_J s'_J$.
	Thus $s_J = s'_J$ in the set of $\sim_J$ equivalence-classes $S_T^J$.
\end{proof}
\begin{rmk}
	This mapping is \emph{not} necessarily well-defined for the quotient $S_T^{\CDAlg}$
	(if $J$ is not complete).
	For example on finite data test-results may not be consistent with any single graph, and many algorithms
	use some kind of conflict resolution \citep{conservativePC}.
	Additionally $\CDAlg$ will typically not even inspect every statement in $\IStructs$.
	On the other hand, the mapping $\IStructOracle(M_*):S_T \rightarrow \IStructs$
	(before taking any quotients) is generally not injective as states
	could be different with the same independence-structure,
	for example two nodes and states $X\rightarrow Y$ and $X\leftarrow Y$.
\end{rmk}

We need a formal object corresponding to the \inquotes{pseudo cit}
used in Algo.\ \ref{algo:core}, which helps to represent a given $J$-state:

\begin{definition}[$J$-Resolved Pseudo Independence-Atoms]\label{def:J_resolved_state_space}
	The $J$-resolved marked independence oracle $\IStructPseudo{M,s_J}$ for $M$ associated to $s_J \in S_T^J$
	is an extension (as a mapping) of $s_J|_J$ from $J\subseteq\MultiindexSet$
	to all of $\MultiindexSet$ by the marked independence oracle $\IStructOracleM(M)$
	of example \ref{example:oracle_marked} 
	as a marked independence structure (Def.\ \ref{def:marked_independence}), \ie:
	\begin{equation*}
		\IStructPseudo{M,s_J} : \MultiindexSet \rightarrow \{0,1,\mCIToutR\} \halfquad,\quad
		\vec{i} \mapsto
		\begin{cases}
			s(\vec{i}) & \txt{ if } \vec{i}\in J\\
			\IStructOracleM(M)(\vec{i}) & \txt{ otherwise.}
		\end{cases}
	\end{equation*}
	Further, define the trivial choice\footnote{The choice of mapping $\mCIToutR$ to $1$ (see footnote in main-text),
	is hence actually a choice of representative of the $J$-state $s_J$
	(see proofs of next results).} map $c$ as
	\begin{equation*}
		c : \{0,1,\mCIToutR\} \rightarrow \{0,1\},
		\quad
		\begin{cases}
			0 \mapsto 0 \\
			1 \mapsto 1 \\
			\mCIToutR \mapsto 1
		\end{cases}\txt.
	\end{equation*}
\end{definition}

The core-algorithm will not converge before $J$ is complete:

\begin{lemma}[Incomplete $J$ Will Grow]\label{lemma:J_must_grow}
	If $J$ is not complete,
	then $\exists \vec{j} \in \Lookup_{\CDAlg}(c\circ\IStructPseudo{M, s_J})$,
	\ie in the lookup-region (Def.\ \ref{def:lookup_region}),
	with $\vec{j} \notin J$ and
	$\IStructPseudo{M, s_J}(\vec{j}) = \mCIToutR$.
\end{lemma}
\begin{proof}
	If $J$ is not complete, then by definition $\exists s,s'\in S_T$ with
	$s \sim_J s'$ but $s \not\sim_{\CDAlg} s'$.
	By definition of $\sim_{\CDAlg}$, thus
	$\CDAlg(\IStructOracle(M_s)) \neq \CDAlg(\IStructOracle(M_{s'}))$.
	We denote the image of $s$ in $S_T^J$ by $s_J$,
	then by $s \sim_J s'$ we have $s_J = s'_J$, thus $\IStructPseudo{M, s_J}=\IStructPseudo{M, s'_J}$.
	Thus without loss of generality we may assume
	\begin{equation*}
		\tag{$*$}
		G_{s} := \CDAlg(\IStructOracle(M_s))
		\halfquad\neq\halfquad
		\CDAlg(c\circ\IStructPseudo{M, s_J}) =: G_{s_J}^{\txt{pseudo}}
		\txt,
	\end{equation*}
	because if $\CDAlg(c \circ \IStructPseudo{M, s_J} ) = \CDAlg(\IStructOracle(M_s))$
	then $\CDAlg(c \circ \IStructPseudo{M, s'_J} ) = \CDAlg(c \circ \IStructPseudo{M, s_J} ) = \CDAlg(\IStructOracle(M_s)) \neq \CDAlg(\IStructOracle(M_{s'}))$ (by the above results),
	thus switching the roles of $s, s'$ ensures $(*)$ in this case.
	
	By contraposition of the defining property of lookup-regions in Def.\ \ref{def:lookup_region},
	equation $(*)$ implies $\exists \vec{j} \in \Lookup_{\txt{pseudo}} := \Lookup_{\CDAlg}(c\circ\IStructPseudo{M, s_J})$ with
	\begin{equation*}
		\tag{$**$}
		c\circ\IStructPseudo{M, s_J}(\vec{j}) \neq \IStructOracle(M_{s})(\vec{j})
		\txt.
	\end{equation*}
	By definition of $\IStructOracleM(M)$ in example \ref{example:oracle_marked} (cf.\ Rmk.\ \ref{rmk:induced_marked_structure}),
	the marked oracle  $\IStructOracleM(M)(\vec{j})$ takes the value $0$ (or $1$) if and only if all states
	share this value $0$ (or $1$), in which case in particular the state $s$ has the oracle-value
	$\IStructOracle(M_{s})(\vec{j})=0$ (or $1$); thus the above in-equality $(**)$
	can only be satisfied for
	$\IStructOracleM(M)(\vec{j}) = \IStructPseudo{M, s_J}(\vec{j}) = \mCIToutR$ finishing the proof.
	Note that in this case, $c \circ\IStructPseudo{M, s_J}(\vec{j}) = 1$ hence $s$ was a representative
	such that $\IStructOracle(M_{s})(\vec{j})=0$.
\end{proof}

Convergence of the algorithm ensures that on the last iteration's run
no results $=\mCIToutR$ were reported. That means, no pseudo-test, on its lookup-region,
encountered a value $\mCIToutR$. In particular the choice of $c$ (and $R\mapsto 1$)
was irrelevant for the final result.

\begin{lemma}[Complete Pseudo-Atoms are Oracles]\label{lemma:apdx_core_algo_CD_pseudo_agrees}
	If, using $L^s_{\txt{pseudo}}:= L_{\CDAlg}(c \circ\IStructPseudo{M, s_J})$,
	$\forall s_J \in S_T^J$:
	$\mCIToutR\notin \IStructPseudo{M, s_J}(L_{\txt{pseudo}}^s)$
	then
	\begin{equation*}
		\forall s \in S_T: \quad
		\CDAlg\circ\IStructPseudo{M, s_J} = \CDAlg\circ\IStructOracle(M_s)
	\end{equation*}
	In other words, the following diagram commutes:\\[0.2em]
	\begin{minipage}{\textwidth}
		\centering
		\begin{tikzpicture}
			\draw (0,1) node (S) {$S_T$};
			\draw (0,-1) node (J) {$S_T^J$};
			\draw (5,0) node (G) {$\GClasses_{\CDAlg}$};
			
			\draw[->>] (S) -- (J) node [pos=0.5, anchor=north, rotate=270]{$s \mapsto s_J$};
			
			\draw[->] (S) -- (G) node [pos=0.4, anchor=south west]{$\CDAlg\circ\IStructOracle(M_*)$};
			\draw[->] (J) -- (G) node [pos=0.6, anchor=south east]{$\CDAlg\circ\IStructPseudo{M, *}$};
		\end{tikzpicture}
	\end{minipage}
\end{lemma}
\begin{proof}	
	By $\mCIToutR\notin \IStructPseudo{M, s_J}(L_{\CDAlg}(L_{\txt{pseudo}}^s))$,
	we have for $\vec{j}\in L_{\txt{pseudo}}^s\setminus J$ that
	$\IStructPseudo{M, s_J}(\vec{j})=
	\IStructOracleM(M)(\vec{j}) =\IStructOracle(M_s)(\vec{j})=s(\vec{j})$.
	Further on $J$, we have by definition $\IStructOracleM(M)(\vec{j}) =s_J(\vec{j}) = s(\vec{j})$.
	Thus $c \circ \IStructPseudo{M, s_J}|_{L_{\txt{pseudo}}^s}
	=\IStructPseudo{M, s_J}|_{L_{\txt{pseudo}}^s}
	=\IStructOracle(M_s)|_{L_{\txt{pseudo}}^s}$.
	By definition \ref{def:lookup_region} of lookup regions,
	thus $\CDAlg(c \circ \IStructPseudo{M, s_J}) = \CDAlg(\IStructOracle(M_s))$.
\end{proof}

Finally, we have to employ correctness of the CD-algorithm.
This is formally required, because $S_T^{\CDAlg}$ was defined
via $s \sim_{\CDAlg} s' :\Leftrightarrow g_s \sim_{\CDAlg} g_{s'}$
in Def.\ \ref{def:identifiable_states}, \ie while the equivalence
relation on graphs is $\CDAlg$-specific, the graphs are the true graphs.

\begin{lemma}[Reformulate Consistency]\label{lemma:apdx_core_algo_CD_consistent}
	If $\CDAlg$ is consistent (Def.\ \ref{def:cd_alg_abstract_consistent}),
	then $\forall s \in S_T: G_s = \CDAlg(\IStructOracle(M_s))$, 
	where $G_s$ is the equivalence-class in $\GClasses_{\CDAlg}$
	of the ground truth graph $g_s \in \GClassesGeneral$.
	The mapping $G_*$ factors through $S_T^{\CDAlg}$.
	In other words, the following diagram commutes:\\[0.2em]
	\begin{minipage}{\textwidth}
		\centering
		\begin{tikzpicture}
			\draw (0,1) node (S) {$S_T$};
			\draw (0,-1) node (CD) {$S_T^{\CDAlg}$};
			\draw (5,0) node (G) {$\GClasses_{\CDAlg}$};
			
			\draw[->] (S) -- (G) node [pos=0.4, anchor=south west]{$\CDAlg\circ\IStructOracle(M_*)$};
			\draw[->] (CD) -- (G) node [pos=0.5, below]{$G_*$};
			\draw[->>] (S) -- (CD);
		\end{tikzpicture}
	\end{minipage}
\end{lemma}
\begin{proof}
	This is immediate from the definition of consistence, which specified precisely
	$g_s \in \CDAlg(\IStructOracle(M_s))$.
	The factorization of $G_*$ is by definition of $S_T^{\CDAlg}$
	as the quotient of $s\sim_{\CDAlg}s' \Leftrightarrow G_s=G_{s'}$.
\end{proof}

Now we can assemble these observations into our main technical result:

\begin{lemma}[Complete Recovery is Oracle-Consistent]\label{lemma:apdy_core_algo_on_the_fly}
	If $\CDAlg$ is consistent,
	$J$ is complete
	and using $L^s_{\txt{pseudo}}:= L_{\CDAlg}(c \circ\IStructPseudo{M, s_J})$, then
	$\forall s_J \in S_T^J$ with
	$\mCIToutR\notin \IStructPseudo{M, s_J}(L_{\txt{pseudo}}^s)$:
	\begin{enumerate}[label=(\alph*)]
		\item		
		$\forall s_J \in S_T^J: G_s = \CDAlg(c \circ\IStructPseudo{M, s_J})$, 
		where $G_s$ is the equivalence-class in $\GClasses_{\CDAlg}$
		of the ground truth graph $g_s \in \GClassesGeneral$
		for any representative $s\in s_J$.
		\item 
		The images agree $\im\big(\CDAlg\circ c\circ\IStructPseudo{M, *}\big)
		=\im\big(G_*\big)$, thus
		\begin{equation*}
			\Big\{ \CDAlg(c\circ\IStructPseudo{M, s_J}) | s_J \in S_T^J \Big\}
			=
			\Big\{ G_s | s \in S_T^{\CDAlg} \Big\}
			\txt.
		\end{equation*}
	\end{enumerate}
	
	Indeed, the following diagram commutes (we already know from Lemma
	\ref{lemma:J_complete_factor_S} that the left-hand side factors like this,
	from Lemma \ref{lemma:apdx_core_algo_CD_consistent} that the outer triangle commutes
	and from Lemma \ref{lemma:apdx_core_algo_CD_pseudo_agrees} that the upper triangle
	commutes; the claim above states that the lower triangle also commutes):\\[0.2em]
	\begin{minipage}{\textwidth}
		\centering
		\begin{tikzpicture}
			\draw (0,1) node (S) {$S_T$};
			\draw (0,-1) node (J) {$S_T^J$};
			\draw (0,-2) node (CD) {$S_T^{\CDAlg}$};
			\draw (5,0) node (G) {$\GClasses_{\CDAlg}$};
			
			\draw[->>] (S) -- (J);
			\draw[->>] (J) -- (CD);
			
			\draw[->] (S) -- (G) node [pos=0.4, anchor=south west]{$\CDAlg\circ\IStructOracle(M_*)$};
			\draw[->] (J) -- (G) node [pos=0.6, anchor=south east]{$\CDAlg\circ c\circ\IStructPseudo{M, *}$};
			
			\draw[->] (CD) -- (G) node [pos=0.5, below]{$G_*$};
		\end{tikzpicture} 
	\end{minipage}
\end{lemma}
\begin{proof}
	Let $s_J \in S_T^J$ and a representative $s \in S_T$ of $s_J$ be arbitrary.
	By Lemma \ref{lemma:apdx_core_algo_CD_consistent},
	$G_s = \CDAlg(\IStructOracle(M_s))$.
	By Lemma \ref{lemma:apdx_core_algo_CD_pseudo_agrees},
	$\CDAlg(\IStructOracle(M_s))=\CDAlg(c\circ \IStructPseudo{M, s_J})$.
	This shows (a).
	
	Here well-definedness (applying $G_s$ to any representative $s$) follows indirectly,
	but it is also evident from Lemma \ref{lemma:J_complete_factor_S}.
	By this Lemma it is also clear that this choice of $s$ can be
	replaced by applying the surjective quotient-mapping
	$q': S_T^J \twoheadrightarrow S_T^{\CDAlg}$ to $s_J$.
	The previous result then reads $G_{q'(s_J)} = \CDAlg(c\circ \IStructPseudo{M, s_J})$.
	In particular, the images agree $\im(G_* \circ q')=\im(\CDAlg \circ c \circ \IStructPseudo{M, *})$
	because the mappings agree.
	With $q'$ being surjective, also the claim (b) follows.
\end{proof}

The marked independence oracle (and thus the pseudo tests) are recovered from data
by the mCIT (see §\ref{sec:dyn_indep_tests}), so
partial state-space construction needs to provide guarantees
only about finding the correct $S_T^J \subseteq \Map(J, \{0,1\})$.

\begin{definition}[State-Space Construction]
	\label{def:state_space_construction_specification}
	Fix a non-stationary model $M$.
	A state-space construction is a mapping	$S^*$
	that assigns to each set $J\subseteq \MultiindexSet$ of marked independencies
	($\forall \vec{j}\in J: \IStructOracleM(M)(\vec{j})=\mCIToutR$)
	a subset $S^J \subseteq \Map(J, \{0,1\})$ by using a suitable
	indicator-relation structure
	(\eg $\IStructR$ from Def.\ \ref{def:state_construction_partial_orders} in the case
	of the implementation for acyclic models in §\ref{apdx:state_space_construction_J}).
	\begin{equation*}
		S^* : \subsets\Big(\IStructM^{-1}(\{\mCIToutR\})\Big) \times \mathcal{I}_R
		\rightarrow \subsets\Big(\Map(J, \{0,1\})\Big)
		\halfquad,\quad
		(J, \IStructR) \mapsto S^J(\IStructR)\txt.
	\end{equation*}
	Given a $M$-oracle for this indicator-relation structure (\eg $\IStructOracleRel(M)$
	of example \ref{example:oracle_implication})
	we call $S^*:=S^*(\IStructOracleRel(M))$ $M$-sound if $\forall J: S^J\subseteq S^J_T$,
	and $M$-complete if $\forall J: S^J \supseteq S^J_T$.
	
	Fix a class of non-stationary models $\mathcal{M}$.
	We call $S^*$ $\mathcal{M}$-sound\Slash{}$\mathcal{M}$-complete
	if it is $M$-sound\Slash{}$M$-complete $\forall M\in\mathcal{M}$.
	
	If $J=\emptyset$, then the result is trivial $S^J=\Map(\emptyset, \{0,1\})$,
	otherwise the state-space construction may assume
	(in the oracle case)
	$J \supseteq \Lookup_{\CDAlg}(c\circ\IStructOracleM(M)) \cap \IStructOracleM(M)^{-1}(\{\mCIToutR\})$
	(see proof of Thm.\ \ref{thm:core_algo} in §\ref{apdx:core_algo_proof} below;
	the inclusion of this restriction on the form of $J$ into the definition
	will be further discussed in Rmk.\ \ref{rmk:state_space_constr_restricted_J}).
\end{definition}

\subsection{The Actual Algorithm}\label{apdx:core_algo_proof}

The results in §\ref{apdx:core_algo_on_the_fly} already provide the core technical requirements.
It remains to relate these results to the algorithmic implementation and
assemble the formal claims of the theorem.

\theoremstyle{plain}
\newtheorem*{thm_core_algo}
{Thm \ref{thm:core_algo}}
\begin{thm_core_algo}
	Given
	a non-stationary model $M$ (Def.\ \ref{def:model_nonstationary_scm}\Slash{}\ref{def:model_nonstationary_scm_nonadd}),
	a marked independence oracle $\IStructOracleM(M)$ (Def.\ \ref{example:oracle_marked}; invoked as
	{\normalfont{\algMarkedIndependence{}}} in algorithm \ref{algo:core}),
	an abstract CD-algorithm $\CDAlg$ (Def.\ \ref{def:cd_alg_abstract}; invoked as
	{\normalfont{\texttt{CD}}} in algorithm \ref{algo:core}),
	which is consistent (Def.\ \ref{def:cd_alg_abstract_consistent})
	on a set of models $\mathcal{M}$ including all reached states (Def.\ \ref{def:reached_states})
	$\forall s\in S_T: M_s \in \mathcal{M}$
	and
	a sound and complete on a model-class containing $M$,
	state-space construction (Def.\ \ref{def:state_space_construction_specification}; invoked as
	{\normalfont{\algConstructStateSpace{}}} in algorithm \ref{algo:core}).
	
	Then algorithm \ref{algo:core} terminates after a finite number of iterations
	and its output is a set (\ie no ordering or characterization of states by model properties
	is implied) which consists of precisely one graph (a $\CDAlg$ equivalence class, Def.\ \ref{def:cd_alg_abstract_target}) per $\CDAlg$-identifiable state $s\in S_T^{\CDAlg}$
	\begin{equation*}
		\Big\{ G_s \big| s\in S_T^{\CDAlg} \Big\}
		\quad\text{such that}\quad
		\forall \tilde{s} \in s:
		g_{\tilde{s}} \in G_s
		\text,
	\end{equation*}
	where $g_{\tilde{s}}$ is the ground-truth graph in state $\tilde{s}$
	(which is a representative of the $\CDAlg$-identifiable state $s$).
\end{thm_core_algo}

\begin{proof}
	\textbf{The algorithm terminates after a finite number of iterations}
	(see also Rmk.\ \ref{rmk:core_algo_convergence}):
	Let $J_r$ denote the set $J$ in run $r$ (starting with $r=0$).
	Then $J_{r+1}=J_r \cup J'_r$, and if $J_{r+1}=J_r$ the algorithm terminates at step $r$.
	Further $J_r \subseteq J_r \cup J'_r = J_{r+1}$, \ie as long as the algorithm does not terminate
	the size of the set $J_r$ strictly increases with $r$, thus $\setelemcount{J_r} \geq r$ (because $\setelemcount{J_0} =0$).
	Since $J_r \subseteq \MultiindexSet$ and $\MultiindexSet$ is finite, the algorithm terminates after a finite number of iterations.
	
	\textbf{Relation of algorithmic to formal description:}
	In the marked independence oracle case (see hypothesis), the
	function \algPseudoCIT{} in the sub-algorithm \algRunCD{}
	(given $s_J$ as input)
	computes $c \circ \IStructPseudo{M, s_J}$.
	Thus \algRunCD{} returns (besides $J'$) the graph output
	$\CDAlg(c\circ \IStructPseudo{M, s_J})$.
	
	Concerning $J'_s$, we note that $J'_s$ contains precisely
	those indices $\vec{j'}\in\MultiindexSet$ that
	\begin{enumerate}[label=(\roman*)]
		\item were not already contained in $J$, \ie $J'_s \cap J = \emptyset$,
		\item are encountered by running $\CDAlg$, which by definition 
			are those in $L^s_{\txt{pseudo}}:= L_{\CDAlg}(c \circ\IStructPseudo{M, s_J})$,
			\ie $J'_s \subseteq L^s_{\txt{pseudo}}$
			and
		\item on which the mCIT (in the oracle case the marked oracle $\IStructOracleM$)
			returned the result $\mCIToutR$,
			\ie $J'_s \subseteq \IStructOracleM^{-1}(\{\mCIToutR\})$.
	\end{enumerate}
	Finally, by soundness and completeness of the state-space construction (by hypothesis),
	$S_r = S^{J_r}_T$. For applicability of the state-space construction,
	see \inquotes{restricted form of $J$ for state-space construction} near the end
	of this proof.
	
	Thus the $R\notin \IStructPseudo{M, s_J}(L_{\txt{pseudo}}^s)$ part of the hypothesis of
	Lemma \ref{lemma:apdy_core_algo_on_the_fly} is satisfied.
	But additionally it also shows, if $J'=\emptyset$ (thus if the algorithm terminates),
	the contra-position of Lemma \ref{lemma:J_must_grow} applies
	hence $J$ is complete.
	So in this case (if $J'=\emptyset$)
	Lemma \ref{lemma:apdy_core_algo_on_the_fly} applies
	(for this $s_J$, but the core-algorithm unifies all $J'$ returned
	by invocations of \algRunCD{} for all $s_J$, so $J'$ in the core-algorithm
	is only empty, if this is satisfied for all $s_J\in S^J=S_T^J$, see below).

	\textbf{Correctness of output:}
	We inspect the last round executed before termination (whose results are returned).
	In particular $J' = \emptyset$, thus $\forall s\in S_r: J'_s = \emptyset$
	because $J' = \cup_{s\in S_r} J'_s$.
	
	We first show, by contradiction, that $J$ is complete. 
	Assume $J$ were not complete. By Lemma \ref{lemma:J_must_grow} (and comparison
	to the formal description via points (i-iii) of $J'_s$ given above),
	there
	$\exists \vec{j}' \in \Lookup_{\CDAlg}(c\circ\IStructPseudo{M, s_J})=L^s_{\txt{pseudo}}$
	(ii),
	with $\vec{j}' \notin J$ (i) and $\IStructPseudo{M, s_J}(\vec{j}') = \mCIToutR$ (iii),
	thus $\vec{j}'\in J'_s$. This contradicts $J'_s = \emptyset$ (see above),
	and $J$ must be complete.
	
	Next, we show that Lemma \ref{lemma:apdy_core_algo_on_the_fly} applies.
	Consistency of $\CDAlg$ is given by hypothesis and completeness was shown above so
	it remains to show $\forall s_J \in S_T^J = S_r$:
	$R\notin \IStructPseudo{M, s_J}(L_{\txt{pseudo}}^s)$
	(the equality $S_T^J = S_r$ is by soundness and completeness of the state-space
	construction, see above).
	We again show this by contradiction. Assume there were $s_J\in S_r$ with
	$R\in \IStructPseudo{M, s_J}(L_{\txt{pseudo}}^s)$,
	\ie assume there were a $s_J$ and $\vec{j}' \in L_{\txt{pseudo}}^s$ (ii) with
	$\IStructPseudo{M, s_J}(\vec{j}') = \mCIToutR$ (iii).
	Note that $\IStructPseudo{M, s_J}(\vec{j}')=\mCIToutR$, by definition
	(Def.\ \ref{def:J_resolved_state_space}), implies $\vec{j}' \notin J$ (i).
	Thus by the above translation of algorithmic terms to formal terms all
	conditions (i-iii) are satisfied and $\vec{j'} \in J'_s$.
	This contradicts $J'_s = \emptyset$ (see above),
	and the hypothesis of Lemma \ref{lemma:apdy_core_algo_on_the_fly} must be satisfied.
	
	By Lemma \ref{lemma:apdy_core_algo_on_the_fly}b,
	\begin{equation*}
		\Big\{ \CDAlg(c\circ\IStructPseudo{M, s_J}) | s_J \in S_T^J \Big\}
		=
		\Big\{ G_s | s \in S_T^{\CDAlg} \Big\}
		\txt.
	\end{equation*}
	The left-hand side is (by the relation of algorithms to formal descriptions above)
	the set of (graph) outputs produced by \algRunCD{} on $S_r=S_T^J$.
	The core-algorithm applies \algRunCD{}
	precisely to $S_r=S_T^J$, thus the set of (graph) outputs produced is indeed
	the set on the left-hand-side. Thus the algorithm outputs the set
	$\{ G_s | s \in S_T^{\CDAlg} \}$ as claimed.
	
	\textbf{Restricted form of $J$ for state-space construction:}
	Our specification of state-space construction
	(Def.\ \ref{def:state_space_construction_specification}),
	allows the state-space construction to assume
	\begin{equation*}
		J \supseteq \Lookup_{\CDAlg}(c\circ\IStructOracleM(M)) \cap \IStructOracleM(M)^{-1}(\{\mCIToutR\})
	\end{equation*}
	in non-trivial ($J\neq \emptyset$) cases.
	We show that this property of $J$ is always ensured by the core-algorithm.
	In the initial round $J_0=\emptyset$.
	By our convention that $c$ maps $\mCIToutR$ to $1$,
	and the definition of lookup-regions (Def.\ \ref{def:lookup_region}),
	the initial round will execute the pseudo-cit (in the oracle case)
	on the elements of $\vec{j}\in\Lookup_{\CDAlg}(c\circ\IStructOracleM(M))$.
	In the oracle case, those $\vec{j}$ with additionally
	$\IStructOracleM(M)(\vec{j})=\mCIToutR$,
	are collected into 
	$J_1 = \Lookup_{\CDAlg}(c\circ\IStructOracleM(M)) \cap \IStructOracleM(M)^{-1}(\{\mCIToutR\})$.
	In subsequent iterations, $J$ can only grow, thus for $r\geq 1$,
	$J_r \supseteq J_1$.
	Therefore, the state-space construction is only invoked from the
	core-algorithm with $J$ satisfying this restriction\Slash{}assumption
	on $J$ as stated in Def.\ \ref{def:state_space_construction_specification}.
\end{proof}

\begin{rmk}\label{rmk:core_algo_convergence}
	In practice the convergence is very fast (the number of outer iterations
	in the core-algorithm is low).
	As the reader may convince themselves, the
	first iteration (by specification of the pseudo-cit) produces the union-graph, see Lemma \ref{lemma:states_and_union_graph} (or its acyclification in the
	cyclic case). In the acyclic case (see §\ref{sec:indicator_translation}), all states can be found on
	tests performed on the union-graph, if no finite-sample errors occur.
	Thus the algorithm terminates after $r=1$ in this case.
\end{rmk}

\subsection{Illustration on an Example}
\label{apdx:illustrate_algo_evaluation}

Recall the motivating example for state-space construction:
\theoremstyle{definition}
\newtheorem*{example_induced_indicators}
{Example \ref{example:induced_indicators}}
\begin{example_induced_indicators}[Induced Indicators]
	Consider $X \leftarrow Y \rightarrow Z$ with only $R^{\text{model}}_{XY}$
	non-trivial. This will, besides $R_{XY|\emptyset}$ non-trivial, also
	lead to $R_{XZ|\emptyset}$ non-trivial,
	even though $\Rmodel_{XZ}$ is trivial.
	Further, while both are non-trivial,
	$R_{XZ|\emptyset} \equiv R_{XY|\emptyset}$; the underlying
	model only has two states (as opposed to the four potential values the pairs
	$(R_{XY|\emptyset}(t), R_{XZ|\emptyset}(t))$ could take): One with the link $X \leftarrow Y$ active,
	one with the link $X \leftarrow Y$ not active.
\end{example_induced_indicators}

The core algorithm (Algo.\ \ref{algo:core}), in the oracle case,
will process data from this model as follows:
\begin{enumerate}[label=\arabic*)]
	\item line 17: Initializes $J=\emptyset$.
	\item line 19: \algConstructStateSpace{} trivially finds $S=\{s_\emptyset\}$
		(the one-element set containing the unique map $s_\emptyset:\emptyset \rightarrow \{0,1\}$).
	\item line 20: The for-loop iterates over the one element $s_\emptyset \in S$.
	\item line 21: Initializes $J'_\emptyset = \emptyset$
	(using the notation $J'_\emptyset := J'_{s_\emptyset}$).
	\item line 22: First the functor \algPseudoCIT{} is instantiated
		with the local values $s_\emptyset$, $J'_\emptyset = \emptyset$
		and $\texttt{data}$. The resulting
		\algPseudoCIT{}[$s=s_\emptyset$, $J=\emptyset$, $J'_\emptyset = \emptyset$, \texttt{data}] satisfies:
		\begin{itemize}
			\item line 2: Since $J = \emptyset$, the if-clause
				always evaluates to false.
			\item lines 5--12: In the oracle-case the return-value of
				\algPseudoCIT{} is the return-value of a CI-oracle $\IStructOracle(M_{\text{union}})$
				of the joint\Slash{}union model,
				because a true regime is ground-truth \emph{dependent} in the joint model.
			\item line 10: Marked Independencies (results $=\mCIToutR$) are recorded into $J'_\emptyset$ (see below).
		\end{itemize}
		Applying the, assumed
		consistent (Def.\ \ref{def:cd_alg_abstract_consistent}), $\CDAlg$
		to the instantiated functor \algPseudoCIT{}[$s=s_\emptyset$, $J=\emptyset$, $J'_\emptyset = \emptyset$, \texttt{data}]
		\begin{itemize}
			\item
				produces
				the union graph $G_\emptyset=G_{\text{union}}$,
				because as a mapping \algPseudoCIT{} agrees with
				$\IStructOracle(M_{\text{union}})$
				(see above).
			\item 
				evaluates \algPseudoCIT{} for values of $\vec{j}$ depending
				on the choice of $\CDAlg$.
				For PC-stable, and example \ref{example:induced_indicators} above,
				all pairwise unconditional tests are evaluated.
				$R_{XY|\emptyset}$ and $R_{XZ|\emptyset}$ are non-trivial,
				thus recorded into $J_\emptyset'$.
				With all pairs found dependent, the second iteration
				of PC-stable evaluates all conditional tests,
				in particular finding $X \independent Z |Y$ (thus deleting the direct edge
				between $X$ and $Z$) and records
				the non-trivial $R_{XY|Z}$ in $J_\emptyset'$.
		\end{itemize}
		\item line 24: $J'$ now contains $(X,Y|\emptyset), (X,Z|\emptyset)$
			and $(X,Y|\{Z\})$.
		\item line 25: $J$ is updated to $\emptyset\cup J'=J'$.
		\item line 26: Since $J'\neq \emptyset$, we repeat from line 18.
		\item line 19: $J$ now contains $(X,Y|\emptyset), (X,Z|\emptyset)$
			and $(X,Y|\{Z\})$, corresponding to the two links $X$--$Y$
			and $X$--$Z$.
			\begin{itemize}
				\item 
					From $X \independent Z |Y$ (and removal of
					the edge $X$--$Z$) the model-indicator on $X$--$Z$ must be
					constantly identical zero $\Rmodel_{XZ}\equiv 0$, thus trivial.
					We are left with one
					potentially non-trivial model-indicator on $X$--$Y$
					(the only link with non-trivial tests, after excluding the one between
					$X$ and $Z$ by the argument above).
				\item 
					This model-indicator can be represented by either
					$(X,Y|\emptyset)$ or $(X,Y|\{Z\})$
					in the sense that
					$\Rmodel_{XY}\equiv R_{XY|\emptyset} \equiv R_{XY|Z}$.
					These two independencies
					always co-occur in the model of the example, so the choice
					is (in the oracle case) insubstantial. Our algorithm
					will in practice prefer the one which was executed first
					(the smaller conditioning set for PC), here this is
					$(X,Y|\emptyset)$ and $\Rmodel_{XY}\equiv R_{XY|\emptyset}$.
				\item 
					With only one model indicator, the state of the model is
					fully described by its value. The removal of a link
					cannot open causal paths (see §\ref{apdx:helpers_indicator_logic}),
					so the value of all tests in $J$ is simultaneously either
					$0$ (independent) or $1$ (dependent).
				\item Formally this means we have $S=\{s_0, s_1\}$ where
					$s_0$ and $s_1$ are maps $J \rightarrow \{0,1\}$,
					such that $s_0\equiv 0$ is identically $0$, while
					$s_1 \equiv 1$ is identically $1$.
			\end{itemize} 			
			This kind of logic is formalized, generalized and automated by the results in
			§\ref{sec:indicator_translation}.
		\item line 20: The for-loop now iterates over $s_0$ and $s_1$.
		\item line 21: Initialize $J_0'=\emptyset$ (we write the index $s_i$ as $i$, so
			$J_0':=J_{s_0}'$)
		\item line 22: \algPseudoCIT{} is instantiated with 
			the local values $s=s_0$, $J=\{ (X,Y|\emptyset)$, $(X,Z|\emptyset)$,
			$(X,Y|\{Z\})\}$, $J'_0 = \emptyset$
			and $\texttt{data}$. The resulting
			\algPseudoCIT{}[$s=s_0$,
				$J=\{ (X,Y|\emptyset)$, $(X,Z|\emptyset)$, $(X,Y|\{Z\})\}$,
				$J'_0 = \emptyset$, \texttt{data}] satisfies:
			\begin{itemize}
				\item line 2: If $\vec{j} \in J$, then the if-clause
					evaluates to true,
					as a result the value $s_0(\vec{j})=0$ (independent is returned).
					The CD-algorithms sees
					$X\independent Y$, $X \independent Z$ and
					$X\independent Y | Z$.
				\item There are no new tests (compared to the first iteration)
					that could be executed. Thus by construction of $J'_\emptyset$
					and thus $J$ above, the mCIT behaves as a standard CIT
					for tests not in $J$; the new $J_0'$ remains empty.
				\item lines 5--12: In the oracle-case the return-value of
					\algPseudoCIT{} is the return-value of a CI-oracle $\IStructOracle(M_{\tilde{s}})$ for tests not in $J$
					for any $\tilde{s}\in S_T$ (all states agree on non-(oracle-)marked tests).
			\end{itemize}
			Applying the, assumed
			consistent (Def.\ \ref{def:cd_alg_abstract_consistent}), $\CDAlg$
			to \algPseudoCIT{} produces the graph $G_0=$ $X\quad Y$--$Z$,
			as the return-values of \algPseudoCIT{} agree with the oracle for
			this graph.
		\item line 21--22: For $s=s_1$, the steps 11 and 12 are repeated.
			The only difference being, that $s_1(\vec{j})=1$ for tests in $\vec{j}$,
			thus $\CDAlg$ sees 
			$X\dependent Y$, $X \dependent Z$ and $X\dependent Y | Z$.
			It will thus return $G_1=$ $X$--$Y$--$Z$.
			Note, that by a JCI-argument (§\ref{apdx:JCI}),
			the left edge could be oriented as $X \leftarrow Y$--$Z$.
		\item line 24--26: After the second iteration
			$J' = J_0' \cup J_1' = \emptyset \cup \emptyset$,
			thus the algorithm terminates.
		\item line 27: the returned result is the set $\{G_0, G_1\}$.
\end{enumerate}

\section{Details on State-Space Construction}\label{apdx:state_space_construction}

The main purpose of this section is to approach -- and under suitable assumptions
to resolve -- the question of representation of states and detected indicators,
with the eventual goal of implementing \algConstructStateSpace{}
(Def.\ \ref{def:state_space_construction_specification}) for use in the core algorithm.
We first recall the basic problem of induced indicators from §\ref{sec:extended_indep_struct_subsection},
which was illustrated by:
\begin{example_induced_indicators}[Induced Indicators]
	Consider $X \leftarrow Y \rightarrow Z$ with only $R^{\text{model}}_{XY}$
	non-trivial. This will, besides $R_{XY|\emptyset}$ non-trivial, also
	lead to $R_{XZ|\emptyset}$ non-trivial,
	even though $\Rmodel_{XZ}$ is trivial.
	Further, while both are non-trivial,
	$R_{XZ|\emptyset} \equiv R_{XY|\emptyset}$; the underlying
	model only has two states (as opposed to the four potential values the pair
	$(R_{XY|\emptyset}, R_{XZ|\emptyset})$ could take): One with the link $X \leftarrow Y$ active,
	one with the link $X \leftarrow Y$ not active.
\end{example_induced_indicators}
We will approach this problem by first finding and representing states,
then expressing detected indicators relative to these states.
For the above example, there are two states corresponding to presence\Slash{}absence
of the link $X \leftarrow Y$. They can be \inquotes{represented} by (the value of)
the independence $X \independent Y$. Finally, the detected indicators can
be expressed as functions of the (binary) state, thus relative to the representing
test $X \independent Y$ (in the example, all non-stationary independencies, agree with
this representing one) rather than as a time-resolved function.

We will start by discussing assumptions that allow for a systematic resolution
of this problem.
Next, there is an additional, so far neglected, step: Formulating a hypothesis test,
that can be tested directly (without requiring time-resolution of any indicators)
and is simple enough to be accessible in practice, yet general enough that its
(oracle-)results suffice to recover from a marked independence structure
the full multi-valued independence-structure.
From there, we can follow the program laid out above
to produce an algorithm \algConstructStateSpace{} that
satisfies the requirements of the core algorithm (Algo.\ \ref{algo:core}).
Finally, we give a simple approach to transfer edge-orientations between states,
which allows for simple output-representation by a labeled union-graph,
discuss questions relating to order-dependence
and lay out relevant topics for future work.

\subsection{Assumptions}\label{apdx:state_constr_ass}

For simplicity, we focus on the acyclic case. This will simplify the representation
of model-indicators, the required hypothesis-testing and many other arguments.
\begin{assumption}[Acyclicity]\label{ass:aycyclicity}
	The (true) union graph is ancestral (an AG), \ie it contains no
	directed or almost directed cycles.
	Per-state graphs are additionally maximal (MAGs; see Rmk.\ \ref{rmk:MAGass}).
	We use definitions of these terms as given for example in \citep{ZhangFCIorientationrules}.
\end{assumption}
\begin{rmk}\label{rmk:MAGass}
	The necessity of state-graphs to be MAGs is not immediately evident.
	It arises due to the following problem: Assume there are precisely two states,
	with the denser (graph's) one not being a MAG. Than this denser graph
	contains one or multiple inducing paths. In the sparser (graph's) state
	some of these inducing paths may vanish additional to the true model indicator's link.
	In this case \emph{multiple} links can change at once, even for modular changes
	(see Ass.\ \ref{ass:states_modular_and_reached} below).
	It is possible to remedy these issues by suitable post-processing
	(see §\ref{apdx:non_modular_changes}), but this would
	require a more complex (in terms of implementation and number of tests required)
	approach.
\end{rmk}
We mostly ignored the details of \emph{how} causal discovery algorithms
discover graphs. Here we \emph{do} require that the CD-algorithm
is of a particular (albeit very general) form (see Rmk.\ \ref{rmk:ass_CD_local_edge_removal}
and example \ref{example:ass_CD_local_edge_removal} below)
\begin{assumption}[Conventional Skeleton Discovery]\label{ass:conventional_skeleton_phase}
	The CD-algorithm $\CDAlg$ (on the given model class)
	is such that
	\begin{enumerate}[label=(\alph*)]
		\item
		Resolved Skeleton:
		If $g$ and $g'$ (which can occur as state-graphs in our model class)
		have different skeleta, then $g \not\sim_{\CDAlg} g'$.
		\item 
		Exploitation of Sparsity (in skeleton phase):
		if $\IStruct \leq \IStruct'$,
		then $\Lookup_{\CDAlg}(\IStruct) \subseteq \Lookup_{\CDAlg}(\IStruct')$.
		\item
		Local Edge Removal:
		if $\CDAlg(\IStruct)$ contains no edge
		between $X$ and $Y$, then 
		$\exists Z$ such that $(i_x,i_y, \vec{i}_z)\in\Lookup_{\CDAlg}(\IStruct)$
		and
		$X \independent_{\IStruct} Y | Z$.
	\end{enumerate}
\end{assumption}
\begin{rmk}\label{rmk:ass_CD_local_edge_removal}
	The authors are not aware of any commonly used constraint-based algorithm that would
	actually violate this assumption.
	At least the ones used in the numerical
	experiments in this paper abide by it, simply by virtue of how skeleta
	are discovered.
	Parts (a) and (b) are assumed for simplicity, essentially avoiding the discussion
	of partial (incomplete $J$) discovery in §\ref{apdx:state_space_construction_J};
	they are very likely not necessary for our method, just convenient for the proofs.
	There are plausible scenarios where one might want to consider CD-algorithms
	violating (a), for example when aiming for the estimation of a particular effect,
	it may not be necessary to recover the correct skeleton everywhere.
	Note, that in the presence of hidden confounders, (a) can be violated
	because the representatives of the same MAG can have different skeleta at inducing
	paths (see Rmk.\ \ref{rmk:MAGass} and §\ref{apdx:state_translation});
	Ass.\ \ref{ass:aycyclicity} excludes this.
	Part (b) could be plausibly violated if for denser graphs more work has to be done
	to find edge-orientations by additional tests. In practice, part (b) could
	be weakened (given appropriate definitions) for example to a statement about
	lookup-regions of the skeleton phase (see its application in the proof of
	Lemma \ref{lemma:representing_model_indicator}).
	
	The condition (c) is required for the way model-indicators are currently
	represented, see Lemma \ref{lemma:representing_model_indicator}.
	It is even for MAGs \emph{possible} (at least in principle) to violate this assumption (c),
	by creative construction of a CD-algorithm,
	see example \ref{example:ass_CD_local_edge_removal}. The practical usefulness of
	algorithms violating (c) seems to very strongly hinge on prior knowledge about
	certain independencies being much more difficult to assess than others.
	For example, if we knew that the link $X$ to $Y$ in example \ref{example:ass_CD_local_edge_removal},
	if it exists, might have a 
	very complicated functional form, while other parts of the model are linear,
	the logic of the example below could be favorable.
\end{rmk}
\begin{example}[Unconventional CD-Logic]\label{example:ass_CD_local_edge_removal}
	Consider an acyclic, causally sufficient model with the following causal graph:\\[-0.2em]
	\begin{minipage}{\textwidth}
		\centering
		\begin{tikzpicture}
			\draw (0,0) node (X){$X$};
			\draw (-1,0) node (Z1){$Z$};
			\draw (0,1) node (Z2){$Z'$};
			
			\draw[->] (Z1) -- (X);
			\draw[->] (Z2) -- (X);
			
			\draw (1,0) node (Y){$Y$};
			\draw (2,0) node (W1){$W$};
			\draw (2,1) node (W2){$W'$};
			
			\draw[->] (Y) -- (W1);
			\draw[->] (W2) -- (W1);		
			
			\draw[->, thick, dashed] (X) -- (Y);
		\end{tikzpicture}
	\end{minipage}\\[0.2em]
	Any paths from the left sub-graph ($X,Z,Z'$) to the right
	sub-graph ($Y,W,W'$) can be blocked by either $X$ or $Y$, and the colliders
	are both unshielded, thus we can find the two subgraphs and the
	absence of any edge (ignoring $X$--$Y$ for now) connecting them
	indirectly (other than through $X \rightarrow Y$).
	Further, by $Z' \independent Y |X$, there cannot be an edge $X\leftarrow Y$.
	Hence the presence or absence of the edge $X \rightarrow Y$ can be
	determined non-locally, by testing $Z\independent Y$.
	Thus we need not test any independence of the form $X \independent Y | S$ for any $S$
	to consistently distinguish the two graphs with\Slash{}without $X\rightarrow Y$.
	In particular, there exists a constraint based algorithm (for example,
	try the above trick, if it does not work apply PC) that is consistent
	on causally sufficient acyclic models, yet does \emph{not}
	satisfy Ass.\ \ref{ass:conventional_skeleton_phase}c.
\end{example}
Further, we assume changes are modular (see §\ref{apdx:non_modular_changes}).
\begin{assumption}[Modular and Reached]\label{ass:states_modular_and_reached}
	Changes are modular (Def.\ \ref{def:causal_modularity_order}),
	that is links change independently of each other, thus the reachable
	states are all combinations (presence\Slash{}absence) of changing links
	(non-trivial indicators).
	Further we assume all reachable states are reached in the given data.
\end{assumption}

\subsection{Notations}

From §\ref{apdx:core_algorithm}, we use:
\begin{notation}
	To interpret marked independence structures we occasionally use
	the choice map $c$ from Def.\ \ref{def:J_resolved_state_space}:
	\begin{equation*}
		c : \{0,1,\mCIToutR\} \rightarrow \{0,1\},
		\quad
		\begin{cases}
			0 \mapsto 0 \\
			1 \mapsto 1 \\
			\mCIToutR \mapsto 1
		\end{cases}\txt.
	\end{equation*}
\end{notation}

Acyclicity allows to simplify the notation for model-indicators:
\begin{notation}
	Given assumption Ass.\ \ref{ass:aycyclicity} (acyclicity),
	there is a total order of nodes
	(extending the partial order $X \leq Y :\Leftrightarrow X \in \Anc(Y)$)
	called a causal order.
	By acyclicity, if $X_i=X \leq Y=X_j$, then $Y \notin \Pa(X)$, thus
	$R_{ji} \equiv 0$.
	So all information about this link is in $R_{ij}$.
	Thus we will only consider
	\begin{equation*}
		\Rmodel_{XY}(t) := \Rmodel_{ij}(t) + \Rmodel_{ji}(t)
		= \max( \Rmodel_{ij}(t), \Rmodel_{ji}(t) )
		\quad \txt{for $X=X_i$, $Y=X_j$.}
	\end{equation*}
	By definition of states $S$ (Def.\ \ref{def:states}),
	all model-indicators can be written as functions of $S$ (instead of $T$).
	Since by definition (Def.\ \ref{def:model_nonstationary_scm}\Slash{}\ref{def:model_nonstationary_scm_nonadd})
	all qualitative changes in the model come from the
	changes in state (see Lemma \ref{lemma:graphs_factor_through_states}),
	also all detected indicators an be written as a function of $S$.
	We will still denote these by $\Rmodel_{XY}$ and $R_{XY|Z}$ etc.,
	but explicitly write the argument as $s$, \ie $\Rmodel_{XY}(s)$, $R_{XY|Z}(s)$
	and so on.
\end{notation}

It will turn out that many of the subsequent results have elegant expressions relative to the
following simple ordering-relations:
\begin{definition}\label{def:state_construction_partial_orders}
	On the set of indicators we consider the following partial order:
	\begin{equation*}\tag{a}\leqnomode
		R_1 \leq R_2
		\quad:\Leftrightarrow\quad
		\forall t: R_1(t) \leq R_2(t) \\ 
	\end{equation*}
	For fixed $X$, $Y$ this induces a preorder on $Z\subseteq V \setminus \{X, Y\}$:
	\begin{equation*}\tag{b}\leqnomode
		Z_1 \leq_{XY} Z_2
		\quad:\Leftrightarrow\quad
		R_{XY|Z_1} \leq R_{XY|Z_2}
	\end{equation*}
	States can be partially ordered by density of (active) model-indicators:
	\begin{equation*}\tag{c}\leqnomode
		s_1 \leq s_2
		\quad:\Leftrightarrow\quad
		\Big( \forall \Rmodel: \Rmodel(s_1) \leq \Rmodel(s_2) \Big)
	\end{equation*}
	And finally graphs can be partially\footnote{They can certainly be preordered,
	those identified by constraint-based algorithms can also be partially ordered.
	This distinction will not matter here.}
	ordered by d-connectivity (the reader may note a certain similarity
	to the ordering relation used by \citep{Chickering2002OptimalSI}):
	\begin{equation*}\tag{d}\leqnomode		
		g_1 \leq g_2
		\quad:\Leftrightarrow\quad
		\Big( X \dependent_{g_1} Y | Z \Rightarrow X \dependent_{g_2} Y | Z \Big)
	\end{equation*}
\end{definition}
\begin{notation}
	For binary indicators, part (a) is equivalent to
	an implication (of independencies)
	\begin{equation*}
		R_1 \leq R_2 
		\quad\Leftrightarrow\quad
		\Big( R_2(t)=0 \Rightarrow R_1(t)=0 \Big)
	\end{equation*}
	motivating the notation $(R_2 \Rightarrow R_1) :\Leftrightarrow (R_1 \leq R_2)$.
\end{notation}

\subsection{Required Hypothesis Tests}\label{apdx:required_implication_tests}

Assume we knew the regime-marked independence structure $\IStructM$.
Then any detected state (element of $\IStructExt$) is
given by some combination of truth-values for the 
true regimes in $\IStructM$. There is only a finite (and $N$ independent)
number of such combinations, thus it suffices to exclude finitely many
\emph{not} appearing in $\IStructExt$.
However, there can still be many possible combinations of truth-values.
In practice, under the assumptions of this section,
it will suffice to test implications with the following simplified%
\footnote{In principle, there could be arbitrary boolean relations required
on both sides. Instead it suffices to test with a boolean and on the left
and a single indicator\Slash{}representing independence on the right.} form:
\begin{notation}
	A multi-index $\vec{i} = (i_{X}, i_{Y}, (i_{Z_{1}}, \dots, i_{Z_{k}}))$
	identifies an independence-relation $X \independent Y | Z$.
	We call this test marked if $X \independent_R Y | Z$ with respect to some fixed
	marked independence structure (Def.\ \ref{def:marked_independence}),
	usually we will use the marked oracle (example \ref{example:oracle_marked}).
	If the test at $\vec{i}$ is marked,
	there is an associated detected indicator (example \ref{example:induced_indicators})
	$R_{\vec{i}}$.
\end{notation}
\begin{definition}[Indicator Implication-Relations]\label{def:indicator_implciation}
	Given the marked independencies $\vec{i}_1$, $\vec{i}_2$, \dots, $\vec{i}_{K+1} \in J \subseteq \MultiindexSet$
	we call the statement
	\begin{equation*}
		R_{\vec{i}_1}(t) = R_{\vec{i}_2}(t) = \ldots = R_{\vec{i}_K}(t) = 0
		\quad\Rightarrow\quad
		R_{\vec{i}_{K+1}}(t) = 0
	\end{equation*}
	an indicator implication. This is equivalent to
	\begin{equation*}
		R_{\vec{i}_{K+1}} \leq \sum_{j=1}^{K} R_{\vec{i}_j}
		\txt.
	\end{equation*}
	To keep everything binary, the sum on the right-hand side in the last equation can be replaced by a boolean
	'or'.
	The implication-relation structure $\IStructR$ is then formally a mapping
	\begin{equation*}
		\IStructR: \big(\MultiindexSet\big)^{*+1} \rightarrow \{0,1\}
		\text,
	\end{equation*}
	\ie we assign truth-values (true or false) to statements of the form
	$R_1(t) = R_2(t) = \ldots = R_K(t) = 0
	\Rightarrow
	R_{K+1}(t) = 0$ (cf.\ example \ref{example:oracle_implication}).
\end{definition}
Recall that we had defined a marked independence
structure (Def.\ \ref{def:marked_independence}) and associated
oracle (example \ref{example:oracle_marked}), similarly here:
\begin{example}[Implication-Relation Oracle]\label{example:oracle_implication}
	The implication-relation oracle $\IStructOracleRel$ of a non-stationary SCM $M$
	(Def.\ \ref{def:model_nonstationary_scm}\Slash{}\ref{def:model_nonstationary_scm_nonadd})
	is the implication-relation structure, which assigns to each implication
	statement
	$R_{\vec{i}_1}(t) = R_{\vec{i}_2}(t) = \ldots = R_{\vec{i}_K}(t) = 0\Rightarrow R_{\vec{i}_{K+1}}(t) = 0$
	the truth value of the corresponding expression of the true (taking values determined by
	the induced observable distribution of $M$) indicators $R$ (see Rmk.\ \ref{rmk:indicator_implciation},
	which gives an alternative definition from $\IStructOracleX$).
\end{example}
\begin{rmk}\label{rmk:indicator_implciation}
	The multi-valued independence-structure oracle $\IStructOracleX$,
	implies a implication-relation oracle:
	The relation $R_{\vec{i}_{K+1}} \leq \sum_{j=1}^{K} R_{\vec{i}_j}$
	is true in the implication-relation oracle, \ie $\IStructOracleRel((\vec{i}_1, \ldots,
	\vec{i}_K), \vec{i}_{K+1}) = 1$, if and only if
	\begin{equation*}
		\forall\IStruct \in \IStructOracleX:\quad
		\IStruct(\vec{i}_{K+1}) \leq \sum_{j=1}^K \IStruct(\vec{i}_{j})\txt.
	\end{equation*}
	Thus knowledge of this relation-structure is necessary for the reconstruction of
	$\IStructOracleX$ (which in turn is necessary to solve the HCCD problem by Lemma \ref{lemma:necessity_of_dyn_indep_struct}).
	Here we show that, at least under the assumptions of the present section,
	knowledge of marked independencies plus
	these relations is also sufficient to reconstruct $\IStructOracleX$.
\end{rmk}
While formulated via indicators, this is again a statement whose complexity
does \emph{not}	grow with the sample-size $N$, and can be tested \inquotes{directly}
similar to the marked independencies (§\ref{sec:three_way_testing}).
In practice, we can use constraints from causal modeling to further simplify
what form the left-hand-side in this definition may take, the number of summands on
the left-hand-side and the total number of tests (see also §\ref{sec:ind_locality}).
The implementation of these tests is discussed in §\ref{apdx:implication_testing}.

Cyclic models seem to additionally require a type of \inquotes{union test} (a boolean
'or' on the left hand side of the implications),
see example \ref{example:model_indicators_on_cyclic} and §\ref{apdx:state_space_cyclic_models}.

\subsection{Helpful Observations}\label{apdx:helpers_indicator_logic}

In cases where multiple links might change independently, the situation can quickly become
quite complex as the following example shows:

\begin{figure}[ht]
	\colorlet{link_a}{black!30!green}
	\colorlet{link_b}{black!10!orange}
	\begin{minipage}{0.65\textwidth}
		\begin{minipage}{0.45\textwidth}
			\raggedright
			(a) basic example:\\[0.2em]
			\begin{tikzpicture}[line width=0.15em]
				\draw (0,0) node [draw, outer sep=0.2em, line width=0.1em] (X) {$X$};
				\draw (1,0) node (M) {$Z$};
				\draw (2,0) node [draw, outer sep=0.2em, line width=0.1em] (Y) {$Y$};
				
				\draw [->, link_a] (M) -- (X);
				\draw [->] (M) -- (Y);
			\end{tikzpicture}
		\end{minipage}
		\begin{minipage}{0.45\textwidth}
			\raggedright
			(b) \inquotes{in sequence}:\\[0.2em]
			\begin{tikzpicture}[line width=0.15em]
				\draw (0,0) node [draw, outer sep=0.2em, line width=0.1em] (X) {$X$};
				\draw (1,0) node (M) {$Z$};
				\draw (2,0) node [draw, outer sep=0.2em, line width=0.1em] (Y) {$Y$};
				
				\draw [->, link_a] (M) -- (X);
				\draw [->, link_b] (M) -- (Y);
			\end{tikzpicture}
		\end{minipage}
		\caption{Different layouts of independently changing (differently colored) links.
			Main point of interest is $R_{XY|\emptyset}$.}
		\label{fig:state_space_boolean_layouts}
	\end{minipage}\hfill
	\begin{minipage}{0.3\textwidth}
		\raggedright
		(c) \inquotes{in parallel}:
		\begin{tikzpicture}[line width=0.15em]
			\draw (0,0) node [draw, outer sep=0.2em, line width=0.1em] (X) {$X$};
			\draw (1,1) node (M) {$M$};
			\draw (2,0) node [draw, outer sep=0.2em, line width=0.1em] (Y) {$Y$};
			
			\draw [->, link_a] (X) -- (M);
			\draw [->] (M) -- (Y);
			\draw [->, link_b] (X) -- (Y);
		\end{tikzpicture}
	\end{minipage}
\end{figure}
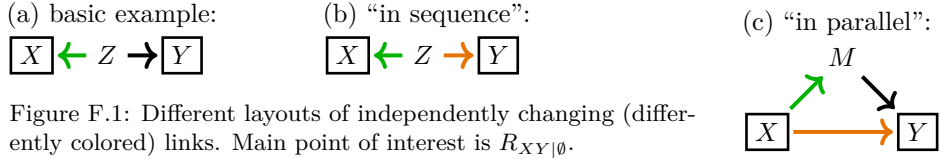
\begin{example}[Requirement For Boolean Combinations]\label{example:in_parallel_in_sequence}
	In Fig.\ \ref{fig:state_space_boolean_layouts}, different types of combinations of 
	multiple indicators are shown. While (a) is the previous, rather basic example,
	(b) and (c) illustrate how different boolean combinations arise.
	First, focus on the empty conditioning set:
	In case (b), there is no path from $X$ to $Y$ iff either the green \emph{or} the orange link
	vanish. In case (c), there is no path from $X$ to $Y$ iff both the green \emph{and} the orange link
	vanish. We return to these cases in example \ref{example:in_parallel_in_sequence_continued}.
\end{example}

While the previous example shows that in general detected indicators can
arise from model-indicators as complicated boolean combinations,
we can always reorder such combinations to first consider all boolean 'and' (and only have to test these),
and afterwards all boolean 'or' (cf.\ §\ref{apdx:detected_indicator_representation})
by the following observation:
Taking away a link can disconnect an open path, but it cannot unblock a closed path.
\begin{lemma}[Monotonicity of State to d-Separation Assignment]\label{lemma:do_not_open_paths}
	Under Ass.\ \ref{ass:aycyclicity} (union acyclicity),
	the assignment of states to state-graphs is monotonically increasing
	with respect to the orders in Def.\ \ref{def:state_construction_partial_orders}c,\,d:	
	\begin{equation*}
		s_1 \leq s_2
		\quad\Rightarrow\quad
		g_{s_1} \leq g_{s_2}
		\quad\overset{\txt{def}}{\Leftrightarrow}\quad
		\Big( X \dependent_{g_{s_1}} Y | Z \Rightarrow X \dependent_{g_{s_2}} Y | Z \Big)
	\end{equation*}
\end{lemma}
\begin{proof}
	Given $s \leq s'$. We have to show
	$X \dependent_{g_{s}} Y | Z \Rightarrow X \dependent_{g_{s'}} Y | Z$.
	
	Let $X$, $Y$, $Z$ be arbitrary.
	Assume $X \dependent_{g_{s}} Y | Z$, otherwise there is nothing to show.	
	By definition, there is an open
	path $\gamma$ in $g_{s}$ between $X$ and $Y$ with $Z$ blocked.
	By $s \leq s'$, all edges contained in $g_s$ are also edges in $g_{s'}$ in
	particular all edges along $\gamma$ (a path in $g_s$) are edges in $g_{s'}$, thus
	$\gamma$ is a well-defined path in $g_{s'}$.
	A path $\gamma$ is open by definition, if (i) all non-colliders on $\gamma$ are not in
	$Z$ and (ii) all colliders on $\gamma$ have a descendant in $Z$ (where we formally define
	descendants to include $Z$ itself: $Z\in\Dec(Z)$).
	With $\gamma$ being open in $g_s$, both criteria are satisfied in $g_s$,
	we have to check them in $g_{s'}$.
	By acyclicity, the distinction collider vs.\ non-collider is the same 
	in $g_s$ and $g_{s'}$.
	
	(i) Since $Z$ is unchanged, all non-colliders are still not in $Z$.
	This condition is still satisfied in $g_{s'}$.
	
	(ii) Let $W$ be an arbitrary collider on $\gamma$. Because $\gamma$ is open (given $Z$)
	in $g_{s}$, $\exists W' \in \Dec_{g_s}(W): W' \in Z$,
	where $\Dec_{g_s}(W)$ are the descendants of $W$ in $g_{s}$.
	It suffices to show $\Dec_{g_s}(W) \subseteq \Dec_{g_{s'}}(W)$,
	then $W' \in \Dec_{g_{s'}}(W)$ and the condition is still satisfied in $g_{s'}$.
	By definition, $\Dec_g(W) := \{$ $\tilde{W}$ $|$ $\exists$ directed path $\gamma'$ in $g$ from
	$W$ to $\tilde{W}$ $\}$. Any directed path $\gamma$ in $g_s$
	is a directed path in $g_{s'}$ from $s' \leq s$ (see above) and acyclicity.
	Thus $\Dec_{g_s}(W) \subseteq \Dec_{g_{s'}}(W)$.
	
	Thus we have seen $\gamma$ is an open path in $g_{s'}$.
	Therefore $X \dependent_{g_{s'}} Y | Z$, concluding the proof.
\end{proof}
Under suitable Markov and faithfulness properties, this result
also applies to independence oracles.
\begin{lemma}[Monotonicity of Detected-Indicators]
	\label{lemma:state_ordering_implies_istruct_ordering}
	Under Ass.\ \ref{ass:aycyclicity} (union acyclicity)
	and Ass.\ \ref{ass:momentary_faithful} (faithfulness),
	if $s \leq s'$,
	then $\IStructOracle(M_{s}) \leq \IStructOracle(M_{s'})$.
	In particular $s \leq s'$ $\Rightarrow$ $R_{XY|Z}(s) \leq R_{XY|Z}(s')$.
\end{lemma}
\begin{proof}
	Let $\vec{j}\in \MultiindexSet$ be arbitrary.
	If $\IStructOracle(M_{s})(\vec{j}) = 1$,
	then $\vec{j}$ is d-connected\footnote{Since $\vec{j}=(i_x,i_y,i_{\vec{z}})$
	corresponds to a test $X_{i_x} \independent X_{i_y} | X_{i_{\vec{z}}}$,
	it makes sense to say \inquotes{$\vec{j}$ is d-connected in $g$}
	if $X_{i_x} \dependent_g X_{i_y} | X_{i_{\vec{z}}}$.} in $g_s$
	by the Markov-property \ref{lemma:momentary_markov}.
	By $s \leq s'$ and the previous Lemma \ref{lemma:do_not_open_paths},
	$\vec{j}$ is also d-connected in $g_{s'}$.
	By faithfulness, thus $\IStructOracle(M_{s'})(\vec{j})=1$.
	Finally, the value $R_{XY|Z}(s)$ is $0$ iff $X \independent_{\IStructOracle(M_{s})} Y | Z$
	(by definition of detected indicators Def.\ \ref{def:detected_indicator}),
	\ie $R_{XY|Z}(s) = \IStructOracle(M_{s})(i_x, i_y, i_{\vec{z}})$,
	so the claim $R_{XY|Z}(s) \leq R_{XY|Z}(s')$ also follows.
\end{proof}
This also characterizes the union-graph and the pseudo-cit:
\begin{lemma}[Union States and Graphs]\label{lemma:states_and_union_graph}
	Under Ass.\ \ref{ass:aycyclicity} (union acyclicity)
	and Ass.\ \ref{ass:momentary_faithful} (faithfulness),
	$c \circ \IStructOracleM(M) = \IStructOracle(M; s_{\txt{union}})$,
	where $s_{\txt{union}}$ is the state where all model-indicators are active ($=1$).
	If additionally $\CDAlg$ is consistent
	then $G_{\txt{union}}=\CDAlg(\IStructOracle(M; s_{\txt{union}}))$
	is the equivalence-class of the union-graph.
\end{lemma}
\begin{proof}
	Note, that $\forall s\in S: s_{\txt{union}} \geq s$ by construction.
	Let $\vec{j} \in \MultiindexSet$ be arbitrary.
	If $\IStructOracleM(M)(\vec{j}) \in \{0,1\}$, then $c \circ \IStructOracleM(M)(\vec{j}) =\IStructOracleM(M)(\vec{j}) = \IStructOracle(M; s_{\txt{union}})(\vec{j})$
	by definition of the marked oracle (example \ref{example:oracle_marked}).
	If $\IStructOracleM(M)(\vec{j})	= \mCIToutR$, then $\exists s',s''$ such that
	$\IStructOracle(M_{s'})(\vec{j})=0$, $\IStructOracle(M_{s''})(\vec{j})=1$, again
	by definition of the marked oracle (example \ref{example:oracle_marked}).
	By maximality $s_{\txt{union}} \geq s''$
	and the previous Lemma \ref{lemma:state_ordering_implies_istruct_ordering},
	we therefore have $c \circ \IStructOracleM(M)(\vec{j})
	= c(R) = 1 = \IStructOracle(M_{s_{\txt{union}}})(\vec{j})$ also in this case.
	
	The second claim about the union-graph follows by definition of the union graph.
\end{proof}

Another simple, but very helpful, observation is that
under Ass.\ \ref{ass:states_modular_and_reached} (modularity, reachedness),
the reached state-space is spanned (as a $\mathbb{Z}/2\mathbb{Z}$ vector-space)
by the model indicators (see Def.\ \ref{def:states}).
Here the assumption is reachedness as opposed to reachability
(Lemma \ref{lemma:model_modular_reachable}).
\begin{lemma}[Modular Reachedness]\label{lemma:states_under_modularity}
	Given Ass.\ \ref{ass:states_modular_and_reached} (modularity, reachedness)
	the reached states $S_T = S$ are given by all combination of 
	model indicators (see Def.\ \ref{def:states}).
\end{lemma}
\begin{proof}
	By modularity (independence of changes), the reachable states are given by the tuples
	of values taken by the non-trivial model-indicators, thus are the elements of the
	product of (binary) value-spaces of those indicators. 
	By Ass.\ \ref{ass:states_modular_and_reached}, reachable states are reached.
\end{proof}

\subsection{Model Indicators}\label{apdx:model_indicators}

Under the acyclicity assumptions of the present section, the model-indicator is contained by all
detected indicators in the following sense:
\begin{lemma}[Model Indicators are Lower Bound]\label{lemma:model_indicator_coarsest}
	Assuming Ass.\ \ref{ass:momentary_faithful} (faithfulness),
	and Ass.\ \ref{ass:aycyclicity} (acyclicity),
	fix an arbitrary pair $X$, $Y$, then $\Rmodel_{XY}$ is a lower bound for detected indicators \wrt the ordering
	from Def.\ \ref{def:state_construction_partial_orders}a.
	That is, for all $Z$
	the detected indicator $R_{XY|Z}$ is comparable to and greater than $\Rmodel_{XY}$:
	\begin{equation*}
		\forall Z: R_{XY|Z} \geq \Rmodel_{XY}
	\end{equation*}
\end{lemma}
\begin{proof}
	Let $Z$, $t$ be arbitrary.
	We can equivalently show $\Rmodel_{XY}(t) \neq 0 \Rightarrow R_{XY|Z}(t) \neq 0$.
	If $\Rmodel_{XY}(t) \neq 0$, the direct path $X \rightarrow Y$ is open,
	thus there is no $Z$ blocking it and therefore $X \dependent_{g_t} Y | Z$.
	By faithfulness, d-separation implies independence, thus $R_{XY|Z}(t) \neq 0$.
\end{proof}

\begin{cor}
	In particular, under the assumptions of the lemma, 
	if there is a $Z$ with $X \independent Y | Z$, and thus $R_{XY|Z}\equiv 0$,
	then $\Rmodel_{XY} \equiv 0$ is also trivial.
\end{cor}

So the model-indicator is a lower bound (\wrt our ordering relation),
further this lower bound is attained,
\ie the model-indicator is realized as the minimum in the following sense
(conventional skeleton discovery is here obviously required to ensure the completeness of
the discovered state-space in part b):
\begin{lemma}[Representation of Model Indicators]\label{lemma:representing_model_indicator}
	Assume
	Ass.\ \ref{ass:conventional_skeleton_phase} (conventional skeleton discovery),
	Ass.\ \ref{ass:aycyclicity} (acyclicity) and Ass.\ \ref{ass:momentary_faithful} (faithfulness).
	Given $X,Y$, and a consistent (Def.\ \ref{def:cd_alg_abstract_consistent})
	$\CDAlg$, then there is a set $Z$, such that
	\begin{enumerate}[label=(\alph*)]
		\item 
		$R_{XY|Z} \equiv \Rmodel_{XY}$ and
		\item 
		The test $X \independent Y | Z$ is in the lookup-region of the union graph:
		$\vec{j}=(i_x,i_y,i_{\vec{z}})\in\Lookup_{\CDAlg}(c\circ\IStructOracleM(M))$.
	\end{enumerate}
\end{lemma}
\begin{proof}
	By Lemma \ref{lemma:states_and_union_graph},
	$c\circ\IStructOracleM(M) = \IStructOracle(M_{s_{\txt{union}}})$,
	where $s_{\txt{union}}$ is the state with all non-trivial model-indicators
	equal to one.
	Let $s_{XY}$ be the state where $\Rmodel_{XY}(s_{XY})=0$,
	while all other model-indicators are $1$ (see Lemma \ref{lemma:states_under_modularity}).
	By Ass.\ \ref{ass:conventional_skeleton_phase}a,
	$\CDAlg(\IStructOracle(M_{s_{\txt{union}}})) \neq \CDAlg(\IStructOracle(M_{s_{XY}}))$.
	By Ass.\ \ref{ass:conventional_skeleton_phase}c,
	this implies the existence of a test of the form
	$\vec{j}=(i_x, i_y, i_{\vec{z}}) \in \Lookup_{\CDAlg}(\IStructOracle(M_{s_{XY}}))$,
	with $R_{\vec{j}}(s_{XY}) = 0$.
	By Ass.\ \ref{ass:conventional_skeleton_phase}b (and $s_{XY}\leq s_{\txt{union}}$
	and Lemma \ref{lemma:state_ordering_implies_istruct_ordering})
	$\Lookup_{\CDAlg}(\IStructOracle(M_{s_{XY}}))\subseteq \Lookup_{\CDAlg}(\IStructOracle(M_{s_{\txt{union}}}))$.
	Since (by Lemma \ref{lemma:states_and_union_graph}, see start of this proof),
	$c\circ\IStructOracleM(M) = \IStructOracle(M_{s_{\txt{union}}})$,
	this shows that the test $\vec{j}$ satisfies requirement (b).
	
	With Lemma \ref{lemma:model_indicator_coarsest},
	it suffices to show $R_{XY|Z} \leq \Rmodel_{XY}$ for part (a).
	Let $s'$ be an arbitrary state.
	If $\Rmodel_{XY}(s')=1$, then $R_{XY|Z} \leq \Rmodel_{XY}$.
	If $\Rmodel_{XY}(s')=0$,
	we have $s' \leq s_{XY}$ (by construction of $s_{XY}$).
	By Lemma \ref{lemma:state_ordering_implies_istruct_ordering}
	thus $\IStructOracle(M_{s'}) \leq \IStructOracle(M_{s_{XY}})$.
	In particular $R_{XY|Z}(s')\leq R_{XY|Z}(s_{XY})=0\leq \Rmodel_{XY}(s')$.
	Since $s'$ was arbitrary, we thus have shown
	$R_{XY|Z} \leq \Rmodel_{XY}$, completing the proof.
\end{proof}

In the presence of cycles, a more complex representation-logic,
including representation by \inquotes{union-tests} would be required:
\begin{example}\label{example:model_indicators_on_cyclic}
	Consider a cyclic model with the following union graph,\\[0.3em]
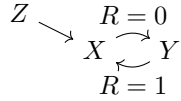
	\begin{minipage}{\textwidth}
		\centering
		\begin{tikzpicture}
			\draw (0,0) node (X) {$X$};
			\draw (1,0) node (Y) {$Y$};		
			\draw (-1,0.5) node (Z) {$Z$};
			
			\draw[->] (X) edge[bend left] node[pos=0.5, above]{$R=0$} (Y);
			\draw[->] (Y) edge[bend left] node[pos=0.5, below]{$R=1$} (X);
			\draw[->] (Z) --(X);
		\end{tikzpicture}
	\end{minipage}\\[0.5em]
	with two regimes, where in one the edge $X\rightarrow Y$ is present,
	while in the other one $X \leftarrow Y$.
	Then
	\begin{align*}
		\txt{Regime $R=0$: }
		Z \notindependent Y
		&\txt{ and }
		Z \independent Y | X\\
		\txt{Regime $R=1$: }
		Z \independent Y
		&\txt{ and }
		Z \notindependent Y | X
		\txt,
	\end{align*}
	\ie the independence-statements are complementary (one always holds).
	More generally, if there was an independently changing link $Z \rightarrow Y$,
	then the model-indicator $\Rmodel_{ZX}$ would be represented by the union
	(or product)
	$\Rmodel_{ZX}=R_{ZY|\emptyset}R_{ZY|X}$. 
\end{example}
For acyclic models, we already find one possible
(albeit unlikely optimal in the finite sample case)
procedure to obtain a model indicator:
\begin{cor}
	Assume $\CDAlg$ is consistent (Def.\ \ref{def:cd_alg_abstract_consistent}),
	Ass.\ \ref{ass:conventional_skeleton_phase} (conventional skeleton discovery),
	Ass.\ \ref{ass:aycyclicity} (acyclicity) and Ass.\ \ref{ass:momentary_faithful} (faithfulness).
	Knowing all detected indicators $R_{XY|Z}$, the model indicator on
	$XY$ is the product of all detected indicators
	on this link: $\Rmodel_{XY} = \prod_Z R_{XY|Z}$.
\end{cor}
In practice, it may be more efficient to find a representor whose existence is
guaranteed by lemma \ref{lemma:representing_model_indicator}:
\begin{cor}\label{cor:model_indicator_as_max}
	Assume $\CDAlg$ is consistent (Def.\ \ref{def:cd_alg_abstract_consistent}),
	Ass.\ \ref{ass:conventional_skeleton_phase} (conventional skeleton discovery),
	Ass.\ \ref{ass:aycyclicity} (acyclicity) and Ass.\ \ref{ass:momentary_faithful} (faithfulness).
	Fix $X$, $Y$. Then $\exists \min_{XY}\{Z\subseteq V\setminus\{X,Y\}\}$ with respect to the
	order-relation defined in Def.\ \ref{def:state_construction_partial_orders}b and
	the detected indicator associated to
	this minimum coincides with the model indicator:
	\begin{equation*}
		\Rmodel_{XY} \equiv R_{XY|\min_{XY}\{Z\subseteq V\setminus\{X,Y\}\}}
	\end{equation*}
	Further, at least one $Z$ realizing this minimum is such that $X\independent Y|Z$
	is in the lookup-region $\Lookup_{\CDAlg}(c\circ\IStructOracleM(M))$.
\end{cor}
The search for this (in general non-unique) minimum $\min_{XY}\{Z\subseteq V\setminus\{X,Y\}\}$
can be easily and efficiently implemented (see Algo.\ \ref{algo:discover_model_indicators}): Initialize all model-indicators as trivial $\equiv 1$
(as CD starts from a fully connected graph).
Whenever $X \independent Y | Z$ for some $Z$, update $R_{XY}^{\txt{model}} \equiv 0$,
whenever we find $X \independent_R Y | Z$ thus a non-trivial $R_{XY|Z}$,
update $R_{XY}^{\txt{model}} = R_{XY|Z}$ \emph{if} $R_{XY|Z} \leq R_{XY}^{\txt{model}}$.
This last comparison is either trivial\footnote{If $\Rmodel_{XY}\equiv 1$,
then $R_{XY|Z} \leq \Rmodel_{XY}$ is always true, if
$\Rmodel_{XY}\equiv 0$, then $R_{XY|Z} \leq \Rmodel_{XY}$ is always false.} (if the previous value for $R_{XY}^{\txt{model}}$ was trivial), or a comparison with
$R_{XY}^{\txt{model}} = R_{XY|Z'}$ (\ie a previously found representation by a detected indicator),
so it requires only testing of implications of the form
Def.\ \ref{def:indicator_implciation} between detected indicators
(in fact, here always $K=1$ in Def.\ \ref{def:indicator_implciation}).

\begin{algorithm}
	\caption{Model Indicator Discovery (union-acyclic case, $g_s$ are MAGs)}
	\label{algo:discover_model_indicators}
	\begin{algorithmic}[1]
		\State \textbf{Input:}
		A set $I$ of all independencies tested, a set $J\subseteq I$ of marked independencies.
		\State \textbf{Output:}
		A set $\mathcal{R}$ containing exactly one representative $R_{\vec{j}}=R_{XY|Z}$ for
		each non-trivial model-indicator $\Rmodel_{XY}$, such that $\Rmodel_{XY}\equiv R_{XY|Z}$.
		Enumerating $R_{\vec{j}}$ by $\vec{j}$, we can consider $\mathcal{R} \subseteq J$.
		\State Initialize $\mathcal{R} = \emptyset$.
		\State Define $I_{XY} := \{\vec{i}=(i_X, i_Y, \vec{i}_Z)\in I|
			i_X \txt{ is index of }X, i_Y \txt{ is index of }Y\}$.
		\State Define $J_{XY} := I_{XY} \cap J$.
		\For{each pair\Slash{}link $(X, Y)$ with $J_{XY} \neq \emptyset$}
			\If{$\exists \vec{i}_z$ with $(i_X, i_Y, \vec{i}_Z)\in I_{XY}$ and
				$X \independent Y | Z$}
				\State \textbf{continue} with next link.
			\EndIf
			\State Initialize $\Rmodel_{XY} :\equiv 1$.
			\For{$\vec{j}\in J_{XY}$ with associated $R_{\vec{j}}=R_{XY|Z}$}
				\If{$R_{XY|Z} \leq \Rmodel_{XY}$}
					\Comment{If $\Rmodel_{XY} \equiv 1$ this is always true.}
					\State Replace $\Rmodel_{XY} :\equiv R_{XY|Z}$.
					\Comment{Rmk: By $\vec{j}\in J$, $R_{XY|Z}$ is non-trivial.}
				\EndIf
			\EndFor
			\State Add $\Rmodel_{XY}$ to $\mathcal{R}$.
			\Comment{Rmk: By $J_{XY} \neq \emptyset$, the added indicator is non-trivial.}
		\EndFor
		\State\Return $\mathcal{R}$
	\end{algorithmic}
\end{algorithm}

Finally, if changes are independent (Ass.\ \ref{ass:states_modular_and_reached}),
then the non-trivial model-indicators span the state-space,
see Lemma \ref{lemma:states_under_modularity}.
\Ie knowledge of all model indicators allows in this case to enumerate all states.
Before we continue with characterizing these states by properties of the model
in §\ref{apdx:state_translation}, we briefly apply the above ideas to
an example.

\begin{example}
	[Continuing Example \ref{example:in_parallel_in_sequence}]	\label{example:in_parallel_in_sequence_continued}
	We apply the previous section to example \ref{example:in_parallel_in_sequence}
	(see Fig.\ \ref{fig:state_space_boolean_layouts}). In case (b), $X \independent Y | Z$,
	so the model-indicator $\Rmodel_{XY} \equiv 0$ is correctly concluded to be trivially zero.
	In case (c), $X \independent_R Y | M$ with $R_{XY|M} \leq R_{XY|\emptyset}$
	thus $\{M\} \leq_{XY} \emptyset$, so while $R_{XY|\emptyset}$ depends on both model-indicators,
	we correctly represent $\Rmodel_{XY} \equiv R_{XY|M}$.
\end{example}

\subsection{Interpretation of Discovered States}\label{apdx:state_translation}

When defining model properties in §\ref{sec:models} we in particular
defined states of models (§\ref{sec:model_states}).
While the elements of a multi-valued independence-structure
also correspond to \inquotes{states} of the independence-structure
(see example \ref{example:oracle_extended}),
it is not clear in general, how the elements of the multi-valued independence-structure
can be associated to states on the model.
We will refer to this task as indicator translation.

\begin{definition}[Indicator Translation]
	Given a non-stationary SCM $M$ with state-space $S(M)$ (Def.\ \ref{def:states})
	and extended independence-structure $\IStructOracleX(M)$,
	then the indicator translation is the mapping
	\begin{equation*}
		\psi:
		S(M) \twoheadrightarrow \IStructOracleX(M),
		\quad
		s \mapsto \IStructOracle(M_s)\txt.
	\end{equation*}
\end{definition}
\begin{rmk}
	By definition, $\IStructOracleX(M)$ consists precisely of the
	elements of the form $\IStructOracle(M_s)$, so $\psi$ always exists and is surjective.
	Given only data produced by an SCM $M$, we will typically
	not have prior knowledge of either $S(M)$ or the mapping $\psi$.
	Discovering an indicator translation thus means giving the 
	discovered elements of $\IStructX=\IStructOracleX(M)$
	an interpretation in terms of model properties (see §\ref{sec:indicator_translation}).
	Generally $\psi$ may not be injective, if some model-states lead
	to the same independence structure (for example if the are in the same Markov
	equivalence-class); thus the inverse is not generally well-defined and such an
	interpretation in terms of model properties may not be uniquely identifiable.
\end{rmk}

In the presence of hidden latent confounders (and inducing paths) or cycles,
this is a rather non-trivial question. However, under the present assumptions,
the evident interpretation applies:

\begin{example}[Acyclic Model-Interpretation]
	\label{example:indicator_translation_restricted_case}
	Assume
	Ass.\ \ref{ass:aycyclicity} (acyclicity),
	Ass.\ \ref{ass:momentary_faithful} (faithfulness) and
	Ass.\ \ref{ass:states_modular_and_reached} (modularity and reachedness),
	then the state-translation in this sense is given by the bijection
	$S \xrightarrow{\txt{1:1}} \braket{\Rmodel}$ (by Lemma \ref{lemma:states_under_modularity}).
	In cases where all model-indicators can be discovered (see §\ref{apdx:model_indicators}),
	we can thus interpret states in model-terms:
	Since we describe states by the values taken 
	by model indicators, a state can be interpreted in model-terms,
	by saying the mechanisms at the non-trivial indicators are
	active iff the value of the corresponding model-indicator is one in this state.
\end{example}

\subsection{Representation of Detected Indicators}\label{apdx:detected_indicator_representation}

The next step is to represent detected indicators as functions of the state.
Since states under the present modularity assumptions
(Ass.\ \ref{ass:states_modular_and_reached}) are described by 
the values of the model-indicators (Lemma \ref{lemma:states_under_modularity}),
this can be reduced to relations to model-indicators.
Indeed, by using previous observations (in particular Lemma \ref{lemma:do_not_open_paths}),
and given representing tests of model-indicators (cf.\ Lemma \ref{lemma:representing_model_indicator}),
this can be accomplished by using \emph{only} implication-tests
of the form described in Def.\ \ref{def:indicator_implciation}.

We know,
that $R_{\vec{j}}$ can be written as a function of the state only
(see Lemma \ref{lemma:graphs_factor_through_states}).
Further the state $s=(s_1, \ldots, s_k)$ is described by
the values $s_i$ of the model-indicators.
While both domain (state-space) and target ($\{0,1\}$)
of $R_{\vec{j}}$ have natural structures as $\mathbb{Z}/2\mathbb{Z}$ vector
spaces, the map $R_{\vec{j}}$ is of course not linear (Example \ref{example:in_parallel_in_sequence}).
In particular, given $k$ model indicators, there are not $k$,
but $2^{2^k}$ choices (for each of the $2^k$ states, $R_{\vec{j}}$ can take $2$ possible values).
Fortunately, by Lemma \ref{lemma:state_ordering_implies_istruct_ordering},
$s \leq s'$ $\Rightarrow$ $R_{XY|Z}(s) \leq R_{XY|Z}(s')$ (for all detected indicators
$R_{XY|Z}$), thus we know $R_{\vec{j}}$ is monotonic, which will help us to reduce the
search-space. This is what ultimately makes Algo.\ \ref{algo:indicator_resolution} work.

\begin{lemma}[State-Resolution of Detected Indicators]\label{lemma:resolve_detected_indicators}
	Given a marked independence $\vec{j}\in J$,
	the set $\mathcal{R} \subseteq J$ of model-indicators represented by marked independencies
	(cf.\ §\ref{apdx:model_indicators}),
	and an oracle $\IStructOracleRel$ (see example \ref{example:oracle_implication})
	for implication-tests of the form Def.\ \ref{def:indicator_implciation}
	(used for the test in line 16 of the algorithm).
	Then the output of Algo.\ \ref{algo:indicator_resolution}
	yields the detected indicator $R_{\vec{j}}(s)$ as a function of
	state.
\end{lemma}
\begin{proof}
	\emph{The result corresponds to a function of $s$:}
	The algorithm returns either an element of $\mathcal{R}$ (line~\ref{line:translate_quick_return}),
	or a subset of (the initial) $C$ (line~\ref{line:translate_final_return}).
	In both cases, the result is a set of monomials of elements of $\mathcal{R}\subseteq J$.
	Elements of $\mathcal{R}$ represent the model-indicators, thus
	$\bar{R}$ (see line~\ref{line:translate_output}) is a boolean combination of
	of model-indicators. Model-indicators, by definition of our state-space,
	take a known value $s_i = \pi_i(s)$ on state $s$.
	Thus $\bar{R}(s)$ is given by $s$ as the specified boolean combination
	of known values $s_i$.
	
	\emph{Correctness:}
	Let $s$ be arbitrary.
	Case 1 ($R_{\vec{j}}(s) = 1$):
	Note that either the return-statement in line \ref{line:translate_quick_return}
	was executed or line \ref{line:translate_test} is tested true at least once, but never both
	or none.
	
	Case 1a (line \ref{line:translate_test} is tested true at least once):
	In this subcase $N\neq \emptyset$ and
	If $R_{\vec{j}}(s) = 1$, then
	by line \ref{line:translate_test}
	(rather: by contraposition of the if-statement's condition),
	$\forall r\in N$: $r(s) = 1$, thus $\bar{R}(s)=1$.
	
	Case 1b (the return-statement in line \ref{line:translate_quick_return} was triggered):
	The state $s_0$ with all entries $0$ (equivalently: all model indicators 
	evaluate to $0$ on $s_0$)
	satisfies $\forall s' : s' \geq s_0$ by construction, in particular
	by Lemma \ref{lemma:state_ordering_implies_istruct_ordering},
	$\forall s': R_{\vec{j}}(s')\geq R_{\vec{j}}(s_0)$.
	If it were $R_{\vec{j}}(s_0)=1$, then this would imply $R_{\vec{j}}\equiv 1$
	is non-trivial, which contradicts $\vec{j}\in J$.
	Thus $R_{\vec{j}}(s_0)=0$ and $s\neq s_0$.
	Let $R' := \boolOr_{\Rmodel\in\mathcal{R}} \Rmodel$,
	by construction $R'(s'') = 0 \Leftrightarrow s''=s_0$.
	Therefore the implication $R'(t) = 0 \Rightarrow R_{\vec{j}}(t)=0$
	in line \ref{line:translate_test} would produce a true result
	and we would be in case (a) if it had been tested, thus $R' \in C$ still remains
	in $C$ when we encounter the return statement in line \ref{line:translate_quick_return}.
	Therefore case $\bar{R}(s) = R'(s) = 1$.
	Note that $R'$ was not removed from $C$ by line \ref{line:translate_remove},
	if this line had been triggered, we would be in case (a).

	Case 2 ($R_{\vec{j}}(s) = 0$):
	If $R_{\vec{j}}(s) = 0$, define $S_0 := \{\Rmodel\in\mathcal{R}|\Rmodel(s)=0\}$.
	Further let $R'(\tilde{s}):=\boolOr_{\Rmodel\in S_0} \Rmodel(\tilde{s})$.
	Note, that $R'\in C_{\txt{init}}$ is an element of $C$ (as initialized in 
	line \ref{line:translate_Cinit}) and that $R'(s)=0$ (by construction of $S_0$).
	Indeed given \emph{any} state $s'$ with the $R'(s')=0$,
	then $s' \leq s$: otherwise there would have to be $r\in\mathcal{R}$
	with $r(s') = 1$ and $r(s)=0$, thus $r\in S_0$.
	Then $R'(s') = \boolOr_{\Rmodel\in S_0} \Rmodel(s') \geq r(s')=1$
	contradicting $R'(s')=0$.
	
	Thus for an arbitrary $s'$, we have $R'(s')=0$ $\Rightarrow$ $s' \leq s$
	$\Rightarrow$ $R_{\vec{j}}(s') \leq R_{\vec{j}}(s) = 0$,
	where the last step finally and crucially uses the monotonicity property.
	Since this holds for arbitrary $s'$ it holds for all times $t\in T=\sigma^{-1}(S)$,
	\ie $R'(t)=0 \Rightarrow R_{\vec{j}}(t)=0$.
	With $R'\in C_{\txt{init}}$ there are two cases:
	
	Case 2a ($R'\in C_{\txt{init}}$ is reached by the for-loop):
	In this case the test (oracle evaluation) of $R'(t)=0 \Rightarrow R_{\vec{j}}(t)=0$
	in line \ref{line:translate_test} produces the result \texttt{true} (this
	implication is correct, as was shown above).
	In particular $R'$ gets added to $N$ and is contained in the final result.
	Thus $\bar{R}(s) \leq R'(s) =0$.
	
	Case 2b ($R'$ was removed from $C$ by line \ref{line:translate_remove}):
	The removal in line \ref{line:translate_remove} is only triggered
	if there is $r$ with $R'$ containing all factors of $r$,
	\ie $r(\tilde{s}) = \boolOr_{\Rmodel\in S_0'} \Rmodel(\tilde{s})$ for
	some $S_0' \subseteq S_0$ which was added to $N$, in particular
	this $r$ is in the final result.
	Clearly by $S_0' \subseteq S_0$ we have $\forall \tilde{s}$:
	$r(\tilde{s}) = \boolOr_{\Rmodel\in S_0'} \Rmodel(\tilde{s})
	\leq \boolOr_{\Rmodel\in S_0} \Rmodel(\tilde{s})=R'(\tilde{s})$.
	In particular $\bar{R}(s) \leq r(s) \leq R'(s) = 0$.
	
	Case 2c (the return-statement in line \ref{line:translate_quick_return} was triggered):
	Since $R' \in C_{\txt{init}}$ and neither case (a) or (b) apply,
	$R'\in C$ is still in $C$ and is thus returned.
	In this case actually $\bar{R}(s) = R'(s) = 0$.
	
	Thus in all sub-cases, $\bar{R}(s)=0=R_{\vec{j}}(s)$
	(the last equality is by case-hypothesis of case 2).	
\end{proof}

\begin{algorithm}
	\caption{Indicator Representation (union-acyclic case, $g_s$ are MAGs)}
	\label{algo:indicator_resolution}
	\begin{algorithmic}[1]
		\State \textbf{Input:}
		A set $\mathcal{R}\subseteq J$ of model-indicators represented by (marked) test-indices,
		a marked test-index $\vec{j}=(i_x,i_y,i_z) \in J$
		\State \textbf{Output:}
		A set $N$ of monomials of the form $R_k = (R_{i_1} \boolOr \ldots \boolOr R_{i_n})$ of
		model-indicators, such that the boolean combination
		$\bar{R}_N = \boolAnd_{k\in N} R_k$ satisfies $\bar{R} = R_{\vec{j}}$
		\label{line:translate_output}
		\State Initialize the candidate-set
		$C := \{ R_{i_1} \boolOr \ldots \boolOr R_{i_n} | i_j \neq i_{j'} \txt{ if }j\neq j';
		R_{i_j} \in \mathcal{R}; n \geq 1 \}$
		\label{line:translate_Cinit}
		\State Initialize the necessary set $N := \emptyset$
		\If{$\vec{j}\in \mathcal{R}$ (as a representor)}
			\State\Return $\{ R_{\vec{j}} \}$ \label{line:translate_quick_return}
		\EndIf
		\If{$\exists \vec{j}'=(i_x',i_y',i_z')\in \mathcal{R}$
			with $i_x' = i_x$ and $i_y'=i_y$}
			\State Remove all combinations \emph{not} involving $\vec{j}'$ as a factor from $C$
		\EndIf
		\For{element $r\in C$ ordered by increasing monomial degree\footnotemark}
			\If{$\setelemcount{C} = 1$ and $N = \emptyset$}
				\State\Return $C$
				\Comment{for finite-sample stability, not needed with oracle}
			\EndIf
			\State Remove $r$ from $C$
			\If{test: $r(t)=0 \Rightarrow R_{\vec{j}}(t)=0$}
				\Comment{$r=1$ is necessary for $R_{\vec{j}} = 1$}
				\label{line:translate_test}
				\State remove all combinations involving (containing all factors of) $r$ from $C$
				\label{line:translate_remove}
				\State replace $N := N \cup \{r\}$
			\EndIf		
		\EndFor
		\State\Return $N$ \label{line:translate_final_return}
	\end{algorithmic}
\end{algorithm}
\footnotetext{$C$ contains,
		by definition, boolean combinations of at least $n\geq 1$ different
		(thus $n\leq |\mathcal{R}|$) model-indicators. Here, \inquotes{ordered by increasing monomial
			degree} means we start with all entries with $n=1$, then all entries with $n=2$ etc.\ 
		until we reached all elements of $C$. This is a heuristic choice to keep the left-hand-side
		in implication-tests simple; a specific ordering is not required by the proof\Slash{}for
		correctness.}

\subsection{State-Space-Construction Result}\label{apdx:state_space_construction_J}

We formalize the results of this section as a two phase algorithm,
which is applicable with the core algorithm in §\ref{apdx:core_algorithm}.
This means, we formally show that the previous results can be combined 
to satisfy the formal specification of state-space construction
(Def.\ \ref{def:state_space_construction_specification}) used by the core algorithm:

\theoremstyle{plain}
\newtheorem*{thm_state_space_construction}
{Theorem \ref{thm:state_space_construction}F}
\begin{thm_state_space_construction}[State-Space Construction]
	\makeatletter\def\@currentlabel{\ref*{thm:state_space_construction}F}\makeatother
	\label{apdx:thm:state_space_construction}
	Assume $\CDAlg$ is consistent (Def.\ \ref{def:cd_alg_abstract_consistent}),
	Ass.\ \ref{ass:conventional_skeleton_phase} (conventional skeleton discovery),
	Ass.\ \ref{ass:aycyclicity} (acyclicity) and Ass.\ \ref{ass:momentary_faithful} (faithfulness) and Ass.\ \ref{ass:states_modular_and_reached} (modularity).
	Given are further oracles for marked independencies $\IStructOracleM$ (example \ref{example:oracle_marked})
	and for indicator-implications $\IStructOracleRel$ (example \ref{example:oracle_implication}).

	If $J=\emptyset$, then $S^J=\Map(\emptyset, \{0,1\})$.
	Otherwise,
	apply first Algo.\ \ref{algo:discover_model_indicators}
	to find model-indicators $\mathcal{R}$,
	and their product-space $S_M$.
	Then Algo.\ \ref{algo:indicator_resolution} applied
	with this $\mathcal{R}$ as argument
	yields representations $R_{\vec{j}}(s)$ for each $\vec{j}\in J$. The mapping
	\begin{equation*}
		S_M \hookrightarrow \Map(J,\{0,1\}),
		s \mapsto (\vec{j} \mapsto R_{\vec{j}}(s))
	\end{equation*}
	is injective, in particular we can identify $S_M$ with its image
	$S^J:=\{ \vec{j} \mapsto R_{\vec{j}}(s) | s\in S_M \}$.
	
	Then, the described assignment $J \mapsto S^J \subseteq \Map(J,\{0,1\})$
	is a sound and complete state-space construction (Def.\ \ref{def:state_space_construction_specification})
	for any model satisfying the assumptions listed in the hypothesis.
\end{thm_state_space_construction}
\begin{proof}
	\textbf{Phase I: Model Indicators.}
	Let $R_{XY}$ be an arbitrary model indicator.
	
	Case 1 ($R_{XY}$ is non-trivial):
	By Lemma \ref{lemma:representing_model_indicator},
	there is a representing test for each model indicator (part (a)) $\Rmodel_{XY} \equiv R_{\vec{j}}$
	and by part (b), this test is in $\vec{j}\in\Lookup_{\CDAlg}(c\circ\IStructOracleM(M))$.
	By non-triviality of $\Rmodel_{XY}$ (and $\Rmodel_{XY} \equiv R_{\vec{j}}$),
	also $R_{\vec{j}}$ is non-trivial, thus $\IStructOracleM(M)(\vec{j})=\mCIToutR$
	(by definition, see example \ref{example:oracle_marked}).
	According to Def.\ \ref{def:state_space_construction_specification},
	we may assume 
	$J \supseteq \Lookup_{\CDAlg}(c\circ\IStructOracleM(M)) \cap \IStructOracleM(M)^{-1}(\{\mCIToutR\})$.
	Therefore $\vec{j}\in J$
	and the result $\vec{j}'$ of Algo.\ \ref{algo:discover_model_indicators}
	satisfies $R_{\vec{j}'} \leq R_{\vec{j}}$ if both are comparable (the algorithm finds a minimum).
	By Lemma \ref{lemma:model_indicator_coarsest} the model-indicator is minimal
	and comparable to all detected indicators on the link $X$--$Y$,
	in particular $R_{\vec{j}} \leq R_{\vec{j}'}$ and both are comparable (thus by the above also
	$R_{\vec{j}'} \leq R_{\vec{j}}$).
	Thus $R_{\vec{j}'} \equiv R_{\vec{j}} = \equiv \Rmodel_{XY}$.
	Note, that while the partial order induces a preorder on $J$, this preorder
	will typically not be anti-symmetric, \ie we cannot conclude $\vec{j}=\vec{j}'$.
	This is not a problem as we only require \emph{a} test $\vec{j}$ to represent $\Rmodel_{XY}$.
	
	Case 2a ($\Rmodel_{XY} \equiv 0$):
	By Ass.\ \ref{ass:conventional_skeleton_phase} (conventional skeleton discovery),
	there is a test $X\independent Y |Z$ in the lookup-region of the union-graph
	$\Lookup_{\CDAlg}(c\circ\IStructOracleM(M))$.
	In particular
	Algo. \ref{algo:discover_model_indicators} will discard this link and
	not include it into its output.
	
	Case 2b ($\Rmodel_{XY} \equiv 1$):
	By Lemma \ref{lemma:model_indicator_coarsest}, the model-indicator is minimal,
	that is for all $Z$: $R_{XY|Z} \geq \Rmodel_{XY} \equiv 1$,
	and the marked-independence oracle always returns $1$ on this link.
	In particular
	Algo. \ref{algo:discover_model_indicators} will never consider this link and
	not include it into its output.
	
	Thus we obtain the correct state-space $S_M$ (Lemma \ref{lemma:states_under_modularity})
	with valid representatives of all non-trivial model-indicators.
	
	\textbf{Phase II: Resolution of Detected Indicators.}
	By Lemma \ref{lemma:resolve_detected_indicators}, Algo.\ \ref{algo:indicator_resolution}
	given a marked test $\vec{j} \in J$ produces a representation
	$\bar{R}_{\vec{j}}(s)$ such that $\bar{R}_{\vec{j}} \equiv R_{\vec{j}}$.
	In particular for any state $s \in S_M$, we have
	$\IStructOracle(M_s)(\vec{j}) = \bar{R}_{\vec{j}}(s)$.
	This is (see Notation \ref{notation:core_algo_states}) denoted as $s(\vec{j})$
	in Lemma\Slash{}Notation \ref{lemma:JstateDescr} where
	the interpretation of $S_T^J$ as a subset of $\Map(J,\{0,1\})$ is
	then specified precisely as $\vec{j} \mapsto s(\vec{j})$.
	Therefore the image of the mapping
	\begin{equation*}
		S_M \hookrightarrow \Map(J,\{0,1\}),
		s \mapsto (\vec{j} \mapsto R_{\vec{j}}(s))
	\end{equation*}
	which is called $S^J$ in the statement of the proposition,
	satisfies indeed $S^J = S^J_T$ as required.
	(If we identify $S_M$ and the true $S_T$ in a suitable way -- see
	§\ref{apdx:state_translation} --
	then this mapping is the same as the mapping in Lemma\Slash{}Notation \ref{lemma:JstateDescr}).
\end{proof}

\begin{rmk}\label{rmk:state_space_constr_restricted_J}
	The inclusion of the constraint 
	$J \supseteq \Lookup_{\CDAlg}(c\circ\IStructOracleM(M)) \cap \IStructOracleM(M)^{-1}(\{\mCIToutR\})$ into 
	Def.\ \ref{def:state_space_construction_specification} is a particular example
	of the requirement for coordination between the core-algorithm (and also the choice of $c$),
	the state-space construction and the used CD-algorithm $\CDAlg$.
	Some degree of fine-tuning concerning these interfacing-requirements
	in future work is likely not avoidable.
	The main point of our presentation is not the strict low-level enforcement of
	a particular interface between modules, but rather the clear specification of
	logically separate modules: mCIT testing, state-space construction, core-algorithm
	and underlying CD-algorithm solve different problems and can be largely \emph{understood}
	individually. They have to be fit together in practice, but this is primarily a lower level
	(implementation type) question, as opposed to the higher lever (design type)
	solution of the respective main tasks.
\end{rmk}

\subsection{Edge-Orientation Transfer}\label{sec:edge_orientation_transfer}

By Ass.\ \ref{ass:aycyclicity} (acyclicity), orientations cannot change between states.
Additionally, under our assumptions, the model indicators are found and
attributed correctly (example \ref{example:indicator_translation_restricted_case});
all changes manifest in \emph{edge-removals} only (on the true graphs).
Thus, when in practice only partial graphs (with some edges unoriented)
are recovered for the $G_s$, we can transfer orientations between states
in a straight-forward manner:
\begin{lemma}[Orientation Transfer]\label{lemma:transfer_of_orientations}
	Given Ass.\ \ref{ass:aycyclicity},
	if $X\circ$--$*Y$ in $G_s$ with a definitive edge mark $*$ at $Y$ (arrow-head or tail)
	and any edge mark $\circ$ at $X$, and $X$, $Y$ are adjacent in $G_{s'}$,
	then $X\circ$--$*Y$ in $G_{s'}$, \ie edge marks are preserved.
\end{lemma}
\begin{proof}
	By union-acyclicity, edge-marks cannot change: An edge-flip would immediately lead
	to a union-cycle, a change $\leftrightarrow$ to $\rightarrow$ would imply the existence
	of an almost directed cycle in the union graph.
	With the union graph being a AG (Ass.\ \ref{ass:aycyclicity}), neither are possible.
	Thus, if the edge is present in $s'$,
	then it has the same edge marks as in $s$.
\end{proof}
Note, that JCI \citep{JCI} arguments may allow for orientations of edges in the union-graph.
By our assumptions, the union-graph does correspond to a reached space
(see Lemma \ref{lemma:states_and_union_graph}) so the lemma above
applies.
Indeed, the same arguments of the last proof work just as well on the union-graph,
thus leading to an alternative (more direct) proof of:
\begin{lemma}[Orientation Transfer from Union-Graph]
	If $X \circ$--$*Y$ in $\gUnion{G}$, and $X$, $Y$ are adjacent in $G_{s}$,
	then $X \circ$--$*Y$ in $G_{s}$.
\end{lemma}
\begin{proof}
	Apply Lemma \ref{lemma:states_and_union_graph} to obtain a representing state
	$s_{\txt{union}}$. Then apply Lemma \ref{lemma:transfer_of_orientations}
	with $s_{\txt{union}}$ and $s$.
\end{proof}
Finally, note that orientations transferred to a state $G_s$ in one of these ways
may make new orientation-rules applicable in that state. Thus it may be
necessary to iterate orientation-phases and transfer-phases till convergence
(until none of the graphs changes anymore; since we only get \emph{more} orientations
[monotonicity]
and there is a finite number of edges to orient in the graph [boundedness],
this iteration is guaranteed to converge after a finite number of rounds).
The uniformity of edge marks across states and union-graph, means
the causal system can be summarized (under the assumptions of the present section)
by a single graph with edges with non-trivial model-indicator marked as varying:
\begin{notation}
	Our final output (under the assumptions of the present section)
	is represented by a union-graph, with edges corresponding to
	non-trivial model-indicators marked (\eg in color), see Fig.\ \ref{fig:intro_toy_model}.
\end{notation}
Under the assumption of modular\Slash{}independent changes
(Ass.\ \ref{ass:states_modular_and_reached}),
all links vary independently and there is (in principle) no reason to distinguish them
(use different colors).
In practice there are at least two important reasons to use individual markings (different colors):
One is the potential post-processing for relations between states (§\ref{apdx:non_modular_changes}),
the other is the potential post-processing for (time-)resolutions (§\ref{apdx:indicator_resolution}).

\subsection{Order Independence}

The question of order independence -- that is ensuring that the result produced
does not depend on the order in which variables are given in the data-set --
has seen substantial
interest for CD-algorithms after its discussion in \citep{PCstable}.

Mostly, the order-independence of our HCCD approach is decided by the
order-independence of the employed underlying (stationary) CD-algorithm.
With most modern CD-algorithms (like PC-stable or PCMCI and its derivatives)
ensuring order-independence, applying our framework with these algorithms
should ensure order-independence of state-graphs.

However, there is a subtlety related to state-space construction as presented above:
The order of variables in the data-set \emph{may} affect which
indicator-implication tests (§\ref{apdx:required_implication_tests}) are executed in Algos
\ref{algo:discover_model_indicators} and \ref{algo:indicator_resolution},
and thus how model-indicators are represented by tests and
how detected indicators are resolved over states.
This can potentially change the multi-valued independence structure
$\IStructExt$ (§\ref{sec:extended_indep_struct_subsection}) concluded,
and thereby "poison" the per-state independencies actually presented
to the CD-algorithm by the pseudo-cit.

This kind of order-dependence only arises if indicator-implication tests
are required, which currently is the case only if two or more
model-indicators are present (indeed for two model-indicators,
there should only be two relevant orderings, so computing
results for both checking for consistency, see below, is plausible).
As was seen in §\ref{apdx:implication_testing}, situations of multiple
(many) model-indicators interacting should be avoided from a practical
perspective also for statistical performance reasons.
Therefore the most relevant and plausible next step
with regard to this problem appears to be to leverage graphical
knowledge to simplify the problem by blocking paths from
(non-local) model-indicators (see §\ref{sec:ind_locality}).

In principle there are at least two evident possibilities to
eradicate order-dependence more generally and directly:
(a) evaluate Algos
\ref{algo:discover_model_indicators} and \ref{algo:indicator_resolution}
for all (relevant)
variable orders and check for consistency (or use a majority vote),
or (b) attach a score (measuring the confidence of a indicator-implication
test, similar to scores in score-based causal discovery) and
proceed in Algos
\ref{algo:discover_model_indicators} and \ref{algo:indicator_resolution}
in each step greedily, using the most confident test.

\subsection{Future Work}\label{apdx:indicator_translation_future_work}

Our state-space construction requires many assumptions, in particular
acyclicity, and may require many complicated tests for large numbers of model-indicators.
The presented approach provides a proof of concept and shows that a systematic formal
treatment is possible. Yet, many interesting questions remain open.
We summarize some of them in this section.

\subsubsection{Locality of Translation}\label{sec:ind_locality}

Model indicators are local in the model, in the sense of only affecting a single link.
Detected indicators are \inquotes{semi-local} in that they only depend on the collection of
paths between two nodes $X$, $Y$.
So how local is semi-local? \Ie do detected indicators for tests $X\independent Y |Z$
really depend on model-indicators \inquotes{far away} in the graph from both $X$ and $Y$?
There are multiple different reasons that make us believe that
indeed semi-local here means more local than global:
\begin{enumerate}[label=(\alph*)]
	\item
	A test becomes independent only if \emph{all} paths from $X$ to $Y$ vanish,
	if $X$ and $Y$ are far apart in the graph, this either requires the graph
	to be very sparse (compared to model-indicator density) or to be
	separated by a bottleneck. Both should be visible from the (union) graph skeleton.
	\item 
	A regime is only detected in practice, if the remaining (after possibly many 'hops')
	signal-to-noise ratio remains high. This limits the plausible distance
	between $X$, $Y$ to be considered.
	This observation is known for CD more generally to be relevant,
	future progress on this problem for the general case may be
	useful to the present problem.
	\item
	Causal discovery aims at leveraging local separating-sets, for example
	PC searches for subsets of parents. The remaining independencies do not
	have to be tested, as they follow from the local properties of the graph.
	Similarly, a suitable analysis of the causal (union-)graph --
	which in the acyclic case is known already at the point where indicator-translation
	is performed -- should allow to reduce the search for model indicators and
	relations of model-indicators to their neighborhood in the graph.
	\item
	Sufficient knowledge about the causal graph can imply knowledge of the
	form of detected indicators, as in example \ref{example:in_parallel_in_sequence}.
	It may thus be possible to extract representations of detected indicators
	or at least restrict possible explanations (representations via model indicators).
	\item
	So far, our regime detection happens on-the-fly on tests that are executed
	by the CD-algorithm anyway. It is of course possible, to inject between
	model-indicator discovery and representation of detected indicators,
	an additional testing phase to block overlapping influences of multiple indicators.
\end{enumerate}
Points (a) and (b) could be approached in practice by testing for dependence on
model-indicators, for example in a JCI-sense, which one might want to do anyway.
Points (c-e) could offer a very general approach, but require the systematic
incorporation of d-separation rules into indicator-translation.
We leave these to future work.

These non-local issues in state reconstruction are made explicit and attainable
by our framework. They are, of course, an inherent aspect of the HCCD problem,
so while less evident in other approaches, they always affect results.
For example for global approaches (one context taking exponentially many
values, Fig.\ \ref{fig:illustrate_global_vs_local_structure}), the comparison
of the (many) resulting per-context graphs
(for example to improve stability or analyze for the
origin of changes) requires arguments regarding these questions.
Given a very large number of (noisy) per-context graphs, such a post-processing
step usually is required for global methods, in practice, to interpret results.

\subsubsection{Cyclic Models}\label{apdx:state_space_cyclic_models}

As was demonstrated in example \ref{example:model_indicators_on_cyclic},
cyclic models require novel \inquotes{union tests} besides implication tests.
Somewhat surprisingly, these may be easier to realize on data: 
The null hypothesis can be the existence of a split of the data, such
that in each subset (at least) one of two independencies holds.
Other than the weak regime (and implication) testing problems,
this does \emph{not} require the combination of a rejection
(of homogeneity) with acceptance of independence (in the weak regime),
instead this single null hypothesis can be accepted or rejected.

\subsubsection{Inducing Paths (and Almost Cycles)}\label{apdx:state_space_inducing_paths}

We excluded the almost cyclic case (in presence of latent variables)
together with the cyclic case, because similar problems occur:
\begin{example}\label{example:inducing_path_indicator}
	Inducing paths on non-AG union graphs:\\
	\begin{minipage}{0.6\textwidth}
		On the right-hand side, assume $L$ to be a hidden latent (\ie unobserved),
		and the dashed arrow to be active only in regime A but not in B.
	\end{minipage}
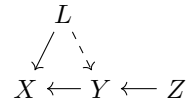
	\begin{minipage}{0.35\textwidth}
		\centering
		\begin{tikzpicture}
			\draw (0,0) node (X) {$X$};
			\draw (1,0) node (Y) {$Y$};
			\draw (2,0) node (Z) {$Z$};
			\draw (0.5,1) node (L) {$L$};
			\draw[->] (Y) -- (X);
			\draw[->] (Z) -- (Y);
			\draw[->] (L) -- (X);
			\draw[->,dashed] (L) -- (Y);
		\end{tikzpicture}
	\end{minipage}\\
	In this case in regime B, $X \independent Z | Y$, but
	in regime A, $X \dependent Z | Y$ as the path through $L$ is now open.
	$X \dependent Z$ in both cases. The non-deletable link $Z \rightarrow X$
	in regime A is called an \inquotes{inducing path} (see \eg \citep{spirtes2001causation, ZhangFCIorientationrules}).
\end{example}

There are also some encouraging observations, however.
For example, a JCI-argument unveils $Y$ as the only child of the indicator $R$.
If the link $Z \rightarrow X$ were real and would change, then either $X$
or $Z$ would have to be a child of $R$. Thus, in this particular example,
we actually \emph{know} that we found an inducing path.
Conversely one may also be able to draw the conclusion
that a some link cannot be an inducing path if there is no candidate for
a confounder inducing it:
\begin{example}\label{example:inducing_path_excluded}
	No inducing path possible:\\
	\begin{minipage}{0.6\textwidth}
		Compared to the previous example, the arrow $Y \rightarrow X$ was
		removed, there is no almost cycle.
	\end{minipage}
	\begin{minipage}{0.35\textwidth}
		
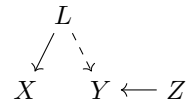
\centering
		\begin{tikzpicture}
			\draw (0,0) node (X) {$X$};
			\draw (1,0) node (Y) {$Y$};
			\draw (2,0) node (Z) {$Z$};
			\draw (0.5,1) node (L) {$L$};
			\draw[->] (Z) -- (Y);
			\draw[->] (L) -- (X);
			\draw[->,dashed] (L) -- (Y);
		\end{tikzpicture}
	\end{minipage}\\
	In this case $X \independent Z$,
	and $X \independent Y | R=B$.
	For the edge $X \rightarrow Y$ to come from an inducing path,
	there would have to be a candidate for an inducing path.
\end{example}	

There are, however, also examples, where no such decision seems to be possible
from independence-structure alone:
\begin{example}\label{example:inducing_path_ambiguous}
	Ambiguous Case:\\
	\begin{minipage}{0.6\textwidth}
		Compared to the first example, only the changing (dashed) arrow
		has been change to be $L\rightarrow X$ instead of $L\rightarrow Y$.
	\end{minipage}
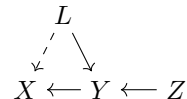
	\begin{minipage}{0.35\textwidth}
		\centering
		\begin{tikzpicture}
			\draw (0,0) node (X) {$X$};
			\draw (1,0) node (Y) {$Y$};
			\draw (2,0) node (Z) {$Z$};
			\draw (0.5,1) node (L) {$L$};
			\draw[->] (Y) -- (X);
			\draw[->] (Z) -- (Y);
			\draw[->,dashed] (L) -- (X);
			\draw[->] (L) -- (Y);
		\end{tikzpicture}
	\end{minipage}\\
	In this case, the JCI-observations of $X$ being the only child of $R$
	no longer excludes the path $Z\rightarrow X$ being real.\\
	\begin{minipage}{0.6\textwidth}
		Actually the model on the right leads to the same
		extended independence-structure.
	\end{minipage}
	\begin{minipage}{0.35\textwidth}
		
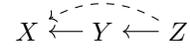
\centering
		\begin{tikzpicture}
			\draw (0,0) node (X) {$X$};
			\draw (1,0) node (Y) {$Y$};
			\draw (2,0) node (Z) {$Z$};
			\draw[->] (Y) -- (X);
			\draw[->] (Z) -- (Y);
			\draw[->,dashed] (Z) edge[bend right] (X);
		\end{tikzpicture}
	\end{minipage}\\
\end{example}	
As the last example \ref{example:inducing_path_ambiguous} shows,
there is no remedy to the problem of inducing paths
in general. This is of course not surprising: Even in the stationary (or IID)
case, inducing paths cannot be identified as such from independence-structure alone
\citep{Verma1991}.
From this perspective it is quite intriguing that sometimes (see examples
\ref{example:inducing_path_indicator}, \ref{example:inducing_path_excluded} and their discussions)
one \emph{can} confirm or exclude the possibility of an inducing path
and thereby narrow down the location of the actual change in the model.

Both the theoretical problem and the practical execution of such a
decision-logic in a stable and reliable
way seem quite complex, so we leave the details of an exhaustive treatment to
future work.

\subsubsection{Non-Modular Changes}\label{apdx:non_modular_changes}

We can of course use our implication tests to compare model-indicators
-- or their representors -- as well (see also §\ref{apdx:JCI}).
This lead to subset-relations between
regimes, thus allow for an analysis of the state-space beyond the assumption
of modularity.
Indeed our approach is not incompatible with non-modular changes;
rather it approximates models around their modular limits,
thus it is only \emph{optimized} for modular changes (the extraction of
any non-modular knowledge requires additional tests).
This is an important difference to \inquotes{global}
approaches that start their approximation in a sense from the opposite limit.

In principle also more sophisticated prior assumptions could be made,
like mechanisms, rather than links, changing modularly; \ie all links going
into a single node change synchronously, while different links going into
different nodes change modularly.
	
	\section{Table of Contents}
	\tableofcontents
\end{document}